%% file: main.tex
\pdfoutput=1
\documentclass[10pt]{article}

\input{preamble}    %
\input{macros}      %
\input{theorems}    %
\usepackage{thm-restate}  %
\usepackage[capitalise,nameinlink,noabbrev,sort&compress]{cleveref}  %
\input{crefnames}
\newenvironment{proofof}[1]{\begin{proof}[Proof of #1]}{\end{proof}}  %

\title{Recalibrating Many Questions with Few Labels:\\ Minimax Rates and Identification Limits}
\author{Xiaoyu Li\thanks{University of New South Wales. \href{mailto:xiaoyulics@gmail.com}{xiaoyulics@gmail.com}}%
  \hspace{3em}%
  Chiwun Yang\thanks{City University of Hong Kong. \href{mailto:christiannyang37@gmail.com}{christiannyang37@gmail.com}}}
\date{\today}

\usepackage{etoolbox}
\makeatletter
\patchcmd{\@maketitle}{\vskip 2em}{\vskip -1.2em}{}{}
\makeatother

\hypersetup{
  pdftitle={Recalibrating Many Questions with Few Labels: Minimax Rates and Identification Limits},
  pdfauthor={Xiaoyu Li, Chiwun Yang},
  pdfsubject={Compound minimax rates and empirical-Bayes identification for probability recalibration},
  pdfkeywords={recalibration, calibration, Jev, decision models, isotonic regression, minimax rates, compound decisions,
    empirical Bayes, identification, moment problems},
}

\begin{document}
\maketitle
\vspace{-1.6em}   %

\begin{abstract}
\input{sections/0_abstract}

\end{abstract}

\vspace{0.6em}
\begin{center}
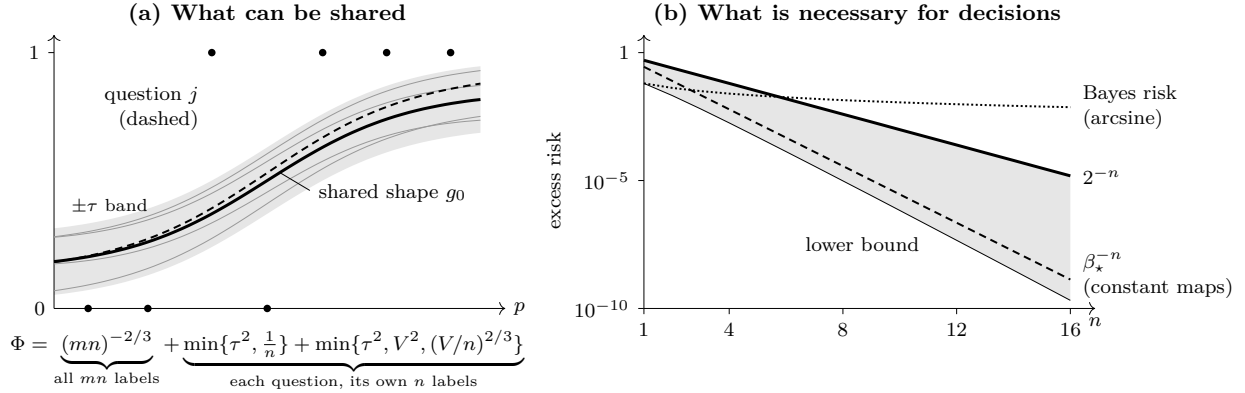

\resizebox{\linewidth}{!}{\input{figures/tikz_hero}}
\captionof{figure}{(a)~$m$ nondecreasing maps (thin) within sup-distance $\tau$ of a shared shape $g_0$
(thick; band $\pm\tau$); dots: the $n$ labels of one question $j$ (dashed); bottom: the rate of
\cref{thm:main-rate}. (b)~Random effects, logarithmic scale. Not knowing the prior costs, in the worst
case over priors on nondecreasing maps, at most $2^{-n}$ when the law of a question's data is known
(solid; \cref{thm:main-eb}), and at least $\max\{1/(8\binom{2n}{n}),\bstar^{-n}/(72\sqrt n)\}$ for every
number of questions (thin; \cref{cor:finite-m}). Dashed: $\bstar^{-n}$, the
upper bound for constant maps, with the regret computed among constant-map priors; its base is exact
(\cref{thm:betastar}). Dotted, for scale: the Bayes risk $1/(8(n+1))$ of the arcsine prior
(\cref{app:practice}).}
\label{fig:hero}
\end{center}

\clearpage
\setcounter{tocdepth}{2}
{\hypersetup{linkcolor=black}\small\tableofcontents}   %
\clearpage

\input{sections/1_intro}
\input{sections/2_prelim}
\input{sections/3_rate}
\input{sections/4_identification}
\input{sections/5_numerics}
\input{sections/6_discussion}

\bibliographystyle{plainnat}   %
\clearpage
\section*{AI Disclosure}
In preparing this work we made use of AI assistants based on large language models. They helped us survey
and summarize the literature, develop and cross-check proof arguments, run and verify the numerical
computations, and draft and edit parts of the text, the figures and the \LaTeX{} source. Every mathematical
statement, proof, citation and numerical claim was checked by the authors, who take full responsibility for
the content of the paper.

\phantomsection\addcontentsline{toc}{section}{References}
\bibliography{refs}

\appendix
\input{sections/9_appendix}

\end{document}

%% file: preamble.tex
\usepackage[T1]{fontenc}
\usepackage[utf8]{inputenc}
\usepackage{lmodern}        %
\usepackage{iftex}          %
\ifPDFTeX                   %
  \usepackage[activate={true,nocompatibility},final]{microtype}
\else                       %
  \usepackage[protrusion=true,final]{microtype}
\fi

\usepackage[margin=1in]{geometry}

\usepackage{authblk}                    %

\usepackage{amsmath}
\usepackage{amssymb,amsfonts}
\usepackage{mathtools}      %
\usepackage{bm}             %
\usepackage{nicefrac}
\allowdisplaybreaks         %

\usepackage{amsthm}

\usepackage{graphicx}
\graphicspath{{figures/}{./}}                       %
\usepackage{booktabs}                               %
\usepackage{multirow,makecell,array}
\usepackage[font=small,labelfont=bf]{caption}       %
\usepackage{subcaption}                             %
\usepackage{algorithm}
\usepackage{algpseudocode}                          %

\usepackage{enumitem}
\setlist[itemize]{leftmargin=2.2em,itemsep=2pt,topsep=2pt}
\setlist[enumerate]{leftmargin=2.2em,itemsep=2pt,topsep=2pt}

\usepackage{xcolor}
\definecolor{LinkColor}{rgb}{0.10,0.40,0.75}        %
\definecolor{CiteColor}{rgb}{0.70,0.25,0.20}        %
\definecolor{UrlColor} {rgb}{0.08,0.30,0.58}        %
\definecolor{TodoColor}{rgb}{0.80,0.30,0.10}        %

\usepackage{tikz}
\usetikzlibrary{positioning,calc,arrows.meta,decorations.pathreplacing}

\usepackage[round,sort&compress]{natbib}

\usepackage{url}
\usepackage{hyperref}
\hypersetup{
  colorlinks=true,
  linkcolor=LinkColor,
  citecolor=CiteColor,
  urlcolor=UrlColor,
  breaklinks=true,
  bookmarksnumbered=true,
}
\usepackage{bookmark}                               %

\numberwithin{equation}{section}

\newif\ifdraft \draftfalse
\ifdraft
  \usepackage{lineno}\linenumbers
  \newcommand{\todo}[1]{\textcolor{TodoColor}{\textbf{[TODO:}~#1\textbf{]}}}
\else
  \newcommand{\todo}[1]{}                            %
\fi

%% file: macros.tex
\newcommand{\R}{\mathbb{R}}

\newcommand{\Z}{\mathbb{Z}}
\newcommand{\Q}{\mathbb{Q}}
\newcommand{\E}{\mathbb{E}}                       %
\renewcommand{\P}{\mathbb{P}}                     %

\newcommand{\eps}{\varepsilon}

\newcommand{\defeq}{\coloneqq}                    %

\DeclareMathOperator{\supp}{supp}

\DeclareMathOperator{\poly}{poly}
\DeclareMathOperator{\Var}{Var}

\DeclareMathOperator*{\argmin}{arg\,min}

\DeclarePairedDelimiter{\abs}{\lvert}{\rvert}
\DeclarePairedDelimiter{\norm}{\lVert}{\rVert}
\DeclarePairedDelimiter{\set}{\{}{\}}
\DeclarePairedDelimiter{\ceil}{\lceil}{\rceil}
\DeclarePairedDelimiter{\floor}{\lfloor}{\rfloor}
\DeclarePairedDelimiterX{\inner}[2]{\langle}{\rangle}{#1,#2}   %

\newcommand{\Mm}{\mathfrak{M}}                    %
\newcommand{\cR}{\mathcal{R}}                     %
\DeclareMathOperator{\TV}{TV}                     %
\DeclareMathOperator{\Bern}{Bern}                 %
\DeclareMathOperator{\clip}{clip}                 %
\DeclareMathOperator{\Law}{Law}                   %
\newcommand{\Fmono}{\mathcal{F}_{\uparrow}}       %
\newcommand{\Dcl}{\mathcal{D}}                    %
\newcommand{\Lmu}{L_2(\mu)}                       %
\newcommand{\nmu}[1]{\norm{#1}_{L_2(\mu)}}        %
\newcommand{\gh}{\widehat{g}}                     %
\newcommand{\dhat}{\widehat{d}}                   %
\newcommand{\fhat}{\widehat{f}}                   %
\newcommand{\gbar}{\bar{g}}                       %
\newcommand{\etabar}{\bar{\eta}}                  %
\newcommand{\Nbr}{N_{[\,]}}                       %
\newcommand{\cG}{\mathcal{G}}                     %
\newcommand{\Gmono}{\mathcal{G}_{\mathrm{mono}}}  %
\newcommand{\Gconst}{\mathcal{G}_{\mathrm{const}}}%
\newcommand{\Reg}{\mathrm{Reg}}                   %
\newcommand{\cU}{\mathcal{U}}                     %
\newcommand{\cH}{\mathcal{H}}                     %
\newcommand{\Gcat}{\mathrm{Cat}}                  %
\newcommand{\bstar}{\beta_{\star}}                %
\DeclareMathOperator{\GM}{GM}                     %
\DeclareMathOperator{\Mah}{Mah}                   %
\newcommand{\Ttil}{\widetilde{T}}                 %

%% file: theorems.tex
\usepackage{aliascnt}
\newcommand{\newsharedtheorem}[2]{%
  \newaliascnt{#1}{theorem}%
  \newtheorem{#1}[#1]{#2}%
  \aliascntresetthe{#1}}

\theoremstyle{plain}

\newsharedtheorem{proposition}{Proposition}
\newsharedtheorem{lemma}{Lemma}
\newsharedtheorem{corollary}{Corollary}
\newsharedtheorem{conjecture}{Conjecture}
\newsharedtheorem{fact}{Fact}

\theoremstyle{definition}
\newsharedtheorem{definition}{Definition}
\newsharedtheorem{assumption}{Assumption}
\newsharedtheorem{example}{Example}
\newsharedtheorem{problem}{Problem}

\theoremstyle{remark}
\newsharedtheorem{remark}{Remark}
\newsharedtheorem{claim}{Claim}

%% file: crefnames.tex
\crefname{theorem}{Theorem}{Theorems}          \Crefname{theorem}{Theorem}{Theorems}
\crefname{proposition}{Proposition}{Propositions}
\Crefname{proposition}{Proposition}{Propositions}
\crefname{lemma}{Lemma}{Lemmas}                \Crefname{lemma}{Lemma}{Lemmas}
\crefname{corollary}{Corollary}{Corollaries}   \Crefname{corollary}{Corollary}{Corollaries}
\crefname{conjecture}{Conjecture}{Conjectures} \Crefname{conjecture}{Conjecture}{Conjectures}
\crefname{fact}{Fact}{Facts}                   \Crefname{fact}{Fact}{Facts}
\crefname{definition}{Definition}{Definitions} \Crefname{definition}{Definition}{Definitions}
\crefname{assumption}{Assumption}{Assumptions} \Crefname{assumption}{Assumption}{Assumptions}
\crefname{example}{Example}{Examples}          \Crefname{example}{Example}{Examples}
\crefname{problem}{Problem}{Problems}          \Crefname{problem}{Problem}{Problems}
\crefname{remark}{Remark}{Remarks}             \Crefname{remark}{Remark}{Remarks}
\crefname{claim}{Claim}{Claims}                \Crefname{claim}{Claim}{Claims}
\crefname{algorithm}{Algorithm}{Algorithms}    \Crefname{algorithm}{Algorithm}{Algorithms}

%% file: sections/0_abstract.tex
Many related prediction problems, which we call questions, come with few labels each. What
information can be shared across them, and what is actually necessary for good decisions? We answer
both for recalibration, motivated by decision models such as Jev~\citep{typesafe2026jev,typesafe2026docs}
that return a probability for any yes/no question a developer writes. A predictor reports a probability $P$ with output law $\mu$ on $m$ questions,
each with $n$ labels and a nondecreasing map $g_j(p)=\P(Y=1\mid P=p)$; the compound risk is the
excess Brier score over the true maps, averaged over questions. Over maps within sup-distance
$\tau\le1$ and total variation $V\le2$ of one nondecreasing shape (fixed effects), the minimax
compound risk is $(mn)^{-2/3}+\min\{\tau^2,1/n\}+\min\{\tau^2,V^2,(V/n)^{2/3}\}$ up to universal
constants, for every non-atomic $\mu$. A pooled isotonic fit with a per-question correction sized
once, on held-out labels of all questions, attains it without knowing $\tau$, $V$ or $\mu$ ($n\ge2$,
$m\ge3$); tuning each question on its own labels loses a factor of order $\min\{n,\log(m^2/n)\}$ when
the maps differ by constants and $m^2/n$ is large. When the maps are drawn from an unknown prior
(random effects), other questions need not identify the Bayes rule. The loss, the
identification regret, is at most $2^{-n}$ once the law of a question's data is known, extending a bound
of \citet{skibinsky1968minimax} for coins, and in the worst case at least $\bstar^{-n}/(72\sqrt n)$ at
every finite $m$, where $\bstar\approx3.583$.
Beyond a few labels per question, what cannot be learned is nearly unnecessary. For constant maps (an
uninformative predictor) and constant-map priors, the exponential base is exactly $\bstar$, not the
$4$ suggested by the moment range; for monotone maps it lies in $[2,\bstar]$.

%% file: figures/tikz_hero.tex
\begin{tikzpicture}[font=\footnotesize]
  \node[font=\small\bfseries,anchor=south] at (3.00,3.85) {(a) What can be shared};
  \fill[black!10] (0.000,1.128) -- (0.100,1.139) -- (0.200,1.152) -- (0.300,1.167) -- (0.400,1.182) -- (0.500,1.199) -- (0.600,1.218) -- (0.700,1.238) -- (0.800,1.259) -- (0.900,1.283) -- (1.000,1.308) -- (1.100,1.336) -- (1.200,1.365) -- (1.300,1.397) -- (1.400,1.431) -- (1.500,1.468) -- (1.600,1.506) -- (1.700,1.548) -- (1.800,1.591) -- (1.900,1.637) -- (2.000,1.686) -- (2.100,1.736) -- (2.200,1.789) -- (2.300,1.844) -- (2.400,1.901) -- (2.500,1.959) -- (2.600,2.019) -- (2.700,2.080) -- (2.800,2.142) -- (2.900,2.205) -- (3.000,2.268) -- (3.100,2.331) -- (3.200,2.394) -- (3.300,2.456) -- (3.400,2.517) -- (3.500,2.577) -- (3.600,2.635) -- (3.700,2.692) -- (3.800,2.747) -- (3.900,2.800) -- (4.000,2.850) -- (4.100,2.899) -- (4.200,2.945) -- (4.300,2.988) -- (4.400,3.030) -- (4.500,3.068) -- (4.600,3.105) -- (4.700,3.139) -- (4.800,3.171) -- (4.900,3.200) -- (5.000,3.228) -- (5.100,3.253) -- (5.200,3.277) -- (5.300,3.298) -- (5.400,3.318) -- (5.500,3.337) -- (5.600,3.354) -- (5.700,3.369) -- (5.800,3.384) -- (5.900,3.397) -- (6.000,3.408) -- (6.000,2.472) -- (5.900,2.461) -- (5.800,2.448) -- (5.700,2.433) -- (5.600,2.418) -- (5.500,2.401) -- (5.400,2.382) -- (5.300,2.362) -- (5.200,2.341) -- (5.100,2.317) -- (5.000,2.292) -- (4.900,2.264) -- (4.800,2.235) -- (4.700,2.203) -- (4.600,2.169) -- (4.500,2.132) -- (4.400,2.094) -- (4.300,2.052) -- (4.200,2.009) -- (4.100,1.963) -- (4.000,1.914) -- (3.900,1.864) -- (3.800,1.811) -- (3.700,1.756) -- (3.600,1.699) -- (3.500,1.641) -- (3.400,1.581) -- (3.300,1.520) -- (3.200,1.458) -- (3.100,1.395) -- (3.000,1.332) -- (2.900,1.269) -- (2.800,1.206) -- (2.700,1.144) -- (2.600,1.083) -- (2.500,1.023) -- (2.400,0.965) -- (2.300,0.908) -- (2.200,0.853) -- (2.100,0.800) -- (2.000,0.750) -- (1.900,0.701) -- (1.800,0.655) -- (1.700,0.612) -- (1.600,0.570) -- (1.500,0.532) -- (1.400,0.495) -- (1.300,0.461) -- (1.200,0.429) -- (1.100,0.400) -- (1.000,0.372) -- (0.900,0.347) -- (0.800,0.323) -- (0.700,0.302) -- (0.600,0.282) -- (0.500,0.263) -- (0.400,0.246) -- (0.300,0.231) -- (0.200,0.216) -- (0.100,0.203) -- (0.000,0.192) -- cycle;
  \draw[->] (0,0) -- (6.35,0) node[right] {$p$};
  \draw[->] (0,0) -- (0,3.85);
  \draw (0,3.60) -- (-0.08,3.60) node[left] {$1$};
  \node[left] at (0,0) {$0$};
  \draw[black!40] (0.000,0.249) -- (0.100,0.264) -- (0.200,0.280) -- (0.300,0.297) -- (0.400,0.315) -- (0.500,0.335) -- (0.600,0.356) -- (0.700,0.379) -- (0.800,0.404) -- (0.900,0.430) -- (1.000,0.459) -- (1.100,0.489) -- (1.200,0.522) -- (1.300,0.556) -- (1.400,0.593) -- (1.500,0.633) -- (1.600,0.674) -- (1.700,0.718) -- (1.800,0.765) -- (1.900,0.814) -- (2.000,0.865) -- (2.100,0.918) -- (2.200,0.974) -- (2.300,1.032) -- (2.400,1.092) -- (2.500,1.153) -- (2.600,1.216) -- (2.700,1.280) -- (2.800,1.345) -- (2.900,1.410) -- (3.000,1.476) -- (3.100,1.542) -- (3.200,1.607) -- (3.300,1.672) -- (3.400,1.736) -- (3.500,1.799) -- (3.600,1.860) -- (3.700,1.920) -- (3.800,1.978) -- (3.900,2.034) -- (4.000,2.087) -- (4.100,2.138) -- (4.200,2.187) -- (4.300,2.234) -- (4.400,2.278) -- (4.500,2.319) -- (4.600,2.359) -- (4.700,2.396) -- (4.800,2.430) -- (4.900,2.463) -- (5.000,2.493) -- (5.100,2.522) -- (5.200,2.548) -- (5.300,2.573) -- (5.400,2.596) -- (5.500,2.617) -- (5.600,2.637) -- (5.700,2.655) -- (5.800,2.672) -- (5.900,2.688) -- (6.000,2.703);
  \draw[black!40] (0.000,0.627) -- (0.100,0.635) -- (0.200,0.643) -- (0.300,0.653) -- (0.400,0.665) -- (0.500,0.677) -- (0.600,0.691) -- (0.700,0.707) -- (0.800,0.724) -- (0.900,0.744) -- (1.000,0.765) -- (1.100,0.788) -- (1.200,0.813) -- (1.300,0.841) -- (1.400,0.870) -- (1.500,0.903) -- (1.600,0.937) -- (1.700,0.974) -- (1.800,1.013) -- (1.900,1.055) -- (2.000,1.099) -- (2.100,1.145) -- (2.200,1.194) -- (2.300,1.244) -- (2.400,1.297) -- (2.500,1.351) -- (2.600,1.407) -- (2.700,1.463) -- (2.800,1.521) -- (2.900,1.579) -- (3.000,1.638) -- (3.100,1.697) -- (3.200,1.755) -- (3.300,1.813) -- (3.400,1.869) -- (3.500,1.925) -- (3.600,1.979) -- (3.700,2.032) -- (3.800,2.082) -- (3.900,2.131) -- (4.000,2.177) -- (4.100,2.221) -- (4.200,2.263) -- (4.300,2.302) -- (4.400,2.339) -- (4.500,2.373) -- (4.600,2.406) -- (4.700,2.435) -- (4.800,2.463) -- (4.900,2.488) -- (5.000,2.511) -- (5.100,2.532) -- (5.200,2.552) -- (5.300,2.569) -- (5.400,2.585) -- (5.500,2.599) -- (5.600,2.611) -- (5.700,2.623) -- (5.800,2.633) -- (5.900,2.641) -- (6.000,2.649);
  \draw[black!40] (0.000,1.002) -- (0.100,1.011) -- (0.200,1.022) -- (0.300,1.033) -- (0.400,1.047) -- (0.500,1.061) -- (0.600,1.077) -- (0.700,1.095) -- (0.800,1.114) -- (0.900,1.135) -- (1.000,1.158) -- (1.100,1.183) -- (1.200,1.211) -- (1.300,1.240) -- (1.400,1.272) -- (1.500,1.306) -- (1.600,1.342) -- (1.700,1.381) -- (1.800,1.422) -- (1.900,1.466) -- (2.000,1.512) -- (2.100,1.560) -- (2.200,1.610) -- (2.300,1.663) -- (2.400,1.717) -- (2.500,1.773) -- (2.600,1.831) -- (2.700,1.890) -- (2.800,1.949) -- (2.900,2.009) -- (3.000,2.070) -- (3.100,2.131) -- (3.200,2.191) -- (3.300,2.250) -- (3.400,2.309) -- (3.500,2.367) -- (3.600,2.423) -- (3.700,2.477) -- (3.800,2.530) -- (3.900,2.580) -- (4.000,2.628) -- (4.100,2.674) -- (4.200,2.718) -- (4.300,2.759) -- (4.400,2.798) -- (4.500,2.834) -- (4.600,2.868) -- (4.700,2.900) -- (4.800,2.929) -- (4.900,2.957) -- (5.000,2.982) -- (5.100,3.005) -- (5.200,3.026) -- (5.300,3.045) -- (5.400,3.063) -- (5.500,3.079) -- (5.600,3.093) -- (5.700,3.107) -- (5.800,3.118) -- (5.900,3.129) -- (6.000,3.138);
  \draw[black!40] (0.000,1.009) -- (0.100,1.022) -- (0.200,1.036) -- (0.300,1.051) -- (0.400,1.067) -- (0.500,1.085) -- (0.600,1.105) -- (0.700,1.126) -- (0.800,1.148) -- (0.900,1.173) -- (1.000,1.199) -- (1.100,1.228) -- (1.200,1.258) -- (1.300,1.291) -- (1.400,1.326) -- (1.500,1.363) -- (1.600,1.403) -- (1.700,1.445) -- (1.800,1.490) -- (1.900,1.537) -- (2.000,1.586) -- (2.100,1.638) -- (2.200,1.692) -- (2.300,1.747) -- (2.400,1.805) -- (2.500,1.865) -- (2.600,1.925) -- (2.700,1.988) -- (2.800,2.050) -- (2.900,2.114) -- (3.000,2.178) -- (3.100,2.242) -- (3.200,2.306) -- (3.300,2.368) -- (3.400,2.431) -- (3.500,2.491) -- (3.600,2.551) -- (3.700,2.609) -- (3.800,2.664) -- (3.900,2.718) -- (4.000,2.770) -- (4.100,2.819) -- (4.200,2.866) -- (4.300,2.911) -- (4.400,2.953) -- (4.500,2.993) -- (4.600,3.030) -- (4.700,3.065) -- (4.800,3.098) -- (4.900,3.128) -- (5.000,3.157) -- (5.100,3.183) -- (5.200,3.208) -- (5.300,3.230) -- (5.400,3.251) -- (5.500,3.271) -- (5.600,3.289) -- (5.700,3.305) -- (5.800,3.320) -- (5.900,3.334) -- (6.000,3.347);
  \draw[very thick] (0.000,0.660) -- (0.100,0.671) -- (0.200,0.684) -- (0.300,0.699) -- (0.400,0.714) -- (0.500,0.731) -- (0.600,0.750) -- (0.700,0.770) -- (0.800,0.791) -- (0.900,0.815) -- (1.000,0.840) -- (1.100,0.868) -- (1.200,0.897) -- (1.300,0.929) -- (1.400,0.963) -- (1.500,1.000) -- (1.600,1.038) -- (1.700,1.080) -- (1.800,1.123) -- (1.900,1.169) -- (2.000,1.218) -- (2.100,1.268) -- (2.200,1.321) -- (2.300,1.376) -- (2.400,1.433) -- (2.500,1.491) -- (2.600,1.551) -- (2.700,1.612) -- (2.800,1.674) -- (2.900,1.737) -- (3.000,1.800) -- (3.100,1.863) -- (3.200,1.926) -- (3.300,1.988) -- (3.400,2.049) -- (3.500,2.109) -- (3.600,2.167) -- (3.700,2.224) -- (3.800,2.279) -- (3.900,2.332) -- (4.000,2.382) -- (4.100,2.431) -- (4.200,2.477) -- (4.300,2.520) -- (4.400,2.562) -- (4.500,2.600) -- (4.600,2.637) -- (4.700,2.671) -- (4.800,2.703) -- (4.900,2.732) -- (5.000,2.760) -- (5.100,2.785) -- (5.200,2.809) -- (5.300,2.830) -- (5.400,2.850) -- (5.500,2.869) -- (5.600,2.886) -- (5.700,2.901) -- (5.800,2.916) -- (5.900,2.929) -- (6.000,2.940);
  \draw[thick,densely dashed] (0.000,0.652) -- (0.100,0.668) -- (0.200,0.685) -- (0.300,0.703) -- (0.400,0.722) -- (0.500,0.743) -- (0.600,0.765) -- (0.700,0.789) -- (0.800,0.815) -- (0.900,0.842) -- (1.000,0.872) -- (1.100,0.903) -- (1.200,0.936) -- (1.300,0.972) -- (1.400,1.010) -- (1.500,1.050) -- (1.600,1.093) -- (1.700,1.138) -- (1.800,1.185) -- (1.900,1.235) -- (2.000,1.287) -- (2.100,1.342) -- (2.200,1.399) -- (2.300,1.457) -- (2.400,1.518) -- (2.500,1.580) -- (2.600,1.644) -- (2.700,1.709) -- (2.800,1.775) -- (2.900,1.841) -- (3.000,1.908) -- (3.100,1.975) -- (3.200,2.041) -- (3.300,2.107) -- (3.400,2.172) -- (3.500,2.236) -- (3.600,2.298) -- (3.700,2.359) -- (3.800,2.417) -- (3.900,2.474) -- (4.000,2.529) -- (4.100,2.581) -- (4.200,2.631) -- (4.300,2.678) -- (4.400,2.723) -- (4.500,2.766) -- (4.600,2.806) -- (4.700,2.844) -- (4.800,2.880) -- (4.900,2.913) -- (5.000,2.944) -- (5.100,2.974) -- (5.200,3.001) -- (5.300,3.027) -- (5.400,3.051) -- (5.500,3.073) -- (5.600,3.094) -- (5.700,3.113) -- (5.800,3.131) -- (5.900,3.148) -- (6.000,3.164);
  \fill (0.480,0.000) circle (1.5pt);
  \fill (1.320,0.000) circle (1.5pt);
  \fill (2.220,3.600) circle (1.5pt);
  \fill (3.000,0.000) circle (1.5pt);
  \fill (3.780,3.600) circle (1.5pt);
  \fill (4.680,3.600) circle (1.5pt);
  \fill (5.580,3.600) circle (1.5pt);
  \node[anchor=west] at (3.60,1.63) {shared shape $g_0$};
  \draw[thin] (3.60,1.66) -- (3.18,1.91);
  \node[anchor=east,align=right] at (2.16,2.81) {question $j$\\ (dashed)};
  \node[anchor=west,align=left,font=\scriptsize] at (0.12,1.44) {$\pm\tau$ band};
  \node[anchor=north] at (3.00,-0.18) {$\displaystyle\Phi=\underbrace{(mn)^{-2/3}}_{\text{all $mn$ labels}}+\underbrace{\min\{\tau^2,\tfrac1n\}+\min\{\tau^2,V^2,(V/n)^{2/3}\}}_{\text{each question, its own $n$ labels}}$};
  \begin{scope}[xshift=8.30cm]
  \node[font=\small\bfseries,anchor=south] at (3.00,3.85) {(b) What is necessary for decisions};
  \fill[black!10] (0.000,3.492) -- (0.400,3.383) -- (0.800,3.275) -- (1.200,3.167) -- (1.600,3.058) -- (2.000,2.950) -- (2.400,2.841) -- (2.800,2.733) -- (3.200,2.625) -- (3.600,2.516) -- (4.000,2.408) -- (4.400,2.300) -- (4.800,2.191) -- (5.200,2.083) -- (5.600,1.974) -- (6.000,1.866) -- (6.000,0.114) -- (5.600,0.326) -- (5.200,0.538) -- (4.800,0.749) -- (4.400,0.959) -- (4.000,1.169) -- (3.600,1.379) -- (3.200,1.588) -- (2.800,1.795) -- (2.400,2.002) -- (2.000,2.207) -- (1.600,2.410) -- (1.200,2.611) -- (0.800,2.807) -- (0.400,2.995) -- (0.000,3.167) -- cycle;
  \draw[->] (0,0) -- (6.35,0) node[below] {$n$};
  \draw[->] (0,0) -- (0,3.85);
  \draw (0,3.600) -- (-0.08,3.600) node[left] {$1$};
  \draw (0,1.800) -- (-0.08,1.800) node[left] {$10^{-5}$};
  \draw (0,0.000) -- (-0.08,0.000) node[left] {$10^{-10}$};
  \draw (0.000,0) -- (0.000,-0.08) node[below] {$1$};
  \draw (1.200,0) -- (1.200,-0.08) node[below] {$4$};
  \draw (2.800,0) -- (2.800,-0.08) node[below] {$8$};
  \draw (4.400,0) -- (4.400,-0.08) node[below] {$12$};
  \draw (6.000,0) -- (6.000,-0.08) node[below] {$16$};
  \node[rotate=90,anchor=south] at (-1.05,1.80) {excess risk};
  \draw[densely dotted,thick] (0.000,3.167) -- (0.400,3.103) -- (0.800,3.058) -- (1.200,3.023) -- (1.600,2.995) -- (2.000,2.971) -- (2.400,2.950) -- (2.800,2.931) -- (3.200,2.915) -- (3.600,2.900) -- (4.000,2.886) -- (4.400,2.874) -- (4.800,2.862) -- (5.200,2.851) -- (5.600,2.841) -- (6.000,2.832);
  \draw[very thick] (0.000,3.492) -- (0.400,3.383) -- (0.800,3.275) -- (1.200,3.167) -- (1.600,3.058) -- (2.000,2.950) -- (2.400,2.841) -- (2.800,2.733) -- (3.200,2.625) -- (3.600,2.516) -- (4.000,2.408) -- (4.400,2.300) -- (4.800,2.191) -- (5.200,2.083) -- (5.600,1.974) -- (6.000,1.866);
  \draw[thick,densely dashed] (0.000,3.400) -- (0.400,3.201) -- (0.800,3.001) -- (1.200,2.802) -- (1.600,2.602) -- (2.000,2.403) -- (2.400,2.203) -- (2.800,2.004) -- (3.200,1.804) -- (3.600,1.605) -- (4.000,1.405) -- (4.400,1.206) -- (4.800,1.006) -- (5.200,0.806) -- (5.600,0.607) -- (6.000,0.407);
  \draw[thin] (0.000,3.167) -- (0.400,2.995) -- (0.800,2.807) -- (1.200,2.611) -- (1.600,2.410) -- (2.000,2.207) -- (2.400,2.002) -- (2.800,1.795) -- (3.200,1.588) -- (3.600,1.379) -- (4.000,1.169) -- (4.400,0.959) -- (4.800,0.749) -- (5.200,0.538) -- (5.600,0.326) -- (6.000,0.114);
  \node[anchor=west,align=left] at (6.050,2.832) {Bayes risk\\(arcsine)};
  \node[anchor=west] at (6.050,1.866) {$2^{-n}$};
  \node[anchor=west,align=left] at (6.050,0.487) {$\bstar^{-n}$\\(constant maps)};
  \node[anchor=north east] at (4.000,1.129) {lower bound};
  \end{scope}
\end{tikzpicture}

%% file: sections/1_intro.tex
\section{Introduction}
\label{sec:intro}

Suppose we face many related statistical problems, which we call \emph{questions}, and observe each
of them only a few times. Having $m$ questions with $n$ observations each is not the same as having
$mn$ observations of one: information can be shared across questions only to the extent that their
structure is shared. How much can then be learned collectively---what does the heterogeneity between
the questions cost each of them? And what remains impossible to learn about an individual question,
however many others are observed? In short: what information can be shared across questions, and
what information is actually necessary for good decisions? For one nonparametric problem, the
recalibration of probabilities, we answer both halves: with \emph{fixed effects}, where the questions
are arbitrary members of a class of heterogeneous problems and we take the worst case, and with
\emph{random effects}, where they are exchangeable draws from an unknown prior.

The first answer is a minimax rate, exact up to universal constants: the shared shape is learned from
all $mn$ labels, and each question pays for its own deviation at the rate of its own $n$ labels. The
second answer concerns exchangeable questions, where the target is the Bayes risk. Other questions
reveal the law of one question's data, but in general not the Bayes rule. The part they cannot reveal
costs at most $2^{-n}$ when that law is known, and at least $\bstar^{-n}/(72\sqrt n)$ in the
worst case, where $\bstar\approx3.583$ (\cref{fig:hero}). Beyond a few labels per question, what
cannot be learned is nearly unnecessary. The setting is that of compound decision
theory~\citep{robbins1951asymptotically} and empirical
Bayes~\citep{robbins1956empirical,zhang2003compound}, and of learning many related
tasks~\citep{baxter2000model}.

\paragraph{Recalibration with few labels per question.}
A predictor reports a probability $P\in[0,1]$ for a binary event and is used on many questions:
sub-populations or slices of a dataset~\citep{li2024fewshot}, raters~\citep{raj2026pebs},
groups~\citep{ji2020trust}, or yes/no questions that developers write for decision models such as Jev
and Laya~\citep{typesafe2026jev,typesafe2026docs,laya2026}, systems that return a probability for each
such question. Calibration of the predictor, if it holds, is a property of the population of
questions and says nothing about a particular question. For question $j$ the relevant object is its
\emph{recalibration map} $g_j(p)=\P(Y=1\mid P=p,\ \text{question }j)$. It is the identity exactly when
the predictor is calibrated on question $j$, and we assume it nondecreasing: the predictor ranks
cases correctly within a question even when its numbers are off. The map is observed through $n$
pairs $(P_{jk},Y_{jk})$ with $P_{jk}\sim\mu$, the \emph{output law}, shared by all questions
(\cref{def:model}). An estimator sees all $mn$ pairs and returns maps $\gh_1,\dots,\gh_m$; its compound
risk, the average over $j$ of $\E\,\nmu{\gh_j-g_j}^2$, is the average excess Brier score over ideal
per-question recalibration. Heterogeneity is measured against a \emph{shared shape} $g_0$, a
nondecreasing map: the deviations $g_j-g_0$ have sup-norm at most $\tau$ and total variation at most
$V$ (the class $\Theta(\tau,V)$ of \cref{def:class}). We write $\Mm_{m,n}(\tau,V;\mu)$ for the minimax
compound risk over $\Theta(\tau,V)$.

\paragraph{What can be shared: the shape, and the size of the deviations.}
Our first answer is the minimax rate (\cref{thm:main-rate}): for every non-atomic $\mu$, all
$m,n\ge1$, $\tau\in[0,1]$ and $V\in[0,2]$, with constants independent of $\mu,m,n,\tau,V$,
\begin{equation}\label{eq:intro-rate}
  \Mm_{m,n}(\tau,V;\mu)\;\asymp\;\Phi(m,n,\tau,V)=(mn)^{-2/3}
    +\min\Bigl\{\tau^2,\frac1n\Bigr\}
    +\min\Bigl\{\tau^2,V^2,\Bigl(\frac Vn\Bigr)^{2/3}\Bigr\}.
\end{equation}
All questions share the cost of the shape. Each question either learns its offset from its $n$
labels, at price $1/n$, or ignores it, at price $\tau^2$; likewise it learns the rest of its
deviation, its wiggle, at price $(V/n)^{2/3}$, or ignores it, at price $\min\{\tau^2,V^2\}$. The shape
is what pooling buys, and heterogeneity is paid question by question. With $m=300$ questions and
$n=16$ labels each (orders of magnitude, constants ignored), a question alone resolves its map only at
$n^{-2/3}\approx0.16$ (the rate for $m=1$, non-atomic $\mu$), the pool resolves the shared shape at
$(mn)^{-2/3}\approx0.0035$, and each offset costs at most $1/n\approx0.06$. In the \emph{shift model}, where every deviation $g_j-g_0$ is a
constant ($V=0$), the rate is $(mn)^{-2/3}+\min\{\tau^2,1/n\}$, parametric in $n$ even when $\tau=1$:
heterogeneity is paid in variation, not in distance.

One procedure attains the rate without knowing $\tau$, $V$ or $\mu$ (\cref{thm:adaptive}): a pooled
isotonic fit with a bounded-variation correction, whose two size parameters are chosen once, for all
questions, on held-out labels (the \emph{compound hold-out}; plain isotonic fits when $n=1$ or
$m\le2$). That choice shares one more piece of information across questions: the \emph{size} of the
deviations. This sharing is necessary once $n$ and $m^2/n$ are large. In the shift model with
$m^2\ge216n$, even the best rule that tunes each question on its own data, and knows the shared shape, has
worst-case ratio to the rate of order $\min\{n,\log(m^2/n)\}$ (\cref{prop:separation}), while the
compound hold-out keeps it bounded.

\paragraph{What is necessary for decisions: an identification limit.}
Now let the maps $g_1,g_2,\dots$ be drawn independently from a prior $G$ on nondecreasing maps. The
natural target is then the Bayes risk $R^*_n(G)$ of the posterior mean, which can be far below the
minimax rate. If every map is the constant $1/4$ or $3/4$ with probability $1/2$, predicting $3/4$ when
at least half of a question's labels are positive and $1/4$ otherwise errs by $1/2$ with probability at most
$e^{-n/8}$ (Hoeffding's inequality), so $R^*_n(G)\le e^{-n/8}/4$; yet these maps lie in $\Theta(1/4,0)$,
whose minimax risk is at least a universal constant times $\min\{1/16,1/n\}$ for non-atomic $\mu$
(\cref{thm:main-rate}). Reaching this target needs the Bayes rule, which other questions need not
reveal: with infinitely many questions the law of one question's data is known exactly, but $G$ need not
be, and $n$ labels per question need not determine the Bayes rule. One label and constant maps (coins) already
show this: the priors $\delta_{1/2}$ and $\frac12\delta_0+\frac12\delta_1$ give the same label law,
$\Bern(1/2)$, but the Bayes rule predicts $1/2$ under the first and $Y$ under the second, and every rule
loses at least $1/16$ against one of them~\citep{robbins1964empirical,skibinsky1968minimax}. With $n$
labels the law of a question's data determines the product moments $\E_G[g(t_1)\cdots g(t_r)]$ for
$r\le n$, while the posterior mean needs $r=n+1$; for coins, \citet{skibinsky1968minimax} showed that,
among the priors with given first $n$ moments, the minimax risk exceeds each prior's Bayes risk by at
most $2^{-n}$. Let $\Gmono$ be the class of priors on nondecreasing maps. We measure the price by the
\emph{identification regret} $\Reg_n(G;\Gmono)$, the smallest excess over the Bayes risk that one rule
can guarantee simultaneously against every prior in $\Gmono$ (the second argument) inducing the same
law of the data as $G$. For random monotone maps and every output law (\cref{thm:main-eb}),
\begin{equation}\label{eq:intro-eb}
  \max\Bigl\{\frac{1}{8\binom{2n}{n}},\ \frac{\bstar^{-n}}{72\sqrt n}\Bigr\}
  \;\le\;\sup_{G\in\Gmono}\Reg_n(G;\Gmono)\;\le\;2^{-n},
  \qquad \bstar=2e^{2\Gcat/\pi}\approx3.583,
\end{equation}
where $\Gcat$ is Catalan's constant. The upper bound assumes the law of a question's data known
($m=\infty$); the lower bounds hold for every finite number of questions (\cref{cor:finite-m}).

\paragraph{Contributions (\cref{tab:results}).}
\begin{enumerate}[leftmargin=*,itemsep=1pt]
\item \textbf{The compound minimax rate and compound tuning}
  (\cref{thm:main-rate,thm:adaptive,prop:separation}). The rate~\eqref{eq:intro-rate} is exact up to
  universal constants, with no logarithmic factor, in every regime with $\tau\le1$ and $V\le2$, for every
  non-atomic $\mu$. Because all questions share the output law $\mu$, the pooled isotonic fit tracks the
  \emph{average} of $m$ different maps at the rate of all its labels, and a localized least-squares
  bound for independent, non-identically distributed pairs (\cref{prop:erm2}) removes the logarithmic
  factor that a global entropy bound puts on the $1/n$ offset term.
\item \textbf{The identification regret for monotone maps}
  (\cref{prop:chebradius,thm:main-eb,cor:finite-m}). The regret is the squared Chebyshev radius---the
  radius of the smallest enclosing ball---of the set of unidentified top moments in an explicit weighted
  $L_2$ space: a data-averaged, functional form of posterior-regret $\Gamma$-minimax estimation. The
  bound $2^{-n}$ has a short reason: the unidentified $(n+1)$-th moment ranges over an interval of
  length at most $4^{-n}$ (classical for coins; for monotone maps we prove it with a diagonal-domination
  inequality for sorted vectors), and each of the $2^n$ label patterns contributes at most this length.
  It holds for every output law, for the posterior mean of any monotone prior that reproduces the law of
  the data.
\item \textbf{The exact base for constant maps} (\cref{thm:betastar,cor:skibinsky}). For priors on
  constant maps (an uninformative predictor), with the regret computed among such priors, its supremum
  lies between $\bstar^{-n}/(72\sqrt n)$ and $\bstar^{-n}$: the natural guess, $4^{-n}$ up to polynomial
  factors, is false. The arcsine law attains the largest moment range, $4^{-n}$, which suggests base
  $4$; near-extremal priors instead trade range for an exponentially rare label pattern, which gives the
  smaller base $\bstar$. For monotone maps the base lies in $[2,\bstar]$ (\cref{conj:monotone-base}: it
  is $\bstar$). The lower bound also bounds Skibinsky's minimax-minus-Bayes gap, whose base appears
  numerically to be $2$.
\end{enumerate}

\begin{table}[t]
\centering\footnotesize
\setlength{\tabcolsep}{3pt}
\renewcommand{\arraystretch}{1.1}
\newcolumntype{L}[1]{>{\raggedright\arraybackslash}p{#1}}
\begin{tabular}{@{}L{1.32cm}L{1.95cm}L{1.15cm}L{2.3cm}L{4.85cm}L{2.53cm}@{}}
\toprule
Result & Benchmark & $m$ & Class & Bound & Tight? \\
\midrule
\multicolumn{6}{@{}l}{\emph{What can be shared} (fixed effects)}\\
Thm.~\ref{thm:main-rate} & minimax compound risk & every $m\ge1$
  & $\Theta(\tau,V)$, $\tau\le1$, $V\le2$, non-atomic $\mu$ & $\Mm_{m,n}(\tau,V;\mu)\asymp\Phi(m,n,\tau,V)$ & yes, up to universal constants \\
Prop.~\ref{prop:separation} & worst-case ratio $\kappa_{g_0}$ to $\Phi$ & $m^2\ge216n$
  & per-question rules, shift model $V=0$ & best per-question rule: $\asymp\min\{n,\log(m^2/n)\}$;
  compound hold-out: $O(1)$ (non-atomic $\mu$) & yes, in order \\
\midrule
\multicolumn{6}{@{}l}{\emph{What is necessary for decisions} (random effects)}\\
Thm.~\ref{thm:main-eb} & identification regret $\Reg_n$ & $m=\infty$
  & $\Gmono$, every $\mu$ & $\Reg_n(G;\Gmono)\le2^{-n}$ at every prior $G$, guaranteed by the posterior
  mean of any data-equivalent prior & open: base in $[2,\bstar]$ \\
Cor.~\ref{cor:finite-m} & worst-case excess over $R^*_n$ & every $m\ge1$
  & $\Gmono$, every $\mu$ & $\ge1/(8\tbinom{2n}{n})$ and $\ge\bstar^{-n}/(72\sqrt n)$
  & lower bound only (open problem~2, Sec.~\ref{sec:open}) \\
Thm.~\ref{thm:betastar} & $\Reg_n$ within $\Gconst$ & $m=\infty$
  & $\Gconst$ (constant maps) & $\bstar^{-n}/(72\sqrt n)\le\sup_H\Reg_n\le\bstar^{-n}$
  & exact base $\bstar$; factor $72\sqrt n$ open \\
Cor.~\ref{cor:skibinsky} & Skibinsky's gap $W_n$ (Rem.~\ref{rem:skibinsky}) & $m=\infty$
  & $\Gconst$ & $\bstar^{-n}/(72\sqrt n)\le\sup_HW_n\le2^{-n}$ & no: base in $[2,\bstar]$;
  numerically $2$ (E2) \\
Conj.~\ref{conj:monotone-base} & identification regret $\Reg_n$ & $m=\infty$
  & $\Gmono$, every $\mu$ & $(\sup_G\Reg_n)^{1/n}\to1/\bstar$ & open \\
\bottomrule
\end{tabular}
\caption{The results at a glance. \emph{Fixed effects}: the maps are arbitrary members of the class;
\emph{random effects}: they are independent draws from a prior, and $m=\infty$ means that the law of
one question's data is known. ``In order'': up to universal constants, which are not optimized.
Evidence grades (\cref{sec:numerics}): E1, seeded simulation sweep; E2, global search, values attained
but optimality not certified; E3, exact rational arithmetic.}
\label{tab:results}
\end{table}

\paragraph{Related work.}
Risk bounds for isotonic least squares~\citep{zhang2002risk,chatterjee2015risk,guntuboyina2018shape}
hold in the empirical norm; we prove population-norm bounds for independent, non-identically
distributed pairs by bracketing chaining~\citep{vandervaart1996weak,vandegeer2000empirical}. The rate~\eqref{eq:intro-rate} follows the
shared-plus-local template of multitask learning~\citep{baxter2000model}. Its closest parametric
instance is the adaptive and robust estimator of~\citet{duan2023adaptive}, with compound risk of order
$d/(mn)+\min\{\delta^2,d/n\}$ up to logarithmic factors for $d$-dimensional parameters within distance
$\delta$ of a common centre; \cref{thm:main-rate} is the shape-constrained, Bernoulli analogue, with
universal constants and no logarithmic factor. Limits of adaptation in multitask
learning~\citep{hanneke2022nofreelunch,hanneke2026limits} concern unknown task relatedness, while
\cref{prop:separation} prices tuning each question on its own data in a known class where one compound
hold-out adapts.

Our identification regret differs from the empirical-Bayes regret of~\citet{polyanskiy2021sharp}, the
excess risk of a data-driven procedure over the oracle that knows the prior: ours is measured against a
Bayes risk that the data cannot identify~\citep{skibinsky1989admissible}, so it is a floor, not the
estimation cost of a procedure. It is the posterior regret of $\Gamma$-minimax
estimation~\citep{riosinsua1995posterior}, with $\Gamma$ the priors that induce the same law of the
data, averaged over the data before the supremum over $\Gamma$ is taken. For one real estimand the
minimax action is the midpoint of the range of posterior means, and \cref{prop:chebradius} is the
functional version of this rule. Methods for recalibration with few labels per unit
exist~\citep{li2024fewshot,ji2020trust,raj2026pebs}, and calibration across groups is a large
literature~\citep{zhao2021calibrating,gopalan2024omnipredictors}; we ask what is achievable in the
minimax sense. \Cref{sec:related} gives a more detailed account.

\paragraph{Scope and organization.}
The risk is the excess Brier risk. It upper-bounds the $L_2$ calibration error, so our upper bounds
transfer to calibration error and our lower bounds do not (\cref{app:calibration}). We assume
monotone maps and one output law shared by all questions; \cref{thm:main-rate} also needs that law to
be non-atomic, while the empirical-Bayes results hold for every $\mu$. The empirical-Bayes upper bound
is a population statement, and a finite-$m$ procedure is open; \cref{app:scope} lists all limitations.
\Cref{sec:prelim} sets up the model and the risk. \Cref{sec:results} gives the minimax rate, the
procedure that attains it adaptively and the separation from per-question tuning, with proof outlines.
\Cref{sec:identification} gives the identification regret: its bounds, its representation as a
Chebyshev radius, the confinement for monotone maps, the exact base for constant maps, and its relation
to Skibinsky's gap. \Cref{sec:numerics} illustrates the results numerically, and \cref{sec:discussion}
compares the two benchmarks, draws practical conclusions and lists open problems. \Cref{app:supp} gives
the extended related work and the limitations, and \cref{app:rate,app:eb} the proofs.

%% file: sections/2_prelim.tex
\section{Setting}
\label{sec:prelim}

Throughout, $\mu$ is a Borel probability measure on $[0,1]$, $[n]\defeq\set{1,\dots,n}$, and
$\nmu{f}^2=\int_0^1 f^2\,d\mu$.

\subsection{Many questions, few labels}
\label{sec:model}

\begin{definition}[Model]\label{def:model}
Let $m,n\ge 1$. Question $j\in[m]$ carries a nondecreasing \emph{recalibration map}
$g_j:[0,1]\to[0,1]$ and $n$ labelled pairs $(P_{jk},Y_{jk})$, $k\in[n]$, with $P_{jk}\sim\mu$ and,
conditionally on $P_{jk}$, $Y_{jk}\sim\Bern(g_j(P_{jk}))$. All $mn$ pairs are mutually independent,
within and across questions.
\end{definition}

The output law $\mu$ is shared by all questions; only the maps differ. The upper bound uses this: with a
common $\mu$ the pooled fit tracks the plain average of the maps (\cref{app:model-remarks}). We write $\Fmono$ for the set of
nondecreasing maps $[0,1]\to[0,1]$. For $d:[0,1]\to\R$ the \emph{variation} $\TV(d)$ is the supremum of
$\sum_{l=1}^{L}\abs{d(x_l)-d(x_{l-1})}$ over all $L\ge1$ and $0\le x_0<x_1<\dots<x_L\le1$, a seminorm
(constant functions have variation zero).

\begin{definition}[Heterogeneity class]\label{def:class}
For $\tau,V\ge0$ let $\Theta(\tau,V)$ be the set of $m$-tuples $(g_1,\dots,g_m)$ of nondecreasing maps
$[0,1]\to[0,1]$ for which there is one nondecreasing \emph{witness} $g_0:[0,1]\to\R$ with
\[
  \sup_{x\in[0,1]}\abs{g_j(x)-g_0(x)}\le\tau
  \quad\text{and}\quad
  \TV(g_j-g_0)\le V
  \qquad\text{for all } j\in[m].
\]
The \emph{deviation class} is
$\Dcl(\tau,V)\defeq\set{d:[0,1]\to\R \text{ measurable} : \sup\abs{d}\le\tau,\ \TV(d)\le V}$; thus
$(g_j)\in\Theta(\tau,V)$ if and only if $g_j-g_0\in\Dcl(\tau,V)$ for all $j$ and some nondecreasing $g_0$.
\end{definition}

We call $g_0$ the \emph{shared shape}. The parameter $\tau$ measures how far the maps are from it, and
$V$ how much the deviations oscillate. For $\tau=0$ the problem is a single isotonic regression with
$mn$ samples; for $V=0$ every deviation $g_j-g_0$ is a constant of size at most $\tau$: the \emph{shift
model}. Requiring the shared shape to take values in $[0,1]$ gives a smaller class; for it, $\sup\abs{g_j-g_0}\le1$ and
$\TV(g_j-g_0)\le\TV(g_j)+\TV(g_0)\le2$, so the ranges $\tau\in[0,1]$ and $V\in[0,2]$ cover every case
for that class (\cref{app:model-remarks} explains the choice of sup-distance and variation).

\subsection{Risk and the minimax rate}
\label{sec:risk}
\label{sec:minimax}

An \emph{estimator} $\gh=(\gh_1,\dots,\gh_m)$ is a jointly measurable map from the $mn$ pairs (and,
possibly, an independent auxiliary randomization) to $m$ real functions on $[0,1]$; each $\gh_j$ may
use the data of all questions. Its \emph{compound risk} at $g=(g_1,\dots,g_m)$ is
\begin{equation}\label{eq:risk}
  \cR(\gh;g)\defeq\frac1m\sum_{j=1}^m\E\,\nmu{\gh_j-g_j}^2 .
\end{equation}
For a fresh pair $(P,Y)$ of question $j$, independent of the data,
$\E[(\gh_j(P)-Y)^2\mid\text{data}]-\E[(g_j(P)-Y)^2]=\nmu{\gh_j-g_j}^2$, because $\E[Y\mid P]=g_j(P)$
removes the cross term. So $\cR$ is the average excess Brier score of the recalibrated forecasts
$\gh_j(P)$ over the oracle forecasts $g_j(P)$. It is at least their average $L_2$ calibration error,
and it also asks that the recalibrated forecasts keep the resolution of the original model
(\cref{app:calibration}). The \emph{minimax compound risk} is
\begin{equation}\label{eq:minimax}
  \Mm_{m,n}(\tau,V;\mu)\defeq\inf_{\gh}\ \sup_{(g_j)\in\Theta(\tau,V)}\ \cR(\gh;g),
\end{equation}
the infimum taken over all estimators. Its order of magnitude will turn out to be
\begin{align}
  \Phi(m,n,\tau,V)&\defeq(mn)^{-2/3}+\min\set{\tau^2,\tfrac1n}+\min\set{\tau^2,V^2,(V/n)^{2/3}},
  \label{eq:Phi}\\
  \Psi(m,n,\tau,V)&\defeq(mn)^{-2/3}+\min\set{\tau^2,\tfrac1n+(V/n)^{2/3}} ,\label{eq:Psi}
\end{align}
two expressions that agree up to a factor $2$ (\cref{thm:main-rate}): $\Psi$ makes one learn-or-ignore
choice for the whole deviation, $\Phi$ one for each of its two parts (\cref{rem:regimes} uses $\Psi$).
Each term of $\Phi$ prices one object at its own sample size, through the metric entropy of the class
it lives in: the logarithm of the number of $\eps$-balls needed to cover that class. The shared shape is a monotone function, with entropy
of order $1/\eps$, learned from all $mn$ labels at rate $(mn)^{-2/3}$. A question's offset is one
number, learned from $n$ labels at rate $1/n$ or ignored at cost $\tau^2$. Its wiggle lives in a
variation ball, with entropy of order $V/\eps$, learned at rate $(V/n)^{2/3}$ or ignored at cost
$\min\set{\tau^2,V^2}$. The law $\mu$ does not enter $\Phi$, and for non-atomic $\mu$ it cannot enter
the minimax risk either: only the order of the outputs matters (\cref{rem:reparam}).

\subsection{Excess Brier risk and calibration error}
\label{app:calibration}

For fixed data, the forecast $q\defeq\gh_j(P)$ has squared calibration error
$\mathrm{CE}_j\defeq\E[(q-\E[Y\mid q])^2]$. Since $\sigma(q)\subseteq\sigma(P)$, the tower property gives
$\E[Y\mid q]=\E[g_j(P)\mid q]$, the orthogonal projection of $g_j(P)$ onto the $\sigma(q)$-measurable
square-integrable variables, a subspace that contains $q$. A projection does not increase distances,
so
\[
  \mathrm{CE}_j=\E\big[\big(\E[q-g_j(P)\mid q]\big)^2\big]\le\E\big[(q-g_j(P))^2\big]=\nmu{\gh_j-g_j}^2 ,
\]
with equality when $q$ determines $P$ (for instance, when $\gh_j$ is strictly increasing). Every upper
bound on $\cR$ in \cref{sec:results} is therefore an upper bound on the average calibration error. Lower bounds do
\emph{not} transfer. At $\tau=0$ all maps equal one map $g$; let every $\gh_j$ be the constant $\bar Y$, the
mean of all $mn$ labels. For fixed data the forecast is a constant, so $\E[Y\mid q]=\int g\,d\mu$ and
$\mathrm{CE}_j=(\bar Y-\int g\,d\mu)^2$. The $mn$ labels are i.i.d.\ $\Bern(\int g\,d\mu)$, so the expected
calibration error is $\Var(\bar Y)\le1/(4mn)$, far below the lower bound $c\,(mn)^{-2/3}$ of
\cref{thm:main-rate} once $mn$ is large. Calibration alone is cheap; the risk $\cR$ asks in addition that
the recalibrated forecasts keep the resolution of the original model.

\subsection{Remarks on the model}
\label{app:model-remarks}

\paragraph{The shared output law.}
The pooled fit of \cref{sec:rate} tracks the plain average $\gbar$ of the maps, not the shared shape $g_0$,
because every $P_{jk}$ has law $\mu$ (\cref{lem:pooled}); and, because that target is a plain average of
maps in the class, a question's deviation from it stays in the deviation class $\Dcl(2\tau,2V)$ of
\cref{def:class} (\cref{prop:oracle}, Step~2). If question $j$ had its own law $\mu_j$, with density
$p_j$ with respect to a common measure, the population least-squares target of the pooled pairs would
be $\sum_jp_jg_j/\sum_jp_j$, which need not be monotone (take $g_1\equiv0$, $g_2\equiv1$ and a ratio
$p_2/p_1$ that is not monotone), so the pooled isotonic fit would target its monotone projection, not
an average of the maps; this extension is open problem~6.

\paragraph{Why sup-distance and variation.} The two parameters do different jobs. The sup-distance $\tau$
bounds, pointwise and hence for every $\mu$, what a question loses by ignoring its deviation $g_j-g_0$;
measuring it in $\Lmu$ instead would make the class itself depend on the output law. The variation $V$
bounds the metric entropy of the deviations (of order $V/\eps$ at scale $\eps$, besides one offset
direction), and this is what lets $n$ labels estimate a deviation at the rate $(V/n)^{2/3}$. The class is
a nonparametric envelope, and for structured families it is coarse. A per-question temperature and
intercept, $g_j=\operatorname{logit}^{-1}(a_j+b_j\operatorname{logit}g_0)$ with $b_j\ge0$, is a
two-parameter family. Its variation is at most four times its sup-distance $\tau$ (proof below), so for
such a family the variation constraint adds little to the sup constraint, and the class does not
exploit that the family has only two parameters. A warping $g_j=g_0\circ\phi_j$ with $\phi_j$ a
nondecreasing self-map of $[0,1]$ always satisfies $\TV(g_j-g_0)\le\TV(g_j)+\TV(g_0)\le2$ for a
$[0,1]$-valued $g_0$, but its sup-distance can be as large as the largest jump of $g_0$, even for
$\phi_j$ close to the identity. \Cref{thm:main-rate} is the price of assuming none of these structures.

\paragraph{The factor $4$ for a temperature and an intercept.} In the coordinate
$t=\operatorname{logit}g_0(x)\in[-\infty,\infty]$ the deviation is
$\Delta(t)=\operatorname{logit}^{-1}(a+bt)-\operatorname{logit}^{-1}(t)$; since
$x\mapsto\operatorname{logit}g_0(x)$ is nondecreasing, the variation of the deviation in $x$ is the
variation of $\Delta$ along a subset of the extended line, and its sup-distance is the sup of
$\abs\Delta$ there. For $b>0$, $\Delta(\pm\infty)=0$ and, as the derivative of $\operatorname{logit}^{-1}$
at $u$ is $1/(4\cosh^2(u/2))$, $\Delta'(t)=0$ if and only if $\cosh((a+bt)/2)=\sqrt b\cosh(t/2)$; for
$b\ne1$ this is an exponential sum with exponent-ordered coefficient signs $(+,-,-,+)$ or $(-,+,+,-)$, so
it has at most two real zeros (Descartes' rule of signs for exponential sums; for $b=1$ two terms of
opposite sign remain, so at most one zero, unless $a=0$ and $\Delta\equiv0$). So $\Delta$ rises from $0$,
falls and rises back to $0$ (or the mirror image, or is one-signed), and along any subset of the line
its variation is at most $4$ times its sup there (adjoin $\pm\infty$, where $\Delta=0$; this neither
decreases the variation nor changes the sup, and a sequence from $0$ to $0$ with at most two turns has
variation $2(\max-\min)\le4\sup\abs\Delta$). For $b=0$, $\Delta$ is monotone and the ratio is at most
$2$. The constant $4$ is attained on the full line: for $a=0$ and $b\notin\{0,1\}$, $\Delta$ is odd, so
its maximum and minimum have the same size.

\begin{remark}[Only the order of the outputs matters]\label{rem:reparam}
Let $\mu$ be non-atomic, with distribution function $F$ and quantile function
$F^{\leftarrow}(u)\defeq\inf\set{x\in[0,1]:F(x)\ge u}$. Composition with a nondecreasing map preserves
monotonicity and does not increase sup-distances or variations; $F(P)$ is uniform for $P\sim\mu$;
$F^{\leftarrow}(F(P))=P$ almost surely; and $F(F^{\leftarrow}(u))=u$. Hence the substitutions
$(P,Y)\mapsto(F(P),Y)$, $g\mapsto g\circ F^{\leftarrow}$ and, in the other direction,
$(U,Y)\mapsto(F^{\leftarrow}(U),Y)$, $g\mapsto g\circ F$, carry the model, the class $\Theta(\tau,V)$ and
the risk~\eqref{eq:risk} for $\mu$ to those for the uniform law on $[0,1]$ and back (Step~0 of the proof of
\cref{lem:lower} spells out one direction). So $\Mm_{m,n}(\tau,V;\mu)$ is the same for every non-atomic
$\mu$, and constants free of $\mu$ in \cref{thm:main-rate} come at no extra cost. The procedure of
\cref{def:procedure} is built from pool-adjacent-violators and a variation-constrained fit, whose fitted
values depend on the design points only through their order.
\end{remark}

%% file: sections/3_rate.tex
\section{What can be shared: the compound minimax rate}
\label{sec:results}

This section answers the first question, for fixed effects. In one sentence: pooling buys the shared shape at the rate $(mn)^{-2/3}$ of all $mn$ labels, each
question still pays $\min\set{\tau^2,1/n}$ for its offset and the variation cost of its own deviation, and
a single compound hold-out attains this without knowledge of $\tau$, $V$ or $\mu$.
\Cref{sec:procedure} defines the procedure, \cref{sec:rate} states the minimax rate and outlines its
proof, \cref{sec:erm} isolates the least-squares bound that removes the logarithm from the offset term,
\cref{sec:adaptive} shows that the procedure attains the rate without knowing $\tau$, $V$ or $\mu$, and
\cref{sec:separation} shows why its tuning must be shared across questions.

\subsection{The procedure}
\label{sec:procedure}

\begin{definition}[Pooled isotonic fit, two-stage estimator, compound hold-out]\label{def:procedure}
\leavevmode
\begin{enumerate}[label=(\roman*)]
\item \emph{Isotonic fit.} The isotonic fit to a set of pairs is the least-squares fit over $\Fmono$
  taken as the right-continuous step interpolation of the pool-adjacent-violators values at the sorted
  design points, constant below the smallest and above the largest design point; on the null event of
  tied design points, the tied points are pooled into one weighted pool-adjacent-violators block, a rule
  that does not use $\mu$.
\item \emph{Two-stage estimator with radii $(\tau',V')\in(0,1]\times(0,2]$} (for $m\ge2$). For each
  $j$: let $\gh_0^{(-j)}$ be the isotonic fit to the pooled pairs of all questions $i\ne j$; let
  $\dhat_j$ be a non-randomized, jointly measurable least-squares selection over $\Dcl(\tau_2,V_2)$,
  $\tau_2\defeq\min\set{2\tau',1}$, $V_2\defeq\min\set{2V',2}$, fitted to the residuals
  $Y_{jk}-\gh_0^{(-j)}(P_{jk})$, $k\in[n]$, that does not use $\mu$ (for instance the right-continuous step
  interpolation of the unique fitted values at the sorted design points); and let
  $\gh_j\defeq\clip_{[0,1]}\big(\gh_0^{(-j)}+\dhat_j\big)$.
\item \emph{Compound hold-out} (for $m\ge3$, $n\ge2$). Split the pairs of question $j$ into $A_j$ (the
  first $\floor{n/2}$ pairs) and $B_j$ (the rest). Let $K\defeq\ceil{\log_2(mn)}$ and
  $\Lambda\defeq\set{(2^{-a},2^{1-b}):a,b\in\set{0,1,\dots,K}}$. For $\lambda\in\Lambda$ let
  $\gh^{\lambda}=(\gh^\lambda_1,\dots,\gh^\lambda_m)$ be the two-stage estimator with parameters $\lambda$
  computed from the $A$-halves only (its pooled stage uses the $A$-halves of the other questions). Select
  \[
    \widehat\lambda\in\argmin_{\lambda\in\Lambda}\ \sum_{j=1}^m\sum_{k\in B_j}\big(Y_{jk}-\gh^\lambda_j(P_{jk})\big)^2
  \]
  (ties broken by a fixed order) and output $\gh_j^{\widehat\lambda}$, $j\in[m]$.
\item \emph{The adaptive procedure $\gh^{\mathrm{ad}}$.} If $n=1$, every $\gh^{\mathrm{ad}}_j$ is the isotonic
  fit to all $m$ pairs; else, if $m\le2$, $\gh^{\mathrm{ad}}_j$ is the isotonic fit to the $n$ pairs of
  question $j$; else ($m\ge3$, $n\ge2$), $\gh^{\mathrm{ad}}$ is the compound hold-out estimator of (iii).
\end{enumerate}
\end{definition}

The procedure fits the shape on the pooled data and corrects each question on its own labels
(\cref{fig:procedure}). The radii are doubled because the target of the correction, the deviation
$g_j-\gbar^{(-j)}$ of question $j$ from the average map $\gbar^{(-j)}$ of the other questions, lies in
$\Dcl(2\tau,2V)$ when $(g_j)\in\Theta(\tau,V)$; the caps $1$ and $2$ are the largest sup-norm and
variation of a difference of two nondecreasing $[0,1]$-valued maps.

\begin{figure}[htb]
\centering
\resizebox{\linewidth}{!}{\input{figures/tikz_procedure}}
\caption{The compound hold-out procedure of \cref{def:procedure}(iii) ($m\ge3$, $n\ge2$). The shape is
learned from the pooled $A$-halves of the other questions, each question fits its own constrained
correction for every candidate $\lambda$, and a single $\widehat\lambda$, the size of the deviations, is
chosen from the hold-out halves of all questions together. For $n=1$ or $m\le2$ the procedure is a plain
isotonic fit (\cref{def:procedure}(iv)).}
\label{fig:procedure}
\end{figure}
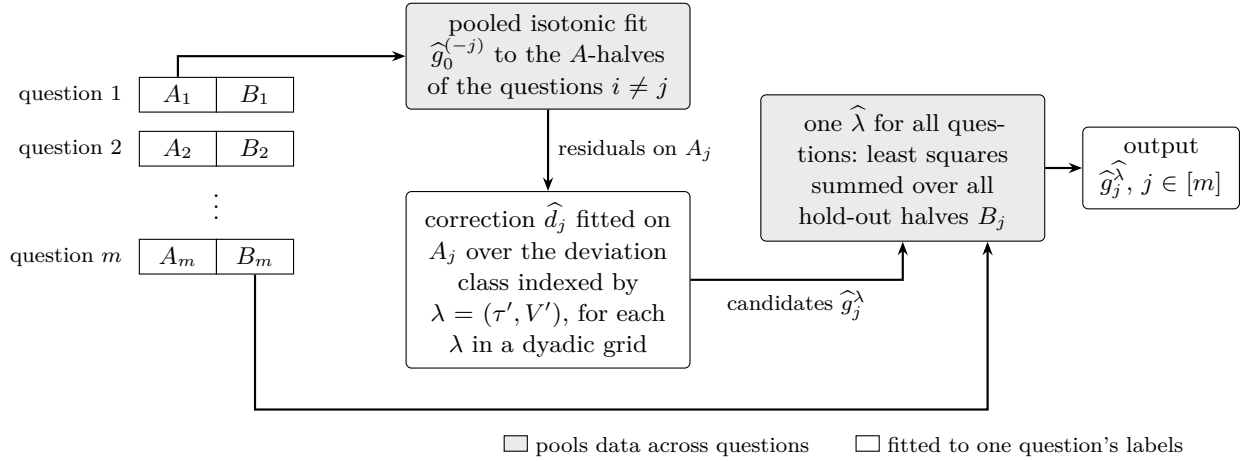

\subsection{The exact compound minimax rate}
\label{sec:rate}

\begin{restatable}[Exact compound minimax rate]{theorem}{ThmRate}\label{thm:main-rate}
There are universal constants $0<c\le C<\infty$ such that for every non-atomic Borel probability measure
$\mu$ on $[0,1]$, all $m,n\ge1$, $\tau\in[0,1]$ and $V\in[0,2]$,
\[
  c\,\Phi(m,n,\tau,V)\ \le\ \Mm_{m,n}(\tau,V;\mu)\ \le\ C\,\Phi(m,n,\tau,V),
\]
whether or not the witness $g_0$ of \cref{def:class} is required to take values in $[0,1]$. The upper
bound is attained, up to the constant $C$, by one non-randomized jointly measurable procedure that uses
only the data and $(m,n)$, and neither $(\tau,V)$ nor $\mu$: a pooled isotonic fit with a
bounded-variation correction, tuned by a compound hold-out, when $m\ge3$ and $n\ge2$; the pooled isotonic
fit when $n=1$; and per-question isotonic fits when $m\le2$ and $n\ge2$ (\cref{def:procedure}). Moreover $\Phi/2\le\Psi\le2\Phi$.
\end{restatable}

Each question learns its offset and its wiggle from its own labels or ignores them, whichever is
cheaper (\cref{sec:minimax}); \cref{rem:regimes} and \cref{fig:regimes} read the rate in its regimes.

\paragraph{The picture behind the rate.}
\Cref{sec:minimax} prices each term of $\Phi$ through the metric entropy of one object at its own sample
size. Heterogeneity never enters the first term, and $m$ never enters the last two. The lower bound
makes this exact with a compound Assouad argument, one hypercube per object. On the upper side, two
facts carry the argument. First, least squares only needs independent observations with a common output
law, so the pooled fit tracks the average of $m$ different maps, and heterogeneity averages out. Second,
the offset is a single direction, and it should cost $1/n$, not $\log(en)/n$: the logarithm that a
global entropy count charges disappears once entropy is counted locally, around the target
(\cref{prop:erm2}). Finally, tuning each question on its own data cannot adapt. A rule that does so and
is to stay within a constant of the rate when all questions coincide must keep its risk at order
$(mn)^{-2/3}$, far below its own parametric $1/n$; by a constrained-risk
inequality~\citep{brown1996constrained} it then remains biased for shifts of order
$\sqrt{\log(m^2/n)/n}$, which its $n$ labels cannot reliably detect. A compound hold-out pays its
selection error with the hold-out labels of all questions, about $mn/2$ of them.

\begin{remark}[Reading the rate]\label{rem:regimes}
\leavevmode
\begin{enumerate}[label=(\alph*)]
\item \emph{Heterogeneity is paid in variation, not in distance.} If $\tau^2\ge 1/n+(V/n)^{2/3}$, then
  $\Psi=(mn)^{-2/3}+1/n+(V/n)^{2/3}$. A single question on its own has rate of order $n^{-2/3}$ (take
  $m=1$: $n^{-2/3}\le\Phi(1,n,\tau,V)\le4n^{-2/3}$). Pooling therefore helps through $V$: in the shift
  model $V=0$ the rate is $(mn)^{-2/3}+\min\set{\tau^2,1/n}$, parametric in $n$ even when $\tau=1$.
\item \emph{One label per question.} For $n=1$, $m^{-2/3}+\tau^2\le\Phi(m,1,\tau,V)\le m^{-2/3}+2\tau^2$,
  and the optimal procedure is the pooled isotonic fit, which ignores the identity of the questions.
\item \emph{Template.} $\Phi$ is the shape-constrained, nonparametric counterpart of the multi-task rate
  $d/(mn)+\min\set{\delta^2,d/n}$ (up to logarithmic factors and an outlier term) for $d$-dimensional
  Gaussian means within distance $\delta$ of a common centre \citep{duan2023adaptive}: the parametric
  $d/(mn)$ becomes the isotonic rate $(mn)^{-2/3}$, and the local term splits into an offset and a
  variation part.
\end{enumerate}
\end{remark}

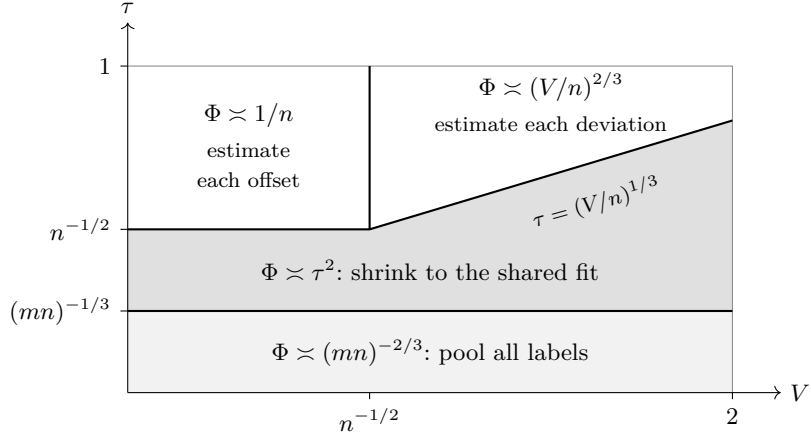
\begin{figure}[t]
\centering
\input{figures/tikz_regimes}
\caption{Which term of $\Phi$ dominates, for fixed $m,n$ with $m^2\gg n$ (schematic; logarithmic axes;
boundaries up to constant factors, read off $\Psi$ in \cref{thm:main-rate}). The shift model $V=0$ is the
left edge. In the two shaded regions the pooled isotonic fit alone attains the rate (for non-atomic $\mu$), since its risk is
$O((mn)^{-2/3}+\tau^2)$ (\cref{lem:pooled}, and $\abs{g_j-\gbar}\le2\tau$); above them each question pays for
its own offset ($V\le n^{-1/2}$) or for its own deviation ($V\ge n^{-1/2}$). When $m^2\lesssim n$ the
pooling term $(mn)^{-2/3}\ge1/n$ absorbs the offset region.}
\label{fig:regimes}
\end{figure}

\paragraph{Outline of the proof of \cref{thm:main-rate}.} The lower bound is a compound Assouad argument over three hypercubes inside
$\Theta(\tau,V)$ (\cref{lem:lower}): bumps of height of order $(mn)^{-1/3}$ shared by all questions,
offsets $\pm\min\set{\tau,1/(8\sqrt n)}$ chosen independently per question, and per-question bumps whose
number $K_b\asymp\max\set{1,(nV^2)^{1/3}}$ and height of order $\min\set{\tau,V/K_b}$ exhaust the variation
budget; a quantile transform reduces every non-atomic $\mu$ to Lebesgue measure, which is where
non-atomicity enters the lower bound. The upper bound analyses a two-stage estimator. A leave-one-out pooled isotonic
fit tracks the average map $\gbar^{(-j)}$ of the other questions at rate $(mn)^{-2/3}$, although no
single pair has $\gbar^{(-j)}$ as its regression function (\cref{lem:pooled}); the deviation
$g_j-\gbar^{(-j)}$ lies in $\Dcl(2\tau,2V)$, and a constrained least-squares correction on question $j$'s
own labels pays the pooled error plus $\min\set{\tau^2,1/n+\min\set{V^2,(V/n)^{2/3}}}$ by an oracle
inequality (\cref{lem:local,prop:oracle}). The difficulty concentrates in the term $1/n$: a global
entropy bound (\cref{lem:erm1}) charges the offset direction $\log(en)/n$, and the localized bound of
\cref{prop:erm2} removes this logarithm, which is what makes the rate exact rather than exact up to
$\log n$. Adaptation costs only a constant factor because the tuning parameter is chosen once, for all
questions together (\cref{thm:adaptive}).

\subsection{Least squares with non-identical pairs and a common design}
\label{sec:erm}

Both stages of the estimator are least-squares fits to independent pairs that are \emph{not} identically
distributed: the pooled stage mixes questions with different maps, and the correction stage regresses
residuals around an estimated offset. What the pairs do share is the law of the design points. The bound
below is the analogue, for this design, of the localized rate theorem for least squares
\citep[Theorem~3.4.1]{vandervaart1996weak}, which we need in $\Lmu$ for independent, non-identically
distributed pairs; we state it for completeness. For a
class $\mathcal{F}$ of functions on $[0,1]$, the bracketing number $\Nbr(\eps,\mathcal{F},\Lmu)$ is the
least number of pointwise $\eps$-brackets covering $\mathcal{F}$: pairs $[l,u]$ of real measurable
functions with $\nmu{u-l}\le\eps$, where $f\in[l,u]$ means $l\le f\le u$ at every point of $[0,1]$; the
brackets need not lie in $\mathcal{F}$ \citep{vandervaart1996weak}.

\begin{restatable}[Localized least squares, non-identical pairs, common design]{proposition}{PropERMtwo}\label{prop:erm2}
Let $(P_i,Y_i)$, $i\in[N]$, be independent random pairs with values in $[0,1]\times[-1,1]$ such that
$P_i\sim\mu$ for every $i$; let $\eta_i(p)\defeq\E[Y_i\mid P_i=p]$ (a version) and
$\etabar\defeq\frac1N\sum_i\eta_i$. Let $\mathcal{F}$ be a nonempty convex class of measurable functions
$[0,1]\to[-1,1]$, closed in $\Lmu$, let $f^*$ be the $\Lmu$-projection of $\etabar$ onto $\mathcal{F}$,
and let $\fhat$ be a measurable selection of $\argmin_{f\in\mathcal{F}}\sum_{i=1}^N(Y_i-f(P_i))^2$
(assumed to exist). For $\delta>0$ let $\mathcal{F}(\delta)\defeq\set{f\in\mathcal{F}:\nmu{f-f^*}\le\delta}$,
and assume that for some $A,B\ge0$
\[
  \log\Nbr\big(\eps,\mathcal{F}(\delta),\Lmu\big)\ \le\ \frac{A}{\eps}+B\log\frac{e\delta}{\eps}
  \qquad\text{for all } 0<\delta\le2 \text{ and } 0<\eps\le\delta .
\]
Then, with a universal constant $C$ (the expectation may be read as an outer expectation, which equals
$\E$ whenever $\nmu{\fhat-f^*}$ is measurable),
\[
  \E\,\nmu{\fhat-f^*}^2\ \le\ C\Big[\Big(\frac{A}{N}\Big)^{2/3}+\frac{B+1}{N}\Big].
\]
\end{restatable}

Three features of \cref{prop:erm2} matter for the first part of the paper.
\emph{(i) The target is the average.} The expected empirical risk is $N$ times the average of the
individual risks, whose minimizer over $\mathcal{F}$ is the projection of $\etabar$; the empirical
process is centred summand by summand, and its natural norm
$(\frac1N\sum_i\E h(P_i,Y_i)^2)^{1/2}$ reduces to the $\Lmu$ norm on functions of the design point
because every $P_i\sim\mu$. The heterogeneity of the $\eta_i$ therefore never enters the entropy, and this
is why the pooled fit tracks the average map (\cref{lem:pooled}).
\emph{(ii) Localization removes the logarithm.} As in the identically distributed case, a bracketing
chaining bound with Bernstein truncation, here proved for independent non-identically distributed summands,
controls the empirical process on $\mathcal{F}(\delta)$, and peeling over $\delta$ locates the fixed
point $\sqrt N\delta^2\asymp\sqrt{A\delta}+\delta\sqrt B$. The geometric quantity is the local entropy
integral: since $\int_0^\delta\sqrt{B\log(e\delta/\eps)}\,d\eps\le\sqrt2\,\delta\sqrt B$ is
scale-free, the finite-dimensional part costs $(B+1)/N$; under a global bound $B\log(e/\eps)$
(\cref{lem:erm1}) the same integral carries a factor $\sqrt{\log(e/\delta)}$, and the bound becomes
$(B+1)\log(eN)/N$.
\emph{(iii) The deviation class.} For $\Dcl(\tau,V)$ the hypothesis holds with $A$ of order $V$ and
$B=2$; the term $B\log(e\delta/\eps)$ is the entropy of the one-dimensional offset direction inside the
localized ball. This gives the correction stage its log-free rate $(V/n)^{2/3}+1/n$ (\cref{lem:local}).

\subsection{Adaptation by a compound hold-out}
\label{sec:adaptive}

\Cref{thm:adaptive} is the upper half of \cref{thm:main-rate}, stated for the procedure itself; we write
$C_{\mathrm{AD}}$ for its constant $C$.

\begin{restatable}[Adaptive upper bound]{theorem}{ThmAdaptive}\label{thm:adaptive}
Let $\mu$ be non-atomic, $m,n\ge1$, and let $\tau\in[0,1]$ and $V\in[0,2]$ be unknown. The procedure
$\gh^{\mathrm{ad}}$ of \cref{def:procedure}(iv) is non-randomized, jointly measurable in (data, $p$), and
uses only the data and $(m,n)$. There is a universal constant $C$, independent of $\mu,m,n,\tau,V$, such
that for every $(g_1,\dots,g_m)\in\Theta(\tau,V)$,
\[
  \frac1m\sum_{j=1}^m\E\,\nmu{\gh^{\mathrm{ad}}_j-g_j}^2\ \le\ C\,\Phi(m,n,\tau,V).
\]
\end{restatable}

\paragraph{Outline of the proof of \cref{thm:adaptive}.} For $m\ge3$ and $n\ge2$, Bernstein's inequality applied conditionally on the
$A$-halves, with a union bound over the $\abs{\Lambda}=(K+1)^2$ candidates, shows that the selected
candidate has average risk at most $5/3$ times the best candidate's plus $25(\log(2\abs\Lambda)+1)/(m\,\abs{B_j})$
in expectation (\cref{prop:holdout}). The grid contains a point within a factor $2$ of
$(\max\set{\tau,2^{-K}},\max\set{V,2^{1-K}})$, at which the two-stage estimator on half-samples attains
$C\Phi$ (\cref{prop:oracle}); the selection cost is of order $\log\log(mn)/(mn)$, below $(mn)^{-2/3}$. The
point of choosing \emph{one} $\lambda$ for all questions is that the selection cost is divided by the
total hold-out size $m\abs{B_j}\ge mn/2$; the same argument applied per question only gives an overhead
of order $\log\abs\Lambda/n$, which exceeds $\Phi$ when $\tau$ is small and $m$ is large. For $n=1$ the pooled fit pays
$m^{-2/3}$ plus a bias of at most $4\tau^2$, and for $m\le2$ the per-question fit pays $n^{-2/3}$; both are
$O(\Phi)$ in their regimes. \Cref{prop:separation} shows that the loss of per-question tuning is not an
artefact of this analysis.

\subsection{Per-question tuning cannot adapt}
\label{sec:separation}

The obvious alternative to a compound hold-out tunes each question on
its own labels; it cannot adapt. In words, a per-question estimator computes the map of question $j$
from that question's data and a randomization fixed in advance. Formally, call an estimator
\emph{per-question} if $\gh_j=T_j(D_j,\xi,\cdot)$ for every $j\in[m]$, where
$D_j\defeq((P_{jk},Y_{jk}))_{k\in[n]}$ is the data of question $j$; $\xi$ is an auxiliary random element
with values in a measurable space $\Xi$, independent of the $mn$ pairs, whose law is fixed in advance (it
depends on none of $(g_j)$, $\tau$, $V$, and it may be shared by the questions); and each
$T_j:([0,1]\times\set{0,1})^n\times\Xi\times[0,1]\to\R$ is jointly measurable. The maps $T_j$ may differ
across $j$ and may depend on $(m,n,\mu,j)$ and on anything fixed in advance, but not on the data of
other questions, on $(\tau,V)$, or on $(g_j)$. The lower bound is proved on shifts of a shared shape that the rules may know: for
a nondecreasing $g_0:[0,1]\to[\frac14,\frac34]$ and $\tau\in[0,\frac14]$, ranges that keep every shifted
map in $[0,1]$, let $\Theta_{g_0}(\tau)$ be the set of tuples $(g_0+c_1,\dots,g_0+c_m)$ with
$c_j\in[-\tau,\tau]$, and define the \emph{adaptivity ratio}, the worst-case ratio of the risk to the rate,
\[
  \kappa_{g_0}(\gh)\defeq\sup_{\tau\in[0,1/4]}\ \sup_{g\in\Theta_{g_0}(\tau)}\
  \frac{\cR(\gh;g)}{\Phi(m,n,\tau,0)}\ \in[0,\infty],
  \quad \Phi(m,n,\tau,0)=(mn)^{-2/3}+\min\set{\tau^2,\tfrac1n}.
\]

\begin{restatable}[Per-question tuning cannot adapt]{proposition}{PropSeparation}\label{prop:separation}
Let $\mu$ be any Borel probability measure on $[0,1]$, let $m,n\ge1$, let $g_0:[0,1]\to[\frac14,\frac34]$
be nondecreasing, and put $M\defeq(m^2/n)^{1/3}$.
\begin{enumerate}[label=(\roman*)]
\item (Lower bound.) If $m^2\ge216n$, then
  $\kappa_{g_0}(\gh)\ge\frac1{256}\min\set{3n,\log(m^2/n)}$ for every per-question estimator $\gh$.
  Consequently every per-question $\gh$, whether
  or not it uses $g_0$, satisfies
  \[
    \sup_{\tau\in[0,1]}\ \sup_{g\in\Theta(\tau,0)}\ \frac{\cR(\gh;g)}{\Phi(m,n,\tau,0)}
    \ \ge\ \frac1{256}\min\Bigl\{3n,\log\frac{m^2}n\Bigr\}.
  \]
\item (Tightness within per-question rules.) Let $\hat c_j\defeq\frac1n\sum_{k}(Y_{jk}-g_0(P_{jk}))$. The
  per-question rule $\gh_j\defeq g_0+\hat c_j\mathbf 1\set{\abs{\hat c_j}>t}$, with $t\defeq0$ if $M<e$ and
  $t\defeq(28\log M/(3n))^{1/2}$ if $M\ge e$, has $\kappa_{g_0}\le e/4$ if $M<e$ and
  $\kappa_{g_0}\le4+50\log(m^2/n)$ if $M\ge e$; the rule $\gh_j\defeq g_0$ has $\kappa_{g_0}\le\max\set{1,n/16}$.
  Hence, if $m^2\ge216n$,
  \[
    \frac1{256}\min\Bigl\{3n,\log\frac{m^2}n\Bigr\}\ \le\ \inf_{\gh\ \text{per-question}}\kappa_{g_0}(\gh)
    \ \le\ 51\min\Bigl\{n,\log\frac{m^2}n\Bigr\},
  \]
  and if $m^2<e^3n$ (that is, $M<e$), the rule above (with $t=0$) has $\kappa_{g_0}\le e/4$.
\item (Compound tuning adapts.) Let $\mu$ be non-atomic and let $C_{\mathrm{AD}}$ be the constant of
  \cref{thm:adaptive}. The procedure $\gh^{\mathrm{ad}}$ of \cref{def:procedure}(iv), which uses neither
  $g_0$, $\tau$, $V$ nor $\mu$ and for $m\ge3$ is not per-question, has
  $\kappa_{g_0}(\gh^{\mathrm{ad}})\le C_{\mathrm{AD}}$. Hence, if $m^2\ge216n$, every per-question procedure,
  even one given $g_0$ and $\mu$, has $\kappa_{g_0}(\gh)/C_{\mathrm{AD}}\ge\min\set{3n,\log(m^2/n)}/(256\,C_{\mathrm{AD}})$,
  a ratio that tends to infinity as $n\to\infty$ and $m^2/n\to\infty$; in particular, when
  $\min\set{3n,\log(m^2/n)}>256\,C_{\mathrm{AD}}$, no per-question procedure matches the guarantee of
  $\gh^{\mathrm{ad}}$.
\end{enumerate}
\end{restatable}

\paragraph{Why per-question tuning fails.} At $\tau=0$ the class $\Theta_{g_0}(0)$ is the single tuple
$(g_0,\dots,g_0)$, so a per-question rule $\gh$ has risk at most $\kappa_{g_0}(\gh)(mn)^{-2/3}$ there,
far below the $1/n$ at which one question's own labels can estimate its offset; it must shrink every
estimated offset almost to zero. The constrained-risk inequality of
\citet[Theorem~1]{brown1996constrained}, carried to a Bernoulli shift family, measures the cost. For one
question, the likelihood ratio between shift $c$ and shift $0$ has second moment at most
$e^{16nc^2/3}$, so the mean, under shift $c$, of the offset that the rule estimates is at most
$e^{8nc^2/3}$ times its root risk under shift $0$. For $nc^2$ of order $\log(m^2/n)$ these means
stay below $c/2$ on average over the questions, and the rule pays $c^2/4$, of order $\log(m^2/n)/n$, against
$\Phi(m,n,c,0)\le2/n$. A compound procedure escapes the trade-off because the radius it uses is
selected on the hold-out labels of all questions together, about $mn/2$ of them
(\cref{thm:adaptive}).

\begin{remark}[Scope of \cref{prop:separation}]\label{rem:separation}
\leavevmode
\begin{enumerate}[label=(\alph*)]
\item The separation is asymptotic: since $C_{\mathrm{AD}}$ is not explicit, the proposition gives no
  concrete $(m,n)$ at which per-question rules lose to $\gh^{\mathrm{ad}}$; \cref{sec:numerics} gives
  evidence at moderate sizes.
\item Per-question rules use their own data only. Tuning each question around a shape estimated from
  the pooled data is not covered, since pooled statistics can carry information about the size of the
  shifts (\cref{app:scope}(vii); open problem~7 of \cref{sec:open}).
\item The inequality behind (i) is the constrained-risk inequality
  of~\citet[Theorem~1]{brown1996constrained}, with an exact chi-square computation for the Bernoulli
  shift family; we prove the form we use (\cref{app:separation}).
\end{enumerate}
\end{remark}

%% file: figures/tikz_procedure.tex
\begin{tikzpicture}[font=\small,
    box/.style={draw,rounded corners=2pt,align=center,inner sep=4pt,text width=3.4cm},
    shared/.style={box,fill=black!8},
    arr/.style={-{Stealth[length=5pt]},thick},
    cell/.style={draw,minimum height=0.45cm,minimum width=1.0cm,inner sep=0pt}]
  \foreach \lab/\y in {1/1.5,2/0.8,m/-0.6} {
    \node[cell] (A\lab) at (0.5,\y) {$A_{\lab}$};
    \node[cell]              (B\lab) at (1.5,\y) {$B_{\lab}$};
    \node[anchor=east,font=\footnotesize] at (-0.1,\y) {question $\lab$};
  }
  \node at (1.0,0.15) {$\vdots$};
  \node[shared] (pool) at (5.3,2.0)
    {pooled isotonic fit $\gh_0^{(-j)}$ to the $A$-halves of the questions $i\ne j$};
  \node[box] (corr) at (5.3,-0.9)
    {correction $\dhat_j$ fitted on $A_j$ over the deviation class indexed by $\lambda=(\tau',V')$, for each $\lambda$ in a dyadic grid};
  \node[shared] (hold) at (9.9,0.55)
    {one $\widehat\lambda$ for all questions: least squares summed over all hold-out halves $B_j$};
  \node[box,text width=1.75cm] (out) at (13.25,0.55) {output\\ $\gh_j^{\widehat\lambda}$, $j\in[m]$};
  \draw[arr] (A1.north) |- (pool.west);
  \draw[arr] (pool) -- node[right,font=\footnotesize,align=left] {residuals on $A_j$} (corr);
  \draw[arr] (corr.east) -| node[pos=0.25,below,font=\footnotesize] {candidates $\gh_j^{\lambda}$} (hold.south);
  \draw[arr] (Bm.south) -- ++(0,-1.75) -| ([xshift=1.1cm]hold.south);
  \draw[arr] (hold) -- (out);
  \node[font=\footnotesize,anchor=west] at (4.6,-3.05)
    {\tikz\draw[fill=black!8] (0,0) rectangle (0.3,0.2);\ pools data across questions\qquad
     \tikz\draw (0,0) rectangle (0.3,0.2);\ fitted to one question's labels};
\end{tikzpicture}

%% file: figures/tikz_regimes.tex
\begin{tikzpicture}[x=0.8cm,y=0.72cm,font=\small]
  \fill[black!5]  (0,0) rectangle (10,1.5);
  \fill[black!12] (0,1.5) -- (10,1.5) -- (10,5) -- (4,3) -- (0,3) -- cycle;
  \draw[black!45] (0,6) -- (10,6) -- (10,0);
  \draw[->] (0,0) -- (10.8,0) node[right] {$V$};
  \draw[->] (0,0) -- (0,6.8) node[above] {$\tau$};
  \draw[thick] (0,1.5) -- (10,1.5);
  \draw[thick] (0,3) -- (4,3) -- (10,5);
  \draw[thick] (4,3) -- (4,6);
  \foreach \y/\l in {1.5/{(mn)^{-1/3}},3/{n^{-1/2}},6/{1}} {\draw (0,\y) -- (-0.12,\y) node[left] {$\l$};}
  \foreach \x/\l in {4/{n^{-1/2}},10/{2}} {\draw (\x,0) -- (\x,-0.12) node[below] {$\l$};}
  \node at (5,0.75) {$\Phi\asymp(mn)^{-2/3}$:\ pool all labels};
  \node at (5,2.25) {$\Phi\asymp\tau^2$:\ shrink to the shared fit};
  \node[align=center] at (2,4.55) {$\Phi\asymp1/n$\\[2pt] {\footnotesize estimate}\\ {\footnotesize each offset}};
  \node[align=center] at (7.0,5.35) {$\Phi\asymp(V/n)^{2/3}$\\[2pt] {\footnotesize estimate each deviation}};
  \node[rotate=15.6,font=\footnotesize,inner sep=1pt] at (7.75,3.5) {$\tau=(V/n)^{1/3}$};
\end{tikzpicture}

%% file: sections/4_identification.tex
\section{What is necessary for decisions: the identification regret}
\label{sec:identification}

This section answers the second question, for random effects. In one sentence: not knowing the prior over monotone recalibration maps costs at most $2^{-n}$ in excess
Brier score, for every output law and for the posterior mean of every monotone prior consistent with the
law of the data, and in the worst case it costs at least $\bstar^{-n}/(72\sqrt n)$, with
$\bstar=2e^{2\Gcat/\pi}\approx3.583$.
\Cref{sec:eb} sets up random effects and defines the identification regret, and \cref{sec:barrier}
states the bounds. \Cref{sec:chebradius,sec:monotone} explain the upper bound: the regret is a Chebyshev
radius, and for monotone maps the unidentified moment is confined by a diagonal-domination inequality.
\Cref{sec:finite-m} shows that the lower bounds hold for every finite number of questions,
\cref{sec:constant} pins the exponential base for constant maps, \cref{app:skibinsky} relates the regret
to Skibinsky's minimax-minus-Bayes gap, and \cref{sec:conjecture} states our conjecture for monotone maps.

\subsection{Random effects and the identification regret}
\label{sec:eb}

Now the $g_j$ are independent with a common law $G$ on monotone maps. The Bayes rule uses $G$, which
other questions reveal only through the law of one question's data; the second part measures what is
lost even when this law is known exactly (the idealization $m=\infty$).

\begin{definition}[Random-effects setting]\label{def:seb}
Let $\mu$ be any Borel probability measure on $[0,1]$ (atoms allowed) and $n\ge1$. A \emph{prior} $G$
is the law of a random map $g:[0,1]\to[0,1]$ that is jointly measurable in (sample point, argument).
One question's data is $D=(X,Y)$, where $X=(P_1,\dots,P_n)$ has i.i.d.\ coordinates with law $\mu$ and
is independent of $g$ (precisely: $g$ is $\mathcal{A}\otimes\mathcal{B}([0,1])$-measurable for a
$\sigma$-field $\mathcal{A}$ independent of $X$), and, given $\mathcal{A}\vee\sigma(X)$, the labels
$Y_1,\dots,Y_n$ are independent with $Y_k\sim\Bern(g(P_k))$. Write $s(Y)\defeq\sum_kY_k$ and, for
$r\ge1$, $t\in[0,1]^r$, $x\in[0,1]^n$, $y\in\set{0,1}^n$, $p\in[0,1]$,
\begin{gather*}
  M^G_r(t)\defeq\E_G\big[g(t_1)\cdots g(t_r)\big],\qquad
  Z_G(x,y)\defeq\E_G\Big[\prod_{k=1}^n g(x_k)^{y_k}(1-g(x_k))^{1-y_k}\Big],\\
  N_G(p;x,y)\defeq\E_G\Big[g(p)\prod_{k=1}^n g(x_k)^{y_k}(1-g(x_k))^{1-y_k}\Big].
\end{gather*}
$Z_G$ is the density of $D$ with respect to $\mu^{\otimes n}\otimes(\text{counting measure on }\set{0,1}^n)$.
A \emph{rule} $\delta$ assigns to each $D$ a function $\delta(D):[0,1]\to[0,1]$ such that
$(p,D)\mapsto\delta(D)(p)$ is jointly measurable. Its risk is
$R(\delta,G)\defeq\E_G\E\,\nmu{\delta(D)-g}^2$, and the Bayes risk is
$R^*_n(G)\defeq\inf_\delta R(\delta,G)$. The \emph{posterior mean} is the rule
$\delta^*_G(D)(p)\defeq N_G(p;D)/Z_G(D)$ on $\set{Z_G(D)>0}$ and $0$ elsewhere.
For a convex set $\cG$ of priors and $G\in\cG$, the \emph{identified set} and the
\emph{identification regret} are
\begin{equation}\label{eq:reg}
  [G]\defeq\set{G'\in\cG:\Law_{G'}(D)=\Law_G(D)},
  \qquad
  \Reg_n(G;\cG)\defeq\inf_{\delta}\ \sup_{G'\in[G]}\big(R(\delta,G')-R^*_n(G')\big).
\end{equation}
Conventions: $0\cdot(+\infty)\defeq0$ and $1/0\defeq+\infty$.
\end{definition}

The posterior mean attains the Bayes risk, $R(\delta^*_G,G)=R^*_n(G)$ (\cref{prop:chebradius}(i)). The
regret $\Reg_n(G;\cG)$ is a minimax excess risk over the priors that the law of the data cannot
distinguish from $G$. At $G$ alone the question would be empty, since the posterior mean of $G$ has
zero excess there. It is a population quantity: that law is known exactly.

\paragraph{Two classes of priors.}
$\Gmono$ is the set of priors admitting a jointly measurable version whose paths are nondecreasing maps
$[0,1]\to[0,1]$. It is convex: a mixture of two such priors is the law of $Bg_1+(1-B)g_2$ with $B$ an
independent Bernoulli variable. $\Gconst\subseteq\Gmono$ is the set of priors on \emph{constant} maps,
$g\equiv\theta$ with $\theta\sim H$ for a Borel probability measure $H$ on $[0,1]$; we identify the
prior with $H$. This is the case of an \emph{uninformative predictor}, whose reported probability
carries no information about the label: the classical binomial empirical-Bayes
problem \citep{robbins1956empirical}.

\paragraph{Constant maps: moments.}
For constant maps the regret has a closed form (\cref{sec:constant}), built from two quantities. For $H\in\Gconst$
and $s=0,\dots,n$ let $b_s(H)\defeq\int\theta^s(1-\theta)^{n-s}\,dH(\theta)$, the prior probability of
one label sequence with $s$ positive labels, and let $\rho_n(H)$ be the length of the interval of
values of $\int\theta^{n+1}dH'$ over probability measures $H'$ on $[0,1]$ with the same first $n$
moments as $H$: the range of the one moment that $n$ labels leave free. Always $\rho_n(H)\le4^{-n}$,
with equality for the \emph{arcsine law} $\omega(dt)\defeq dt/(\pi\sqrt{t(1-t)})$ on $[0,1]$
(\citealp{skibinsky1967range}; see also \citealp[eq.~(17)]{skibinsky1968minimax} and \cref{prop:binomial}). For measurable $f>0$ we write
$\GM_\omega(f)\defeq\exp\int\log f\,d\omega$ for the geometric mean under $\omega$, and $\mathcal{P}_n$
for the real polynomials of degree at most $n$. Finally,
$\Gcat\defeq\sum_{k\ge0}(-1)^k/(2k+1)^2=0.91596\ldots$ is Catalan's constant, and
$\bstar\defeq2e^{2\Gcat/\pi}=3.58324\ldots$, which we round to $3.583$.

\paragraph{The coin example.} One label and constant maps (coins) already show the identification
problem. The priors $\delta_{1/2}$ (every question is a fair
coin) and $\frac12\delta_0+\frac12\delta_1$ (every question is deterministic, half of them positive)
give the same label law, $\Bern(1/2)$. Under the first the label is worthless and the Bayes rule
predicts $1/2$; under the second it is decisive and the Bayes rule predicts $Y$. A rule that predicts
$a_y$ after the label $y$ has excess risk
$\frac12(a_0-\frac12)^2+\frac12(a_1-\frac12)^2$ against the first prior and
$\frac12a_0^2+\frac12(1-a_1)^2$ against the second. The two sum to at least $1/8$, so one
of them is at least $1/16$, with equality at $(a_0,a_1)=(\frac14,\frac34)$; this rule and
the value $1/16$ go back to \citet[Section~5]{robbins1964empirical} and
\citet[Section~5]{skibinsky1968minimax}.

\subsection{The identification regret for monotone maps}
\label{sec:barrier}

For exchangeable questions the target is the Bayes risk, which can be exponentially smaller than the
minimax rate (\cref{sec:intro}), and reaching it needs the Bayes rule. In the setting of
\cref{def:seb}, one question's $n$ labels identify the product moments $M^G_r$ with $r\le n$, while the
posterior mean needs $M^G_{n+1}$. The next theorem measures the resulting loss.

\begin{restatable}[Identification regret for empirical-Bayes recalibration]{theorem}{ThmEB}\label{thm:main-eb}
In the setting of \cref{def:seb}, let $[G]$ denote identified sets within $\Gmono$. For every Borel
probability measure $\mu$ on $[0,1]$ (atoms allowed) and every $n\ge1$,
\[
  \max\set*{\frac{1}{8\binom{2n}{n}},\ \frac{\bstar^{-n}}{72\sqrt n}}
  \ \le\ \sup_{G\in\Gmono}\Reg_n(G;\Gmono)\ \le\ 2^{-n},
\]
and for every $G\in\Gmono$ the posterior-mean rule of any $G''\in[G]$ has excess risk at most $2^{-n}$
against every $G'\in[G]$: $R(\delta^*_{G''},G')-R^*_n(G')\le2^{-n}$. For priors on constant maps, with
the identification regret computed within $\Gconst$, the exponential base is exact:
\[
  \frac{\bstar^{-n}}{72\sqrt n}\ \le\ \sup_{H\in\Gconst}\Reg_n(H;\Gconst)\ \le\ \bstar^{-n},
\]
so that $\big(\sup_{H\in\Gconst}\Reg_n(H;\Gconst)\big)^{1/n}\to1/\bstar$ as $n\to\infty$. At $n=1$ the lower
bound $1/(8\binom21)=1/16$ equals the Bayes risk $R^*_1$ of the arcsine constant-map prior.
\end{restatable}

\paragraph{Why it is true.} The coin example of \cref{sec:intro} is the case $n=1$. With $n$ labels
every product moment of order at most $n$ is identified, and only the $(n+1)$-point moment is free. For
constant maps it ranges over an interval of length at most $4^{-n}$ (\cref{sec:eb}), and for monotone
maps \cref{lem:ch1} gives the same confinement (see below). Recall that $Z_G(x,y)$ is the probability of
the label pattern $y$ at outputs $x$, and $N_G/Z_G$ is the posterior mean. By
\cref{prop:chebradius}(i)--(ii) below, the excess risk of the posterior mean of $G$ against another prior
in $[G]$ is a sum over the $2^n$ label patterns, integrated over the outputs, of $U^2/Z_G$, where $U$ is
the gap between the two top moments. Write $U^2/Z_G=\abs U\cdot(\abs U/Z_G)$. The first factor is at most
$4^{-n}$; the second is the gap between two posterior means, so it is at most $1$. Summing gives
$2^n\cdot4^{-n}=2^{-n}$.

\begin{remark}[Scope of \cref{thm:main-eb}]\label{rem:eb-scope}
\leavevmode
\begin{enumerate}[label=(\alph*)]
\item The trivial bound is $1/4$ (the rule $\delta\equiv1/2$ has risk at most $1/4$), so the upper bound
  $2^{-n}$ is informative for $n\ge3$.
\item Of the two lower bounds, the first is larger for $n\le62$ and the second for $n\ge63$
  (\cref{rem:eb-lower-crossover}). For every $n\le60$, explicit
  constant-map priors have regret, computed within $\Gconst$, between $0.22\,\bstar^{-n}$ and
  $0.44\,\bstar^{-n}$; for instance $0.43\,\bstar^{-n}$, $0.41\,\bstar^{-n}$ and $0.37\,\bstar^{-n}$ at
  $n=10,20,30$ (E3 lower bounds; \cref{sec:numerics}). So
  $\bstar^n\sup_H\Reg_n(H;\Gconst)\ge0.22$ for $n\le60$, against the $1/(72\sqrt n)\le0.014$ that the
  theorem guarantees.
\item The exact base $\bstar$ is a statement about the class $\Gconst$, with the regret computed within
  $\Gconst$. \Cref{app:scope}, items (iv) and (v), gives the scope of the upper bound ($m=\infty$) and of
  the base for monotone priors.
\end{enumerate}
\end{remark}

\subsection{The regret as a Chebyshev radius}
\label{sec:chebradius}

\paragraph{Why the price of non-identification is exponentially small.}
\Cref{prop:chebradius} turns the identification regret into geometry: it is the squared Chebyshev
radius of the set of unidentified top moments, in an $L_2$ space weighted by the inverse
probabilities of the $2^n$ label patterns. Upper bounds thus reduce to bounding the diameter of the
unidentified moment, and lower bounds to finding a wide identified set at a prior with a rare label
pattern. For coins the diameter is classical: given the first $n$ moments of a law on $[0,1]$, the
$(n+1)$st ranges over an interval of length at most $4^{-n}$ (\cref{prop:binomial}(b)), a product of
canonical moments~\citep{skibinsky1967range,dette1997canonical}. For monotone maps the role of the
coin is played by the values of the map at $n+1$ sorted points, a sorted vector in $[0,1]^{n+1}$, and
the needed one-dimensional fact is that the monic shifted Chebyshev polynomial of degree $n+1$ has
sup-norm $2^{-(2n+1)}$ on $[0,1]$. Our key lemma (\cref{lem:ch1}) transfers this bound to sorted
vectors: the multi-affine product of matched Chebyshev factors is dominated by its values on the
diagonal. \Cref{sec:barrier} explains how the confinement and the regret representation combine.
The \emph{Chebyshev radius} of a set, which appears next, is the radius of the smallest ball containing it
(unrelated to the Chebyshev polynomials below).

\begin{restatable}[Identification regret as a Chebyshev radius]{proposition}{PropChebRadius}\label{prop:chebradius}
In the setting of \cref{def:seb}, let $\cG$ be a convex set of priors and $G\in\cG$. For $G'\in[G]$ put
$U_{G'}(p,x)\defeq M^{G'}_{n+1}(p,x_1,\dots,x_n)-M^{G}_{n+1}(p,x_1,\dots,x_n)$.
\begin{enumerate}[label=(\roman*)]
\item (Pythagoras.) For every rule $\delta$ and every $G'\in[G]$,
  $R(\delta,G')-R^*_n(G')=\E_{D\sim G}\nmu{\delta(D)-\delta^*_{G'}(D)}^2$.
\item (Only the top moment is unidentified.) For every $G'\in[G]$ and $\mu\otimes\Law_G(D)$-almost every
  $(p,D)$ with $Z_G(D)>0$,
  \[
    \delta^*_{G'}(D)(p)-\delta^*_G(D)(p)=(-1)^{n-s(Y)}\,\frac{U_{G'}(p,X)}{Z_G(X,Y)} .
  \]
\item (Chebyshev radius.) Let $w_G(x)\defeq\sum_{y\in\set{0,1}^n}1/Z_G(x,y)$, and let $\cU[G]$ be the set of
  $\mu^{\otimes(n+1)}$-a.e.\ classes of the functions $U_{G'}$, $G'\in[G]$. Every $U\in\cU[G]$ satisfies
  $\abs{U(p,x)}\le\min_yZ_G(x,y)$ for $\mu^{\otimes(n+1)}$-a.e.\ $(p,x)$, and
  \[
    \Reg_n(G;\cG)=\inf_{c\in\cH}\ \sup_{U\in\cU[G]}\ \int\!\!\int w_G(x)\big(c(p,x)-U(p,x)\big)^2\,d\mu(p)\,d\mu^{\otimes n}(x),
  \]
  where $\cH\defeq\set{c:[0,1]^{n+1}\to\R\text{ measurable}:\int\!\int w_Gc^2<\infty}$ (so $c=0$ a.e.\ on
  $\set{w_G=\infty}$, and $0\cdot\infty=0$ there). That is, $\Reg_n(G;\cG)$ is the squared Chebyshev radius
  of $\cU[G]$ in the weighted space $\cH$. The infimum defining $\Reg_n$ is unchanged if real-valued rules
  are admitted. A Chebyshev centre $c^*$ exists, and for a Borel representative of it, the rule equal to
  $\clip_{[0,1]}\big(\delta^*_G(D)+(-1)^{n-s(Y)}c^*(\cdot,X)/Z_G(X,Y)\big)$ on $\set{Z_G(D)>0}$ and to
  $\delta^*_G(D)$ elsewhere is minimax.
\item (Monotonicity in the prior class.) If $\cG_1\subseteq\cG_2$ are convex sets of priors and
  $G\in\cG_1$, then $\Reg_n(G;\cG_1)\le\Reg_n(G;\cG_2)$.
\end{enumerate}
\end{restatable}

The price of not knowing the prior is geometric: it is the squared radius of the set of top moments that
the data cannot rule out, in a norm that weights each label pattern by its inverse probability ($w_G$).
The prior class enters only through $[G]$.

\paragraph{Proof idea.} (i) is the Pythagorean identity for the posterior mean, an $L_2$ projection. For
(ii), expanding the products $\prod_k(1-g(x_k))$ writes $N_G$ as a signed sum of moments $M^G_r$ at
sub-tuples of $(p,x)$; every moment of order $r\le n$ is identified, and only the full $(n+1)$-point
moment survives. Substituting (ii) into (i) and summing over the $2^n$ label patterns gives (iii).

\paragraph{Details of \cref{prop:chebradius}.}
Two details complete the proof idea: in (i), the expectation may be taken
under $G$ because $G'$ and $G$ induce the same law of $D$; and (iv) holds because a smaller class has a
smaller identified set. In decision-theoretic terms, $\Reg_n(G;\cG)$ is the posterior regret of
$\Gamma$-minimax estimation with $\Gamma=[G]$, averaged over the common law of the data, and (iii)
replaces the classical midpoint rule for one real estimand by a Chebyshev centre of a set of functions;
see \cref{sec:related}.

\subsection{Monotone maps: diagonal domination}
\label{sec:monotone}

Fix outputs $t_0\le\dots\le t_n$ and constants $\zeta_0,\dots,\zeta_n$. Since
every product moment of order at most $n$ is identified, the expectations of
$\prod_{i=0}^n(g(t_i)-\zeta_i)$ under two priors in $[G]$ differ exactly by the difference of their top
moments, which is therefore at most twice the largest value of $\abs{\prod_i(g(t_i)-\zeta_i)}$. For
constant maps, taking the $\zeta_i$ to be the zeros of the monic shifted Chebyshev polynomial of degree
$n+1$, of sup-norm $2^{-(2n+1)}$ on $[0,1]$, gives the classical $4^{-n}$. A nondecreasing map takes
sorted values at sorted outputs, and \cref{lem:ch1}, a diagonal-domination inequality, shows that with
the zeros matched to the values in increasing order the product is never larger in absolute value than
at some point of the diagonal, where it is the Chebyshev polynomial. So two monotone priors with the
same law of the data have top moments within $4^{-n}$ almost everywhere, for every $\mu$
(\cref{thm:monotone-eb}; \cref{app:ch1,app:monotone-eb}).

\begin{restatable}[Diagonal domination for sorted sequences]{lemma}{LemCHone}\label{lem:ch1}
Let $n\ge0$, let $x_0\le x_1\le\dots\le x_n$ and $y_0\le y_1\le\dots\le y_n$ be real numbers, and put
$P(t)\defeq\prod_{i=0}^n(t-y_i)$.
\begin{enumerate}[label=(\alph*)]
\item $\prod_{i=0}^n\abs{x_i-y_i}\le\max_{0\le k\le n}\abs{P(x_k)}$.
\item In particular, if $y_i\defeq\sin^2\!\big(\tfrac{(2i+1)\pi}{4n+4}\big)=\tfrac12\big(1-\cos\tfrac{(2i+1)\pi}{2n+2}\big)$,
  $i=0,\dots,n$ (the zeros, listed increasingly, of the monic shifted Chebyshev polynomial
  $\Ttil_{n+1}(t)\defeq2^{-(2n+1)}T_{n+1}(2t-1)$), and $x_0\le\dots\le x_n$ lie in $[0,1]$, then
  $\prod_{i=0}^n\abs{x_i-y_i}\le2^{-(2n+1)}$.
\end{enumerate}
\end{restatable}

Part (a) says that, among sorted vectors, the product of coordinatewise distances is dominated by its
values on the diagonal. The proof is an induction that peels off the largest pair. If $y_n$ is at least
as near to $x_n$ as to the point $x_k$ that works for the first $n$ pairs, then $x_k$ works for all $n+1$;
otherwise $y_n$ lies left of the midpoint of $x_k$ and $x_n$, every $y_i$ does so too, and replacing $x_k$
by $x_n$ only increases every factor. Both sequences must be sorted: $x=(1,0)$,
$y=(0,1)$ gives $1$ on the left and $0$ on the right, and, by symmetry of $P$ in $y$, $x=(0,1)$,
$y=(1,0)$ shows the same for $y$. The value $2^{-(2n+1)}$ in (b) is the least sup-norm
on the ordered simplex $\set{0\le x_0\le\dots\le x_n\le1}$ among polynomials of degree $n+1$ in $n+1$
variables whose degree-$(n+1)$ coefficients sum to $1$ \citep[Theorem~3.1 and Example~3.3]{dressler2026least};
(b) shows that the \emph{multi-affine} product with matched Chebyshev roots is such a least polynomial.
Multi-affinity is what the moment argument needs: every lower-order monomial of
$\prod_i(\theta_i-y_i)$ is a product over a proper subset of distinct coordinates.

\begin{restatable}[Exponentially small regret for monotone maps]{theorem}{ThmMonotoneEB}\label{thm:monotone-eb}
In the setting of \cref{def:seb}, let $\mu$ be any Borel probability measure on $[0,1]$ (atoms allowed),
$n\ge1$, $G\in\Gmono$, and let $[G]$ be its identified set within $\Gmono$.
\begin{enumerate}[label=(\alph*)]
\item (Diameter of the unidentified moment.) For all $G',G''\in[G]$,
  $\abs{M^{G'}_{n+1}(t)-M^{G''}_{n+1}(t)}\le4^{-n}$ for $\mu^{\otimes(n+1)}$-a.e.\ $t\in[0,1]^{n+1}$; in
  particular $\abs{U_{G'}}\le4^{-n}$ a.e.\ for every $G'\in[G]$.
\item $R(\delta^*_G,G')-R^*_n(G')\le2^{-n}$ for every $G'\in[G]$; hence $\Reg_n(G;\Gmono)\le2^{-n}$.
\item (Plug-in robustness.) For all $G'',G'\in[G]$, $R(\delta^*_{G''},G')-R^*_n(G')\le2^{-n}$.
\end{enumerate}
\end{restatable}

\paragraph{Proof idea.} Fix $t=(t_0,\dots,t_n)$ and sort its coordinates; on a nondecreasing path the
values $\theta_i=g(t_{\pi(i)})$ form a sorted vector. Write $\theta_0\cdots\theta_n=Q(\theta)-q(\theta)$
with $Q(\theta)=\prod_i(\theta_i-a_i)$, $a_i$ the Chebyshev nodes of \cref{lem:ch1}(b); $q$ is a combination of
products over proper subsets, whose expectations are identified. Hence the two top moments differ by
$\E_{G'}Q-\E_{G''}Q$, and $\abs Q\le2^{-(2n+1)}$ by \cref{lem:ch1}(b), giving (a). For (b), both posterior
means lie in $[0,1]$, so $\abs{U_{G'}}/Z_G\le1$ by \cref{prop:chebradius}(ii) and $U_{G'}^2/Z_G\le\abs{U_{G'}}$;
summing over the $2^n$ label patterns in \cref{prop:chebradius}(i) gives $2^n\cdot4^{-n}$. Part (c) is (b)
applied with $G''$ in place of $G$, since $[G'']=[G]$. Non-atomicity of $\mu$ is never used: only distinct
\emph{indices} of the tuple are selected, and coincident values are allowed.

\subsection{Every finite number of questions}
\label{sec:finite-m}

\Cref{thm:main-eb} assumes the law of one question's
data known ($m=\infty$); its lower bounds do not need this.

\begin{restatable}[The identification regret is paid at every finite $m$]{corollary}{CorFiniteM}\label{cor:finite-m}
In the setting of \cref{def:seb}, let $\cG$ be a convex set of priors and $m\ge1$. Under a prior
$G'\in\cG$, the compound data are $m$ independent copies $(g_1,D_1),\dots,(g_m,D_m)$ of $(g,D)$ and an
auxiliary variable $\xi_0\sim\mathrm{Uniform}[0,1]$ independent of them (precisely: the blocks
$(\mathcal A_j,X_j,Y_j)$, $j\in[m]$, each as in \cref{def:seb} with $g_j$ measurable for
$\mathcal A_j\otimes\mathcal B([0,1])$, and $\xi_0$ are independent). A \emph{compound procedure}
$\gh=(\gh_1,\dots,\gh_m)$ consists of jointly measurable maps $(p,D_1,\dots,D_m,\xi_0)\mapsto\gh_j(p)\in[0,1]$;
its \emph{compound excess risk} is
\[
  \mathcal E_m(\gh,G')\defeq\frac1m\sum_{j=1}^m\E_{G'}\nmu{\gh_j-g_j}^2-R^*_n(G').
\]
\begin{enumerate}[label=(\roman*)]
\item For every $G\in\cG$, every $m\ge1$ and every compound procedure $\gh$,
  $\sup_{G'\in[G]}\mathcal E_m(\gh,G')\ge\Reg_n(G;\cG)$.
\item Consequently, for every Borel probability measure $\mu$ on $[0,1]$ and all $n,m\ge1$,
  \[
    \inf_{\gh}\sup_{G\in\Gmono}\mathcal E_m(\gh,G)\ \ge\
    \max\set*{\frac1{8\binom{2n}n},\ \frac{\bstar^{-n}}{72\sqrt n}},
    \qquad
    \inf_{\gh}\sup_{H\in\Gconst}\mathcal E_m(\gh,H)\ \ge\ \frac{\bstar^{-n}}{72\sqrt n}.
  \]
\end{enumerate}
\end{restatable}

\paragraph{Proof idea.} Under every $G'\in[G]$ the other questions, together with $\xi_0$, have the same
law. Averaging $\gh_j$ over this law (Jensen's inequality) and then over $j$ yields one rule $\bar\delta$
for a single question with $\mathcal E_m(\gh,G')\ge R(\bar\delta,G')-R^*_n(G')$ for every $G'\in[G]$. So
however many other questions a procedure observes, it does not beat the identification regret in the
worst case over priors.

\paragraph{The finite-$m$ lower bound and the fixed-effects model.}
Under every $G'\in[G]$ the other
questions and $\xi_0$ have the law $\Law_G(D)^{\otimes(m-1)}\otimes\mathrm{Uniform}[0,1]$, because all
priors in $[G]$ induce the same law of $D$: other questions inform a procedure only through the law of
the data, and that law is the same for all the priors over which the regret is taken.

\Cref{cor:finite-m} is a random-effects statement: the $g_j$ are i.i.d.\ from $G'$, and the excess is
over the Bayes risk $R^*_n(G')$. As a fixed-effects bound it gives only a weak consequence at $\tau=1/2$,
$V=1$: every tuple of nondecreasing $[0,1]$-valued maps lies in $\Theta(1/2,1)$, with shared shape
$g_0\equiv1/2$, and the compound risk is at least the excess. For non-atomic $\mu$, \cref{thm:main-rate}
already dominates this consequence; fixed-effects estimators reduce to compound procedures by clipping
to $[0,1]$ and realizing their randomization from $\xi_0$. The matching finite-$m$ upper bound, an
empirical-Bayes procedure with excess at most $2^{-n}+o(1)$ as $m\to\infty$, is open problem~2 of
\cref{sec:open}.

\subsection{Constant maps: the exact base}
\label{sec:constant}

For a prior $H\in\Gconst$ the data inform about the constant $\theta$ only
through the number of positive labels, and the identification regret has a closed form (\cref{prop:binomial}(a), proved in \cref{app:binomial}):
\begin{equation}\label{eq:binomial-formula}
  \Reg_n(H;\Gconst)=\Big(\frac{\rho_n(H)}2\Big)^2\sum_{s=0}^n\frac{\binom ns}{b_s(H)},
\end{equation}
with the right side read as $0$ when $\rho_n(H)=0$: the squared half-width of the free moment interval,
times the sum over the $2^n$ label sequences of their inverse prior probabilities. Rare label sequences
amplify the regret. The arcsine law $\omega$, whose range is the largest, $4^{-n}$, has regret
$1/(8\binom{2n}n)$ (\cref{prop:binomial}(c)). Formula~\eqref{eq:binomial-formula} is the midpoint case of
the bound in \citet[eq.~(24)]{skibinsky1968minimax}; \cref{app:binomial} traces its classical roots.

\paragraph{Why the base is $\bstar$ and not $4$.}
For constant maps, with the regret computed within constant-map priors, the regret has a closed
form, $(\rho_n/2)^2$ times a weight $\sum_s\binom ns/b_s$
built from the prior probabilities $b_s$ of the label sequences, where $\rho_n$ is the length of the
moment range~\eqref{eq:binomial-formula}. The arcsine law maximizes the range, $\rho_n=4^{-n}$, and
its regret $1/(8\binom{2n}{n})\asymp\sqrt n\,4^{-n}$ suggests a base of $4$. But the arcsine law
leaves every $b_s$ moderate, and near-extremal priors (optimal up to a polynomial factor) trade an
exponential factor in the range for an exponentially rare label pattern. A Szegő-type inequality
for the moment range measures this trade at the exponential scale, and Catalan's constant enters through an integral
of $\log\sin$ over the central half of the arcsine law; the lower-bound priors reweight the arcsine
law on a set of half its mass (\cref{thm:betastar} and its proof idea).

{\setlength{\emergencystretch}{12em}%
\begin{restatable}[Exact exponential base for constant maps]{theorem}{ThmBetaStar}\label{thm:betastar}
In the setting of \cref{def:seb} with priors $H\in\Gconst$ on constant maps and any Borel $\mu$, for
every $n\ge1$:
\begin{enumerate}[label=(\roman*)]
\item $\Reg_n(H;\Gconst)\le\bstar^{-n}$ for every $H\in\Gconst$;
\item $\sup_{H\in\Gconst}\Reg_n(H;\Gconst)\ge\bstar^{-n}/(72\sqrt n)$; for even $n=2k$ it is at least
  $\frac1{16}\binom{2k}{k}4^{-k}\bstar^{-n}$, and an explicit prior $H_k$ (constructed in the proof) has
  regret at least this value;
\item (Szegő-type sandwich.) For every $H$ and every $P\in\mathcal{P}_n$ with $P>0$ on $[0,1]$,
  $\rho_n(H)\le4^{-n}\int P\,dH/\GM_\omega(P)$; and if $H=h\,\omega$ with $\operatorname{ess\,inf}h>0$, then
  $\rho_n(H)\ge\frac12\,4^{-n}\GM_\omega(h)$.
\end{enumerate}
Consequently $\big(\sup_{H}\Reg_n(H;\Gconst)\big)^{1/n}\to1/\bstar$, and no bound of the form
$\sup_H\Reg_n(H;\Gconst)\le C\,n^\alpha4^{-n}$ holds, for any constants $C,\alpha$.
\end{restatable}
}

\paragraph{Proof idea.} The sandwich (iii), a finite-$n$ form of Szegő's theorem (\cref{rem:szego}),
prices the trade between moment range and rare label patterns described above. For (i), apply its
upper half with the test polynomial $Q_s=1+v_s/b_s(H)$, $v_s(t)=t^s(1-t)^{n-s}$, so that
$\int Q_s\,dH=2$. Since $\log(1+u)\ge\max\{0,\log u\}$, restricting to the central half
$\mathcal C=[\sin^2\frac\pi8,\cos^2\frac\pi8]$ of the arcsine law gives
$\GM_\omega(Q_s)\ge b_s^{-1/2}e^{\int_{\mathcal C}\log v_s\,d\omega}$, where Catalan's constant enters
through $\int_{\mathcal C}\log v_s\,d\omega=n(\Gcat/\pi-\log2)$. Hence
$\rho_n^2\le4\cdot4^{-n}\,b_s\,e^{-2\Gcat n/\pi}$, and summing the terms of~\eqref{eq:binomial-formula}
gives $2^n4^{-n}e^{-2\Gcat n/\pi}=\bstar^{-n}$. For (ii), a prior equal to $\omega$ off a half-mass
superlevel set of $v_s$ and proportional to $\omega/v_s$ on it makes one term of the sum match this
bound up to a polynomial factor, by the lower half of (iii). The exact base matters because it shows
that the regret is not governed by the moment range alone.
Already at $n=2$ a symmetric prior with a smaller range than the arcsine law has regret
$(5\sqrt5-11)/8\approx0.0225$, above the arcsine value $1/48\approx0.0208$ (\cref{sec:numerics}).

\begin{remark}[Szegő asymptotics and the constant $\bstar$]\label{rem:szego}
Part (iii) is not new: both halves can be read off from the variational formula for the moment range of
\citet[Theorem~3.1 and Remark~3.1]{gamboa2004large}, which gives the lower half without the factor
$\frac12$; we include a self-contained proof because (i) and (ii) consume exactly these two inequalities.
It is a finite-$n$ form of the Szegő asymptotics $4^n\rho_n(h\,\omega)\to\GM_\omega(h)$, which in
canonical-moment form, $\prod_k4p_k(1-p_k)\to\GM_\omega(h)$, is the Szegő sum rule on $[0,1]$
\citep[Section~2.3]{gamboa2016sum}; on the unit circle, the $(N+1)$-th moment given the first $N$ ranges
over a disk of radius $\prod_{j<N}(1-\abs{\alpha_j}^2)$ in the Verblunsky coefficients $\alpha_j$, and
Szegő's theorem identifies the limit of the product with $\exp\int\log w\,\frac{d\theta}{2\pi}$
\citep[Sections~3 and~8]{simon2005opuc}. The new point of (i)--(ii), stated in \cref{sec:barrier}, is that
the regret is \emph{not} governed by the moment range alone: the arcsine law maximizes the range,
$\rho_n(\omega)=4^{-n}$, and its regret $1/(8\binom{2n}n)$ has base $4$ (\cref{prop:binomial}), while the
supremum has the smaller base $\bstar$, so near-extremal priors (optimal up to a polynomial factor in
$n$) trade moment range for weight in the product $\rho_n^2\sum_s\binom ns/b_s$. Two
equivalent forms of the base are
\[
  \bstar=2\Mah(1+x+y-xy)=4\exp\Big(\tfrac12\int_0^1\log\max\set{t,1-t}\,d\omega(t)\Big),
\]
where $\Mah$ is the Mahler measure; the first is the formula $\Gcat=\frac\pi2\log\Mah(1+x+y-xy)$
\citep[eq.~(1.8)]{kurokawa2006quantum}, and the second follows from the substitution $t=(1-\cos\vartheta)/2$,
which gives $\int_0^1\log\max\set{t,1-t}\,d\omega=\frac8\pi\int_{\pi/4}^{\pi/2}\log\sin u\,du=\frac{4\Gcat}\pi-2\log2$.
The second form makes the gap to the base $4$ of the moment range explicit: $\bstar/4\approx0.896$.
\end{remark}

\paragraph{Relative to the Bayes risk.} For priors on constant maps, the price relative to the Bayes risk is
first-order only for very few labels: Skibinsky's gap is $O(n^{-1/2})$ times the Bayes risk, uniformly
over priors~\citep{joshi1975minimax}, and for the arcsine prior, with the regret computed within
constant-map priors, the ratio $\Reg_n/R^*_n$ is $(n+1)/\binom{2n}{n}$, the reciprocal of the $n$-th
Catalan number (unrelated to Catalan's constant $\Gcat$)---$1$, $\frac12$, $\frac15$, $\frac1{14}$ for
$n=1,2,3,4$ (\cref{app:practice}). The Bayes risk is the benchmark of the second answer and may be
far below the minimax rate of the first (\cref{app:practice}).

\subsection{Skibinsky's gap and the identification regret}
\label{app:skibinsky}

\begin{remark}[Skibinsky's minimax-minus-Bayes gap]\label{rem:skibinsky}
\citet{skibinsky1968minimax} estimates $\theta$ by a rule $a(s)\in[0,1]$ of the count $s=s(Y)$; with
$R(a,H')$ its risk under $H'$ and $[H]$ the priors with the first $n$ moments of $H$ (the identified set
within $\Gconst$), his gap is $W_n(H)\defeq\inf_a\sup_{H'\in[H]}R(a,H')-R^*_n(H)$, and he proves
$\sup_HW_n(H)\le2^{-n}$ \citep[Theorem~3]{skibinsky1968minimax}. Since every function of the count is a
rule, the two quantities are ordered (as explained below):
\begin{equation}\label{eq:skibinsky-order}
  \Reg_n(H;\Gconst)\ \le\ \sup_{H'\in[H]}W_n(H')\qquad\text{for every }H\in\Gconst .
\end{equation}
\end{remark}

\Cref{rem:skibinsky} defines Skibinsky's gap $W_n$ and orders it above the identification regret. The
class of priors with the same first $n$ moments as $H$ is the identified set $[H]$ within $\Gconst$
(Step~1 of the proof of \cref{prop:binomial}), and $R^*_n(H)$ is also the least risk among rules of the
count, because the posterior mean of a constant-map prior depends on the data only through $s(Y)$. The
order~\eqref{eq:skibinsky-order} holds because every function of the count is a rule, so
\[
  \Reg_n(H;\Gconst)\ \le\ \inf_a\sup_{H'\in[H]}\bigl(R(a,H')-R^*_n(H')\bigr)
  \ \le\ \inf_a\sup_{H'\in[H]}R(a,H')-\inf_{H'\in[H]}R^*_n(H'),
\]
and the right side equals $\sup_{H'\in[H]}W_n(H')$ because $[H']=[H]$ for every $H'\in[H]$. So
Skibinsky's theorem gives $\Reg_n(H;\Gconst)\le2^{-n}$ for every $H\in\Gconst$, which
\cref{thm:betastar}(i) improves to $\bstar^{-n}$, and our lower bound transfers to his gap
(\cref{cor:skibinsky}). The two quantities answer different questions. An empirical-Bayes procedure
competes with the Bayes risk of the \emph{true} prior, which the data locate only up to its identified
set; $\Reg_n$ therefore charges the excess over each candidate prior's own Bayes risk, uniformly over
$[H]$, while $W_n$ compares the minimax risk over $[H]$ with the Bayes risk of the single prior $H$. For
the regret $\Reg_n$ the exact base over constant maps is $\bstar$ (\cref{thm:betastar}); for Skibinsky's
gap, \cref{cor:skibinsky} places the base in $[2,\bstar]$. To our knowledge no base other than $2$ had
been stated for either quantity. Relative to the Bayes risk, \citet{joshi1975minimax} proved Robbins'
conjecture that the gap is negligible: $W_n(H)/R^*_n(H)=O(n^{-1/2})$ uniformly over priors.

\begin{corollary}[A lower bound on Skibinsky's gap]\label{cor:skibinsky}
For every $n\ge1$, with the supremum over Borel probability measures $H$ on $[0,1]$,
\[
  \frac{\bstar^{-n}}{72\sqrt n}\ \le\ \sup_HW_n(H)\ \le\ 2^{-n}.
\]
Hence the $\liminf$ and $\limsup$ of $\big(\sup_HW_n(H)\big)^{-1/n}$ lie in $[2,\bstar]$.
\end{corollary}

\begin{proof}
The upper bound is \citet[Theorem~3]{skibinsky1968minimax}. Since every $H'\in[H]$ is itself a Borel
probability measure on $[0,1]$, \eqref{eq:skibinsky-order} gives
$\sup_{H''}W_n(H'')\ge\sup_{H'\in[H]}W_n(H')\ge\Reg_n(H;\Gconst)$ for every $H$; taking the supremum over
$H$ and using the constant-map clause of \cref{thm:main-eb},
$\sup_HW_n(H)\ge\sup_H\Reg_n(H;\Gconst)\ge\bstar^{-n}/(72\sqrt n)$. Take $n$-th roots and use
$(72\sqrt n)^{1/n}\to1$.
\end{proof}

\begin{remark}[Skibinsky's gap appears to decay with base $2$; numerical evidence, E2]\label{rem:skibinsky-base}
The two notions of the price of an unknown prior appear to differ at the exponential scale, and
Skibinsky's analysis shows why (\cref{fig:envelope}). On the identified set the Bayes risk is a concave
quadratic in the free moment $\int\theta^{n+1}dH'$, with leading coefficient $-\sum_s\binom ns/b_s(H)$;
the minimax risk over $[H]$ is its maximum over the moment interval, of length $\rho_n(H)$, and so the gap
maximized over $[H]$ is the range of the Bayes risk over the identified set
\citep[eqs.~(8)--(9), Theorem~1 and eq.~(24)]{skibinsky1968minimax}. Our
regret~\eqref{eq:binomial-formula}, $(\rho_n/2)^2\sum_s\binom ns/b_s$, is the range this quadratic would
have if its vertex sat at the midpoint of the interval: the curvature part alone. When the vertex lies
outside the interval, the range is instead $\rho_n(H)$ times the slope at the midpoint; Skibinsky's proof
of his Theorem~3 bounds the range by $2^n\rho_n(H)\le2^{-n}$. So $W_n$ also counts how the Bayes risk
varies with the free moment, while $\Reg_n$ counts only what a rule loses against each candidate prior's
own Bayes risk. This is why $\Reg_n$ is the empirical-Bayes quantity, and why its exact base $\bstar$
needed an argument of its own (\cref{thm:betastar}).

The numbers agree with this picture: the lower bound of \cref{cor:skibinsky} is valid but appears weak.
A global search over priors (E2; $n\le44$, with the values re-evaluated in exact arithmetic for
$n\le40$) indicates that $\sup_HW_n(H)$ decays with base $2$ up to a polynomial factor, and for $n\ge4$
the numerical maximizers have their vertex outside the interval, far outside for larger $n$. Write
$S^W_n$ for the best values found, which are lower estimates of $\sup_HW_n$: $2^nS^W_n$ falls only from
$0.125$ at $n=1$ to $0.0010$ at $n=44$, with a local power exponent near $2$ for $30\le n\le44$; the local
base $S^W_n/S^W_{n+1}$ is $2.08$ and $2.13$ at $n=40$ and $41$ and still falling on each parity class; and
a fit with a power term on $37\le n\le44$ gives base $2.008$.\footnote{The range $31\le n\le44$ and its
fit windows were added post hoc; on the run declared before the computations, $n\le30$, the same fit
gives base $2.002$ on $25\le n\le30$ and $2.046$ on $20\le n\le30$, and the decision rule fixed in
advance outputs base~$2$.} So Skibinsky's $2^{-n}$ appears to be sharp in the base, and the transferred
bound $\bstar^{-n}/(72\sqrt n)$ appears exponentially loose for $W_n$: at $n=40$ the search value
$S^W_{40}$ exceeds $\bstar^{-40}$, our upper bound on $\sup_H\Reg_{40}(H;\Gconst)$, by a factor of about
$1.7\cdot10^7$. Proving $\sup_HW_n(H)\ge2^{-n}/\poly(n)$ is open problem~9.
\end{remark}

\begin{figure}[t]
\centering
\input{figures/tikz_envelope}
\caption{The two prices of an unknown prior, for priors on constant maps (schematic). On the identified set
$[H]$ the risks depend on $H'\in[H]$ only through the free moment $\int\theta^{n+1}dH'$, which ranges over
an interval of length $\rho_n(H)\le4^{-n}$, and the Bayes risk is a concave quadratic in it
\citep[eqs.~(8)--(9)]{skibinsky1968minimax}; the risk of a fixed rule is affine in it, being linear in the
moments of $H'$. The Bayes rule at the midpoint therefore has the tangent there as its risk, and its
largest excess over the Bayes risk, at the endpoints, is the curvature term
$(\rho_n/2)^2\sum_s\binom ns/b_s=\Reg_n(H;\Gconst)$ of \eqref{eq:binomial-formula}. Skibinsky's gap,
maximized over $[H]$, is the whole range of the Bayes risk over the interval
\citep[Theorem~1]{skibinsky1968minimax}; when the vertex lies outside the interval, as for the numerical
maximizers with $n\ge4$ (\cref{rem:skibinsky-base}), the range is $\rho_n(H)$ times the absolute slope at
the midpoint and can be much larger.}
\label{fig:envelope}
\end{figure}
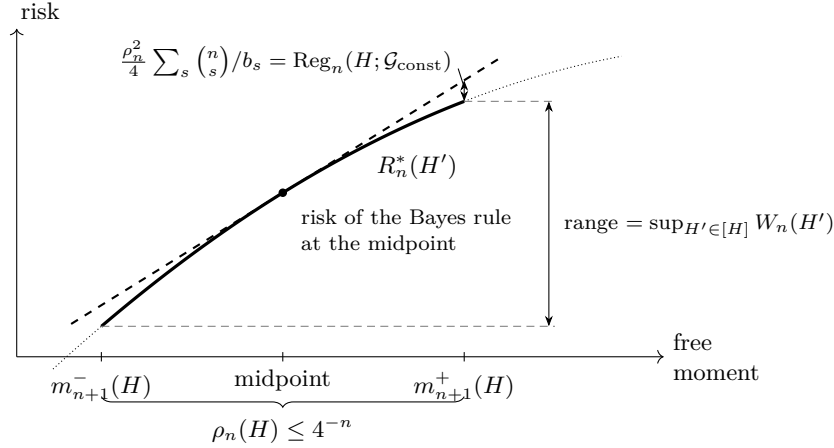

\subsection{The monotone base}
\label{sec:conjecture}

Numerically, Skibinsky's gap appears to have base $2$ (\cref{rem:skibinsky-base}). For monotone maps
the base of the regret lies in $[2,\bstar]$ in the same sense, by \cref{thm:main-eb}, and we conjecture
that monotone maps are no harder than constant ones (evidence below).

\begin{conjecture}[The monotone base is $\bstar$]\label{conj:monotone-base}
For every Borel probability measure $\mu$ on $[0,1]$,
$\lim_{n\to\infty}\big(\sup_{G\in\Gmono}\Reg_n(G;\Gmono)\big)^{1/n}=1/\bstar$.
\end{conjecture}

\paragraph{The conjecture on the monotone base.}
By \cref{thm:main-eb} the $\liminf$ is at least $1/\bstar$, so \cref{conj:monotone-base} asks for the matching upper
bound. \emph{Evidence} (numerical, evidence level E2; not a proof): for $n\le4$, a pointwise version of
the monotone problem---maximize the weighted squared range of $\E\prod_i\theta_i$ over pairs of laws on
ordered tuples $0\le\theta_0\le\dots\le\theta_n\le1$ with equal multilinear moments over every proper
subset of coordinates---was optimized on grids (from $401$ points per coordinate at $n=1$ to $21$ at
$n=4$), and no value above its restriction to the diagonal $\theta_0=\dots=\theta_n$, which is the
constant-map problem, was found; the grid values are lower estimates. What resists is the lift: the
proof of \cref{thm:monotone-eb} transfers the unweighted Chebyshev bound to sorted vectors through
matched-root products, while the base $\bstar$ comes from the weighted bounds of \cref{thm:betastar}, and
we have no transfer of these weighted bounds to sorted vectors.

%% file: figures/tikz_envelope.tex
\begin{tikzpicture}[x=0.8cm,y=0.62cm,font=\small,
    arr/.style={{Stealth[length=4pt]}-{Stealth[length=4pt]}}]
  \draw[->] (0.6,-0.4) -- (11.3,-0.4) node[right,align=left] {free\\moment};
  \draw[->] (0.6,-0.4) -- (0.6,6.6) node[above right=0pt and -2pt] {risk};
  \draw (2,-0.5) -- (2,-0.3) node[below=4pt] {$m^-_{n+1}(H)$};
  \draw (8,-0.5) -- (8,-0.3) node[below=4pt] {$m^+_{n+1}(H)$};
  \draw (5,-0.5) -- (5,-0.3) node[below=3pt] {midpoint};
  \draw[decorate,decoration={brace,amplitude=4pt,mirror}] (2,-1.25) -- (8,-1.25)
    node[midway,below=4pt] {$\rho_n(H)\le4^{-n}$};
  \draw[densely dotted,domain=1.2:2] plot (\x,{6.3-0.05*(\x-13)^2});
  \draw[densely dotted,domain=8:10.6] plot (\x,{6.3-0.05*(\x-13)^2});
  \draw[very thick,domain=2:8,samples=40] plot (\x,{6.3-0.05*(\x-13)^2});
  \node[anchor=north west] at (6.4,{6.3-0.05*(6.4-13)^2}) {$R^*_n(H')$};
  \draw[dashed,thick] (1.5,{3.1+0.8*(1.5-5)}) -- (8.6,{3.1+0.8*(8.6-5)});
  \fill (5,3.1) circle (1.6pt);
  \node[anchor=north west,font=\footnotesize,align=left] at (5.15,2.95) {risk of the Bayes rule\\ at the midpoint};
  \draw[arr] (8,5.05) -- (8,5.5);
  \node[anchor=east,font=\footnotesize,fill=white,inner sep=1pt] at (7.85,5.95) {$\tfrac{\rho_n^2}{4}\sum_s\tbinom ns/b_s=\Reg_n(H;\Gconst)$};
  \draw[thin] (7.9,5.75) -- (8,5.35);
  \draw[densely dashed,black!50] (2,0.25) -- (9.5,0.25);
  \draw[densely dashed,black!50] (8,5.05) -- (9.5,5.05);
  \draw[arr] (9.4,0.25) -- (9.4,5.05);
  \node[anchor=west,font=\footnotesize,align=left] at (9.5,2.4) {range $=\sup_{H'\in[H]}W_n(H')$};
\end{tikzpicture}

%% file: sections/5_numerics.tex
\section{Numerical illustrations}
\label{sec:numerics}

These computations illustrate the theorems; no proof uses them. They are graded E1--E3 as in
\cref{tab:results}; an E3 value that is the identification regret of an explicit prior is a rigorous
lower bound on the supremum. Seeds, parameter ranges and fit windows were fixed before each run (we
call such a run \emph{declared}). Later additions are marked \emph{post hoc}, and every fitted slope is an
in-sample fit to the runs it describes.

\paragraph{Rates (E1).}
Outputs were uniform on $[0,1]$ with Bernoulli labels around one fixed nondecreasing shape, and the
questions differed from it by random shifts or by random bump perturbations whose number of bumps is
the least-favorable one of the lower-bound construction; $n\in\{8,16,\dots,256\}$, $m\in\{10,30,100,300\}$, five
values of $\tau$ and two of $V$, and $30$ replications per configuration ($10{,}800$ runs). The terms of
$\Phi$ appear where the declared range can show them. At $\tau=0$ the pooled isotonic fit has risk
decaying in $mn$ with fitted exponent $-0.68$ ($R^2=0.998$), against the predicted $-2/3$. At
$(\tau,V)=(0.3,0.3)$ a two-stage estimator tuned with the true $(\tau,V)$ has exponent $-0.633$ in $n$
(bootstrap $95\%$ interval $[-0.64,-0.63]$), against $-2/3$. The interval excludes $-2/3$; we read the gap
as pre-asymptotic rather than as a different exponent. The risk of a specific estimator mixes terms with
different constants (at $n=8$ the variation term is still capped by $\tau^2=V^2=0.09$), so one slope
fitted over $n\in[8,256]$ need not equal any single exponent, and slopes on sliding windows of three
values of $n$ range from $-0.62$ to $-0.64$. Its risk is nearly free of $m$: $2.99$, $2.68$, $2.44$ and
$2.43\times10^{-3}$ at $n=256$ for $m=10,30,100,300$. With three constants fitted in sample to the same
runs, the three-term formula $\Phi$ reproduces the risk of the two-stage estimator tuned with the true
$(\tau,V)$ in every configuration within a factor $2.3$ (median deviation $12\%$); this describes the
runs and is not an out-of-sample test. The regime $\min\{\tau^2,V^2\}=V^2$ was not tested, and
per-question isotonic regression has not reached its asymptotic slope by $n=256$ (local exponent
$-0.81$).

\paragraph{The $1/n$ term (E1, post hoc).}
The term $1/n$ can dominate only when $V<\tau$, which the declared range did not contain. A post-hoc
supplement with pure shifts ($V=0$, $\tau=0.3$) gives exponents $-0.92$ and $-0.96$ for two
shift-correcting estimators, against $-1$, and at $\tau=0.1$ the local slope moves from $-0.51$ to
$-0.81$ across $n\approx1/\tau^2=100$, the predicted crossover from the cap $\tau^2$ to $1/n$.

\paragraph{Adaptivity (E1).}
\Cref{fig:e01-adaptivity} compares data-driven tuning rules. Choosing among the fixed
estimators separately for each question, by two-fold cross-validation on that question's
labels (panels (a) and (b)), costs a median factor $2.5$ to $2.6$ and up to $14.6$ over the
best fixed estimator, and at
$\tau=0$ its risk decays in $m$ with exponent only $-0.15$ to $-0.18$ instead of $-2/3$: the
selection noise of $n$ labels does not shrink with $m$. Choosing once for all questions, by
the cross-validation error summed over questions, stays within a factor $1.40$ of the best
fixed estimator everywhere in the declared range (median $1.05$), without knowledge of $(\tau,V)$
(panel (c)). This compound rule is a close relative of the compound hold-out of \cref{thm:adaptive}
(two-fold cross-validation in place of one split, a fixed grid of seven radius pairs in place of the
dyadic grid); it was added post hoc, after the failure of per-question tuning had been seen.

\begin{figure}[tbp]
  \centering
  \includegraphics[width=0.92\textwidth]{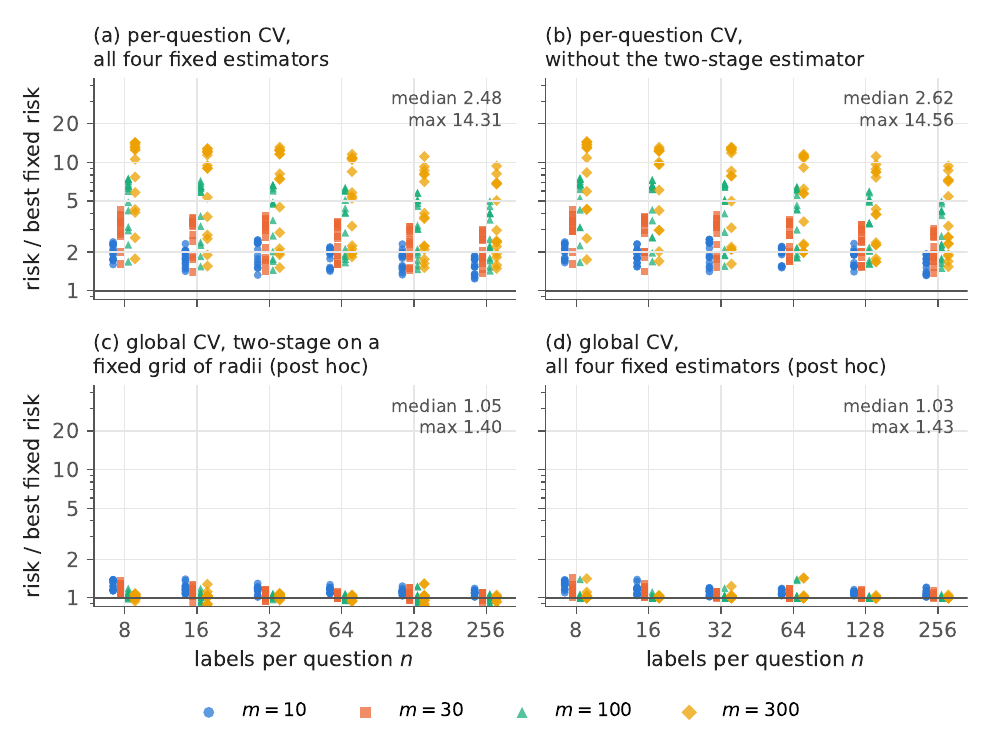}
  \caption{Per-question versus compound tuning (numerical evidence, E1; 30 replications per
  configuration of the declared range). Each marker is one configuration, colored and shaped
  by $m$: the risk of a data-driven rule divided by the smallest risk among four fixed
  estimators (the pooled isotonic fit; the pooled fit plus a per-question mean shift; the
  two-stage estimator with the true $(\tau,V)$; per-question isotonic regression). (a),~(b):
  one choice per question by two-fold cross-validation on that question's labels, (a) among
  all four fixed estimators, (b) without the two-stage estimator. (c),~(d): one choice for
  all questions by the cross-validation error summed over questions, (c) among the pooled
  fit, the pooled fit plus shift, per-question isotonic regression and the two-stage
  estimator on a fixed grid of seven radius pairs (no knowledge of $(\tau,V)$), (d) among all
  four fixed estimators. Panels (c) and (d) were added post hoc. The insets give the median
  and the maximum ratio over the range.}
  \label{fig:e01-adaptivity}
\end{figure}

\paragraph{The constant-map identification regret (E2, with E3 lower bounds).}
For priors on constant maps the regret has the closed form of \cref{prop:binomial}, which
we maximized over priors parametrized by their canonical moments: $68$ starts plus
differential evolution for each $n\le30$, and, post hoc, $10$ starts for each
$31\le n\le60$. One basin was found at every $n$. The best prior at each $n$ was
re-evaluated in exact rational arithmetic, so each reported value $S_n$ is a certified
lower bound on $\sup_H\Reg_n(H;\Gconst)$ (E3), while its optimality is not certified (E2).
Within $\Gconst$, the arcsine law is not extremal for $2\le n\le60$ (exact values, E3) and
for $n\ge62$: at $n=62$ the even-$n$ bound of \cref{thm:betastar}(ii) exceeds $1/(8\binom{124}{62})$
by a factor $3.3$, and for $n\ge63$ the bound $\bstar^{-n}/(72\sqrt n)$ exceeds $1/(8\binom{2n}n)$
(\cref{rem:eb-lower-crossover}: the ratio of the two is $1.017\ldots$ at $n=63$ and increases from
then on); $n=61$ is not covered. At $n=2$ the symmetric prior with second canonical moment
$(\sqrt5-1)/2$ has regret $(5\sqrt5-11)/8\approx0.02254$, above the arcsine value
$1/48\approx0.02083$. \Cref{fig:e03-barrier} places the values against
\cref{thm:betastar}. The normalized values $\bstar^nS_n$ rise to $0.44$ at $n=12$ and then
decrease slowly, to $0.29$ at $n=60$, well inside the range $[1/(72\sqrt n),1]$ that the
theorem allows, while the arcsine regret, normalized the same way, decays like
$\sqrt n\,(\bstar/4)^n$. The local base $S_n/S_{n+1}$ levels off far below $4$: it peaks at
$3.6205$ at $n=24$ and decreases to $3.607$ at $n=59$, within $0.7\%$ of $\bstar\approx3.583$ and
still decreasing. A base-$4$ model with a polynomial prefactor, $S_n\approx Cn^\alpha4^{-n}$, does
not fit: the fitted $\alpha$ doubles, from $1.37$ to $2.68$, across the declared fit windows. The
numerically optimal priors are symmetric (all odd canonical moments equal $1/2$); their even
canonical moments alternate above and below $1/2$ at low indices and approach $1/2$ at high
indices, so that their moment range falls short of $4^{-n}$ by a factor between about
$e^{-0.06n}$ and $e^{-0.1n}$ for $10\le n\le60$---the trade of moment range for a rare label
pattern behind \cref{thm:betastar}. Numerical evidence for \cref{conj:monotone-base} from the
same experiment is described in \cref{sec:conjecture}.

\begin{figure}[tbp]
  \centering
  \includegraphics[width=\textwidth]{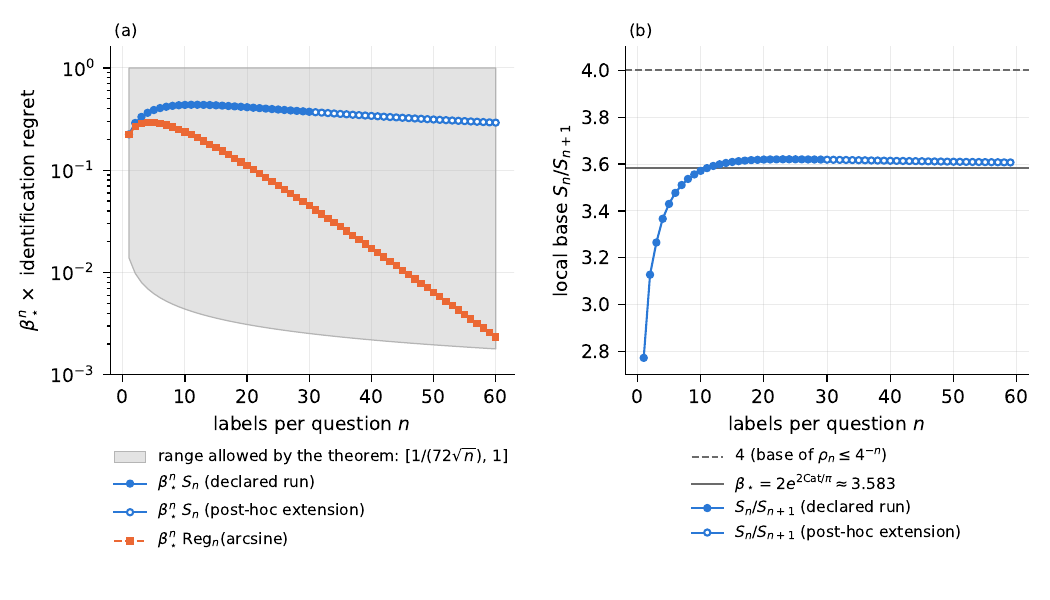}
  \caption{The constant-map identification regret against \cref{thm:betastar} (numerical
  evidence). $S_n$ is the largest value of $\Reg_n(H;\Gconst)$ found by global search; each
  $S_n$ is an exact-arithmetic lower bound on the supremum (E3), its optimality is not
  certified (E2). Filled markers: declared run ($n\le30$); open markers: post-hoc extension
  ($31\le n\le60$). (a) $\bstar^nS_n$, the range $[1/(72\sqrt n),1]$ allowed by
  \cref{thm:betastar}, and the arcsine regret $\bstar^n/(8\binom{2n}{n})$. (b) The local base
  $S_n/S_{n+1}$, with the base $4$ of the moment range and $\bstar$.}
  \label{fig:e03-barrier}
\end{figure}

%% file: sections/6_discussion.tex
\section{Discussion}
\label{sec:discussion}

Across questions, the shape is learned from all $mn$ labels, and the size of the deviations is tuned
once, for all questions (\cref{thm:main-rate,prop:separation}); within a question, what $n$ labels never
reveal about a random monotone map costs at most $2^{-n}$ ($m=\infty$), and no number of other questions
removes this cost (\cref{thm:main-eb,cor:finite-m}). The two answers use different benchmarks: the rate
$\Phi$ guards against every heterogeneous configuration, while for exchangeable questions the Bayes risk
can be exponentially smaller, and chasing it costs at most $2^{-n}$ at $m=\infty$.

\subsection{Two benchmarks, and a practical reading}
\label{app:practice}

\paragraph{Two benchmarks for one problem.}
Both parts study the same compound problem, $m$ questions with $n$ labels each, against two
benchmarks. \Cref{thm:main-rate} is a worst-case statement over $\Theta(\tau,V)$ and needs no
exchangeability: the maps are arbitrary members of the class. \Cref{thm:main-eb} applies when the
questions are exchangeable draws from a prior $G$; the natural target is then the Bayes risk
$R^*_n(G)$, and it can be far below $\Phi$: for the prior that picks the constant map $1/4$ or $3/4$
with probability $1/2$ each, $R^*_n(G)\le\frac14e^{-n/8}$, while the minimax risk of $\Theta(1/4,0)$, which
contains these maps, is at least a universal constant times $\min\{1/16,1/n\}$ for non-atomic $\mu$
(\cref{sec:intro}). The second part prices what
non-identification costs a practitioner who chases this Bayes target: at $m=\infty$, the posterior mean
of any monotone prior that reproduces the law of the data loses at most $2^{-n}$ against it
(\cref{thm:monotone-eb}(c)), and for every finite $m$ no procedure does better than the
identification regret in the worst case over priors (\cref{cor:finite-m}). With finitely many questions
the law of the data must itself be estimated, and a procedure that attains $2^{-n}$ up to a vanishing
term is open problem~2 of \cref{sec:open}.

\paragraph{What the results say for practice.}
Three recommendations follow (constants not optimized).

\emph{Pool the shape.} With $m=300$ questions and $n=16$ labels each (orders of magnitude,
ignoring constants), a question on its own resolves its map only at the scale
$n^{-2/3}\approx0.16$ (the rate for $m=1$, \cref{rem:regimes}), while the pool resolves the
shared shape at $(mn)^{-2/3}\approx0.0035$; what is left per question is its offset, of order
$\min\{\tau^2,1/n\}$ with $1/n\approx0.06$, and its wiggle. A pooled fit followed by a
per-question correction captures both, and the correction must be allowed to vanish: when
the questions differ by less than $n^{-1/2}$ in sup-distance, ignoring the differences is
rate-optimal.\footnote{By \cref{lem:pooled} (for non-atomic $\mu$, and with
$\log(eN)/N\le3e^{-2/3}N^{-2/3}$ for $N=mn$), the pooled isotonic fit $\gh_0$ to all $mn$ labels
satisfies $\E\nmu{\gh_0-\gbar}^2\le C_1(mn)^{-2/3}$ for the average map $\gbar$ and a universal $C_1$,
and $\abs{g_j-\gbar}\le2\tau$; so each question's risk is at most
$2C_1(mn)^{-2/3}+8\tau^2\le\max\{2C_1,8\}\,\Phi$ when $\tau^2\le1/n$.}

\emph{Tune once, for all questions.} The radii of the correction should be chosen once, from the
hold-out error summed over questions (\cref{def:procedure}). \Cref{thm:adaptive} covers this rule, and
\cref{prop:separation} shows that tuning each question on its own data cannot adapt: once $m^2/n$ and
$n$ are large, the best worst-case loss factor of such rules is of order $\min\{n,\log(m^2/n)\}$. The
simulations of \cref{sec:numerics} show per-question cross-validation among estimators that are mostly
built on a pooled shape, a class the proposition does not cover, failing by factors up to $14.6$ at
moderate sizes.

\emph{Empirical-Bayes identification is a few-label phenomenon.} By \cref{thm:main-eb},
not knowing the prior over monotone maps costs at most $2^{-n}$: less than $0.008$ for
$n\ge7$ and less than $10^{-3}$ for $n\ge10$. Relative to the Bayes risk the decay is even
sharper. For the arcsine prior on constant maps, $\theta\sim\mathrm{Beta}(\frac12,\frac12)$,
the posterior after $s$ positive labels is $\mathrm{Beta}(\frac12+s,\frac12+n-s)$, with mean
$(\frac12+s)/(n+1)$. Since $\Var\theta=\frac18$ and the Beta--binomial count has
$\Var s=n(n+1)/8$,
$R^*_n=\E\Var(\theta\mid s)=\Var\theta-\Var\bigl(\tfrac{1/2+s}{n+1}\bigr)=\frac18-\frac{n}{8(n+1)}=\frac1{8(n+1)}$.
Its identification regret within $\Gconst$ is $1/(8\binom{2n}{n})$ (\cref{prop:binomial}), so the ratio is
$(n+1)/\binom{2n}{n}$: $1$, $1/2$, $1/5$, $1/14$ and $1/42$ for $n=1,\dots,5$. So with one to
three labels per question the identification loss can be of the order of the Bayes risk
itself, and no amount of pooling removes it (\cref{cor:finite-m}); from about ten labels per question
on, the posterior mean of any monotone prior that reproduces the law of the data is within $10^{-3}$
of the Bayes risk (\cref{thm:monotone-eb}(c)). The ratio $(n+1)/\binom{2n}n$ is specific to the arcsine
prior; uniformly over priors on constant maps, Joshi's result says that Skibinsky's gap is
$O(n^{-1/2})$ times the Bayes risk~\citep{joshi1975minimax}. These upper bounds hold at $m=\infty$;
estimating the law of the data from finitely many questions is open problem~2.

\subsection{Open problems}
\label{sec:open}

\Cref{app:scope} lists the limitations of our results. In particular, the rate theorems bound the average
risk over questions, not the worst question's. The open problems they leave are the following.

\paragraph{1. The monotone base.}
Prove \cref{conj:monotone-base}: for priors on monotone maps the identification regret
decays with base $\bstar$, as for constant maps, rather than with base $2$. A
quantitative form would be $\sup_{\Gmono}\Reg_n\le\poly(n)\sup_{\Gconst}\Reg_n$. The upper
bound of \cref{thm:monotone-eb} uses only the diameter of the unidentified moment, which
\cref{lem:ch1} controls on sorted vectors; the base $\bstar$ comes instead from a weighted
inequality, the Szegő-type bound of \cref{thm:betastar}, in which the moment range is traded
against the probabilities of the label patterns. The natural route is a weighted analogue
of \cref{lem:ch1} on the ordered simplex, and we do not know one. Even for constant maps the
polynomial prefactor is open: our bounds differ by $72\sqrt n$, and the numerics of
\cref{fig:e03-barrier} suggest a prefactor that decays, but slowly.

\paragraph{2. Empirical Bayes with finitely many questions.}
Give an empirical-Bayes procedure with compound risk at most
$R^*_n(G)+2^{-n}+\eps_m$, with an explicit $\eps_m\to0$, for every prior $G\in\Gmono$---for
instance the posterior mean of a nonparametric maximum-likelihood prior fitted to the
pooled data, at least for a finite output grid. \Cref{thm:monotone-eb}(c) handles the limit
case, a fitted prior that reproduces the law of the data exactly, and \cref{cor:finite-m} shows that
the term $2^{-n}$ cannot be replaced by anything below the identification regret, whatever $m$. What is missing is
stability: if $\Law_{G''}(D)$ is only $\eps$-close to $\Law_G(D)$, say in Hellinger distance,
how far is the posterior mean of $G''$ from the Bayes rule of $G$? The difficulty is that the
posterior mean divides by the pattern probabilities $Z_G$, and the rare label patterns that
drive the lower bound of \cref{thm:betastar} are exactly where small errors in the law of the
data are amplified.

\paragraph{3. Discrete outputs.}
Deployed models often report probabilities on a finite grid, and then \cref{thm:main-rate}
does not apply: for a point mass $\mu=\delta_{1/2}$ and $\tau=0$, the pooled label mean
already has risk at most $1/(4mn)$, far below $(mn)^{-2/3}$. Determine the compound minimax
rate when $\mu$ has $L$ atoms, as a function of $(m,n,L,\tau,V)$. Per-atom averaging suggests a
parametric regime when $L$ is small compared with $(mn)^{1/3}$, and the answer should
interpolate between it and $\Phi$. Our lower bound does not help: it transports every
non-atomic $\mu$ to the uniform law by a quantile transform, which atoms break, and its
shared hypercube uses about $(mn)^{1/3}$ bumps, which need that many distinct output values.

\paragraph{4. Lower bounds for calibration error.}
Our lower bounds are for the excess Brier risk; for calibration error alone no lower bound
of the order of $\Phi$ can hold, since at $\tau=0$ the pooled label mean has expected
calibration error at most $1/(4mn)$ (\cref{app:calibration}). The meaningful
question is a trade-off: among estimators whose excess Brier risk is within a constant
factor of $\Phi$, how small can the compound calibration error be, in $L_2$ or in the $L_1$
and smooth calibration errors used in practice? Our upper bounds give $\Phi$ for the $L_2$
version, and the isotonic fits we use are step functions, for which calibration error can be
strictly smaller than the risk. The obstacle to a lower bound is that calibration error
depends on the level sets of the estimate and not only on its distance to the truth, while
the Assouad constructions of \cref{thm:main-rate} control distances.

\paragraph{5. Is monotonicity necessary?}
\Cref{rem:monotone-needed} shows that without monotonicity in a known direction the
unidentified moment is no longer confined to an interval of length $4^{-n}$. Does the
identification regret still decay exponentially, or at all, for priors on arbitrary
measurable maps $[0,1]\to[0,1]$? Without an order, the values of the map at $n+1$ points form
an arbitrary point of the cube, and \cref{lem:ch1} has no analogue; an answer in either
direction would show whether shape constraints are what makes empirical-Bayes
recalibration nearly free of identification loss.

\paragraph{6. Questions with their own output laws.}
All our results for the compound rate assume one output law $\mu$ shared by all questions. With
question-specific laws $\mu_j$ the pooled isotonic fit targets the monotone projection of a weighted
average of the maps, not their average, and our argument does not control the deviations of the
questions from it (\cref{app:model-remarks}). Determine the
compound minimax rate when the $\mu_j$ differ, for instance when they have densities with
respect to a common law that are bounded above and below by constants.

\paragraph{7. Per-question tuning around a pooled shape.}
\Cref{prop:separation} covers procedures that tune each question on its own data only. A natural
intermediate procedure estimates the shape from the pooled data and then tunes each question's
correction on that question's labels alone. Does such a rule attain $\Phi$ adaptively, or does it also
lose a factor that grows with $m^2/n$? Our lower bound does not apply: the pooled shape estimate is a
statistic of all questions and can carry information about the size of the shifts, so the
likelihood-ratio argument of \cref{app:separation}, which compares the laws of one question's data
alone, breaks down. The evidence points to a loss: panels (a) and (b) of \cref{fig:e01-adaptivity} (E1)
show such per-question choices costing up to a factor $14.6$ at moderate sizes, but no rate in $m^2/n$
can be read off a box of four values of $m$.

\paragraph{8. The worst question.}
Our fixed-effects theorems bound the average risk over questions (\cref{app:scope}(viii)): the compound
hold-out of \cref{def:procedure} selects one pair of radii by the hold-out error summed over questions, and its
guarantee (\cref{prop:holdout}) controls only the average. Does some adaptive procedure attain
$\max_j\E\nmu{\gh_j-g_j}^2\le C\,\Phi(m,n,\tau,V)$ over $\Theta(\tau,V)$, and if not, what is the price of
a worst-question guarantee?

\paragraph{9. The base of Skibinsky's gap.}
Prove that $\sup_HW_n(H)\ge2^{-n}/\poly(n)$, so that Skibinsky's bound $2^{-n}$ is sharp in the base, as
the numerics of \cref{rem:skibinsky-base} suggest; \cref{cor:skibinsky} gives only
$\bstar^{-n}/(72\sqrt n)$. The
lower bound needs priors for which the Bayes risk changes steeply along the moment interval while the
interval stays within a polynomial factor of its largest length $4^{-n}$. The numerical maximizers (E2)
point to a shape: their posterior means jump from about $1/4$ to about $3/4$ near $s=n/2$, as for a prior
with two clusters, and their canonical moments of high order are close to $1/2$.

%% file: sections/9_appendix.tex
\section{Related work and limitations}
\label{app:supp}
\input{sections/9s1_related}
\input{sections/9s2_scope}

\section{Proofs for the compound minimax rate and the separation}
\label{app:rate}

\Cref{fig:roadmap} shows how the results depend on each other. This appendix proves the results of
\cref{sec:results}, and \cref{app:eb} those of \cref{sec:identification}.

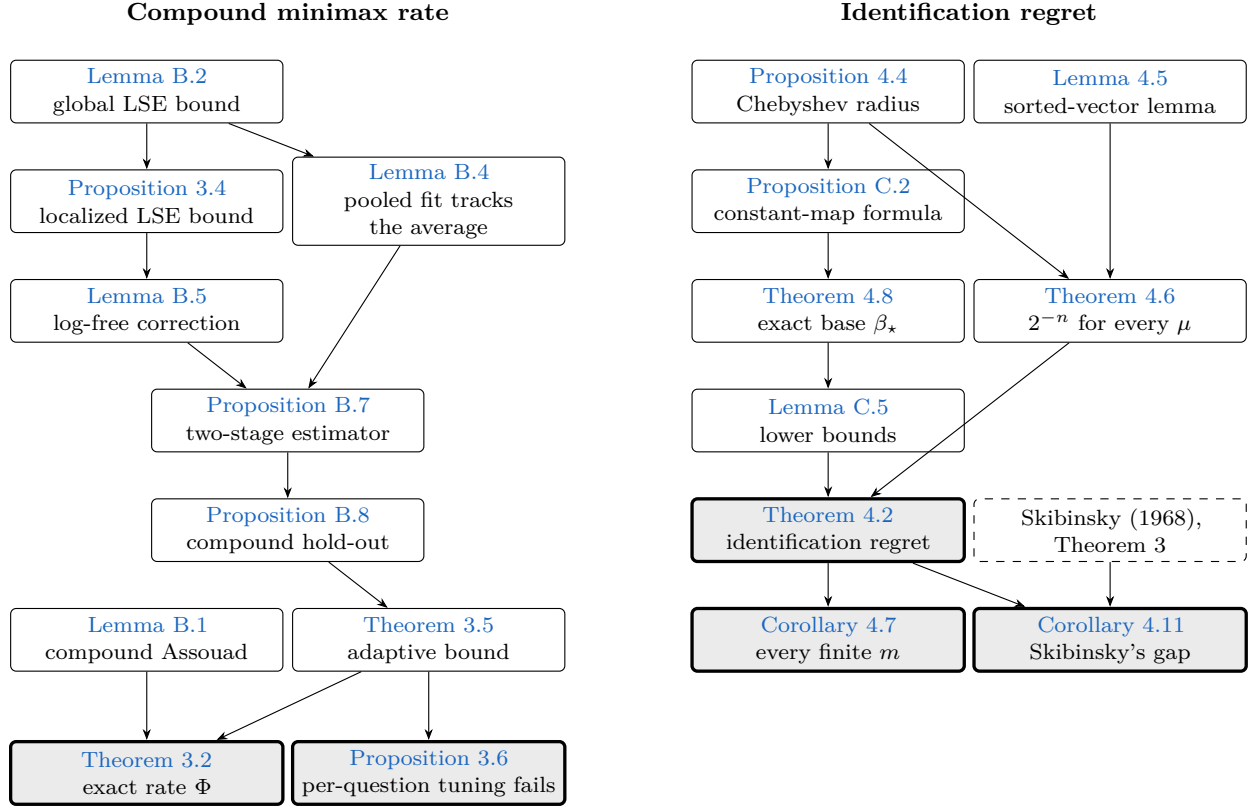
\begin{figure}[t]
\centering
\resizebox{\linewidth}{!}{\input{figures/tikz_roadmap}}
\caption{How the results depend on each other: an arrow points from a result to a result whose proof
uses it (arrows implied by transitivity are omitted). Shaded boxes are the main statements; the dashed
box is a cited theorem.}
\label{fig:roadmap}
\end{figure}

\input{sections/9a_proofs_rate}

\section{Proofs for the identification regret}
\label{app:eb}

\input{sections/9b_proofs_eb}

%% file: sections/9s1_related.tex
This appendix gives a more detailed account of related work (\cref{sec:related}) and the full list of the
limitations of our results (\cref{app:scope}).

\subsection{Extended related work}
\label{sec:related}

\paragraph{Shape-constrained regression.}
The worst-case risk of isotonic least squares, of order $(V/n)^{2/3}+\log(en)/n$ for a
nondecreasing signal of variation $V$ and unit noise variance, is due to~\citet{zhang2002risk} (for Gaussian errors it
appeared in~\citealp{meyer2000degrees}); adaptive bounds of order $(k+1)\log(en)/n$ for monotone
signals with $k$ jumps are due to~\citet{chatterjee2015risk}, analogous bounds for total-variation
denoising to~\citet{guntuboyina2020adaptive}, and isotonic regression in dimension $d\ge2$, on a
lattice or with random design, is treated by~\citet{han2019isotonic}; see~\citet{guntuboyina2018shape}
for a survey. These results bound the risk at the design points, in the empirical norm (which is
$L_2(P_n)$ in the random-design case). We need the population norm $L_2(\mu)$ and observations that
are independent but not identically distributed, so we prove the bounds we use
(\cref{lem:erm1,prop:erm2}) by bracketing chaining with Bernstein's inequality, followed by peeling, in
the tradition of~\citet{vandervaart1996weak} and~\citet{vandegeer2000empirical}. The shift model
$V=0$, $g_j=g_0+c_j$, is the vertical-shift special case of the shape-invariant model
of~\citet{kneip1995model}, in which every curve is a common shape under a linear change of both axes;
there the shape and the curve parameters are estimated by alternating least squares and kernel
smoothing, and the results concern asymptotic distributions and asymptotic bias and variance.

\paragraph{Multitask and transfer learning.}
Learning many related tasks together goes back at least to~\citet{baxter2000model}, and the shape
of~\eqref{eq:intro-rate} follows the shared-plus-local template. Its closest parametric instance is
the adaptive and robust multitask estimator of~\citet{duan2023adaptive}: for $m$ tasks with
$d$-dimensional parameters within distance $\delta$ of a common center, apart from a small fraction of
outlier tasks, it attains compound risk of order $d/(mn)+\min\{\delta^2,d/n\}$ (plus a term for the
outliers) adaptively, up to logarithmic factors, and this matches a minimax lower bound for multitask
Gaussian mean estimation. \citet{huang2025optimal} obtain optimal rates, up to logarithmic factors,
for multitask linear regression when the task parameters differ from a common one in at most $s$
coordinates; \citet{tian2025learning} give nearly minimax rates, adaptive to the degree of
similarity, when the tasks share a linear representation up to a perturbation;
and~\citet{chen2023minimax} find the phase transition between federated averaging and local
training in personalized federated learning. In nonparametric transfer learning,
\citet{cai2021classification} give the minimax rate and an adaptive classifier for classification,
and~\citet{cai2024transfer} minimax rates, up to logarithmic factors, and an adaptive procedure for
regression under posterior drift. \Cref{thm:main-rate} is a shape-constrained, random-design,
Bernoulli analogue of this template (\cref{rem:regimes}): the rate holds with universal constants and
with no logarithmic factors, and adaptation is achieved by one compound hold-out choice.

Limits of adaptation in multitask learning have a different source. In multitask classification with a
common optimal classifier, \citet{hanneke2022nofreelunch} show that the minimax rate improves with the
number of tasks but that, without distributional information, no algorithm can adapt in general: for
noisy classes (Bernstein parameter $\beta<1$) and fewer than $2/\beta-1$ samples per task it is stuck at
the rate of the target task alone; \citet{hanneke2026limits}
extend this impossibility to arbitrarily large samples per task when the tasks are super-exponentially
many. For parametric multi-source transfer with biased sources, \citet{chakraborty2026cost} define an
intrinsic cost of adaptation, the smallest, over bias-agnostic estimators, worst-case ratio of the risk
to the oracle minimax rate; it is free with one source and has a phase transition with two or more. These limits
concern unknown task relatedness or source quality. \Cref{prop:separation} is of another kind: in a
known shared-plus-local class, where one compound hold-out adapts, it is the cost, of order
$\min\{n,\log(m^2/n)\}$, of tuning each question on its own data.

\paragraph{Compound decisions and empirical Bayes.}
Robbins' compound decision problem~\citep{robbins1951asymptotically} treats many independent problems
of the same form jointly, and his empirical-Bayes program~\citep{robbins1956empirical} models the
unknown parameters as draws from an unknown prior; \citet{zhang2003compound} surveys both and
stresses that stochastically independent experiments need not be uninformative about one another.
Our two parts are the compound and the empirical-Bayes view of one problem. For binomial data the
empirical-Bayes rule is not identified: after $n$ trials the Bayes rule depends on the first $n+1$
moments of the prior~\citep{skibinsky1989admissible}, while only $n$ are estimable.
\citet[Section~5]{robbins1964empirical} works out the case $n=1$. \citet{skibinsky1968minimax}
quantifies the price for every $n$: over the class of priors with given first $n$ moments, the
minimax risk exceeds the Bayes risk of each member by at most $2^{-n}$. He records Robbins'
conjecture that this gap is negligible relative to the Bayes risk, which \citet{joshi1975minimax}
proved, with relative gap $O(n^{-1/2})$ uniformly over priors. \Cref{rem:skibinsky} and
\cref{app:skibinsky} relate these quantities to ours. Our identification regret differs from the
empirical-Bayes regret of \citet{polyanskiy2021sharp}, the total excess risk of a data-driven procedure
over $m$ coordinates relative to $m$ times the Bayes risk of the oracle that knows the prior; Robbins
showed that it can be made $o(m)$, and they determine its optimal order for Poisson models and give
lower bounds for normal means. Ours is measured against a
Bayes risk that the data cannot identify even as the number of questions grows, so no procedure beats
it at any $m$ (\cref{cor:finite-m}): it is a floor, not the estimation cost of a procedure.
\citet{sivaganesan1993robust} study the range of posterior means over all priors with the estimable
moments, and~\citet{ignatiadis2022confidence}, constructing confidence intervals for nonparametric
empirical-Bayes estimands, note that the binomial posterior mean is only partially identified. Two
contrasts are instructive. \citet{tian2017learning} show that a population of coin biases is
recoverable from $t$ flips per coin only to Wasserstein error of order $1/t$, and the maximum
likelihood estimator attains this order when $t$ is small compared with the logarithm of the number
of coins~\citep{vinayak2019maximum}; the Bayes rule, in contrast, is identified up to an
exponentially small regret---identifying the prior and identifying the Bayes rule are very different
problems. In a replicated-data setting with nonparametric noise, \citet{ignatiadis2023aurora} bound
the loss from using $K-1$ of $K$ replicates by a quantity of order $1/K^2$ in regular problems.
\citet{chen2025compound} study compound estimation for binomials; a bias bound of order $2^{-n}$ in
their work has the same source as ours, the unidentified $(n+1)$-th moment.

\paragraph{Partial identification and $\Gamma$-minimax.}
Restricting attention to the priors consistent with the data is the moment class of robust Bayesian
analysis~\citep{berger1994overview,sivaganesan1993robust}. Our regret is the posterior regret of
$\Gamma$-minimax estimation~\citep{riosinsua1995posterior}, with $\Gamma=[G]$ the identified set,
averaged over the common law of the data and with the supremum over $\Gamma$ taken after averaging.
Under squared loss and for a single real estimand, the posterior-regret $\Gamma$-minimax action is the
midpoint of the range of posterior means over $\Gamma$ (see, e.g.,
\citealp[Example~1]{jafarijozani2013intrinsic}),
just as the minimax-regret predictor under partial identification is the midpoint of the identified
interval~\citep{manski2021econometrics}; minimax-regret rules for treatment choice under partial
identification are studied by~\citet{montielolea2026decision}. \Cref{prop:chebradius} is the
integrated, functional version of the midpoint rule: the identified Bayes rules form a set of
functions of $(p,X)$ governed by one unidentified $(n+1)$-point moment function, the regret is
integrated over the data with an explicit weight $w_G$, the midpoint becomes a Chebyshev centre in a
weighted $L_2$ space, and a measurable centre, hence a minimax rule, exists.

\paragraph{Moment spaces and Szegő asymptotics.}
The moment range is a product of canonical moments,
$\rho_n=\prod_{k\le n}p_k(1-p_k)$~\citep{skibinsky1967range}; see~\citet{dette1997canonical} for the
theory, in which the arcsine law, all of whose canonical moments equal
$1/2$~\citep{skibinsky1969striking,dette2017hankel}, is the center of the moment space. The
Szegő-type sandwich in \cref{thm:betastar}(iii) is classical: it can be read off from
\citet[Theorem~3.1 and Remark~3.1]{gamboa2004large}, and through the identification of canonical
moments with Verblunsky coefficients~\citep{gamboa2010canonical,dette2010matrix} it is a finite-$n$
form of Szegő's theorem for orthogonal polynomials on the unit
circle~\citep{simon2005opuc,gamboa2016sum}; see \cref{rem:szego}. What is new is its use: the
supremum of the identification regret, its exact exponential base, and near-extremal priors.
Finally, the value $2^{-(2n+1)}$ in \cref{lem:ch1} is the least Chebyshev norm on the ordered simplex
computed by~\citet{dressler2026least}, whose exhibited extremal polynomials are univariate; we did not
find in
the literature the diagonal-domination inequality for sorted vectors, or the fact that the matched
multi-affine product, which is what our moment argument needs, attains this value.

\paragraph{Calibration and recalibration.}
Calibration of a single predictor across groups and decisions is a large
literature~\citep{zhao2021calibrating,gopalan2024omnipredictors}; our object is different---many
recalibration maps, one per question, each with few labels. Recalibration with little labelled data
per unit has been approached by predicting a slice's recalibration curve from a few unlabelled
examples~\citep{li2024fewshot}, by hierarchical Bayesian models for per-group calibration with few
labels~\citep{ji2020trust}, and, in a workshop paper, by empirical-Bayes shrinkage of per-rater affine
calibrators of reward models~\citep{raj2026pebs}; these works propose and evaluate methods, while we
ask what is achievable in the minimax sense. \citet{emmenegger2026ppi} study
prediction-powered inference across many tasks. That $k$ labels per instance identify $k$ moments of the conditional law also
drives higher-order calibration~\citep{ahdritz2025provable} and the equivalence between second-order
calibration and calibration over pairs of labels~\citep{johnson2024experts},
and~\citet{ciosek2026minimax} determine the minimax rate for estimating a perturbed second-order
calibration error.

\paragraph{Applications.}
\Cref{sec:intro} lists the applications that motivate the model; in each, a shared predictor is
audited with a few labels per unit, and personalized prediction more broadly has the same shape. For
the decision models Jev and Laya, the probabilities have been examined empirically for coherence and
for calibration across questions~\citep{joy2026addup,sharma2026typedbench}.

%% file: sections/9s2_scope.tex
\subsection{Scope and limitations}
\label{app:scope}

\paragraph{Limitations.}
\begin{enumerate}[label=(\roman*)]
\item The risk is the excess Brier, or $L_2$ estimation, risk of the recalibration maps. It bounds the
  $L_2$ calibration error of the recalibrated forecasts from above, so our upper bounds transfer to
  calibration error; our lower bounds do not: at $\tau=0$ the forecast equal to the mean of all $mn$
  labels has expected calibration error at most $1/(4mn)$, of smaller order than $(mn)^{-2/3}$
  (\cref{app:calibration}).
\item One output law $\mu$ shared by all questions, and monotone maps, are assumptions; the upper bound
  uses the first (\cref{app:model-remarks}). The lower bound of \cref{thm:main-rate} is a supremum over
  $\Theta(\tau,V)$ whose first term comes from perturbing the shared shape; it says nothing about
  sub-models in which the average map is known, for instance the identity when the predictor is
  calibrated on the same mixture of questions.
\item \Cref{thm:main-rate} needs a non-atomic output law. On a coarse output grid the rate is
  different: for a point mass and $\tau=0$ the pooled label mean has risk at most $1/(4mn)$ (open
  problem~3). The empirical-Bayes results hold for every $\mu$.
\item The empirical-Bayes results are random-effects statements. The upper bound $2^{-n}$ assumes that
  the law of a question's data is known (the population idealization $m=\infty$); a finite-$m$
  empirical-Bayes procedure, and a rate in $m$, are open problem~2. The lower bounds hold for every $m$
  (\cref{cor:finite-m}); they concern $\sup_G\Reg_n(G;\Gmono)$, while the upper bound $2^{-n}$ holds
  prior by prior. $\Reg_n(G;\Gmono)$ is a minimax quantity over the identified set of $G$, not the
  excess risk at the true prior, where the question is empty. These results concern all monotone
  priors, not $\Theta(\tau,V)$, and they are not the $m=\infty$ limit of \cref{thm:main-rate}
  (\cref{app:practice} compares the two benchmarks).
\item For monotone priors the exponential base lies in $[2,\bstar]$. The hard instances of our lower
  bound are priors on constant maps, an uninformative predictor; whether a lower bound of the same order
  holds in subclasses that exclude them (for instance, maps with slope bounded below) is open. For
  constant maps the base is exactly $\bstar$, and the polynomial prefactor is open: our two bounds differ
  by the factor $72\sqrt n$.
\item ``Exact'' means up to universal constants, which we do not optimize; nor do we optimize the
  constants $1/256$ of \cref{prop:separation} and $1/72$ of \cref{thm:main-eb}. The bound $2^{-n}$ is
  below the trivial value $1/4$ only for $n\ge3$ (\cref{rem:eb-scope}).
\item \Cref{prop:separation} concerns procedures that treat each question on its own data; per-question
  tuning around a shape estimated from the pooled data is open problem~7. For that class there is
  numerical evidence (E1) that it also loses at moderate sizes (\cref{fig:e01-adaptivity}(a),(b)).
\item The fixed-effects theorems concern the average risk over questions,
  $\frac1m\sum_j\E\nmu{\gh_j-g_j}^2$. A worst-question guarantee, on $\max_j\E\nmu{\gh_j-g_j}^2$, for
  the adaptive procedure is open problem~8.
\end{enumerate}
The open problems are those of \cref{sec:open}.

%% file: figures/tikz_roadmap.tex
\begin{tikzpicture}[font=\footnotesize,x=1cm,y=1cm,
    res/.style={draw,rounded corners=2pt,align=center,inner sep=2.5pt,text width=3.3cm,minimum height=0.8cm},
    head/.style={res,very thick,fill=black!8},
    ext/.style={res,dashed},
    arr/.style={-{Stealth[length=4pt]}}]
  \node[font=\small\bfseries] at (3.5,1.0) {Compound minimax rate};
  \node[res] (erm1) at (1.7,0)    {\Cref{lem:erm1}\\ global LSE bound};
  \node[res] (erm2) at (1.7,-1.4) {\Cref{prop:erm2}\\ localized LSE bound};
  \node[res] (pool) at (5.3,-1.4) {\Cref{lem:pooled}\\ pooled fit tracks the average};
  \node[res] (loc)  at (1.7,-2.8) {\Cref{lem:local}\\ log-free correction};
  \node[res] (orc)  at (3.5,-4.2) {\Cref{prop:oracle}\\ two-stage estimator};
  \node[res] (hold) at (3.5,-5.6) {\Cref{prop:holdout}\\ compound hold-out};
  \node[res] (low)  at (1.7,-7.0) {\Cref{lem:lower}\\ compound Assouad};
  \node[res] (ad)   at (5.3,-7.0) {\Cref{thm:adaptive}\\ adaptive bound};
  \node[head] (rate) at (1.7,-8.7) {\Cref{thm:main-rate}\\ exact rate $\Phi$};
  \node[head] (sep)  at (5.3,-8.7) {\Cref{prop:separation}\\ per-question tuning fails};
  \draw[arr] (erm1) -- (erm2);  \draw[arr] (erm1) -- (pool);
  \draw[arr] (erm2) -- (loc);
  \draw[arr] (loc) -- (orc);    \draw[arr] (pool) -- (orc);
  \draw[arr] (orc) -- (hold);   \draw[arr] (hold) -- (ad);
  \draw[arr] (low) -- (rate);   \draw[arr] (ad) -- (sep);
  \draw[arr] (ad) -- (rate);
  \begin{scope}[xshift=8.7cm]
  \node[font=\small\bfseries] at (3.5,1.0) {Identification regret};
  \node[res] (cheb) at (1.7,0)    {\Cref{prop:chebradius}\\ Chebyshev radius};
  \node[res] (ch1)  at (5.3,0)    {\Cref{lem:ch1}\\ sorted-vector lemma};
  \node[res] (bin)  at (1.7,-1.4) {\Cref{prop:binomial}\\ constant-map formula};
  \node[res] (beta) at (1.7,-2.8) {\Cref{thm:betastar}\\ exact base $\bstar$};
  \node[res] (mono) at (5.3,-2.8) {\Cref{thm:monotone-eb}\\ $2^{-n}$ for every $\mu$};
  \node[res] (ebl)  at (1.7,-4.2) {\Cref{lem:eb-lower}\\ lower bounds};
  \node[head] (eb)  at (1.7,-5.6)  {\Cref{thm:main-eb}\\ identification regret};
  \node[ext] (sk68) at (5.3,-5.6)  {Skibinsky (1968),\\ Theorem 3};
  \node[head] (fin) at (1.7,-7.0) {\Cref{cor:finite-m}\\ every finite $m$};
  \node[head] (skc) at (5.3,-7.0) {\Cref{cor:skibinsky}\\ Skibinsky's gap};
  \draw[arr] (cheb) -- (bin);   \draw[arr] (bin) -- (beta);
  \draw[arr] (cheb) -- (mono);  \draw[arr] (ch1) -- (mono);
  \draw[arr] (beta) -- (ebl);
  \draw[arr] (ebl) -- (eb);     \draw[arr] (mono) -- (eb);
  \draw[arr] (eb) -- (fin);     \draw[arr] (eb) -- (skc);
  \draw[arr] (sk68) -- (skc);
  \end{scope}
\end{tikzpicture}

%% file: sections/9a_proofs_rate.tex
\paragraph{Conventions.}
Throughout this appendix $\lambda$ is Lebesgue measure on $[0,1]$ and $\log$ is the natural logarithm.
For a Borel probability $\mu$ on $[0,1]$, $\norm{\cdot}$ is the norm of $L^2(\mu)$ and $\inner{\cdot}{\cdot}$
its inner product, unless another measure is indicated. The variation of $d:[0,1]\to\R$ is
$\mathrm{TV}(d)\defeq\sup\sum_{l=1}^{L}\abs{d(x_l)-d(x_{l-1})}$, the supremum over finite partitions
$0\le x_0<\dots<x_L\le 1$; it is a seminorm, and constants have variation $0$. We write
$\mathcal F_{\uparrow}$ for the nondecreasing measurable $f:[0,1]\to[0,1]$, and
$\mathcal D(\tau,V)\defeq\set{d:[0,1]\to\R \text{ measurable}:\ \sup\abs{d}\le\tau,\ \mathrm{TV}(d)\le V}$.

In Model~M there are $m$ questions with $n$ pairs each. All $mn$ pairs $(P_{jk},Y_{jk})$ are mutually
independent, $P_{jk}\sim\mu$ and $Y_{jk}\mid P_{jk}\sim\mathrm{Bern}(g_j(P_{jk}))$. The class
$\Theta(\tau,V)$ consists of the tuples $(g_1,\dots,g_m)$ of nondecreasing maps $[0,1]\to[0,1]$ for which
one nondecreasing witness $g_0:[0,1]\to\R$ satisfies $\sup_{[0,1]}\abs{g_j-g_0}\le\tau$ and
$\mathrm{TV}(g_j-g_0)\le V$ for all $j$. The subclass $\Theta_{\mathrm{LB}}(\tau,V)$ is defined by the same
predicate with $g_0$ in addition $[0,1]$-valued. An estimator is a jointly measurable map from the $mn$ pairs
(and possibly an independent auxiliary randomization with a parameter-free law) to $m$ functions on $[0,1]$.
Its compound risk is
$\mathcal R_\mu(\hat g;g)\defeq\frac1m\sum_{j=1}^m\E\norm{\hat g_j-g_j}^2_{L^2(\mu)}$. We write
$\mathfrak M_{m,n}(\tau,V)\defeq\inf_{\hat g}\sup_{g\in\Theta(\tau,V)}\mathcal R_\mu(\hat g;g)$ (with
$\mu$ fixed), and
$\Phi(m,n,\tau,V)\defeq(mn)^{-2/3}+\min\set{\tau^2,1/n}+\min\set{\tau^2,V^2,(V/n)^{2/3}}$.

Brackets are \emph{pointwise}, following \citet{vandervaart1996weak}. For a seminorm $\rho$, an
$\eps$-bracket $[l,u]$ is a pair of real measurable functions $l\le u$ with $\rho(u-l)\le\eps$. A function
$f$ lies in $[l,u]$ if $l\le f\le u$ at \emph{every} point. $N_{[\,]}(\eps,\mathcal F,\rho)$ is the least
number of $\eps$-brackets that cover $\mathcal F$; the brackets need not lie in $\mathcal F$.

\paragraph{Roadmap.}
We give the proofs in dependency order. \Cref{app:rate-lower} proves the lower bound (\cref{lem:lower}).
\Cref{app:rate-engine} proves a least-squares bound for independent, non-identically distributed
observations that share a design law (here the output law $\mu$). It is proved first under global bracketing entropy (\cref{lem:erm1}),
then under localized bracketing entropy (\cref{prop:erm2}). \Cref{app:rate-stages} applies these bounds to
the pooled isotonic stage (\cref{lem:pooled}) and to the local bounded-variation stage (\cref{lem:local}).
\Cref{app:rate-adaptive} assembles the oracle estimator (\cref{prop:oracle}), the compound hold-out
(\cref{prop:holdout}), the adaptive bound for all $m,n$ (\cref{thm:adaptive}) and \cref{thm:main-rate}.
\Cref{app:separation} proves that per-question tuning cannot match this guarantee (\cref{prop:separation}).

\subsection{The lower bound}
\label{app:rate-lower}

\begin{lemma}[Compound minimax lower bound]\label{lem:lower}
There is a universal constant $c>0$ such that, for every non-atomic Borel probability $\mu$ on $[0,1]$
($c$ does not depend on $\mu$), all $m\ge1$, $n\ge1$ and all $\tau,V\in[0,1/8]$,
\[
  \inf_{\hat g}\ \sup_{g\in\Theta_{\mathrm{LB}}(\tau,V)}\ \mathcal R_\mu(\hat g;g)
  \;\ge\; c\,\Bigl[(mn)^{-2/3}+\min\set{\tau^2,1/n}+\min\set{\tau^2,V^2,(V/n)^{2/3}}\Bigr],
\]
where the infimum is over all estimators of Model~M, randomized ones included.
\end{lemma}

The proof reduces the problem to $\mu=\lambda$. It then uses three hypercubes of monotone maps, one for
each term of $\Phi$, and a product form of Assouad's lemma in which flipping one coordinate changes the data
of a single question.

\begin{proof}
\emph{Conventions.} For probability measures $P,Q$ with densities $p,q$ with respect to any dominating
$\nu$, $d_{\mathrm{TV}}(P,Q)\defeq\sup_A\abs{P(A)-Q(A)}=\frac12\int\abs{p-q}\,d\nu=1-\int\min(p,q)\,d\nu$,
because $\min(p,q)=(p+q-\abs{p-q})/2$. We write $\mathrm{kl}(p\Vert q)\defeq p\log(p/q)+(1-p)\log((1-p)/(1-q))$
for the Bernoulli divergence, and $\mathcal R_\lambda$ for the compound risk in Model~M with $\mu=\lambda$.
Randomized estimators are covered: adding an independent $\xi$ with a parameter-free law $R$ to the data
leaves every Kullback--Leibler and total-variation bound below unchanged, since
$\mathrm{KL}(P\otimes R\Vert P'\otimes R)=\mathrm{KL}(P\Vert P')$.

\smallskip\noindent\textbf{Step 0 (reduction to $\mu=\lambda$).}
Let $\mu$ be non-atomic and $H\subseteq\Theta_{\mathrm{LB}}(\tau,V)$. We show
\begin{equation}\label{eq:lower-reduction}
  \inf_{\hat g}\sup_{g\in\Theta_{\mathrm{LB}}(\tau,V)}\mathcal R_\mu(\hat g;g)
  \;\ge\;\inf_{\hat g'}\sup_{\tilde g\in H}\mathcal R_\lambda(\hat g';\tilde g).
\end{equation}
\emph{(i) Quantile map.} $F(x)\defeq\mu([0,x])$ is nondecreasing with $F(1)=1$. It is continuous because
$\mu$ is non-atomic, and in particular $F(0)=\mu(\set{0})=0$. Put
$F^{\leftarrow}(u)\defeq\inf\set{x\in[0,1]:F(x)\ge u}$ for $u\in[0,1]$. The set contains $1$ and is closed
($F$ is continuous), so the infimum is attained. Galois relation: $F^{\leftarrow}(u)\le x$ if and only if
$u\le F(x)$. Indeed, if $u\le F(x)$ then $x$ lies in the set; conversely
$F(x)\ge F(F^{\leftarrow}(u))\ge u$. Hence
$\lambda(F^{\leftarrow}\le x)=\lambda(u\le F(x))=F(x)$, that is, $F^{\leftarrow}_{\#}\lambda=\mu$.
Also $F(F^{\leftarrow}(u))=u$ on $[0,1]$. Attainment gives $F(F^{\leftarrow}(u))\ge u$. If $u=0$ then
$F^{\leftarrow}(0)=0$ and $F(0)=0$. If $u>0$ then $F^{\leftarrow}(u)>0$, every $x<F^{\leftarrow}(u)$ has
$F(x)<u$, and continuity gives $F(F^{\leftarrow}(u))\le u$. (Non-atomicity is used exactly for $F(0)=0$
and $F(F^{\leftarrow}(u))=u$.)

\emph{(ii) Membership.} Let $\tilde g\in H$ with witness $\tilde g_0$, and put $g_j\defeq\tilde g_j\circ F$,
$g_0\defeq\tilde g_0\circ F$. A composition of nondecreasing maps is nondecreasing, values stay in $[0,1]$,
and $\sup_x\abs{(\tilde g_j-\tilde g_0)(F(x))}\le\tau$. Variation does not increase. Let
$\tilde d\defeq\tilde g_j-\tilde g_0$ and take any partition $0\le x_0<\dots<x_L\le1$. The points
$u_l\defeq F(x_l)$ are nondecreasing in $[0,1]$. Removing repetitions removes only zero terms and leaves an
admissible partition, so $\sum_l\abs{\tilde d(u_l)-\tilde d(u_{l-1})}\le\mathrm{TV}(\tilde d)\le V$. With the
witness $g_0=\tilde g_0\circ F$ (nondecreasing and $[0,1]$-valued) we get
$g\in\Theta_{\mathrm{LB}}(\tau,V)$.

\emph{(iii) Data transport.} Let $x'=(U_{jk},Y_{jk})$ be data of the $\lambda$-model under $\tilde g$, and put
$T(x')\defeq(F^{\leftarrow}(U_{jk}),Y_{jk})\eqqcolon(P_{jk},Y_{jk})$. By (i) the $P_{jk}$ are i.i.d.\ $\mu$
and $U_{jk}=F(P_{jk})$. Hence $\sigma(P_{jk})=\sigma(U_{jk})$ and
$Y_{jk}\mid P_{jk}\sim\mathrm{Bern}(\tilde g_j(F(P_{jk})))=\mathrm{Bern}(g_j(P_{jk}))$, with every
independence preserved. So $T(x')$ has exactly the law of the $\mu$-model under $g=\tilde g\circ F$.

\emph{(iv) Risks.} Given $\hat g$, define $\hat g'_j(x',u)\defeq\hat g_j(T(x'),F^{\leftarrow}(u))$ (for a
randomized $\hat g$ with auxiliary $\xi$, $\hat g'_j(x',\xi,u)\defeq\hat g_j(T(x'),\xi,F^{\leftarrow}(u))$;
the argument is unchanged). By (i), $\tilde g_j(u)=g_j(F^{\leftarrow}(u))$, and $F^{\leftarrow}_{\#}\lambda=\mu$
gives
\[
  \int_0^1\bigl(\hat g'_j(x',u)-\tilde g_j(u)\bigr)^2du
  =\int_0^1\bigl(\hat g_j(T(x'),F^{\leftarrow}(u))-g_j(F^{\leftarrow}(u))\bigr)^2du
  =\norm{\hat g_j(T(x'),\cdot)-g_j}^2_{L^2(\mu)}.
\]
Taking expectations and using (iii), $\mathcal R_\lambda(\hat g';\tilde g)=\mathcal R_\mu(\hat g;\tilde g\circ F)$.
Hence $\sup_{\Theta_{\mathrm{LB}}}\mathcal R_\mu(\hat g;\cdot)\ge\sup_{\tilde g\in H}\mathcal R_\mu(\hat g;\tilde g\circ F)
=\sup_{\tilde g\in H}\mathcal R_\lambda(\hat g';\tilde g)$. Taking the infimum over $\hat g$ proves
\eqref{eq:lower-reduction}.

\smallskip\noindent\textbf{Step 1 (tools).}
\emph{Tool (A): Assouad's lemma, compound product form.} Let $N\ge1$, $\Omega=\set{-1,+1}^N$, let $\omega^{(i)}$ be
$\omega$ with coordinate $i$ flipped, and let $(P_\omega)_{\omega\in\Omega}$ be probability measures on
$(\mathcal X,\mathcal A)$. Assume:
\begin{itemize}
\item for each $i$ and $s\in\set{\pm1}$, $f_i^s:\mathcal X\to[0,\infty]$ is measurable, with
$f_i^{+1}(x)+f_i^{-1}(x)\ge\alpha_i\ge0$ for all $x$;
\item $\max_\omega d_{\mathrm{TV}}(P_\omega,P_{\omega^{(i)}})\le\beta_i$;
\item $L_\omega:\mathcal X\to[0,\infty]$ is measurable, with $L_\omega(x)\ge\sum_if_i^{\omega_i}(x)$.
\end{itemize}
Then $\max_\omega\E_{P_\omega}L_\omega\ge2^{-N}\sum_\omega\E_{P_\omega}L_\omega\ge\sum_i(\alpha_i/2)(1-\beta_i)$.
\emph{Proof of (A).} The first inequality holds because a maximum is at least a mean. By Tonelli,
$2^{-N}\sum_\omega\E_{P_\omega}L_\omega\ge\sum_i\frac12\bigl[\int f_i^{+1}d\bar P_i^{+}+\int f_i^{-1}d\bar P_i^{-}\bigr]$
with $\bar P_i^{s}\defeq2^{-(N-1)}\sum_{\omega:\omega_i=s}P_\omega$. Let $\nu\defeq\bar P_i^++\bar P_i^-$ with
densities $p^{\pm}$. Then
$\int f^+p^+d\nu+\int f^-p^-d\nu\ge\int(f^++f^-)\min(p^+,p^-)\,d\nu\ge\alpha_i\bigl(1-d_{\mathrm{TV}}(\bar P_i^+,\bar P_i^-)\bigr)$.
Finally, $\omega\mapsto\omega^{(i)}$ is a bijection from $\set{\omega_i=+1}$ onto $\set{\omega_i=-1}$. So
$\bar P_i^-=2^{-(N-1)}\sum_{\omega_i=+1}P_{\omega^{(i)}}$, and
$\abs{\bar P_i^+(A)-\bar P_i^-(A)}\le2^{-(N-1)}\sum_{\omega_i=+1}\abs{P_\omega(A)-P_{\omega^{(i)}}(A)}\le\beta_i$
for every $A$.
We apply Tool~(A) to a fixed estimator. $P_\omega$ is the joint law of all $mn$ pairs, and $L_\omega$ is the
question-averaged loss. The uniform-prior Bayes risk $2^{-N}\sum_\omega\E_{P_\omega}L_\omega$ is at most the
maximum over the hypercube, which is at most the supremum of $\mathcal R_\lambda$ over the hypercube family.
The lower bound does not depend on the estimator, so it bounds $\inf_{\hat g}\sup$.

\emph{Tool (B): an information bound.} Let $a\in(0,1/2]$ and let $g,g':[0,1]\to[a,1-a]$ be measurable. Let $Q_g$ be the
law of $(U,Y)$ with $U\sim\lambda$ and $Y\mid U\sim\mathrm{Bern}(g(U))$. Then
$\mathrm{KL}(Q_g\Vert Q_{g'})\le(a(1-a))^{-1}\int_0^1(g-g')^2du$, and the divergence of a product measure is
the sum of the divergences of its factors.
\emph{Proof of (B).} $Q_g$ has density $g(u)^y(1-g(u))^{1-y}$ with respect to $\lambda\otimes$counting,
with values in $[a,1-a]$. So log-ratios are bounded, and
$\mathrm{KL}(Q_g\Vert Q_{g'})=\int\mathrm{kl}(g(u)\Vert g'(u))\,du$. Jensen's inequality ($\log$ is concave; the
expectation is under $\mathrm{Bern}(p)$) gives $\mathrm{kl}(p\Vert q)\le\log\bigl(p^2/q+(1-p)^2/(1-q)\bigr)$.
Moreover $p^2/q+(1-p)^2/(1-q)-1=(p-q)^2/(q(1-q))$, so
$\mathrm{kl}(p\Vert q)\le\log\bigl(1+(p-q)^2/(q(1-q))\bigr)\le(p-q)^2/(q(1-q))\le(p-q)^2/(a(1-a))$. For a product,
the log-ratio is a sum of bounded factor log-ratios, each depending on one coordinate, and it integrates to
the sum of the factor divergences.

By Pinsker's inequality \citep[see, e.g.,][]{tsybakov2009introduction},
$d_{\mathrm{TV}}(P,Q)\le\sqrt{\mathrm{KL}(P\Vert Q)/2}$. So if $\mathrm{KL}(P_\omega\Vert P_{\omega^{(i)}})\le1/2$
for all $\omega$, then $\beta_i=1/2$ is admissible in Tool~(A).

\emph{Tents.} Fix an integer $K\ge1$. Let $I_b\defeq[(b-1)/K,b/K)$ for $b<K$ and $I_K\defeq[(K-1)/K,1]$; let
$\phi_b(u)\defeq(1-\abs{2Ku-(2b-1)})_+$ and $\ell(u)\defeq1/4+u/2$.
\begin{itemize}
\item[(T1)] $\phi_b$ vanishes outside $\overline{I_b}$, and $0\le\phi_b\le1$.
\item[(T2)] $\phi_b$ is continuous and piecewise linear, with slopes in $\set{0,\pm2K}$ and breakpoints in
$(1/(2K))\Z$.
\item[(T3)] $\int\phi_b^2=2\int_0^{1/(2K)}(2Kt)^2dt=1/(3K)$.
\item[(T4)] For signs $s_b$ and $h\ge0$, the function $d\defeq h\sum_bs_b\phi_b$ has $\sup\abs d\le h$ and
$\mathrm{TV}(d)\le2hK$. Indeed, $d$ is continuous and monotone on each $[t_{i-1},t_i]$, $t_i\defeq i/(2K)$, with
$\abs{d(t_i)-d(t_{i-1})}=h$. Adding all $t_i$ to a partition does not decrease the variation sum (triangle
inequality), and within each of the $2K$ pieces the refined sum telescopes to at most $h$.
\item[(T5)] If $2Kh\le1/2$, then $\ell+d$ is continuous and piecewise linear with slopes at least
$1/2-2Kh\ge0$. Hence it is nondecreasing, with values in $[1/4-h,3/4+h]$.
\end{itemize}
Separation: on $I_b$ we have $\ell+d=\ell+s_bh\phi_b$ by (T1), and $(y-t)^2+(y+t)^2\ge2t^2$ for reals $y,t$.

\smallskip\noindent\textbf{Step 2 (the term $(mn)^{-2/3}$).}
Let $K\defeq\ceil{(mn)^{1/3}}\ge1$, so that $K^3\ge mn$ and $K\le2(mn)^{1/3}$, and let $h\defeq1/(8K)$. For
$\sigma\in\set{\pm1}^K$ let $g_j\defeq g^\sigma\defeq\ell+h\sum_b\sigma_b\phi_b$ for every $j$.
\emph{Membership} holds for every $\tau,V\ge0$. Since $2Kh=1/4$, $g^\sigma$ is nondecreasing by (T5), with
values in $[1/8,7/8]$. The witness $g_0\defeq g^\sigma$ gives deviations $0$ (supremum $0$, variation $0$).
\emph{Information.} $P_\sigma=Q_{g^\sigma}^{\otimes mn}$, and flipping $b$ changes $g^\sigma$ only on $I_b$, by
$2h\phi_b$. Tool~(B) with $a=1/8$ and (T3) give
$\mathrm{KL}\le mn\cdot\frac{64}{7}\cdot\frac{4h^2}{3K}=\frac{256}{21}\,\frac{mnh^2}{K}
=\frac4{21}\,\frac{mn}{K^3}\le\frac4{21}\le\frac12$.
\emph{Loss.} Let $f_b^s(x)\defeq\frac1m\sum_j\int_{I_b}(\hat g_j(x,u)-\ell(u)-sh\phi_b(u))^2\,du$. The $I_b$ are
disjoint, so $L_\sigma\ge\sum_bf_b^{\sigma_b}$, and separation gives $\alpha_b=2h^2/(3K)$. Tool~(A) yields
$\max_\sigma\mathcal R_\lambda\ge K\cdot\frac{h^2}{3K}\cdot\frac12=\frac{h^2}6=\frac1{384K^2}\ge\frac{(mn)^{-2/3}}{1536}$.
This holds for all $m,n\ge1$.

\smallskip\noindent\textbf{Step 3 (the term $\min\set{\tau^2,1/n}$).}
Let $\theta\defeq\min\set{\tau,1/(8\sqrt n)}$, $\sigma\in\set{\pm1}^m$, $g_j\equiv1/2+\theta\sigma_j$, and take
the witness $g_0\equiv1/2$. \emph{Membership.} Constant maps are nondecreasing, with values in $[3/8,5/8]$;
$\norm{g_j-g_0}_\infty=\theta\le\tau$; and $\mathrm{TV}(g_j-g_0)=0$ because $\mathrm{TV}$ is a seminorm.
\emph{Information.} $P_\sigma=\bigotimes_jQ_{g_j}^{\otimes n}$, and flipping $j$ changes only factor $j$. Tool~(B)
with $a=3/8$ gives $\mathrm{KL}\le n\cdot\frac{64}{15}\cdot4\theta^2\le\frac4{15}\le\frac12$.
\emph{Loss.} Let $f_j^s(x)\defeq\frac1m\int_0^1(\hat g_j-1/2-s\theta)^2du$, so $\alpha_j=2\theta^2/m$. Tool~(A) gives
Bayes risk at least $m\cdot\frac{\theta^2}m\cdot\frac12=\frac{\theta^2}2=\frac12\min\set{\tau^2,\frac1{64n}}
\ge\frac{\min\set{\tau^2,1/n}}{128}$.

\smallskip\noindent\textbf{Step 4 (the term $\min\set{\tau^2,V^2,(V/n)^{2/3}}$).}
\emph{Regimes.}
\begin{itemize}
\item[(c1)] If $nV^2<2/25$: $K\defeq1$ and $h\defeq\min\set{\tau,V/2}$.
\item[(c2)] If $nV^2\ge2/25$: $K_*\defeq(25nV^2/2)^{1/3}\ge1$, $K\defeq\floor{K_*}$ and
$h\defeq\min\set{\tau,V/(2K_*)}$.
\end{itemize}
In regime (c2), $K\in[K_*/2,K_*]$: for $1\le K_*<2$, $\floor{K_*}=1>K_*/2$; for $K_*\ge2$,
$\floor{K_*}>K_*-1\ge K_*/2$. Four properties follow.
\begin{itemize}
\item[(Q1)] $h\le\tau$ in both regimes.
\item[(Q2)] $2hK\le V$. In (c1), $2h\le V$. In (c2), $2hK\le VK/K_*\le V$.
\item[(Q3)] $nh^2/K\le1/25$. In (c1), $nh^2\le nV^2/4<1/50$. In (c2),
$nh^2/K\le(2/K_*)\,nV^2/(4K_*^2)=nV^2/(2K_*^3)=1/25$.
\item[(Q4)] $h^2\ge\frac1{22}\min\set{\tau^2,V^2,(V/n)^{2/3}}$. In (c1), $h^2=\min\set{\tau^2,V^2/4}$. In (c2),
$(V/(2K_*))^2=(V/n)^{2/3}/(4\cdot12.5^{2/3})$ and $4\cdot12.5^{2/3}<21.55\le22$, so
$h^2=\min\set{\tau^2,(V/(2K_*))^2}\ge\min\set{\tau^2,(V/n)^{2/3}}/22$.
\end{itemize}
\emph{Construction.} The coordinates are $(j,b)\in[m]\times[K]$, so $N=mK$. For $\omega\in\set{\pm1}^{mK}$ let
$g_j^\omega\defeq\ell+h\sum_b\omega_{jb}\phi_b$, with witness $g_0\defeq\ell$.
\emph{Membership.} $2Kh\le V\le1/8<1/2$, so $g_j^\omega$ is nondecreasing by (T5). Since $h\le\tau\le1/8$, its
values lie in $[1/8,7/8]$. By (T4), (Q1) and (Q2), $\sup\abs{g_j^\omega-\ell}\le h\le\tau$ and
$\mathrm{TV}(g_j^\omega-\ell)\le2hK\le V$.
\emph{Information.} $P_\omega=\bigotimes_jQ_{g_j^\omega}^{\otimes n}$, and flipping $(j,b)$ changes only question
$j$'s factor, and only on $I_b$. Tool~(B) ($a=1/8$), (T3) and (Q3) give
$\mathrm{KL}\le n\cdot\frac{64}7\cdot\frac{4h^2}{3K}=\frac{256}{21}\,\frac{nh^2}K\le\frac{256}{525}<\frac12$.
\emph{Loss.} Let $f_{jb}^s(x)\defeq\frac1m\int_{I_b}(\hat g_j-\ell-sh\phi_b)^2du$, so $\alpha=2h^2/(3Km)$. Tool~(A)
gives Bayes risk at least $mK\cdot\frac{h^2}{3Km}\cdot\frac12=\frac{h^2}6\ge\frac{\min\set{\tau^2,V^2,(V/n)^{2/3}}}{132}$,
by (Q4). If $\tau=0$ or $V=0$, then $h=0$ and $\min\set{\tau^2,V^2,(V/n)^{2/3}}=0$, so there is
nothing to prove.

\smallskip\noindent\textbf{Step 5 (combination).}
The three families of Steps 2--4 are subsets of $\Theta_{\mathrm{LB}}(\tau,V)$, with maps defined on $[0,1]$
and data generated under $\lambda$. Let
$T_1\defeq(mn)^{-2/3}/1536$, $T_2\defeq\min\set{\tau^2,1/n}/128$ and $T_3\defeq\min\set{\tau^2,V^2,(V/n)^{2/3}}/132$.
By \eqref{eq:lower-reduction}, applied with each family as $H$, and by Steps 2--4,
\[
  \inf_{\hat g}\sup_{\Theta_{\mathrm{LB}}(\tau,V)}\mathcal R_\mu\;\ge\;\max\set{T_1,T_2,T_3}
  \;\ge\;\frac{T_1+T_2+T_3}{3}\;\ge\;\frac{\Phi(m,n,\tau,V)}{4608}.
\]
So the lemma holds with $c\defeq1/4608$, which does not depend on $\mu$, $m$, $n$, $\tau$ or $V$.
\end{proof}

\subsection{Least squares with a common design law}
\label{app:rate-engine}

\begin{lemma}[Least squares, independent non-identically distributed pairs]\label{lem:erm1}
Let $\mu$ be a Borel probability on $[0,1]$. Let $(P_i,Y_i)$, $i=1,\dots,N$, be \emph{independent} random
pairs with values in $[0,1]\times[-1,1]$ such that $P_i\sim\mu$ for every $i$. Let
$\eta_i(p)\defeq\E[Y_i\mid P_i=p]$ (a version) and $\bar\eta\defeq\frac1N\sum_i\eta_i$. Let $\mathcal F$ be a
nonempty convex class of measurable functions $[0,1]\to[-1,1]$, closed in $L^2(\mu)$, and let $f^*$ be the
$L^2(\mu)$-projection of $\bar\eta$ onto $\mathcal F$. Assume the pointwise bracketing entropy bound
\begin{equation}\label{eq:erm1-entropy}
  \log N_{[\,]}(\eps,\mathcal F,L^2(\mu))\;\le\;\frac A\eps+B\log\frac e\eps\qquad\text{for all }0<\eps\le2,
\end{equation}
for some $A,B\ge0$. Let $\hat f$ be a least-squares estimator, that is, a measurable selection of
$\argmin_{f\in\mathcal F}\sum_{i=1}^N(Y_i-f(P_i))^2$ (assumed to exist). Then, for a universal constant $C$,
\[
  \E\norm{\hat f-f^*}^2_{L^2(\mu)}\;\le\;C\Bigl[\Bigl(\frac AN\Bigr)^{2/3}+\frac{(B+1)\log(eN)}N\Bigr].
\]
Here $\E$ may be read as the outer expectation $\E^*$, which equals $\E$ whenever $\norm{\hat f-f^*}$ is
measurable.
\end{lemma}

The regression functions $\eta_i$ may differ from pair to pair; only the design law is shared. The proof
follows the standard route: a basic inequality, bracketing chaining, and peeling. Every empirical-process
step is written for independent, non-identically distributed summands. Each summand is centred by its
\emph{own} mean, and the average $\bar\eta$ enters only through the population risk.

\paragraph{Notation for \cref{lem:erm1,prop:erm2}.}
$\norm\cdot\defeq\norm\cdot_{L^2(\mu)}$, $Z_i\defeq(P_i,Y_i)\in\mathcal Z\defeq[0,1]\times[-1,1]$. For bounded
measurable $h:\mathcal Z\to\R$,
\[
  \mathbb G_Nh\defeq N^{-1/2}\sum_{i=1}^N\bigl(h(Z_i)-\E h(Z_i)\bigr),\qquad
  \norm h_N\defeq\Bigl(N^{-1}\sum_{i=1}^N\E h(Z_i)^2\Bigr)^{1/2}.
\]
$\mathbb G_N$ centres each summand by its own mean, and identical distribution is never used. Let
$g_f(p,y)\defeq(y-f(p))^2-(y-f^*(p))^2$ and $\hat D\defeq\norm{\hat f-f^*}\in[0,2]$. The bounds are proved
for the outer expectation $\E^*$, which equals $\E$ whenever $\hat D$ is measurable. Every event below is built
from countably many fixed measurable functions evaluated at $Z_1,\dots,Z_N$, hence it is measurable. The
estimator $\hat f$ enters only through inclusions that hold pointwise in $\omega$.

\paragraph{Bernstein's inequality in quantile form.}
Let $X_1,\dots,X_N$ be independent, with $\E X_i=0$, $\abs{X_i}\le b$ and $\sum_i\Var X_i\le v$. Then
$\P(\sum_iX_i>t)\le\exp\bigl(-t^2/(2(v+bt/3))\bigr)$ for $t\ge0$. This is Bernstein's inequality;
\citet[Thm.~2.8.4]{vershynin2018high} states the two-sided form, and the one-sided form is the upper-tail half
of the same Chernoff argument \citep[see also][]{boucheron2013concentration}. We use three consequences.
\begin{itemize}
\item[(B1)] $\P\bigl(\sum_iX_i>\sqrt{2vu}+2bu/3\bigr)\le e^{-u}$ for $u\ge0$. Indeed, with
$t\defeq\sqrt{2vu}+2bu/3$ we have $t^2-2u(v+bt/3)=(2bu/3)\sqrt{2vu}\ge0$, so $t^2/(2(v+bt/3))\ge u$. (If
$v=b=0$, the probability is $0$.)
\item[(B2)] Let $h$ be a fixed measurable function with $\abs{h(z)-\E h(Z_i)}\le b$ for all $z$ and $i$, and
$\norm h_N\le\sigma$. Then $\sum_i\Var h(Z_i)\le N\sigma^2$ and
$\P\bigl(\mathbb G_Nh>\sigma\sqrt{2u}+2bu/(3\sqrt N)\bigr)\le e^{-u}$. If $0\le h\le c$, then $b=c$ works; if
$\abs h\le c$, then $b=2c$ works.
\item[(B3)] For a finite family $\mathcal H$ with common $\sigma$ and $b$, the union of the events in (B2)
has probability at most $\abs{\mathcal H}e^{-u}$. The two-sided version (apply (B2) to $h$ and to $-h$)
has probability at most $2\abs{\mathcal H}e^{-u}$. Only the independence of $Z_1,\dots,Z_N$ is used.
\end{itemize}

\begin{lemma}[Bracketing chaining, independent non-identically distributed observations]\label{lem:chaining}
Let $Z_1,\dots,Z_N$ be independent random elements of $\mathcal Z$; they need not be identically
distributed. Let $\mathcal G$ be a class of measurable $g:\mathcal Z\to\R$ with $\sup\abs g\le M$ and
$\norm g_N\le\rho$. Let $\tilde H:(0,\rho]\to[0,\infty)$ be nonincreasing. Assume that for each $q\ge0$ the
class $\mathcal G$ is covered by $m_q\le e^{\tilde H(\eps_q)}$ pointwise brackets $[L_{q,j},U_{q,j}]$ with
$\norm{U_{q,j}-L_{q,j}}_N\le\eps_q\defeq2^{-q}\rho$. Put $\tilde J\defeq\int_0^\rho\sqrt{\tilde H(\eps)}\,d\eps$.
Then for every $x\ge0$ there is an event $\Omega_x$ with $\P(\Omega_x)\ge1-2e^{-x}$ on which, for all
$g\in\mathcal G$ simultaneously,
\[
  \abs{\mathbb G_Ng}\;\le\;52\tilde J+26\rho\sqrt x+39\rho+\frac{8M}{3\sqrt N}\bigl(\tilde H(\rho)+x+1\bigr).
\]
\end{lemma}

The proof chains over nested partitions at the dyadic scales $\eps_q$. It applies Bernstein's inequality
to each link and truncates the links whose brackets are too wide, which yields the bracketing entropy
integral $\tilde J$. The scheme is that of \citet[Lemma~3.4.2]{vandervaart1996weak}, with the i.i.d.\ law
replaced by the average $N^{-1}\sum_i\E[\,\cdot\,(Z_i)]$.

\begin{proof}
\emph{(a) Nested partitions.} For $g\in\mathcal G$ let $j_q(g)\defeq\min\set{j:L_{q,j}\le g\le U_{q,j}}$. Write
$g\sim_qg'$ if $j_r(g)=j_r(g')$ for all $r\le q$. These partitions are nested. Level $q$ has
$n_q\le\prod_{r\le q}m_r\le e^{S_q}$ classes, where $S_q\defeq\sum_{r=0}^q\tilde H(\eps_r)$. Fix,
deterministically, a representative in $\mathcal G$ of each class, and let $\pi_qg$ be the representative of
the level-$q$ class of $g$. Let
$\Delta_qg\defeq\min\set{2M,\min_{r\le q}(U_{r,j_r(g)}-L_{r,j_r(g)})}$. Then:
\begin{itemize}
\item[(L1)] $\pi_qg$ and $\Delta_qg$ depend on $g$ only through its level-$q$ class;
\item[(L2)] $0\le\Delta_qg\le\Delta_{q-1}g\le2M$;
\item[(L3)] $\norm{\Delta_qg}_N\le\eps_q$;
\item[(L4)] $\abs{g-\pi_qg}\le\Delta_qg$, because both functions lie in the bracket $j_r(g)$ for every
$r\le q$ and both are bounded by $M$;
\item[(L5)] $\abs{\pi_qg-\pi_{q-1}g}\le\Delta_{q-1}g$, because by nestedness $\pi_qg$ lies in the
level-$(q-1)$ class of $g$, so the argument of (L4) applies at level $q-1$.
\end{itemize}

\emph{(b) Truncation.} Let $u_q\defeq S_q+x+q+1$ and $a_q\defeq\eps_q/\sqrt{u_{q+1}}$. Then $a_q\le\eps_q$ and
$a_{q-1}u_q=2\eps_q\sqrt{u_q}$. Let $A_qg\defeq\mathbf 1\set{\Delta_rg\le\sqrt Na_r\text{ for all }r\le q}$,
$B_0g\defeq1-A_0g$, and $B_qg\defeq A_{q-1}g-A_qg=A_{q-1}g\,\mathbf 1\set{\Delta_qg>\sqrt Na_q}$ for $q\ge1$.
All of them depend on $g$ only through its level-$q$ class.

\emph{(c) An exact identity.} For every $Q\ge0$, pointwise,
\[
  g-\pi_0g=(g-\pi_0g)B_0g+\sum_{q=1}^Q\Bigl[(g-\pi_qg)B_qg+(\pi_qg-\pi_{q-1}g)A_{q-1}g\Bigr]+(g-\pi_Qg)A_Qg.
\]
This follows by induction. The case $Q=0$ is $A_0+B_0=1$. For the inductive step,
$(g-\pi_Qg)A_Qg=(g-\pi_{Q+1}g)(B_{Q+1}g+A_{Q+1}g)+(\pi_{Q+1}g-\pi_Qg)A_Qg$.

\emph{(d) Envelopes.} Let $k_qg\defeq\Delta_qg\,B_qg$ for $q\ge0$, and
$h_qg\defeq(\pi_qg-\pi_{q-1}g)A_{q-1}g$ for $q\ge1$. Then:
\begin{itemize}
\item $\abs{(g-\pi_qg)B_qg}\le k_qg$, by (L4);
\item $0\le k_0g\le2M$, and for $q\ge1$, $0\le k_qg\le\Delta_{q-1}g\,A_{q-1}g\le\sqrt Na_{q-1}$, by (L2) and
$B_q\le A_{q-1}$;
\item $\norm{k_qg}_N\le\eps_q$;
\item $k_qg\le\Delta_qg\,\mathbf 1\set{\Delta_qg>\sqrt Na_q}\le(\Delta_qg)^2/(\sqrt Na_q)$, hence
$N^{-1}\sum_i\E k_qg(Z_i)\le\eps_q^2/(\sqrt Na_q)$;
\item $\abs{h_qg}\le\sqrt Na_{q-1}$ and $\norm{h_qg}_N\le\eps_{q-1}=2\eps_q$, by (L5);
\item $\abs{(g-\pi_Qg)A_Qg}\le\sqrt Na_Q$.
\end{itemize}

\emph{(e) Comparison facts.} If $\abs h\le k$ pointwise, then
$\abs{\mathbb G_Nh}\le\mathbb G_Nk+2N^{-1/2}\sum_i\E k(Z_i)$, because
$\pm h(Z_i)\mp\E h(Z_i)\le k(Z_i)+\E k(Z_i)$. Always $\abs{\mathbb G_Nh}\le2\sqrt N\sup\abs h$.

\emph{(f) Events.} By (L1), as $g$ ranges over $\mathcal G$ there are at most $n_0$ functions $\pi_0g$ and
$k_0g$, and at most $n_q$ functions $k_qg$ and $h_qg$. Apply (B3) with $u=u_q\ge\log n_q+x+q+1$ to the
following families:
\begin{itemize}
\item $\pi_0g$, two-sided, with $\sigma=\rho$ and $b=2M$ (failure probability at most $2e^{-x-1}$);
\item $k_0g$, one-sided, with $\sigma=\rho$ and $b=2M$ (at most $e^{-x-1}$);
\item for $q\ge1$, $k_qg$, one-sided, with $\sigma=\eps_q$ and $b=\sqrt Na_{q-1}$ (at most $e^{-x-q-1}$);
\item for $q\ge1$, $h_qg$, two-sided, with $\sigma=2\eps_q$ and $b=2\sqrt Na_{q-1}$ (at most $2e^{-x-q-1}$).
\end{itemize}
Let $\Omega_x$ be the intersection of these countably many events. Then
$\P(\Omega_x^c)\le3\sum_{q\ge0}e^{-x-q-1}=3e^{-x}/(e-1)<2e^{-x}$.

\emph{(g) Bounds on $\Omega_x$.} On $\Omega_x$, for all $g$, (B2), (d) and (e) give
\begin{align*}
  \abs{\mathbb G_N\pi_0g}&\le\rho\sqrt{2u_0}+\tfrac{4Mu_0}{3\sqrt N},\\
  \abs{\mathbb G_N((g-\pi_0g)B_0g)}&\le\mathbb G_Nk_0g+2\eps_0^2/a_0
    \le\rho\sqrt{2u_0}+\tfrac{4Mu_0}{3\sqrt N}+2\eps_0\sqrt{u_1},\\
  \abs{\mathbb G_N((g-\pi_qg)B_qg)}&\le\eps_q\sqrt{2u_q}+\tfrac23a_{q-1}u_q+2\eps_q^2/a_q
    =\bigl(\sqrt2+\tfrac43\bigr)\eps_q\sqrt{u_q}+2\eps_q\sqrt{u_{q+1}}\;\ (q\ge1),\\
  \abs{\mathbb G_Nh_qg}&\le2\eps_q\sqrt{2u_q}+\tfrac43a_{q-1}u_q=\bigl(2\sqrt2+\tfrac83\bigr)\eps_q\sqrt{u_q}\;\ (q\ge1).
\end{align*}
The remainder satisfies $\abs{\mathbb G_N((g-\pi_Qg)A_Qg)}\le2Na_Q\to0$ as $Q\to\infty$, with $N$ fixed. Sum
the pieces of (c) and let $Q\to\infty$. The truncation-bias terms satisfy
$\sum_{q\ge0}2\eps_q\sqrt{u_{q+1}}=4\sum_{q\ge1}\eps_q\sqrt{u_q}$, so
\[
  \abs{\mathbb G_Ng}\le(3\sqrt2+8)\sum_{q\ge0}\eps_q\sqrt{u_q}+\frac{8Mu_0}{3\sqrt N}
  \le13\sum_{q\ge0}\eps_q\sqrt{u_q}+\frac{8Mu_0}{3\sqrt N}.
\]

\emph{(h) The entropy sum.} We have $\sqrt{u_q}\le\sqrt{S_q}+\sqrt x+\sqrt{q+1}$ and
$\sqrt{S_q}\le\sum_{r\le q}\sqrt{\tilde H(\eps_r)}$. Hence
$\sum_q\eps_q\sqrt{S_q}\le\sum_r\sqrt{\tilde H(\eps_r)}\sum_{q\ge r}\eps_q=2\sum_r\eps_r\sqrt{\tilde H(\eps_r)}$.
Since $\tilde H$ is nonincreasing, $\eps_r\sqrt{\tilde H(\eps_r)}\le2\int_{\eps_{r+1}}^{\eps_r}\sqrt{\tilde H}$, so
$\sum_q\eps_q\sqrt{S_q}\le4\tilde J$. Further, $\sum_q\eps_q\sqrt x=2\rho\sqrt x$,
$\sum_q\eps_q\sqrt{q+1}=\rho\sum_q2^{-q}\sqrt{q+1}\le2.70\rho\le3\rho$, and $u_0=\tilde H(\rho)+x+1$.
Combined with (g), these bounds give the lemma.
\end{proof}

\begin{proofof}{\cref{lem:erm1}}
\textbf{Step 1 (population geometry).}
(i) $\abs{Y_i}\le1$ gives $\abs{\eta_i}\le1$ $\mu$-a.e.\ (as $P_i\sim\mu$), so $\bar\eta\in L^2(\mu)$.
(ii) $\mathcal F$ is nonempty and convex, and its set of $\mu$-classes is closed in $L^2(\mu)$. By the Hilbert
projection theorem a unique nearest class to $\bar\eta$ exists. Fix a representative $f^*\in\mathcal F$; the
quantity in the lemma does not depend on the choice. For $f\in\mathcal F$ and $t\in(0,1]$ we have
$f^*+t(f-f^*)\in\mathcal F$, so $\norm{\bar\eta-f^*-t(f-f^*)}^2\ge\norm{\bar\eta-f^*}^2$, that is,
$-2t\inner{\bar\eta-f^*}{f-f^*}+t^2\norm{f-f^*}^2\ge0$. Dividing by $t$ and letting $t\to0$ gives
$\inner{\bar\eta-f^*}{f-f^*}\le0$.
(iii) Let $\bar R(f)\defeq N^{-1}\sum_i\E(Y_i-f(P_i))^2$. For bounded measurable $f$, the tower property and
$P_i\sim\mu$ give $\E(Y_i-f(P_i))^2=\E Y_i^2-2\inner{\eta_i}f+\norm f^2$. Averaging,
$\bar R(f)=c+\norm{f-\bar\eta}^2$ with $c\defeq N^{-1}\sum_i\E Y_i^2-\norm{\bar\eta}^2$, which does not depend on $f$.
This is where the common design law is used.
(iv) Hence, by (ii), for every $f\in\mathcal F$,
\begin{equation}\label{eq:erm-excess}
  \bar R(f)-\bar R(f^*)=\norm{f-f^*}^2+2\inner{f-f^*}{f^*-\bar\eta}\;\ge\;\norm{f-f^*}^2.
\end{equation}
(v) For every $f\in\mathcal F$,
\begin{equation}\label{eq:erm-mean}
  \textstyle\sum_{i=1}^N\E g_f(Z_i)=N\bigl(\bar R(f)-\bar R(f^*)\bigr).
\end{equation}

\textbf{Step 2 (the loss-difference class).}
We have $g_f=(f^*-f)(2y-f-f^*)$ and $\abs{2y-f-f^*}\le4$. Moreover
$\E g_f(Z_i)^2\le16\,\E(f-f^*)^2(P_i)=16\norm{f-f^*}^2$ for each $i$, since $P_i\sim\mu$. Therefore
\begin{equation}\label{eq:erm-envelope}
  \abs{g_f}\le4\abs{f-f^*}\le8\ \text{pointwise},\qquad\norm{g_f}_N\le4\norm{f-f^*}.
\end{equation}
\emph{Brackets} (no monotonicity of $f\mapsto g_f$ is used). Fix $\eps\in(0,8]$ and a minimal cover of
$\mathcal F$ by $\eps/4$-brackets $[l_j,u_j]$ in $L^2(\mu)$. Drop the brackets that contain no element of
$\mathcal F$. For each remaining bracket put $l'_j\defeq\max(l_j,-1)$ and $u'_j\defeq\min(u_j,1)$. Some
$f\in\mathcal F$ with values in $[-1,1]$ lies in the bracket, so $-1\le l'_j\le u'_j\le1$ and
$u'_j-l'_j\le u_j-l_j$. Every $f\in\mathcal F$ that lies in $[l_j,u_j]$ also lies in $[l'_j,u'_j]$. Define
\[
  U_j(p,y)\defeq\max\set{(y-l'_j(p))^2,(y-u'_j(p))^2}-(y-f^*(p))^2,\quad
  L_j(p,y)\defeq(y-\kappa_j)^2-(y-f^*(p))^2,
\]
where $\kappa_j\defeq\min(\max(y,l'_j(p)),u'_j(p))$. Both are measurable. Since $t\mapsto(y-t)^2$ is convex,
they are the maximum and the minimum of $t\mapsto(y-t)^2-(y-f^*(p))^2$ over $t\in[l'_j(p),u'_j(p)]$. Hence
$f\in\mathcal F\cap[l_j,u_j]$ implies $L_j\le g_f\le U_j$. For the width, take $s,t\in[l'_j,u'_j]\subseteq[-1,1]$.
Then $\abs{(y-s)^2-(y-t)^2}=\abs{s-t}\abs{2y-s-t}\le4(u'_j-l'_j)$, so $0\le U_j-L_j\le4(u_j-l_j)(p)$. Since
$P_i\sim\mu$, $\norm{U_j-L_j}_N^2\le16N^{-1}\sum_i\E(u_j-l_j)^2(P_i)=16\norm{u_j-l_j}^2\le\eps^2$.
\emph{Consequence.} Let $\delta>0$ and $\eps\in(0,8]$. The class
$\mathcal G_\delta\defeq\set{g_f:f\in\mathcal F,\ \norm{f-f^*}\le\delta}$ is covered by at most
$N_{[\,]}(\eps/4,\mathcal F,L^2(\mu))\le e^{H(\eps)}$ pointwise brackets of $\norm\cdot_N$-width at most $\eps$.
Here $H(\eps)\defeq4A/\eps+B\log(4e/\eps)$, from \eqref{eq:erm1-entropy} at $\eps/4\in(0,2]$. $H$ is
nonincreasing and nonnegative on $(0,8]$, because $\log(4e/\eps)\ge\log(e/2)>0$.

\textbf{Step 3 (a localized bound).}
For $\delta\in(0,2]$ and $x\ge0$ there is an event $\Omega_{\delta,x}$ with probability at least $1-2e^{-x}$ on
which, for all $f\in\mathcal F$ with $\norm{f-f^*}\le\delta$,
\begin{equation}\label{eq:erm1-local}
\begin{split}
  \abs{\mathbb G_Ng_f}\le\Gamma_N(\delta,x)\defeq{}&416\sqrt{A\delta}+208\,\delta\sqrt{B(1+\log(e/\delta))}+104\,\delta\sqrt x\\
  &+156\,\delta+\frac{22}{\sqrt N}\Bigl(\frac A\delta+B\log\frac e\delta+x+1\Bigr).
\end{split}
\end{equation}
Indeed, apply \cref{lem:chaining} to $\mathcal G_\delta$ with $M=8$, $\rho=4\delta\le8$ and $\tilde H=H$. By
Step 2, $\norm{g_f}_N\le4\delta$, and the covers of Step 2 exist at every level
$\eps_q=2^{-q}\cdot4\delta\in(0,8]$. By $\sqrt{a+b}\le\sqrt a+\sqrt b$ and Cauchy--Schwarz,
$\int_0^\delta\sqrt{\log(e/s)}\,ds\le\bigl(\delta\int_0^\delta\log(e/s)\,ds\bigr)^{1/2}=\delta\sqrt{1+\log(e/\delta)}$, so
\[
  \tilde J=4\int_0^\delta\sqrt{A/s+B\log(e/s)}\,ds\le4\Bigl[2\sqrt{A\delta}+\delta\sqrt{B(1+\log(e/\delta))}\Bigr].
\]
Also $H(4\delta)=A/\delta+B\log(e/\delta)$, and $64/3\le22$.

\textbf{Step 4 (basic inequality).}
$\hat f$ minimizes the empirical risk and $f^*\in\mathcal F$, so $\sum_ig_{\hat f}(Z_i)\le0$ for every $\omega$. For
each fixed $f\in\mathcal F$ and $\omega$, \eqref{eq:erm-mean} gives
$\sum_ig_f(Z_i)=\sqrt N\,\mathbb G_Ng_f+N(\bar R(f)-\bar R(f^*))$. At $f=\hat f(\omega)$, with \eqref{eq:erm-excess},
$N\hat D^2\le N(\bar R(\hat f)-\bar R(f^*))\le-\sqrt N\,(\mathbb G_Ng_f)|_{f=\hat f}$, hence
\begin{equation}\label{eq:erm-basic}
  \sqrt N\,\hat D^2\;\le\;\abs{\mathbb G_Ng_f}\big|_{f=\hat f(\omega)}\qquad\text{for every }\omega.
\end{equation}
This is the only place where the aggregate centring enters; $Y_i-\bar\eta(P_i)$ is never centred individually.

\textbf{Step 5 (peeling).}
Let $K_0\defeq2\cdot2496^2=12{,}460{,}032$ and
$r(x)^2\defeq K_0\bigl[(A/N)^{2/3}+(B+1)\log(eN)/N+x/N\bigr]$. We claim that there is an event $\Omega'_x$
with $\P(\Omega'_x)\ge1-4e^{-x}$ on which $\hat D\le r\defeq r(x)$. If $r\ge2$, take $\Omega'_x\defeq\Omega$,
since $\hat D\le2$. If $r<2$, then for each $j\ge0$ with $2^jr<2$ let
$\delta_j\defeq\min(2^{j+1}r,2)\in[r,2]$, and let $\Omega'_x\defeq\bigcap_j\Omega_{\delta_j,x+j}$. Then
$\P(\Omega'^c_x)\le\sum_j2e^{-x-j}=(2e/(e-1))e^{-x}<4e^{-x}$. Suppose $\hat D>r$ at some point of $\Omega'_x$.
Since $\hat D\le2$, some $j$ has $2^jr<\hat D\le\delta_j$. By Step 4 and \eqref{eq:erm1-local},
\[
  T\defeq\sqrt N\,4^jr^2<\sqrt N\,\hat D^2\le\Gamma_N(\delta_j,x+j).
\]
We use the following facts. $r^2\ge K_0/N\ge1/N$, so $\log(e/\delta_j)\le\log(e/r)\le1+\frac12\log N\le\log(eN)$
and $1+\log(e/\delta_j)\le2\log(eN)$. Also $\delta_j\le2^{j+1}r$, $2^{j/2}\le2^j\le4^j$,
$\sqrt{x+j}\le2^j(\sqrt x+1)$ and $j+1\le4^j$. Term by term:
\begin{itemize}
\item[(a)] $416\sqrt{A\delta_j}\le416\cdot4^j\sqrt{2Ar}\le T/6$, because $r^3\ge K_0^{3/2}A/N\ge2\cdot2496^2A/N$;
\item[(b)] $208\,\delta_j\sqrt{B(1+\log(e/\delta_j))}\le416\cdot4^jr\sqrt{2B\log(eN)}\le T/6$, because
$r^2\ge K_0B\log(eN)/N$;
\item[(c)] $104\,\delta_j\sqrt{x+j}\le208\cdot4^jr(\sqrt x+1)\le T/6$, because
$r^2\ge K_0(1+x)/N\ge2\cdot1248^2(x+1)/N$;
\item[(d)] $156\,\delta_j\le312\cdot4^jr\le T/6$, because $K_0\ge1872^2$;
\item[(e)] $\frac{22}{\sqrt N}\bigl(A/\delta_j+B\log(e/\delta_j)+x+j+1\bigr)\le\frac{22}{\sqrt N}\bigl(A/r+B\log(eN)+4^j(x+1)\bigr)$,
and each of the three pieces is at most $T/18$, because $r^3\ge396A/N$, $r^2\ge396B\log(eN)/N$ and
$r^2\ge396(x+1)/N$.
\end{itemize}
The total is at most $(4/6+3/18)T=(5/6)T<T$, a contradiction. So $\hat D\le r$ on $\Omega'_x$.

\textbf{Step 6 (from high probability to expectation).}
Let $W(\omega)\defeq\min\bigl(\set{r(k)^2:k\in\Z_{\ge0},\ \omega\in\Omega'_k}\cup\set4\bigr)$. It is measurable
(countably many measurable events) and $\hat D^2\le W$ pointwise. Moreover
$\P(W>r(k)^2)\le\P((\Omega'_k)^c)\le4e^{-k}$. Hence
\[
  \E^*\hat D^2\le\E W=\int_0^\infty\P(W>t)\,dt\le r(0)^2+\sum_{k\ge0}\bigl(r(k+1)^2-r(k)^2\bigr)4e^{-k}
  =r(0)^2+\frac{4e}{e-1}\cdot\frac{K_0}N.
\]
Since $1/N\le(B+1)\log(eN)/N$,
$\E\hat D^2=\E^*\hat D^2\le K_0\bigl(1+\frac{4e}{e-1}\bigr)\bigl[(A/N)^{2/3}+(B+1)\log(eN)/N\bigr]
\le10^8\bigl[(A/N)^{2/3}+(B+1)\log(eN)/N\bigr]$. The constant $C=10^8$ does not depend on $A$, $B$, $N$, $\mu$ or
the $\eta_i$.
\end{proofof}

\paragraph{Localized least squares under a common design.}
We restate \cref{prop:erm2}, discussed in \cref{sec:erm}, and prove it.

\PropERMtwo*

In the proof of \cref{lem:erm1}, the entropy hypothesis is used only to cover the classes
$\mathcal G_\delta$, which involve only those $f$ with $\norm{f-f^*}\le\delta$. Under a hypothesis on these
local classes, the parametric part of the entropy integral over a shell of radius $\delta$ is of order
$\delta\sqrt B$ instead of $\delta\sqrt{B\log(e/\delta)}$. Every peeling threshold is then free of $\delta$,
and the $\log(eN)$ factor disappears.

\begin{proof}
The setting and notation are those of \cref{lem:erm1} and its proof. In addition,
$\mathcal F(\delta)\defeq\set{f\in\mathcal F:\norm{f-f^*}\le\delta}$, and the localized hypothesis reads
$\log N_{[\,]}(\eps,\mathcal F(\delta),L^2(\mu))\le A/\eps+B\log(e\delta/\eps)$ for $0<\delta\le2$ and
$0<\eps\le\delta$.

\emph{Imports.} Four ingredients of the proof of \cref{lem:erm1} use only the common setting: independence,
$P_i\sim\mu$, $\abs{Y_i}\le1$, $\mathcal F$ nonempty, convex, closed and $[-1,1]$-valued, and $\hat f$ a
least-squares estimator. None of them uses the entropy hypothesis \eqref{eq:erm1-entropy}. They are:
\begin{itemize}
\item[(I1)] Step 1: the representative $f^*\in\mathcal F$ exists, and \eqref{eq:erm-excess} and
\eqref{eq:erm-mean} hold;
\item[(I2)] the envelope bounds \eqref{eq:erm-envelope};
\item[(I3)] \cref{lem:chaining}, together with Bernstein's inequality (B1)--(B3);
\item[(I4)] the basic inequality \eqref{eq:erm-basic}.
\end{itemize}
The parts of that proof that used the global hypothesis (the brackets of Step 2, and Steps 3, 5 and 6) are
replaced by Steps 1$'$--4$'$ below.

\textbf{Step 1$'$ (localized loss brackets).}
Fix $\delta\in(0,2]$ and $\eps\in(0,4\delta]$. Since $0<\eps/4\le\delta$, the hypothesis at $(\delta,\eps/4)$
gives a cover of $\mathcal F(\delta)$ by at most $\exp(4A/\eps+B\log(4e\delta/\eps))$ pointwise
$\eps/4$-brackets $[l_j,u_j]$ in $L^2(\mu)$. Drop the brackets that contain no element of $\mathcal F(\delta)$.
For each remaining $j$ put $l'_j\defeq\max(l_j,-1)$ and $u'_j\defeq\min(u_j,1)$. Some $f\in\mathcal F(\delta)$
with values in $[-1,1]$ lies in $[l_j,u_j]$, so $-1\le l'_j\le u'_j\le1$ everywhere and
$0\le u'_j-l'_j\le u_j-l_j$. Every $f\in\mathcal F$ that lies in $[l_j,u_j]$ also lies in $[l'_j,u'_j]$. Define
$U_j$, $L_j$ and $\kappa_j$ as in Step 2 of the proof of \cref{lem:erm1}. For fixed $(p,y)$, the function
$\varphi(t)\defeq(y-t)^2-(y-f^*(p))^2$ is convex in $t$. On $[l'_j(p),u'_j(p)]$ its maximum is attained at an
endpoint, where it equals $U_j(p,y)$. Its minimum is attained at the projection $\kappa_j$ of $y$ onto the
interval, where it equals $L_j(p,y)$. Hence $f\in\mathcal F(\delta)\cap[l_j,u_j]$ implies $L_j\le g_f\le U_j$ on
$\mathcal Z$. For the width, $U_j-L_j=\varphi(s)-\varphi(t)$ for some
$s,t\in[l'_j(p),u'_j(p)]\subseteq[-1,1]$, and $\abs{\varphi(s)-\varphi(t)}=\abs{s-t}\abs{2y-s-t}\le4(u'_j-l'_j)(p)$.
So $0\le U_j-L_j\le4(u_j-l_j)(p)$ and, since $P_i\sim\mu$,
$\norm{U_j-L_j}_N^2\le16N^{-1}\sum_i\E(u_j-l_j)^2(P_i)=16\norm{u_j-l_j}^2\le\eps^2$.
\emph{Consequence.} For every $\eps\in(0,4\delta]$, the class $\mathcal G_\delta\defeq\set{g_f:f\in\mathcal F(\delta)}$
is covered by at most $\exp(\tilde H_\delta(\eps))$ pointwise brackets of $\norm\cdot_N$-width at most $\eps$,
where $\tilde H_\delta(\eps)\defeq4A/\eps+B\log(4e\delta/\eps)$. On $(0,4\delta]$, $\tilde H_\delta$ is
nonincreasing and nonnegative, since $\log(4e\delta/\eps)\ge1$ there.

\textbf{Step 2$'$ (a localized bound).}
For $\delta\in(0,2]$ and $x\ge0$ there is an event $\Omega_{\delta,x}$ with $\P(\Omega_{\delta,x})\ge1-2e^{-x}$
on which, for all $f\in\mathcal F(\delta)$,
\begin{equation}\label{eq:erm2-local}
  \abs{\mathbb G_Ng_f}\le\Gamma'_N(\delta,x)\defeq416\sqrt{A\delta}+208\sqrt2\,\delta\sqrt B+104\,\delta\sqrt x+156\,\delta
  +\frac{22}{\sqrt N}\Bigl(\frac A\delta+B+x+1\Bigr).
\end{equation}
Indeed, apply (I3) to $\mathcal G\defeq\mathcal G_\delta$ with $M\defeq8$, $\rho\defeq4\delta$ and
$\tilde H\defeq\tilde H_\delta$ on $(0,\rho]$. By (I2), $\sup\abs{g_f}\le8$ and $\norm{g_f}_N\le4\delta$ on
$\mathcal F(\delta)$. For each $q\ge0$ we use the cover of Step 1$'$ at $\eps_q=2^{-q}\cdot4\delta\in(0,4\delta]$.
For the entropy integral, substitute $\eps=4s$, use $\sqrt{a+b}\le\sqrt a+\sqrt b$, and then substitute
$s=\delta t$:
\begin{align*}
  \tilde J=4\int_0^\delta\sqrt{A/s+B\log(e\delta/s)}\,ds
  &\le4\Bigl[2\sqrt{A\delta}+\delta\sqrt B\int_0^1\sqrt{1+\log(1/t)}\,dt\Bigr]\\
  &\le8\sqrt{A\delta}+4\sqrt2\,\delta\sqrt B.
\end{align*}
The last step is Cauchy--Schwarz on $[0,1]$:
$\int_0^1\sqrt{1+\log(1/t)}\,dt\le\bigl(\int_0^1(1+\log(1/t))\,dt\bigr)^{1/2}=\sqrt2$. Also
$\tilde H(\rho)=A/\delta+B$. Inserting these into (I3) gives $52\cdot8=416$, $52\cdot4\sqrt2=208\sqrt2$,
$26\cdot4=104$, $39\cdot4=156$ and $8M/3=64/3\le22$.

\textbf{Step 3$'$ (peeling).}
Let $K_0\defeq2\cdot2496^2=12{,}460{,}032$ and $r(x)^2\defeq K_0\bigl[(A/N)^{2/3}+(B+1)/N+x/N\bigr]$. We claim
that there is an event $\Omega'_x$ with $\P(\Omega'_x)\ge1-4e^{-x}$ on which $\hat D\le r\defeq r(x)$. If
$r\ge2$, take $\Omega'_x\defeq\Omega$, since $\hat D\le2$ always. If $r<2$, let $J\defeq\set{j\ge0:2^jr<2}$, which
is finite and contains $0$. Put $\delta_j\defeq\min(2^{j+1}r,2)$, which lies in $(2^jr,2^{j+1}r]\cap(0,2]$, and let
$\Omega'_x\defeq\bigcap_{j\in J}\Omega_{\delta_j,x+j}$. Then
$\P((\Omega'_x)^c)\le\sum_{j\ge0}2e^{-x-j}=(2e/(e-1))e^{-x}<4e^{-x}$.
Suppose $\omega\in\Omega'_x$ and $\hat D>r$. Since $r<\hat D\le2$, the integer $j$ with $2^jr<\hat D\le2^{j+1}r$
satisfies $j\ge0$, $j\in J$ and $\hat D\le\delta_j$, so $\hat f(\omega)\in\mathcal F(\delta_j)$. By (I4) and
\eqref{eq:erm2-local} at $(\delta_j,x+j)$,
\[
  T\defeq\sqrt N\,4^jr^2<\sqrt N\,\hat D^2\le\abs{\mathbb G_Ng_f}\big|_{f=\hat f}\le\Gamma'_N(\delta_j,x+j).
\]
We use the following facts: $r>0$; $r\le\delta_j\le2^{j+1}r$; $2^{j/2}\le2^j\le4^j$;
$\sqrt{x+j}\le\sqrt x+\sqrt j$ and $\sqrt j\le2^j$; $x+j+1\le4^j(x+1)$; and $r^2\ge K_0(1+x)/N$,
$r^2\ge K_0B/N$, $r^3\ge K_0^{3/2}A/N\ge K_0A/N$. Term by term:
\begin{itemize}
\item[(a)] $416\sqrt{A\delta_j}\le416\sqrt2\cdot4^j\sqrt{Ar}\le T/6$, since $Nr^3\ge2\cdot2496^2A=K_0A$;
\item[(b)] $208\sqrt2\,\delta_j\sqrt B\le416\sqrt2\cdot4^jr\sqrt B\le T/6$, since $Nr^2\ge K_0B$;
\item[(c)] $104\,\delta_j\sqrt{x+j}\le208\,r\,2^j(\sqrt x+\sqrt j)\le208\cdot4^jr(\sqrt x+1)\le T/6$. This needs
$Nr^2\ge1248^2(\sqrt x+1)^2$, which follows from $Nr^2\ge2\cdot1248^2(x+1)$ (as $(\sqrt x+1)^2\le2(x+1)$) and
$2\cdot1248^2=3{,}115{,}008\le K_0$;
\item[(d)] $156\,\delta_j\le312\cdot4^jr\le T/6$, since $Nr^2\ge K_0\ge1872^2=3{,}504{,}384$;
\item[(e)] $\frac{22}{\sqrt N}(A/\delta_j+B+x+j+1)\le\frac{22}{\sqrt N}(A/r+B+4^j(x+1))$, and each of the three pieces
is at most $T/18$, since $Nr^3\ge396A$, $Nr^2\ge396B$ (and $4^j\ge1$) and $Nr^2\ge396(x+1)$, all implied by
$K_0\ge396$.
\end{itemize}
The total is at most $4T/6+3T/18=5T/6<T$, a contradiction. So $\hat D\le r$ on $\Omega'_x$.

\textbf{Step 4$'$ (from high probability to expectation).}
Define $W$ from the events $\Omega'_k$ of Step 3$'$ exactly as in Step 6 of the proof of \cref{lem:erm1}. Then
$W$ is measurable and $\hat D^2\le W$ pointwise, by Step 3$'$ and $\hat D\le2$. For
$t\in[r(k)^2,r(k+1)^2)$, the event $\set{W>t}$ is contained in $(\Omega'_k)^c$, so $\P(W>t)\le4e^{-k}$; and
$r(k)^2$ increases to infinity. Hence
$\E^*\hat D^2\le\E W\le r(0)^2+\sum_{k\ge0}(r(k+1)^2-r(k)^2)\,4e^{-k}=r(0)^2+\frac{4e}{e-1}\cdot\frac{K_0}N$. Since
$1/N\le(B+1)/N$,
\[
  \E^*\hat D^2\le K_0\Bigl(1+\frac{4e}{e-1}\Bigr)\Bigl[\Bigl(\frac AN\Bigr)^{2/3}+\frac{B+1}N\Bigr]
  \le9.14\cdot10^7\Bigl[\Bigl(\frac AN\Bigr)^{2/3}+\frac{B+1}N\Bigr].
\]
So the proposition holds with $C=10^8$, which does not depend on $A$, $B$, $N$, $\mu$ or the $\eta_i$.
\end{proof}

\subsection{The pooled stage and the local stage}
\label{app:rate-stages}

\begin{lemma}[Pooled isotonic regression tracks the average map]\label{lem:pooled}
Let $\mu$ be a Borel probability on $[0,1]$ with no atoms. Let $I$ be a finite index set with
$\abs I=q\ge1$, let $n\ge1$, and for $i\in I$ let $g_i:[0,1]\to[0,1]$ be nondecreasing. Observe $N=qn$
independent pairs: for each $i\in I$, $n$ i.i.d.\ pairs $(P_{ik},Y_{ik})$ with $P_{ik}\sim\mu$ and
$Y_{ik}\mid P_{ik}\sim\mathrm{Bern}(g_i(P_{ik}))$. Let $\bar g_I\defeq\frac1q\sum_{i\in I}g_i$. Let $\hat g_0$ be the
isotonic least-squares estimator fitted to all $N$ pairs over $\mathcal F_{\uparrow}$, that is, any $\hat g_0$,
jointly measurable in (data, $p$), such that $\hat g_0(\text{data},\cdot)\in\mathcal F_{\uparrow}$ and
$\hat g_0(P_{(l)})=v_l$ for every $l$. Here $v$ is the PAV solution at the sorted design points
$P_{(1)}<\dots<P_{(N)}$; on the null event of ties, $\hat g_0$ is any jointly measurable least-squares
selection over $\mathcal F_{\uparrow}$, for instance the one that pools tied design points into one
weighted pool-adjacent-violators block, a rule that does not use $\mu$. An example is the right-continuous step interpolation, equal to $v_1$
on $[0,P_{(1)})$ and to $v_N$ on $[P_{(N)},1]$. Then, for a universal constant $C$ (independent of $\mu$, $q$,
$n$ and the $g_i$),
\[
  \E\norm{\hat g_0-\bar g_I}^2_{L^2(\mu)}\;\le\;C\Bigl[N^{-2/3}+\frac{\log(eN)}N\Bigr].
\]
\end{lemma}

The heterogeneity of the $g_i$ does not enter the bound: the deviations $g_i-\bar g_I$ average out, although
each pair's conditional mean $\E[Y_{ik}\mid P_{ik}]=g_i(P_{ik})$ differs from $\bar g_I(P_{ik})$. In the
language of \cref{lem:erm1}, the regression functions differ across pairs but the design law is common.

\begin{proof}
\textbf{Step 1 (the hypotheses of \cref{lem:erm1}).}
The $N=qn$ pairs are independent, take values in $[0,1]\times\set{0,1}\subseteq[0,1]\times[-1,1]$, and every
$P_{ik}\sim\mu$. For a pair of question $i$, $\eta(p)\defeq\E[Y_{ik}\mid P_{ik}=p]=g_i(p)$ (a version). Since every
question has $n$ pairs, $\bar\eta=\bar g_I$. The class $\mathcal F_{\uparrow}$ is convex, and its set of
$\mu$-classes is closed in $L^2(\mu)$. Indeed, let $f_k\in\mathcal F_{\uparrow}$ with $f_k\to f$ in $L^2(\mu)$.
Helly's selection theorem (statement (H) in the proof of \cref{lem:local}, with $V=1$) gives a subsequence that
converges at every point to some $\tilde f\in\mathcal F_{\uparrow}$. A further subsequence converges
$\mu$-a.e.\ to $f$, so $f=\tilde f$ $\mu$-a.e.

\emph{Bracketing entropy.} By \citet[Thm.~2.7.5]{vandervaart1996weak}, the class of monotone functions
$\R\to[0,1]$ satisfies $\log N_{[\,]}(\eps,\cdot,L^r(Q))\le K_r/\eps$ for every probability $Q$ and every
$0<\eps<1$, where $K_r$ depends only on $r$. We apply it with $r=2$ and with $Q=\mu$ viewed as a law on $\R$
with $\mu(\R\setminus[0,1])=0$. The theorem holds for every $Q$, so gaps in $\supp\mu$ need no quantile
reduction. Extend each $f\in\mathcal F_{\uparrow}$ to $\R$ by $f(0)$ on $(-\infty,0)$ and by $f(1)$ on
$(1,\infty)$; the extension is still nondecreasing and $[0,1]$-valued. Restricting a pointwise bracket
$[l,u]$ on $\R$ to $[0,1]$ keeps pointwise inclusion on $[0,1]$ and the same $L^2(\mu)$-width. The width
convention of \citet{vandervaart1996weak} (`$<\eps$') is stricter than ours (`$\le\eps$'), so the count
transfers: $\log N_{[\,]}(\eps,\mathcal F_{\uparrow},L^2(\mu))\le K_{\uparrow}/\eps$ for $0<\eps<1$, with
$K_{\uparrow}\defeq K_2$. For $1\le\eps\le2$, the single bracket $[0,1]$ has $L^2(\mu)$-width $1\le\eps$, so
$\log N_{[\,]}=0\le K_{\uparrow}/\eps$. Hence \eqref{eq:erm1-entropy} holds with $A=K_{\uparrow}$, a universal
constant, and $B=0$.

\textbf{Step 2 (the projection).}
$\bar g_I$ is nondecreasing with values in $[0,1]$, as an average of such maps. So $\bar g_I\in\mathcal F_{\uparrow}$
is its own $L^2(\mu)$-projection onto $\mathcal F_{\uparrow}$: $f^*=\bar g_I$.

\textbf{Step 3 (the estimator is a least-squares estimator over $\mathcal F_{\uparrow}$).}
By definition $\hat g_0\in\mathcal F_{\uparrow}$. Since $\mu$ is non-atomic, the $P_{ik}$ are a.s.\ distinct.
The PAV solution is the unique minimizer of $\sum_l(Y_l-v_l)^2$ over the vectors $v$ that are nondecreasing in
the order of the sorted design points. Its values are averages of consecutive responses, so they lie in $[0,1]$.
Let $f\in\mathcal F_{\uparrow}$ satisfy $f(P_{(l)})=v_l$ for all $l$; such $f$ exist, for example the step
interpolation. Then $f$ minimizes $\sum_l(Y_l-f(P_l))^2$ over $\mathcal F_{\uparrow}$, because for every
$f\in\mathcal F_{\uparrow}$ the vector $(f(P_{(l)}))_l$ is nondecreasing in the sorted order. On the null event
that two design points coincide, replace $\hat g_0$ by a jointly measurable weighted-PAV least-squares
selection. The modified estimator is a least-squares selection everywhere, so \cref{lem:erm1} applies to it,
and it equals $\hat g_0$ a.s., so the expectations agree. Thus $\hat g_0$ is a least-squares estimator over
$\mathcal F_{\uparrow}$. The right-continuous step interpolation is a measurable function of the data jointly
with the evaluation point.

\textbf{Step 4 (conclusion).}
With $C_E$ the constant of \cref{lem:erm1},
$\E\norm{\hat g_0-\bar g_I}^2\le C_E\bigl[(K_{\uparrow}/N)^{2/3}+\log(eN)/N\bigr]\le C\bigl[N^{-2/3}+\log(eN)/N\bigr]$,
where $C\defeq C_E\max\set{K_{\uparrow}^{2/3},1}$.
\end{proof}

\begin{lemma}[Constrained least squares for a bounded-variation deviation]\label{lem:local}
Let $\mu$ be a Borel probability on $[0,1]$ with no atoms, and let $h:[0,1]\to[0,1]$ be a \emph{fixed} Borel
function. Let $(P_k,Y_k)$, $k=1,\dots,n$ ($n\ge1$), be i.i.d.\ with $P\sim\mu$, $Y\in[0,1]$ and
$\E[Y\mid P]=h(P)+f^*(P)$. For $\tau\in[0,1]$ and $V\in[0,2]$, let
$\hat d\in\argmin_{d\in\mathcal D(\tau,V)}\sum_k(Y_k-h(P_k)-d(P_k))^2$ be any measurable selection. Here $\E$ is
read as the outer expectation $\E^*$, which equals $\E$ for jointly measurable selections such as the step
interpolation of \cref{rem:local-selection}. Then, for a universal constant $C$ (independent of $\mu$, $h$,
$f^*$, $n$, $\tau$ and $V$),
\[
  \E\norm{\hat d-f^*}^2_{L^2(\mu)}\;\le\;C\Bigl[\inf_{d\in\mathcal D(\tau,V)}\norm{d-f^*}^2_{L^2(\mu)}
  +\min\Bigl\{\tau^2,\ \frac1n+\min\bigl\{V^2,(V/n)^{2/3}\bigr\}\Bigr\}\Bigr].
\]
\end{lemma}

Write $d$ as its $\mu$-mean plus a centred part. Over the local class, the mean ranges over an interval of
length at most $2\delta$, which costs only $B=2$ in the localized entropy of \cref{prop:erm2}. The centred part
is a bounded-variation class with entropy of order $V/\eps$ and no logarithmic term. \Cref{prop:erm2} then gives
the rate without a logarithm.

\begin{proof}
\emph{Conventions.} $\mu(g)\defeq\int g\,d\mu$ and $\mathcal D\defeq\mathcal D(\tau,V)$. $K_{\uparrow}$ is the
constant of \citet[Thm.~2.7.5]{vandervaart1996weak} at $r=2$, as in the proof of \cref{lem:pooled}, and $C_E$ is
the constant of \cref{prop:erm2}. $f^*$ is the deviation function in the lemma: it is measurable (the
$L^2(\mu)$ norm presupposes this) and $\E[Y\mid P]=h(P)+f^*(P)$ a.s. Brackets are pointwise.

\textbf{Step 0 ($\tau=0$).} Then $\mathcal D=\set0$, $\hat d=0$, and
$\E\norm{\hat d-f^*}^2=\norm{f^*}^2=\inf_{d\in\mathcal D}\norm{d-f^*}^2$, so the claim holds for every $C\ge1$.
From now on, $\tau\in(0,1]$.

\textbf{Step 1 (the setting of \cref{prop:erm2}).}
Put $N\defeq n$ and $\tilde Y_k\defeq Y_k-h(P_k)\in[-1,1]$; recall that $h$ is a fixed measurable $[0,1]$-valued
function and $Y_k\in[0,1]$. The pairs $(P_k,\tilde Y_k)$ are i.i.d., as measurable images of i.i.d.\ pairs. Hence
they are independent, with $P_k\sim\mu$. Also $\eta_k(p)\defeq\E[\tilde Y_k\mid P_k=p]=f^*(p)$ for $\mu$-a.e.\ $p$,
so we may take $\eta_k\defeq f^*$ for every $k$, and then $\bar\eta=f^*$. Since
$\sum_k(Y_k-h(P_k)-d(P_k))^2=\sum_k(\tilde Y_k-d(P_k))^2$, $\hat d$ is a measurable selection of the argmin in
\cref{prop:erm2} with $\mathcal F\defeq\mathcal D$. The class $\mathcal D$ is (i) nonempty, since $0\in\mathcal D$;
(ii) made of measurable functions with values in $[-\tau,\tau]\subseteq[-1,1]$; (iii) convex, because
$\sup\abs\cdot$ and $\mathrm{TV}$ are seminorms ($\mathrm{TV}$ is subadditive and absolutely homogeneous term by term in every
partition sum, hence after the supremum); and (iv) closed in $L^2(\mu)$, by Step 3.

\textbf{Step 2 (Jordan decomposition).}
Let $\mathrm{TV}(d)\le V$, and let $T(x)$ be the variation of $d$ on $[0,x]$, with $T(0)=0$. Let $0\le x<y\le1$. Take a
partition of $[0,x]$ whose last point is $x$; we may do so, since inserting $x$ never decreases the partition
sum, by the triangle inequality. Appending $y$ gives $T(y)\ge T(x)+\abs{d(y)-d(x)}$. Put
$v^+\defeq(T+d-d(0))/2$ and $v^-\defeq(T-d+d(0))/2$. Then:
\begin{itemize}
\item $v^\pm(0)=0$;
\item $v^\pm(y)-v^\pm(x)=\bigl(T(y)-T(x)\pm(d(y)-d(x))\bigr)/2\ge0$, so $v^\pm$ is nondecreasing, hence Borel;
\item $0\le v^\pm\le v^++v^-=T\le\mathrm{TV}(d)\le V$;
\item $d=d(0)+v^+-v^-$ everywhere.
\end{itemize}

\textbf{Step 3 (closedness).}
\emph{(H) Helly's selection theorem.} Every sequence of nondecreasing $\phi_k:[0,1]\to[0,V]$ has a subsequence
that converges at every point to a nondecreasing $[0,V]$-valued function.
\emph{Proof of (H).} A diagonal extraction gives convergence on $\Q_1\defeq\Q\cap[0,1]$, to a limit $\psi$. The
function $\Psi(x)\defeq\sup\set{\psi(q):q\in\Q_1,\ q\le x}$ is nondecreasing, so its set $E$ of discontinuities is
countable. Let $x\in(0,1)\setminus E$, and let $q<x<q'$ be rationals. From $\phi_k(q)\le\phi_k(x)\le\phi_k(q')$ we
get $\Psi(x-)\le\liminf_k\phi_k(x)\le\limsup_k\phi_k(x)\le\Psi(x+)$, and the two ends are equal off $E$ (the points
$0$ and $1$ lie in $\Q_1$). A second diagonal extraction over the countable set $E$ gives convergence
everywhere. A pointwise limit of nondecreasing $[0,V]$-valued functions is nondecreasing and $[0,V]$-valued.

Now let $d_k\in\mathcal D$ with $d_k\to d$ in $L^2(\mu)$. By Step 2, $d_k=d_k(0)+v_k^+-v_k^-$. Apply (H) to
$(v_k^+)$, then to $(v_k^-)$ along that subsequence, then Bolzano--Weierstrass to $d_k(0)\in[-\tau,\tau]$. Along
the resulting subsequence, $v_k^\pm\to v^\pm$ at every point and $d_k(0)\to c_0$, so
$d_k\to\tilde d\defeq c_0+v^+-v^-$ at every point. The limit $\tilde d$ is Borel and
$\abs{\tilde d}\le\tau$. For every partition,
$\sum_l\abs{\tilde d(x_l)-\tilde d(x_{l-1})}=\lim_k\sum_l\abs{d_k(x_l)-d_k(x_{l-1})}\le V$, so $\tilde d\in\mathcal D$.
Since $\abs{d_k-\tilde d}^2\le4\tau^2$, dominated convergence gives $d_k\to\tilde d$ in $L^2(\mu)$ along the
subsequence, so $d=\tilde d$ $\mu$-a.e. Hence the set of $\mu$-classes of $\mathcal D$ is closed, which is the
reading used in \cref{lem:erm1,prop:erm2}.

\textbf{Step 4 (projection and local entropy).}
By Steps 1--3 and the Hilbert projection theorem (as in Step 1(ii) of the proof of \cref{lem:erm1}), the
projection of $\bar\eta=f^*$ onto $\mathcal D$ exists. Fix a representative $f^*_{\mathcal D}\in\mathcal D$; it
plays the role of $f^*$ in \cref{prop:erm2}. Then
$\norm{f^*_{\mathcal D}-f^*}^2=\inf_{d\in\mathcal D}\norm{d-f^*}^2$. Put $m^*\defeq\mu(f^*_{\mathcal D})$ and
$\mathcal D(\delta)\defeq\set{d\in\mathcal D:\norm{d-f^*_{\mathcal D}}\le\delta}$, the class $\mathcal F(\delta)$ of
\cref{prop:erm2}. We claim that, for every $\delta>0$ and $0<\eps\le\delta$,
\begin{equation}\label{eq:local-entropy}
  \log N_{[\,]}\bigl(\eps,\mathcal D(\delta),L^2(\mu)\bigr)\;\le\;\frac{16K_{\uparrow}V}\eps+2\log\frac{e\delta}\eps.
\end{equation}
\emph{(a) Localizing the mean.} For $d\in\mathcal D(\delta)$, the mean $m_d\defeq\mu(d)$ satisfies
$\abs{m_d-m^*}=\abs{\mu(d-f^*_{\mathcal D})}\le\norm{d-f^*_{\mathcal D}}\le\delta$, by Cauchy--Schwarz ($\mu$ is a
probability). Let $G\defeq\set{m^*-\delta+(2i-1)\eps/4:\ i=1,\dots,\ceil{4\delta/\eps}}$. The intervals of radius
$\eps/4$ around the points of $G$ cover $[m^*-\delta,m^*+\delta]$. So some $c\in G$ has $\abs{m_d-c}\le\eps/4$,
and $\abs G\le1+4\delta/\eps$.

\emph{(b) Brackets for $V\cdot\mathcal F_{\uparrow}$.} We transport \citet[Thm.~2.7.5]{vandervaart1996weak}
($r=2$) as in Step 1 of the proof of \cref{lem:pooled}. View $Q\defeq\mu$ as a law on $\R$, and extend each
$\phi\in\mathcal F_{\uparrow}$ to $\R$ by $\phi(0)$ on $(-\infty,0)$ and $\phi(1)$ on $(1,\infty)$. Restricting
brackets on $\R$ to $[0,1]$ keeps them pointwise, with the same $L^2(\mu)$-width, and the strict width
convention `$<\eta$' implies `$\le\eta$'. So, for $0<\eta<1$, $\mathcal F_{\uparrow}$ is covered by at most
$\exp(K_{\uparrow}/\eta)$ pointwise $\eta$-brackets. Drop the brackets that contain no element of
$\mathcal F_{\uparrow}$, and replace each remaining $[l,u]$ by $[\max(l,0),\min(u,1)]$. This keeps the same
elements of $\mathcal F_{\uparrow}$, does not increase the width, and gives $0\le\max(l,0)\le\min(u,1)\le1$. The
family $\mathcal B_V$ of $\eps/8$-brackets for $V\cdot\mathcal F_{\uparrow}$ (which contains every $v^\pm$ of Step 2)
is defined as follows:
\begin{itemize}
\item if $V=0$, $\mathcal B_V\defeq\set{[0,0]}$;
\item if $V>0$ and $\eta\defeq\eps/(8V)<1$, $\mathcal B_V\defeq\set{[Vl,Vu]}$, from the truncated family at $\eta$
(width $V\eta=\eps/8$);
\item if $V>0$ and $\eta\ge1$, $\mathcal B_V\defeq\set{[0,V]}$ (width $V\le\eps/8$).
\end{itemize}
In all cases $\abs{\mathcal B_V}\le\exp(8K_{\uparrow}V/\eps)$, and every bracket satisfies $0\le l\le u\le V$, so
$\mu(l)$ and $\mu(u)$ are finite.

\emph{(c) Centred brackets.} Let $d\in\mathcal D(\delta)$ and write $d=d(0)+v^+-v^-$ (Step 2). Pick
$[l^\pm,u^\pm]\in\mathcal B_V$ containing $v^\pm$, and $c\in G$ with $\abs{m_d-c}\le\eps/4$. Since
$\mu(d)=d(0)+\mu(v^+)-\mu(v^-)$, we have $d=m_d+(v^+-\mu(v^+))-(v^--\mu(v^-))$ everywhere. Integrating the pointwise
brackets gives $\mu(l^\pm)\le\mu(v^\pm)\le\mu(u^\pm)$. Hence $L\le d\le U$ pointwise, where
\[
  L\defeq c-\tfrac\eps4+l^+-\mu(u^+)-u^-+\mu(l^-),\qquad U\defeq c+\tfrac\eps4+u^+-\mu(l^+)-l^-+\mu(u^-).
\]
The width is $U-L=\frac\eps2+(u^+-l^+)+\mu(u^+-l^+)+(u^--l^-)+\mu(u^--l^-)\ge0$. By Minkowski's inequality and
$\mu(g)\le\norm g$ for $g\ge0$, $\norm{U-L}\le\frac\eps2+2\cdot\frac\eps8+2\cdot\frac\eps8=\eps$. The bracket $[L,U]$ is
measurable and is determined by $(c,[l^+,u^+],[l^-,u^-])$; brackets need not lie in $\mathcal D$. The centring
constant is bracketed by integrating the $v$-brackets, so it adds no logarithmic term.

\emph{(d) Count.} $N_{[\,]}(\eps,\mathcal D(\delta),L^2(\mu))\le\abs G\,\abs{\mathcal B_V}^2\le(1+4\delta/\eps)\exp(16K_{\uparrow}V/\eps)$.
With $x\defeq\delta/\eps\ge1$, $\log(1+4x)\le\log5+\log x\le2+2\log x=2\log(ex)$, since $\log5<2$ and $\log x\ge0$.
This proves \eqref{eq:local-entropy}. In particular, the hypothesis of \cref{prop:erm2} holds for all
$0<\delta\le2$ and $0<\eps\le\delta$, with $A\defeq16K_{\uparrow}V$ and $B\defeq2$.

\textbf{Step 5 (\cref{prop:erm2}).}
All hypotheses of \cref{prop:erm2} hold (Steps 1, 3 and 4), with $\mathcal F\defeq\mathcal D$,
$f^*\defeq f^*_{\mathcal D}$ and $\hat f\defeq\hat d$. \Cref{prop:erm2} gives
\[
  \E\norm{\hat d-f^*_{\mathcal D}}^2\le C_E\Bigl[\Bigl(\frac{16K_{\uparrow}V}n\Bigr)^{2/3}+\frac3n\Bigr]
  \le C_1\Bigl[\Bigl(\frac Vn\Bigr)^{2/3}+\frac1n\Bigr],\qquad C_1\defeq C_E\max\set{(16K_{\uparrow})^{2/3},3},
\]
with $\E$ read as in \cref{prop:erm2} ($\E=\E^*$ whenever the norm is measurable).

\textbf{Step 6 (cap and $V^2$ refinement).}
For every $\omega$, both $\hat d(\omega)$ and $f^*_{\mathcal D}$ lie in $\mathcal D$, so
$\norm{\hat d-f^*_{\mathcal D}}^2\le(2\tau)^2=4\tau^2$. Moreover
$(V/n)^{2/3}+1/n\le2\bigl[1/n+\min\set{V^2,(V/n)^{2/3}}\bigr]$. If $V\le n^{-1/2}$, then $(V/n)^{2/3}\le1/n$ and
the left side is at most $2/n$. If $V>n^{-1/2}$, then $V^4>n^{-2}$, so $V^6>V^2n^{-2}$, that is,
$V^2>(V/n)^{2/3}$; the minimum equals $(V/n)^{2/3}$ and the left side equals $1/n$ plus the minimum. With
$X\defeq1/n+\min\set{V^2,(V/n)^{2/3}}$,
\[
  \E\norm{\hat d-f^*_{\mathcal D}}^2\le\min\set{4\tau^2,2C_1X}\le C_2\min\set{\tau^2,X},\qquad C_2\defeq\max\set{4,2C_1}.
\]

\textbf{Step 7 (back to $f^*$).}
Pointwise in $\omega$, $\norm{\hat d-f^*}^2\le2\norm{\hat d-f^*_{\mathcal D}}^2+2\norm{f^*_{\mathcal D}-f^*}^2$, and
$\norm{f^*_{\mathcal D}-f^*}^2=\inf_{d\in\mathcal D}\norm{d-f^*}^2$. Taking (outer) expectations, with
$C\defeq2C_2\ge1$ (so that Step 0 is covered),
$\E\norm{\hat d-f^*}^2\le C\bigl[\inf_{d\in\mathcal D}\norm{d-f^*}^2+\min\set{\tau^2,X}\bigr]$. The constant $C$
depends only on $C_E$ and $K_{\uparrow}$, so it is universal.
\end{proof}

\begin{remark}[A measurable least-squares selection exists]\label{rem:local-selection}
This remark is not used in the proof above; it shows that \cref{lem:local} is not vacuous. For fixed data, the
criterion depends on $d$ only through $(d(P_k))_k$. Let $p_{(1)}<\dots<p_{(r)}$ be the distinct design values,
and let $w_i$ be the common value at $p_{(i)}$. The achievable set is
$S=\set{v:\ v_k=v_l\text{ when }P_k=P_l,\ \abs{v_k}\le\tau,\ \sum_i\abs{w_{i+1}-w_i}\le V}$. Indeed, the partition
$p_{(1)}<\dots<p_{(r)}$ is admissible in $\mathrm{TV}$. Conversely, the step function equal to $w_1$ on $[0,p_{(2)})$, to
$w_i$ on $[p_{(i)},p_{(i+1)})$ and to $w_r$ on $[p_{(r)},1]$ has variation $\sum_i\abs{w_{i+1}-w_i}$. The set $S$ is
nonempty, compact and convex. It depends on the data only through the weak order pattern of the $P_k$, which
takes finitely many values on Borel cells. On each cell, the unique minimizer is the Euclidean projection of
$\tilde Y$ onto a fixed closed convex set, which is $1$-Lipschitz and Borel, and its step interpolation is
jointly Borel in $(\omega,x)$. By Tonelli, $\norm{\hat d-f^*}$ is then measurable.
\end{remark}

\subsection{The oracle estimator, the compound hold-out, and the rate}
\label{app:rate-adaptive}

\begin{proposition}[Oracle-tuned two-stage estimator]\label{prop:oracle}
Consider Model~M with $\mu$ non-atomic, $m\ge2$, $n\ge1$, and $(g_1,\dots,g_m)\in\Theta(\tau,V)$ with
$\tau\in[0,1]$ and $V\in[0,2]$. The following estimator knows $\tau$ and $V$. For each $j$:
\begin{enumerate}
\item[(1)] $\hat g_0^{(-j)}$ is the isotonic least-squares fit over $\mathcal F_{\uparrow}$ to the pooled pairs of
all questions $i\ne j$, pinned as in \cref{lem:pooled} (a jointly measurable $\mathcal F_{\uparrow}$-valued
interpolation of the PAV values);
\item[(2)] $\hat d_j$ is a non-randomized selection, jointly measurable in (data, $p$) and not using $\mu$, of the
constrained least-squares estimator over $\mathcal D(\tau_2,V_2)$ of the residuals
$Y_{jk}-\hat g_0^{(-j)}(P_{jk})$, $k=1,\dots,n$, as in \cref{lem:local}, where $\tau_2\defeq\min\set{2\tau,1}$ and
$V_2\defeq\min\set{2V,2}$ (for example, the right-continuous step interpolation of the unique fitted-value
vector at the sorted design points);
\item[(3)] $\hat g_j\defeq\mathrm{clip}_{[0,1]}\bigl(\hat g_0^{(-j)}+\hat d_j\bigr)$.
\end{enumerate}
Then, for a universal constant $C$ (independent of $\mu$, $m$, $n$, $\tau$ and $V$),
\[
  \frac1m\sum_{j=1}^m\E\norm{\hat g_j-g_j}^2_{L^2(\mu)}\;\le\;C\Bigl[(mn)^{-2/3}+\min\set{\tau^2,1/n}
  +\min\set{\tau^2,V^2,(V/n)^{2/3}}\Bigr].
\]
\end{proposition}

The pooled stage estimates the leave-one-out average map at the rate $(mn)^{-2/3}$ (\cref{lem:pooled}). The
local stage estimates the deviation of question $j$ from that pilot at the rate of a single question
(\cref{lem:local}), and the pilot's error enters only through the approximation term of \cref{lem:local}.

\begin{proof}
Fix $j$. Let $I\defeq\set{1,\dots,m}\setminus\set j$, so that $\abs I=m-1\ge1$. Let
$\bar g^{(-j)}\defeq\frac1{m-1}\sum_{i\in I}g_i$ and $h\defeq\hat g_0^{(-j)}$. Let $C_{\mathrm{pool}}$ and
$C_{\mathrm{loc}}$ be the constants of \cref{lem:pooled,lem:local}.

\textbf{Step 1 (independence).}
$h$ is a measurable function of the pairs of the questions in $I$, and these are independent of the pairs of
question $j$. Moreover $h$ is nondecreasing with values in $[0,1]$: it is an isotonic fit to $[0,1]$-valued
responses, interpolated monotonically. We condition on the $\sigma$-field of the pairs of the questions in $I$
(the ``$I$-data''). By independence and the freezing (Fubini) argument, the conditional law of question $j$'s
pairs is unchanged, and for each fixed value of the $I$-data, $\hat d_j$ is a selection as in \cref{lem:local}.
Then $h$ is a fixed measurable $[0,1]$-valued function (by the pinning of \cref{lem:pooled}), and the pairs
$(P_{jk},Y_{jk})$, $k=1,\dots,n$, are still i.i.d.\ with $P\sim\mu$. In the notation of \cref{lem:local}, the
objects are $(Y,h,f^*)\defeq(Y_{jk},h,g_j-h)$, with $Y\in[0,1]$ and $\E[Y_{jk}\mid P_{jk}]=h(P_{jk})+f^*(P_{jk})$.

\textbf{Step 2 (the class contains the centred deviation).}
Let $g_0$ witness $(g_1,\dots,g_m)\in\Theta(\tau,V)$. Then
$g_j-\bar g^{(-j)}=(g_j-g_0)-\frac1{m-1}\sum_{i\in I}(g_i-g_0)$. Hence $\norm{g_j-\bar g^{(-j)}}_\infty\le2\tau$, and,
since the variation seminorm is subadditive and absolutely homogeneous, $\mathrm{TV}(g_j-\bar g^{(-j)})\le V+V=2V$.
Also $\abs{g_j-\bar g^{(-j)}}\le1$ and $\mathrm{TV}(g_j-\bar g^{(-j)})\le\mathrm{TV}(g_j)+\mathrm{TV}(\bar g^{(-j)})\le2$, because both
maps are nondecreasing with values in $[0,1]$. Hence $g_j-\bar g^{(-j)}\in\mathcal D(\tau_2,V_2)$.

\textbf{Step 3 (\cref{lem:local}).}
Apply \cref{lem:local} conditionally on the $I$-data, with this $h$, with $f^*\defeq g_j-h$, and with
$(\tau_2,V_2)\in[0,1]\times[0,2]$:
\begin{align*}
  \E\bigl[\norm{\hat d_j-(g_j-h)}^2\bigm|I\text{-data}\bigr]\le C_{\mathrm{loc}}\Bigl[&\inf_{d\in\mathcal D(\tau_2,V_2)}\norm{d-(g_j-h)}^2\\
  &+\min\bigl\{\tau_2^2,\ 1/n+\min\set{V_2^2,(V_2/n)^{2/3}}\bigr\}\Bigr].
\end{align*}
By Step 2, the infimum is at most $\norm{(g_j-\bar g^{(-j)})-(g_j-h)}^2=\norm{h-\bar g^{(-j)}}^2$. Since
$\tau_2\le2\tau$ and $V_2\le2V$,
\begin{align*}
  \min\bigl\{\tau_2^2,\ 1/n+\min\set{V_2^2,(V_2/n)^{2/3}}\bigr\}
  &\le4\min\bigl\{\tau^2,\ 1/n+\min\set{V^2,(V/n)^{2/3}}\bigr\}\\
  &\le4\Bigl[\min\set{\tau^2,1/n}+\min\set{\tau^2,V^2,(V/n)^{2/3}}\Bigr],
\end{align*}
using $\min\set{a,x+y}\le\min\set{a,x}+\min\set{a,y}$ for $a,x,y\ge0$ and $\min\set{a,\min\set{b,c}}=\min\set{a,b,c}$.

\textbf{Step 4 (clipping).}
Since $g_j$ takes values in $[0,1]$, and $\mathrm{clip}_{[0,1]}$ is $1$-Lipschitz and fixes $[0,1]$, pointwise
$\abs{\hat g_j-g_j}=\abs{\mathrm{clip}(h+\hat d_j)-\mathrm{clip}(g_j)}\le\abs{h+\hat d_j-g_j}=\abs{\hat d_j-(g_j-h)}$. Hence
\[
  \E\bigl[\norm{\hat g_j-g_j}^2\bigm|I\text{-data}\bigr]\le C_{\mathrm{loc}}\norm{h-\bar g^{(-j)}}^2
  +4C_{\mathrm{loc}}\Bigl[\min\set{\tau^2,1/n}+\min\set{\tau^2,V^2,(V/n)^{2/3}}\Bigr].
\]

\textbf{Step 5 (\cref{lem:pooled}).}
Take expectations over the pairs of the questions in $I$. \Cref{lem:pooled} applies with index set $I$
($q=m-1\ge1$), the same $n$, and the nondecreasing maps $(g_i)_{i\in I}$. It gives
$\E\norm{h-\bar g^{(-j)}}^2\le C_{\mathrm{pool}}[N^{-2/3}+\log(eN)/N]$ with $N\defeq(m-1)n\ge mn/2$ (as $m\ge2$).
Since $\log(eN)/N\le C_3N^{-2/3}$ for all $N\ge1$, where $C_3\defeq\sup_{N\ge1}N^{-1/3}\log(eN)<\infty$, and since
$N^{-2/3}\le2^{2/3}(mn)^{-2/3}$, we get $\E\norm{h-\bar g^{(-j)}}^2\le C_4(mn)^{-2/3}$ with
$C_4\defeq2^{2/3}C_{\mathrm{pool}}(1+C_3)$.

\textbf{Step 6 (conclusion).}
Combining Steps 4 and 5 and averaging over $j=1,\dots,m$ gives the claim with
$C\defeq\max\set{C_{\mathrm{loc}}C_4,4C_{\mathrm{loc}}}$.
\end{proof}

\begin{proposition}[Adaptivity by compound hold-out]\label{prop:holdout}
Consider Model~M with $\mu$ non-atomic and $(g_j)\in\Theta(\tau,V)$, where $\tau\in[0,1]$ and $V\in[0,2]$ are
\emph{unknown}, $m\ge3$ and $n\ge2$. Split the pairs of each question into halves: $A_j$, the first
$n_A\defeq\floor{n/2}$ pairs, and $B_j$, the remaining $n_B\defeq n-n_A$ pairs. Let $K\defeq\ceil{\log_2(mn)}$ and
\[
  \Lambda\defeq\set{(2^{-a},2^{1-b}):\ a,b\in\set{0,1,\dots,K}}.
\]
For each $\lambda=(\tau',V')\in\Lambda$, form the two-stage estimator of \cref{prop:oracle} with parameters
$(\tau',V')$ from the $A$-halves only (its pooled stage uses the $A$-halves of the other questions), and call its
maps $\hat g_j^\lambda$. Select one \emph{global}
\[
  \hat\lambda\in\argmin_{\lambda\in\Lambda}\sum_{j=1}^m\sum_{k\in B_j}\bigl(Y_{jk}-\hat g_j^\lambda(P_{jk})\bigr)^2,
\]
with ties broken by a fixed order, and output $\hat g_j\defeq\hat g_j^{\hat\lambda}$. The procedure is
non-randomized, jointly measurable in (data, $p$), and uses neither $(\tau,V)$ nor $\mu$, the selections of
\cref{prop:oracle} being the $\mu$-free step interpolations of \cref{lem:pooled,prop:oracle}. Then, for a
universal constant $C$ (independent of $\mu$, $m$, $n$, $\tau$ and $V$),
\[
  \frac1m\sum_{j=1}^m\E\norm{\hat g_j-g_j}^2_{L^2(\mu)}\;\le\;C\Bigl[(mn)^{-2/3}+\min\set{\tau^2,1/n}
  +\min\set{\tau^2,V^2,(V/n)^{2/3}}\Bigr],
\]
that is, the rate of \cref{prop:oracle} without knowledge of $(\tau,V)$.
\end{proposition}

\begin{proof}
\emph{Notation.} Let $r_j(\lambda)\defeq\norm{\hat g_j^\lambda-g_j}^2_{L^2(\mu)}$ and
$\bar R(\lambda)\defeq\frac1m\sum_jr_j(\lambda)$, and let $\mathcal A$ be the $\sigma$-field generated by all
$A$-halves. The estimators of \cref{prop:oracle} are non-randomized ($\hat g_0^{(-j)}$ is a jointly measurable
interpolation of the PAV values, and $\hat d_j$ is a non-randomized jointly measurable selection). So every
candidate is $\mathcal A$-measurable. Given $\mathcal A$, every $\hat g_j^\lambda$ is a fixed $[0,1]$-valued
function (\cref{prop:oracle} clips to $[0,1]$ in its step (3)), jointly measurable in ($A$-data, $p$). The
$B$-pairs are independent of $\mathcal A$ and of each other, and $(P_{jk},Y_{jk})$ is distributed as
$P\sim\mu$, $Y\mid P\sim\mathrm{Bern}(g_j(P))$. Every $\hat\lambda$-indexed quantity is measurable: $\hat\lambda$ is
the argmin, with a fixed tie-break, of finitely many measurable criteria, and
$r_j(\hat\lambda)=\sum_\lambda\mathbf 1\set{\hat\lambda=\lambda}r_j(\lambda)$. The grid $\Lambda$ and the selection
rule use neither $(\tau,V)$ nor $\mu$, so the output is non-randomized, jointly measurable, and free of
$(\tau,V,\mu)$.

\textbf{Step 1 (selection inequality).}
For $\lambda\in\Lambda$, $j$ and $k\in B_j$, put
\[
  Z_{jk}(\lambda)\defeq\bigl(Y_{jk}-\hat g_j^\lambda(P_{jk})\bigr)^2-\bigl(Y_{jk}-g_j(P_{jk})\bigr)^2
  =(g_j-\hat g_j^\lambda)(P_{jk})\,\bigl(2Y_{jk}-\hat g_j^\lambda(P_{jk})-g_j(P_{jk})\bigr).
\]
Given $\mathcal A$, $\E[Z_{jk}(\lambda)\mid\mathcal A]=r_j(\lambda)$: condition on $P_{jk}$ and use
$\E[Y\mid P]=g_j(P)$, so that $(g-\hat g)(2g-\hat g-g)=(g-\hat g)^2$. Also $Z_{jk}(\lambda)\in[-1,1]$ (both squares lie
in $[0,1]$), so $\abs{Z-\E Z}\le2$; and $\E[Z_{jk}(\lambda)^2\mid\mathcal A]\le4r_j(\lambda)$, since
$\abs{2Y-\hat g-g}\le2$. Let $S(\lambda)\defeq\sum_j\sum_{k\in B_j}Z_{jk}(\lambda)$ and
$M(\lambda)\defeq\E[S(\lambda)\mid\mathcal A]=n_Bm\bar R(\lambda)$. The conditional variance of $S(\lambda)$ is at
most $4M(\lambda)$. We apply Bernstein's inequality conditionally on $\mathcal A$; the regular conditional law of
the $B$-pairs is the product law. For independent summands with $\abs{X_i-\E X_i}\le2$ and variance sum at most
$v$, the two-sided form \citep[Thm.~2.8.4]{vershynin2018high}, inverted as in (B1), gives for $u>0$
$\P(\abs{S(\lambda)-M(\lambda)}\ge\sqrt{2vu}+\frac43u\mid\mathcal A)\le2e^{-u}$. With $v=4M(\lambda)$,
\[
  \P\Bigl(\abs{S(\lambda)-M(\lambda)}\ge\sqrt{8M(\lambda)u}+\tfrac43u\Bigm|\mathcal A\Bigr)\le2e^{-u}.
\]
Let $L\defeq\abs\Lambda=(K+1)^2$, and for $\delta\in(0,1)$ let $u\defeq\log(2L/\delta)$. Each of the $L$
inequalities fails with conditional probability at most $\delta/L$. So on an event $E_\delta$ of conditional
probability at least $1-\delta$ all of them hold, in particular at $\hat\lambda$. The selection rule minimizes
$S$, because the subtracted square does not depend on $\lambda$. Let $\lambda^*$ be an $\mathcal A$-measurable
minimizer of $M$ over $\Lambda$. On $E_\delta$,
\begin{align*}
  M(\hat\lambda)&\le S(\hat\lambda)+\sqrt{8M(\hat\lambda)u}+\tfrac43u
  \le S(\lambda^*)+\sqrt{8M(\hat\lambda)u}+\tfrac43u\\
  &\le M(\lambda^*)+\sqrt{8M(\lambda^*)u}+\sqrt{8M(\hat\lambda)u}+\tfrac83u.
\end{align*}
By $2\sqrt{xy}\le x+y$ with $x=M/4$ and $y=8u$, $\sqrt{8Mu}\le M/4+8u$. So
$\frac34M(\hat\lambda)\le\frac54M(\lambda^*)+\frac{56}3u$, hence
$M(\hat\lambda)\le\frac53M(\lambda^*)+\frac{224}9u\le\frac53M(\lambda^*)+25u$, that is,
$\bar R(\hat\lambda)\le\frac53\bar R(\lambda^*)+25\log(2L/\delta)/(n_Bm)$ on $E_\delta$. Given $t>0$, choose
$\delta\defeq e^{-tn_Bm/25}\in(0,1)$. With $a\defeq\frac53\min_\lambda\bar R(\lambda)+25\log(2L)/(n_Bm)$, this gives
$\P(\bar R(\hat\lambda)>a+t\mid\mathcal A)\le e^{-tn_Bm/25}$. Since
$\E[X\mid\mathcal A]\le a+\int_0^\infty\P(X>a+t\mid\mathcal A)\,dt$ for $X\ge0$,
\[
  \E[\bar R(\hat\lambda)\mid\mathcal A]\le\tfrac53\min_\lambda\bar R(\lambda)+\frac{25(\log(2L)+1)}{n_Bm}.
\]
Taking expectations and using $\E\min\le\min\E$,
\begin{equation}\label{eq:holdout-select}
  \E\bar R(\hat\lambda)\le\tfrac53\min_{\lambda\in\Lambda}\E\bar R(\lambda)+\frac{25(\log(2L)+1)}{n_Bm}.
\end{equation}

\textbf{Step 2 (a good grid point).}
Put $\bar\tau\defeq\max\set{\tau,2^{-K}}\in[2^{-K},1]$ and $\bar V\defeq\max\set{V,2^{1-K}}\in[2^{1-K},2]$. Let
$a^*\defeq\floor{\log_2(1/\bar\tau)}$ and $b^*\defeq\floor{\log_2(2/\bar V)}$. Both lie in $\set{0,\dots,K}$, since
$1/\bar\tau$ and $2/\bar V$ lie in $[1,2^K]$; this uses $\tau\le1$ and $V\le2$. Then $\tau'\defeq2^{-a^*}$
satisfies $\bar\tau\le\tau'<2\bar\tau$, and $V'\defeq2^{1-b^*}$ satisfies $\bar V\le V'<2\bar V$. So
$\lambda_{\mathrm{good}}\defeq(\tau',V')$ lies in $\Lambda$, with $\tau'\in(0,1]$ and $V'\in(0,2]$. Since
$\tau'\ge\tau$ and $V'\ge V$, any witness of $(g_j)\in\Theta(\tau,V)$ also witnesses $\Theta(\tau',V')$. Also
$2^{-K}\le1/(mn)$. The hypotheses of \cref{prop:oracle} hold for the $A$-half instance at $(\tau',V')$:
\begin{itemize}
\item Model~M holds: the $A$-pairs are a subfamily of the mutually independent pairs, with $P\sim\mu$ and
$Y\mid P\sim\mathrm{Bern}(g_j(P))$;
\item $\mu$ is the same non-atomic law;
\item $m\ge2$ (we assumed $m\ge3$) and $n_A=\floor{n/2}\ge1$ (as $n\ge2$);
\item $\tau'\in[0,1]$, $V'\in[0,2]$ and $(g_j)\in\Theta(\tau',V')$.
\end{itemize}
By construction, the candidate at $\lambda_{\mathrm{good}}$ is the estimator of \cref{prop:oracle} with
parameters $(\tau',V')$ on that instance (its pooled stage uses the other questions' $A$-halves). So, with the
universal constant $C_U$ of \cref{prop:oracle},
\begin{equation}\label{eq:holdout-good}
  \E\bar R(\lambda_{\mathrm{good}})\le C_U\Bigl[(mn_A)^{-2/3}+\min\set{\tau'^2,1/n_A}+\min\set{\tau'^2,V'^2,(V'/n_A)^{2/3}}\Bigr].
\end{equation}

\textbf{Step 3 (back to the target rate).}
For $n\ge2$ we have $n_A\ge n/3$ and $n_B\ge n/2$: if $n=2$, then $n_A=1\ge2/3$; if $n\ge3$, then
$n_A\ge(n-1)/2\ge n/3$; and $n_B=\ceil{n/2}$. Hence:
\begin{itemize}
\item[(a)] $(mn_A)^{-2/3}\le3^{2/3}(mn)^{-2/3}$ and $1/n_A\le3/n$.
\item[(b)] \emph{Shift term.} $\tau'^2<4\bar\tau^2\le4(\tau^2+2^{-2K})\le4\tau^2+4(mn)^{-2}$. Using
$\min\set{x+y,z}\le\min\set{x,z}+y$ for $y\ge0$, and $\min\set{4x,3z}\le4\min\set{x,z}$,
\begin{align*}
  \min\set{\tau'^2,1/n_A}&\le\min\set{4\tau^2+4(mn)^{-2},3/n}\le\min\set{4\tau^2,3/n}+4(mn)^{-2}\\
  &\le4\min\set{\tau^2,1/n}+4(mn)^{-2}.
\end{align*}
\item[(c)] \emph{Third term.} $\tau'^2<4\bar\tau^2$, $V'^2<4\bar V^2$, and
$(V'/n_A)^{2/3}<(6\bar V/n)^{2/3}\le4(\bar V/n)^{2/3}$ (as $6^{2/3}<3.31$). So the third term is at most
$4\min\set{\bar\tau^2,\bar V^2,(\bar V/n)^{2/3}}$. Moreover $\bar\tau^2\le\tau^2+(mn)^{-2}$,
$\bar V^2\le V^2+4(mn)^{-2}$, and
$(\bar V/n)^{2/3}\le(V/n)^{2/3}+(2^{1-K}/n)^{2/3}\le(V/n)^{2/3}+2^{2/3}(mn)^{-2/3}$, using
$(x+y)^{2/3}\le x^{2/3}+y^{2/3}$ and $2^{1-K}/n\le2/(mn)$. Since $\min_i(a_i+e_i)\le\min_ia_i+\max_ie_i$, and here
$\max_ie_i\le2^{2/3}(mn)^{-2/3}+4(mn)^{-2}$,
\[
  \min\set{\tau'^2,V'^2,(V'/n_A)^{2/3}}\le4\min\set{\tau^2,V^2,(V/n)^{2/3}}+4\bigl(2^{2/3}(mn)^{-2/3}+4(mn)^{-2}\bigr).
\]
\item[(d)] \emph{Selection term.} $n_Bm\ge mn/2$, and $L=(K+1)^2\le(\log_2(mn)+2)^2$ because
$K\le\log_2(mn)+1$. So $25(\log(2L)+1)/(n_Bm)\le50\,h(mn)\,(mn)^{-2/3}$, where
$h(x)\defeq x^{-1/3}\bigl(\log(2(\log_2x+2)^2)+1\bigr)$. With $s=\log_2x$,
$h=2^{-s/3}\bigl(1+\log2+2\log(s+2)\bigr)$. For $s\in[0,60]$, an interval evaluation on a grid of step
$5\cdot10^{-4}$ (on each cell, the decreasing factor is evaluated at the left end and the increasing factor at
the right end) gives $h\le3.14162$. This is a computer-assisted bound; any finite constant would do, since $h$
is continuous on $[1,\infty)$ and tends to $0$. For $s\ge60$, $2\log(s+2)\le s+2$, so $h\le2^{-s/3}(2s+6)$, which is
decreasing for $s\ge2$ and at most $1.3\cdot10^{-4}$ on $[60,\infty)$. Hence $\sup_{x\ge1}h(x)\le3.1417$, and the
selection term is at most $158\,(mn)^{-2/3}$. An elementary bound, with no computer assistance, also
suffices: the maximum of $(s+2)^22^{-s}$ over $s\ge0$ is $16/(e\log2)^2<4.51$, attained at
$s=2/\log2-2$, so $(s+2)^2\le9\cdot2^s$ and $2\log(s+2)\le\log9+s\log2$. Then
$h\le2^{-s/3}(1+\log18+s\log2)=e^{-u/3}(1+\log18+u)$ with $u\defeq s\log2\ge0$, which is decreasing in
$u$ because $1+\log18>3$; so $\sup h\le1+\log18<3.9$ and the selection term is at most
$195\,(mn)^{-2/3}$. Step~4 then holds with $195$ in place of $158$; any finite constant suffices.
\end{itemize}

\textbf{Step 4 (assembly).}
Use $\min_\lambda\E\bar R(\lambda)\le\E\bar R(\lambda_{\mathrm{good}})$ in \eqref{eq:holdout-select}, then
\eqref{eq:holdout-good}, (a)--(d) and $(mn)^{-2}\le(mn)^{-2/3}$:
\begin{align*}
  \E\bar R(\hat\lambda)\le{}&\tfrac53C_U\Bigl[\bigl(3^{2/3}+4\cdot2^{2/3}+20\bigr)(mn)^{-2/3}+4\min\set{\tau^2,1/n}\\
  &\qquad+4\min\set{\tau^2,V^2,(V/n)^{2/3}}\Bigr]+158\,(mn)^{-2/3}.
\end{align*}
Since $\frac53(3^{2/3}+4\cdot2^{2/3}+20)=47.383\ldots<48$ and $\frac53\cdot4<48$,
\begin{align*}
  \frac1m\sum_j\E\norm{\hat g_j-g_j}^2_{L^2(\mu)}=\E\bar R(\hat\lambda)
  &\le(48C_U+158)\Bigl[(mn)^{-2/3}+\min\set{\tau^2,1/n}\\
  &\qquad+\min\set{\tau^2,V^2,(V/n)^{2/3}}\Bigr].
\end{align*}
This is the claim, with $C\defeq48C_U+158$, which does not depend on $\mu$, $m$, $n$, $\tau$ or $V$.
\end{proof}

\ThmAdaptive*

The procedure branches on the known pair $(m,n)$. When $n=1$, $\Phi\asymp m^{-2/3}+\tau^2$, and the pooled fit,
which ignores question identity, attains this. When $m\le2$, $\Phi\asymp n^{-2/3}$, and per-question isotonic
fits attain it. Otherwise, the compound hold-out of \cref{prop:holdout} applies.

\begin{proof}
Recall the three branches of the procedure. If $n=1$, every $\hat g_j$ equals the pooled isotonic fit
$\hat g_0$ to all $m$ pairs, pinned as in \cref{lem:pooled}, with tied design points (a null event) pooled
into one weighted pool-adjacent-violators block, a rule that does not use $\mu$. If $n\ge2$ and $m\le2$, $\hat g_j$ is the isotonic fit
to question $j$'s own $n$ pairs, pinned in the same way. If $m\ge3$ and $n\ge2$, the procedure is that of
\cref{prop:holdout}. Let $C_{\mathrm{pool}}$ be the constant of \cref{lem:pooled}, and put
$c_0\defeq\sup_{N\ge1}N^{-1/3}\log(eN)=3e^{-2/3}$, which is attained at $N=e^2$ and is finite. Then
$\log(eN)/N\le c_0N^{-2/3}$ for every $N\ge1$, and \cref{lem:pooled} gives
$\E\norm{\hat g_0-\bar g_I}^2\le C_{\mathrm{pool}}(1+c_0)N^{-2/3}\eqqcolon C_1N^{-2/3}$.

\emph{Case $n=1$.} Apply \cref{lem:pooled} with $I=\set{1,\dots,m}$ ($q=m$), $n=1$ and $N=m$:
$\E\norm{\hat g_0-\bar g}^2\le C_1m^{-2/3}$, where $\bar g\defeq\frac1m\sum_ig_i$. Let $g_0$ witness
$(g_j)\in\Theta(\tau,V)$. For every $j$, pointwise,
$\abs{g_j-\bar g}\le\abs{g_j-g_0}+\frac1m\sum_i\abs{g_0-g_i}\le2\tau$, so $\norm{g_j-\bar g}^2_{L^2(\mu)}\le4\tau^2$. Hence
$\E\norm{\hat g_j-g_j}^2\le2\E\norm{\hat g_0-\bar g}^2+2\norm{\bar g-g_j}^2\le2C_1m^{-2/3}+8\tau^2$. With $n=1$ and
$\tau\le1$, $\Phi(m,1,\tau,V)\ge m^{-2/3}+\min\set{\tau^2,1}=m^{-2/3}+\tau^2$. So the compound risk is at most
$\max\set{2C_1,8}\,\Phi$.

\emph{Case $n\ge2$, $m\le2$.} Apply \cref{lem:pooled} with $I=\set j$ ($q=1$) and $N=n$:
$\E\norm{\hat g_j-g_j}^2\le C_1n^{-2/3}$. Since $m\le2$, $\Phi\ge(mn)^{-2/3}\ge2^{-2/3}n^{-2/3}$. So the compound
risk is at most $2^{2/3}C_1\Phi$.

\emph{Case $m\ge3$, $n\ge2$.} \Cref{prop:holdout} bounds the compound risk by $C_{\mathrm{HO}}\Phi$, with
$C_{\mathrm{HO}}$ universal.

The branch is chosen from $(m,n)$ alone. Each branch is non-randomized, jointly measurable in (data, $p$)
(by the pinned fits of \cref{lem:pooled} and by \cref{prop:holdout}), and uses neither $(\tau,V)$ nor $\mu$. Take
$C\defeq\max\set{2C_1,8,2^{2/3}C_1,C_{\mathrm{HO}}}$.
\end{proof}

\ThmRate*

\begin{proof}
Recall that $\Theta_{\mathrm{LB}}(\tau,V)\subseteq\Theta(\tau,V)$ (the same predicate, with $g_0$ in addition
$[0,1]$-valued), and that $\Theta(\tau,V)$ is nondecreasing in each of $\tau$ and $V$.

\textbf{Lower bound.} Fix a non-atomic $\mu$, $m,n\ge1$, $\tau\in[0,1]$ and $V\in[0,2]$, and put
$\tau'\defeq\min\set{\tau,1/8}$ and $V'\defeq\min\set{V,1/8}$. Since
$\Theta_{\mathrm{LB}}(\tau',V')\subseteq\Theta(\tau',V')\subseteq\Theta(\tau,V)$, the supremum over $\Theta(\tau,V)$ is
at least the supremum over $\Theta_{\mathrm{LB}}(\tau',V')$. \Cref{lem:lower} holds for all $m,n\ge1$, all
$\tau',V'\in[0,1/8]$ and every non-atomic $\mu$, with a constant $c_{\mathrm{LB}}$ independent of $\mu$ and over
the same estimator class. It gives $\mathfrak M_{m,n}(\tau,V)\ge c_{\mathrm{LB}}\Phi(m,n,\tau',V')$. We claim that
$\Phi(m,n,\tau,V)\le256\,\Phi(m,n,\tau',V')$. Term by term:
\begin{itemize}
\item $(mn)^{-2/3}$ is unchanged.
\item If $\tau\le1/8$, the second term is unchanged. Otherwise $\min\set{\tau^2,1/n}\le1/n$ and
$\min\set{\tau'^2,1/n}=\min\set{1/64,1/n}\ge1/(64n)$ (as $n\ge1$), so
$\min\set{\tau^2,1/n}\le64\min\set{\tau'^2,1/n}$.
\item Third term: $\tau'^2\ge\tau^2/64$ (as $\tau\le1$). Also $V'^2\ge V^2/256$: if $V>1/8$, then
$V'^2=1/64\ge V^2/256$, as $V\le2$. For the same reason $V'\ge V/16$, so
$(V'/n)^{2/3}\ge16^{-2/3}(V/n)^{2/3}\ge(V/n)^{2/3}/256$. Hence
$\min\set{\tau'^2,V'^2,(V'/n)^{2/3}}\ge\min\set{\tau^2,V^2,(V/n)^{2/3}}/256$.
\end{itemize}
Summing the three terms gives $\Phi(m,n,\tau,V)\le256\,\Phi(m,n,\tau',V')$, so
$\mathfrak M_{m,n}(\tau,V)\ge(c_{\mathrm{LB}}/256)\,\Phi(m,n,\tau,V)$.

\textbf{Upper bound.} For all $m,n\ge1$, the procedure of \cref{thm:adaptive} is non-randomized, jointly
measurable, and uses neither $(\tau,V)$ nor $\mu$, so it is an admissible estimator. Its compound risk on
$\Theta(\tau,V)$ (unrestricted witness), for $\tau\in[0,1]$ and $V\in[0,2]$, is at most $C_{AD}\Phi(m,n,\tau,V)$
by \cref{thm:adaptive}. Hence $\mathfrak M_{m,n}(\tau,V)\le C_{AD}\Phi$. Take $c\defeq c_{\mathrm{LB}}/256$ and
$C\defeq\max\set{C_{AD},c}$.

\emph{Witness convention.} The lower bound holds on the class with $[0,1]$-valued witness, and the upper bound
on the class with unrestricted witness. A supremum over a larger class dominates the supremum over a smaller
one. So both bounds hold for either convention, and for every class in between.

\textbf{Equivalence of $\Phi$ and $\Psi$.} Here
$\Psi\defeq(mn)^{-2/3}+\min\set{\tau^2,1/n+(V/n)^{2/3}}$.
(i) For $a,x,y\ge0$, $\min\set{a,x}+\min\set{a,y}\le2\min\set{a,x+y}$; and
$\min\set{\tau^2,V^2,(V/n)^{2/3}}\le\min\set{\tau^2,(V/n)^{2/3}}$. So
$\Phi\le(mn)^{-2/3}+2\min\set{\tau^2,1/n+(V/n)^{2/3}}\le2\Psi$.
(ii) $\min\set{\tau^2,1/n+(V/n)^{2/3}}\le\min\set{\tau^2,1/n}+\min\set{\tau^2,(V/n)^{2/3}}$. If
$V^2\ge(V/n)^{2/3}$, then $\min\set{\tau^2,(V/n)^{2/3}}=\min\set{\tau^2,V^2,(V/n)^{2/3}}$. Otherwise $V<n^{-1/2}$
(since $V^2<(V/n)^{2/3}$ if and only if $V^4<n^{-2}$), so $(V/n)^{2/3}<1/n$ and
$\min\set{\tau^2,(V/n)^{2/3}}\le\min\set{\tau^2,1/n}$. In both cases
$\min\set{\tau^2,(V/n)^{2/3}}\le\min\set{\tau^2,V^2,(V/n)^{2/3}}+\min\set{\tau^2,1/n}$. So
$\Psi\le(mn)^{-2/3}+2\min\set{\tau^2,1/n}+\min\set{\tau^2,V^2,(V/n)^{2/3}}\le2\Phi$.
\end{proof}

\input{sections/9a6_separation}

%% file: sections/9a6_separation.tex
\subsection{Per-question tuning: proof of \texorpdfstring{\cref{prop:separation}}{the separation}}
\label{app:separation}

\PropSeparation*

Fix $\mu$, $m,n\ge1$ and a nondecreasing $g_0:[0,1]\to[\frac14,\frac34]$, and write
$\Phi_0(\tau)\defeq\Phi(m,n,\tau,0)=(mn)^{-2/3}+\min\set{\tau^2,1/n}$, $M\defeq(m^2/n)^{1/3}$ and
$L\defeq\log(m^2/n)$. Note that $(mn)^{-2/3}=1/(nM)$. For $c\in(0,\frac14]$ and $s\in\set{\pm1}^m$ put
$g^0\defeq(g_0,\dots,g_0)$ and $g^{c,s}\defeq(g_0+s_jc)_{j\in[m]}$; both lie in
$\Theta_{g_0}(c)$. Let $\gh$ be a per-question estimator, $\gh_j=T_j(D_j,\xi,\cdot)$.

\begin{proof}
\emph{Step 0 (structure).} Let $g=(g_0+c_j)_j$ with $\abs{c_j}\le\frac14$, and let $Q_c$ be the law of one
pair $(P,Y)$ with $P\sim\mu$ and $Y\mid P\sim\Bern(g_0(P)+c)$. The pair $(D_j,\xi)$ has law
$Q_{c_j}^{\otimes n}\otimes\Law(\xi)$: the pairs are independent across questions, and $\xi$ is
independent of the data, with a law that does not depend on the parameters. Hence
$\cR(\gh;g)=\frac1m\sum_jr_j(c_j)$ with $r_j(c)\defeq\E_c\nmu{T_j(D_j,\xi,\cdot)-g_0-c}^2$, which depends on
$c_j$ only.

\emph{Step 1 (reduction to a scalar).} Let $\pi(y)\defeq\min\set{1,\max\set{0,y}}$; for $z\in[0,1]$,
$\abs{\pi(y)-z}\le\abs{y-z}$. Put
$\delta_j\defeq\int\bigl(\pi(T_j(D_j,\xi,p))-g_0(p)\bigr)\,\mu(dp)\in[-1,1]$, which is measurable in
$(D_j,\xi)$ by Tonelli's theorem ($\pi\circ T_j$ is jointly measurable and bounded, and $g_0$ is monotone,
hence Borel). For $c\in[-\frac14,\frac14]$ the map $g_0+c$ is $[0,1]$-valued, so, by Jensen's inequality for
the probability $\mu$,
\[
  \nmu{T_j-g_0-c}^2\ \ge\ \nmu{\pi\circ T_j-g_0-c}^2\ \ge\ (\delta_j-c)^2 .
\]
In particular $\E_0\delta_j^2\le e_j^2\defeq r_j(0)$.

\emph{Step 2 (exact chi-square).} $Q_c$ has density
$q_c(p,y)=(g_0(p)+c)^y(1-g_0(p)-c)^{1-y}$ with respect to $\mu\otimes(\text{counting measure on }\set{0,1})$.
Since $q_0\ge\frac14>0$, $Q_c\ll Q_0$, the density of $Q_c^{\otimes n}$ with respect to $Q_0^{\otimes n}$ is
$\Lambda_c\defeq\prod_kq_c(P_k,Y_k)/q_0(P_k,Y_k)$, and
$\E_0\Lambda_c^2=\bigl(\int\sum_yq_c^2/q_0\,d\mu\bigr)^n$. With $a=g_0(p)$,
\[
  \frac{(a+c)^2}{a}+\frac{(1-a-c)^2}{1-a}=1+\frac{c^2}{a(1-a)},
\]
and $a(1-a)\ge\frac3{16}$ for $a\in[\frac14,\frac34]$. Hence
$\E_0\Lambda_c^2\le(1+\frac{16c^2}3)^n\le e^{16nc^2/3}$ and
$I_c\defeq(\E_0\Lambda_c^2)^{1/2}\le e^{8nc^2/3}$.

\emph{Step 3 (constrained-risk inequality).} For bounded measurable $f$,
$\E_cf(D_j,\xi)=\E_0[\Lambda_c(D_j)f(D_j,\xi)]$, by Fubini's theorem, since $\xi$ is independent of $D_j$ and
has the same law under both parameters. By the Cauchy--Schwarz inequality,
$\abs{\E_c\delta_j}\le I_c(\E_0\delta_j^2)^{1/2}\le I_ce_j$. This is the form of the constrained-risk
inequality of~\citet[Theorem~1]{brown1996constrained} that we use; the argument above is self-contained.

\emph{Step 4 (the trade-off).} We show that for every $c\in(0,\frac14]$ and $s\in\set{\pm1}^m$,
\begin{equation}\label{eq:sep-tradeoff}
  \cR(\gh;g^{c,s})\ \ge\ \Bigl(c-e^{8nc^2/3}\,\cR(\gh;g^0)^{1/2}\Bigr)_+^2 .
\end{equation}
If $\eps\defeq\cR(\gh;g^0)=\frac1m\sum_je_j^2=+\infty$, the right side is $0$. Otherwise, by Step~1,
$r_j(s_jc)\ge\E(\delta_j-s_jc)^2\ge(\E\delta_j-s_jc)^2=a_j^2$, where the expectations are
under the shift $s_jc$ and $a_j\defeq c-s_j\E_{s_jc}\delta_j$. By Steps~2 and~3,
$s_j\E_{s_jc}\delta_j\le e^{8nc^2/3}e_j$ (the same bound holds at $-c$), so
$a_j\ge c-e^{8nc^2/3}e_j$. Hence
$\cR(\gh;g^{c,s})=\frac1m\sum_jr_j(s_jc)\ge\frac1m\sum_ja_j^2\ge(\frac1m\sum_ja_j)^2$ whenever
$\frac1m\sum_ja_j\ge0$, and, by the Cauchy--Schwarz inequality,
$\frac1m\sum_ja_j\ge c-e^{8nc^2/3}\frac1m\sum_je_j\ge c-e^{8nc^2/3}\eps^{1/2}$. If the last
right side is nonnegative, squaring nonnegative numbers gives~\eqref{eq:sep-tradeoff}; otherwise
\eqref{eq:sep-tradeoff} is trivial.

\emph{Step 5 (proof of (i)).} Assume $m^2\ge216n$, so that $M\ge6$. Let $\kappa\defeq\kappa_{g_0}(\gh)$ and suppose
that $\kappa<\kappa_0\defeq\frac3{256}\min\set{n,\log M}=\frac1{256}\min\set{3n,L}$. Since $\Theta_{g_0}(0)=\set{g^0}$,
$\eps\le\kappa\,\Phi_0(0)=\kappa/(nM)$. Put $b\defeq\frac3{32}\log M$ and $c\defeq\min\set{\frac14,(b/n)^{1/2}}$.
Then $nc^2\le b$, $e^{8nc^2/3}\le e^{8b/3}=M^{1/4}$, and
$e^{8nc^2/3}\eps^{1/2}\le(\kappa/n)^{1/2}M^{-1/4}$. This is at most $c/2$: if $c=(b/n)^{1/2}$ it
needs $\kappa\le bM^{1/2}/4=\frac3{128}M^{1/2}\log M$, which holds as $\kappa<\frac3{256}\log M$; if $c=\frac14$ it
needs $\kappa\le nM^{1/2}/64$, which holds as $\kappa<3n/256$. Since $g^{c,(1,\dots,1)}\in\Theta_{g_0}(c)$,
\eqref{eq:sep-tradeoff} gives $\kappa\,\Phi_0(c)\ge c^2/4$, where
$\Phi_0(c)=n^{-1}(M^{-1}+\min\set{nc^2,1})$.
\emph{Case $c=(b/n)^{1/2}$.} Here $b\ge\frac3{32}\log6=0.16798\ldots>\frac16\ge\frac1M$, so
$\Phi_0(c)\le\frac2n\min\set{b,1}$ and $\kappa\ge b/(8\min\set{b,1})\ge b/8=\frac3{256}\log M\ge\kappa_0$, a
contradiction.
\emph{Case $c=\frac14$.} Here $\Phi_0(c)\le\frac7{6n}$, so $\kappa\ge\frac{3n}{224}>\frac{3n}{256}\ge\kappa_0$, a
contradiction.
Hence $\kappa_{g_0}(\gh)\ge\kappa_0$. For the consequence, note that $\Theta_{g_0}(\tau)\subseteq\Theta(\tau,0)$
(each $g_0+c_j$ is nondecreasing and $[0,1]$-valued, $g_0$ is a witness, $\sup\abs{c_j}\le\tau$, and
constants have variation $0$), that $\Phi(m,n,\tau,0)=\Phi_0(\tau)$, and that $[0,\frac14]\subseteq[0,1]$; so
the supremum in (i) is at least $\kappa_{g_0}(\gh)$ for any one $g_0$ as above, for instance $g_0\equiv\frac12$,
whether or not $\gh$ uses it.

\emph{Step 6 (proof of (ii)).} Fix $c$ with $\abs c\le\frac14$ and a question $j$, and let
$X\defeq\hat c_j-c$. The variables $Z_k\defeq Y_{jk}-g_0(P_{jk})-c$ are i.i.d., with mean $0$, $\abs{Z_k}\le1$
and $\E Z_k^2\le\frac14$. Bernstein's inequality~\citep[eq.~(2.10)]{boucheron2013concentration} with $b=1$ and
$v\le n/4$, applied to $(Z_k)$ and to $(-Z_k)$, gives, for $s\le1$,
$\P(\abs X\ge s)\le2\exp\bigl(-ns^2/(\frac12+\frac{2s}3)\bigr)\le2e^{-\nu ns^2}$ with $\nu\defeq\frac67$;
since $\abs X\le1$, this holds for all $s\ge0$. Hence
$\E X^4\le\int_0^\infty8s^3e^{-\nu ns^2}ds=4/(\nu^2n^2)$ and $\E X^2\le1/(4n)$. The risk of the rule at
the shift $c$ is
$r(c)=\E[X^2\mathbf 1\set{\abs{\hat c_j}>t}]+c^2\,\P(\abs{\hat c_j}\le t)$.
(a)~If $\abs c\le t/2$, then $\abs{\hat c_j}>t$ forces $\abs X>t/2$, and by the Cauchy--Schwarz inequality
$r(c)\le c^2+\frac2{\nu n}\bigl(2e^{-\nu nt^2/4}\bigr)^{1/2}$.
(b)~If $\abs c>t/2$, the first term is at most $\E X^2\le1/(4n)$. For the second, if $\abs c\ge2t$ then
$\abs{\hat c_j}\le t$ forces $\abs X\ge\abs c/2$, and
$c^2\P(\abs X\ge\abs c/2)\le\sup_{x\ge0}2xe^{-\nu nx/4}=\frac{28}{3en}$; if $t/2<\abs c<2t$ it is at most
$c^2\le4t^2\le16\min\set{c^2,t^2}$.
When $M\ge e$, $t^2=28\log M/(3n)\ge4/n$, so in case (b) $c^2>1/n$; and $e^{-\nu nt^2/8}=1/M$. Combining,
\[
  r(c)\ \le\ 16\min\set{c^2,t^2}+3.684\min\set{c^2,\tfrac1n}+3.300\,(mn)^{-2/3}.
\]
Every term is nondecreasing in $\abs c$, and $\min\set{\tau^2,k/n}\le k\min\set{\tau^2,1/n}$ for
$k\defeq\frac{28}3\log M\ge1$, so
$\sup_{\abs c\le\tau}r(c)\le(\frac{448}3\log M+3.684)\Phi_0(\tau)\le(4+50L)\Phi_0(\tau)$. The compound risk of a
configuration in $\Theta_{g_0}(\tau)$ is $\frac1m\sum_jr(c_j)\le\sup_{\abs c\le\tau}r(c)$, which gives
$\kappa_{g_0}\le4+50L$ when $M\ge e$.
\emph{Case $M<e$.} Then $t=0$ and $r(c)=\E X^2\le\frac1{4n}<\frac e4(mn)^{-2/3}\le\frac e4\Phi_0(\tau)$.
\emph{The rule $g_0$.} Its risk on $\Theta_{g_0}(\tau)$ is at most $\tau^2$, and
$\tau^2/\Phi_0(\tau)\le\max\set{1,n\tau^2}\le\max\set{1,n/16}$.
\emph{Conclusion.} For $m^2\ge216n$ we have $L\ge\log216=5.375\ldots$, so $4+50L\le51L$, and
$\min\set{\max\set{1,n/16},51L}\le51\min\set{n,L}$. Together with (i) this gives the two-sided bound; the case
$m^2<e^3n$ is the case $M<e$.

\emph{Step 7 (proof of (iii)).} Apply \cref{thm:adaptive} at the non-atomic $\mu$, with
$\tau\in[0,\frac14]\subseteq[0,1]$ and $V=0\in[0,2]$: $\cR(\gh^{\mathrm{ad}};g)\le C_{\mathrm{AD}}\Phi(m,n,\tau,0)$
for every $g\in\Theta(\tau,0)\supseteq\Theta_{g_0}(\tau)$ (Step~5), so $\kappa_{g_0}(\gh^{\mathrm{ad}})\le C_{\mathrm{AD}}$.
Combined with (i), every per-question $\gh$ has
$\kappa_{g_0}(\gh)/C_{\mathrm{AD}}\ge\min\set{3n,L}/(256\,C_{\mathrm{AD}})$ when $m^2\ge216n$.
\end{proof}

%% file: sections/9b_proofs_eb.tex
This appendix proves the empirical-Bayes results in dependency order.
\Cref{app:chebradius} proves \cref{prop:chebradius}: the identification regret
is the squared Chebyshev radius of the set of unidentified top moments. Every
later bound reads off this representation. \Cref{app:binomial} evaluates it for
priors on constant maps, where it reduces to the classical range of the
$(n+1)$-th moment (\cref{prop:binomial}). \Cref{app:ch1} proves the
diagonal-domination inequality \cref{lem:ch1}. \Cref{app:monotone-eb} combines
it with the representation to prove the upper bound $2^{-n}$ for monotone maps
(\cref{thm:monotone-eb}). \Cref{app:betastar} pins the exponential base
$\bstar=2e^{2\Gcat/\pi}$ for constant maps (\cref{thm:betastar}).
\Cref{app:eb-lower} transfers the constant-map lower bounds to monotone priors
(\cref{lem:eb-lower}), \cref{app:main-eb} assembles \cref{thm:main-eb}, and \cref{app:finite-m} shows
that its lower bounds hold for every finite number of questions (\cref{cor:finite-m}).

Throughout, $\mu$ is an arbitrary Borel probability measure on $[0,1]$; atoms
are allowed in every statement of this appendix. A prior is the law of a
jointly measurable random map $g:[0,1]\to[0,1]$. One question's data are
$D=(X,Y)$, where $X=(P_1,\dots,P_n)$ has i.i.d.\ $\mu$ entries independent of
$g$, and, given $g$ and $X$, the labels $Y_k\sim\mathrm{Bernoulli}(g(P_k))$
are independent. Step~0 of the proof of \cref{prop:chebradius} makes this
reading precise.

\input{sections/9b1_chebradius}
\input{sections/9b2_binomial}
\input{sections/9b3_monotone}
\input{sections/9b4_betastar}
\input{sections/9b5_lower_main}
\input{sections/9b6_finite_m}

%% file: sections/9b1_chebradius.tex
\subsection{The identification regret is a Chebyshev radius}
\label{app:chebradius}

The proposition says that the excess risk one cannot avoid without knowing the
prior is a purely geometric quantity: the squared radius of the smallest ball,
in a weighted $L^2$ space, that contains the set of unidentified top moments.
The proof has five steps. Step~0 fixes the measure-theoretic reading of the
setting and the kernels $M^G_r$, $Z_G$, $N_G$. Step~1 is a Pythagorean
identity for the risk. Step~2 shows that two observationally equivalent priors
can differ only in the $(n+1)$-point product moment. Step~3 rewrites the
regret as a Chebyshev-radius problem, and Step~4 compares prior classes.

\PropChebRadius*

\begin{proof}
\textbf{Conventions.} Fix $n\ge 1$ (the argument runs verbatim for $n=0$). Let
$\mathcal D \defeq [0,1]^n\times\{0,1\}^n$ with the $\sigma$-field
$\mathcal B([0,1]^n)\otimes 2^{\{0,1\}^n}$, and
$\nu\defeq\mu^{\otimes n}\otimes(\text{counting measure})$. For
$y\in\{0,1\}^n$ write $s(y)=\sum_k y_k$, $T_1(y)\defeq\{k:y_k=1\}$,
$T_0(y)\defeq[n]\setminus T_1(y)$ and $\sigma(y)\defeq(-1)^{n-s(y)}$; thus
$\sigma(y)^2=1$ and $\abs{T_0(y)}=n-s(y)$. For $h:[0,1]\to[0,1]$ and
$x\in[0,1]^n$ put
$\pi_y(h,x)\defeq\prod_{k}h(x_k)^{y_k}(1-h(x_k))^{1-y_k}$. For
$S=\{k_1<\dots<k_r\}\subseteq[n]$ write $x_S\defeq(x_{k_1},\dots,x_{k_r})$,
and set $M_0\defeq 1$. Let $\operatorname{clip}(t)\defeq\min\{1,\max\{0,t\}\}$.
A rule is written $\delta(x,y)(p)$, and
$\mathrm{regret}(\delta,G')\defeq R(\delta,G')-R^*_n(G')$. We use
$0\cdot\infty\defeq 0$ and $1/0\defeq\infty$, and write
$\mathcal L_G\defeq\mathrm{Law}_G(D)$.

\medskip\noindent
\textbf{Step 0: formalization, kernels, and a jointly measurable posterior mean.}

\emph{0.1. Reading of the setting.} A prior $G$ is realized on a probability
space $(\Omega,\mathcal F,\P)$ carrying a sub-$\sigma$-field $\mathcal H$ such
that:
(F1) $g:\Omega\times[0,1]\to[0,1]$ is $\mathcal H\otimes\mathcal B([0,1])$-measurable;
(F2) $X=(P_1,\dots,P_n)$ is $\mathcal F$-measurable, has law $\mu^{\otimes n}$, and is independent of $\mathcal H$;
(F3) for each $y$, $\P(Y=y\mid\mathcal H\vee\sigma(X))=\pi_y(g,X)$ almost surely.
Here $\pi_y(g,X)$ is $(\mathcal H\vee\sigma(X))$-measurable: the map
$\varpi\mapsto(\varpi,X_k(\varpi))$ is
$(\mathcal H\vee\sigma(X))/(\mathcal H\otimes\mathcal B)$-measurable, because the
preimage of $A\times B$ is $A\cap\{X_k\in B\}$; now compose with $g$. A
canonical realization is $\Omega=\Omega_0\times[0,1]^n\times[0,1]^n$,
$\P=\P_0\otimes\mu^{\otimes n}\otimes\mathrm{Leb}^{\otimes n}$,
$\mathcal H=\mathcal F_0\otimes\{\emptyset,[0,1]^n\}\otimes\{\emptyset,[0,1]^n\}$,
$X$ the second coordinate, and $Y_k\defeq\mathbf 1\{\xi_k<g(\varpi_0,X_k)\}$ with
$\xi$ the third coordinate; Fubini over $\xi$ gives (F3).
The weaker reading ``$X$ is independent of $\sigma(g(t):t\in[0,1])$'' is
inconsistent with the density statement of the setting: take $n=1$,
$\mu=\mathrm{Leb}$, $T\sim\mathrm{Unif}[0,1]$ and
$g(\varpi,t)\defeq\mathbf 1\{t=T(\varpi)\}$. This $g$ is jointly measurable,
$\sigma(g)$ is $\P$-trivial, so $X\defeq T$ is independent of it; yet $g(X)=1$,
while $Z(x,1)=\P(T=x)=0$ for every $x$. Reading (F2) is the one under which the
setting is consistent.

\emph{0.2. Lemma A (freezing).} Let $(\mathbb X,\mathcal X)$ be a measurable
space, $\zeta:\Omega\to\mathbb X$ an $\mathcal F$-measurable map independent of
$\mathcal H$ with law $\kappa$, and $\Upsilon:\Omega\times\mathbb X\to[0,\infty]$
an $\mathcal H\otimes\mathcal X$-measurable function. Then
(a) $\varpi\mapsto\Upsilon(\varpi,\zeta(\varpi))$ is $\mathcal F$-measurable;
(b) $e\mapsto\E\,\Upsilon(\cdot,e)$ is $\mathcal X$-measurable;
(c) $\E\,\Upsilon(\cdot,\zeta)=\int\E\,\Upsilon(\cdot,e)\,\kappa(de)$.
The same holds for bounded real $\Upsilon$.
\emph{Proof of Lemma A.} (a) The map $\iota(\varpi)\defeq(\varpi,\zeta(\varpi))$
satisfies $\iota^{-1}(A\times B)=A\cap\zeta^{-1}(B)\in\mathcal F$, and rectangles
generate $\mathcal H\otimes\mathcal X$. (b) Tonelli on
$(\Omega,\mathcal H,\P)\times(\mathbb X,\mathcal X,\kappa)$. (c) For rectangles,
$\E[\mathbf 1_A\mathbf 1_B(\zeta)]=\P(A)\kappa(B)$ by independence. The class of
$C\in\mathcal H\otimes\mathcal X$ for which (c) holds with $\Upsilon=\mathbf 1_C$
contains $\Omega\times\mathbb X$ and is closed under proper differences and
increasing unions, so by the $\pi$--$\lambda$ theorem it is all of
$\mathcal H\otimes\mathcal X$. Extend by linearity and monotone convergence,
and to bounded $\Upsilon$ via $\Upsilon=\Upsilon^+-\Upsilon^-$. $\triangleleft$

\emph{0.3. Lemma B (disintegration of $D$).} For every
$\Upsilon:\Omega\times\mathcal D\to[0,\infty]$ (or bounded) that is
$\mathcal H\otimes\mathcal B([0,1]^n)\otimes2^{\{0,1\}^n}$-measurable,
\begin{equation}\label{eq:eb1-disint}
  \E\,\Upsilon(\cdot,X,Y)=\sum_{y}\int\E\bigl[\Upsilon(\cdot,x,y)\,\pi_y(g,x)\bigr]\,\mu^{\otimes n}(dx).
\end{equation}
\emph{Proof of Lemma B.} Write
$\E\,\Upsilon(\cdot,X,Y)=\sum_y\E[\Upsilon(\cdot,X,y)\mathbf 1\{Y=y\}]$. For
bounded $\Upsilon$, $\Upsilon(\cdot,X,y)$ is
$(\mathcal H\vee\sigma(X))$-measurable, so (F3) gives
$\E[\Upsilon(\cdot,X,y)\mathbf 1\{Y=y\}]=\E[\Upsilon(\cdot,X,y)\pi_y(g,X)]$. The
function $\Upsilon_y(\varpi,x)\defeq\Upsilon(\varpi,x,y)\,\pi_y(g(\varpi,\cdot),x)$
is $\mathcal H\otimes\mathcal B([0,1]^n)$-measurable by (F1), and Lemma~A(c)
with $\zeta=X$ gives~\eqref{eq:eb1-disint}. For nonnegative $\Upsilon$, apply
this to $\Upsilon\wedge j$ and let $j\to\infty$. $\triangleleft$

\emph{0.4. Kernels.} Let $M^G_r(t)\defeq\E\prod_{i\le r}g(t_i)$,
$Z_G(x,y)\defeq\E\,\pi_y(g,x)$ and $N_G(p;x,y)\defeq\E[g(p)\,\pi_y(g,x)]$; these
are Borel by Tonelli. Pointwise in $\varpi$ we have
$0\le g(p)\pi_y\le\pi_y$; for every $S\subseteq[n]$,
$\prod_{k\in S}g(x_k)=\sum_{y:\,y_S=1}\pi_y(g,x)$ (multiply by
$\prod_{k\notin S}\bigl(g(x_k)+(1-g(x_k))\bigr)=1$ and expand); and
$\prod_{k\in T_0}(1-g(x_k))=\sum_{A\subseteq T_0}(-1)^{\abs A}\prod_{k\in A}g(x_k)$.
Taking expectations,
\begin{align*}
  &\text{(K1)}\quad 0\le N_G(p;x,y)\le Z_G(x,y)\le 1;\qquad
   \text{(K2)}\quad \textstyle\sum_y Z_G(x,y)=1;\\
  &\text{(K3)}\quad M^G_{\abs S}(x_S)=\textstyle\sum_{y:\,y_S=1}Z_G(x,y);\\
  &\text{(K4)}\quad N_G(p;x,y)=\textstyle\sum_{A\subseteq T_0(y)}(-1)^{\abs A}\,
   M^G_{1+\abs{T_1(y)\cup A}}\bigl(p,x_{T_1(y)\cup A}\bigr).
\end{align*}
\emph{Density.} Lemma~B with $\Upsilon=\mathbf 1_B(x)\mathbf 1\{y=y_0\}$ gives
$\P(X\in B,Y=y_0)=\int_B Z_G(x,y_0)\,\mu^{\otimes n}(dx)$, so
$\mathcal L_G=Z_G\cdot\nu$, and hence
(K5) $\P(Z_G(D)=0)=0$.
\emph{Everything depends only on $G$.} For a rule $\delta$, the map
$(\varpi,p)\mapsto(\delta(D(\varpi))(p)-g(\varpi,p))^2$ is
$\mathcal F\otimes\mathcal B$-measurable, and
$\Upsilon(\varpi,d)\defeq\int(\delta(d)(p)-g(\varpi,p))^2\mu(dp)$ is
$\mathcal H\otimes\mathcal B(\mathcal D)$-measurable. Tonelli and Lemma~B give
\[
  R(\delta,G)=\sum_y\int\mu^{\otimes n}(dx)\int\mu(dp)\,
  \E\bigl[(\delta(x,y)(p)-g(p))^2\,\pi_y(g,x)\bigr]\in[0,1],
\]
an expectation of a polynomial in $(g(p),g(x_1),\dots,g(x_n))$. So $R$,
$M^G_r$, $Z_G$, $N_G$ and $\mathcal L_G$ depend only on the finite-dimensional
distributions of $g$.

\emph{0.5. Lemma C (jointly measurable posterior mean).} Let
$\delta^*_G(d)(p)\defeq N_G(p;d)/Z_G(d)$ if $Z_G(d)>0$, and $0$ otherwise. Then:
(a) $\delta^*_G$ is a rule ($[0,1]$-valued by (K1), jointly Borel);
(b) $\delta^*_G(d)(p)\,Z_G(d)=N_G(p;d)$ everywhere (where $Z_G=0$, $N_G=0$ by (K1));
(c) for every $p$ and every bounded measurable $h$ on $\mathcal D$,
$\E[g(p)h(D)]=\E[\delta^*_G(D)(p)h(D)]$, so $\delta^*_G(D)(p)=\E_G[g(p)\mid D]$ a.s.;
(d) for every bounded jointly measurable $\varphi$ on $[0,1]\times\mathcal D$,
$\E\int\varphi(p,D)g(p)\,\mu(dp)=\E\int\varphi(p,D)\,\delta^*_G(D)(p)\,\mu(dp)$;
(e) every jointly measurable $\varphi^*$ with $\varphi^*(p,D)=\E_G[g(p)\mid D]$
a.s.\ for every $p$ equals $\delta^*_G$, $\mu\otimes\mathcal L_G$-a.e.
\emph{Proof of Lemma C.} (c) Lemma~B with $\Upsilon\defeq g(\varpi,p)h(x,y)$ gives
$\E[g(p)h(D)]=\sum_y\int h\,N_G(p;\cdot,y)\,d\mu^{\otimes n}$; by the density and
(b), $\E[\delta^*_G(D)(p)h(D)]=\sum_y\int h\,\delta^*_G(\cdot,y)(p)\,Z_G(\cdot,y)\,d\mu^{\otimes n}$,
which is the same. (d) Fubini on $\P\otimes\mu$ reduces both sides to
$\int\mu(dp)\,\E[\,\cdot\,]$, and these agree by (c) with $h=\varphi(p,\cdot)$.
(e) The set $\{\varphi^*\ne\delta^*_G\}$ is measurable with $\mathcal L_G$-null
$p$-sections; apply Tonelli. $\triangleleft$

By (e), all statements below about $\delta^*_{G'}$ hold for every jointly
measurable version (for $G'\in[G]$, $\mu\otimes\mathcal L_{G'}=\mu\otimes\mathcal L_G$).

\medskip\noindent
\textbf{Step 1: part (i) (Pythagoras).} Fix a prior $G'$ and a rule $\delta$,
and write $\delta^{*\prime}\defeq\delta^*_{G'}$. On $\Omega\times[0,1]$ put
$a\defeq\delta(D)(p)-\delta^{*\prime}(D)(p)$ and
$b\defeq\delta^{*\prime}(D)(p)-g(p)$; both are $\mathcal F\otimes\mathcal B$-measurable
with values in $[-1,1]$, and $(\delta(D)(p)-g(p))^2=a^2+2ab+b^2$. Integrating
against $\P\otimes\mu$,
$R(\delta,G')=\E\int a^2\,d\mu+2\,\E\int ab\,d\mu+R(\delta^{*\prime},G')$. For
the cross term, $\varphi\defeq\delta-\delta^{*\prime}$ is bounded and jointly
measurable, and
$\E\int ab\,d\mu=\E\int\varphi\,\delta^{*\prime}\,d\mu-\E\int\varphi\,g\,d\mu=0$
by Lemma~C(d) for $G'$. Hence
\begin{equation}\label{eq:eb1-pyth}
  R(\delta,G')=R(\delta^{*\prime},G')+\E_{G'}\Gamma_\delta(D),\qquad
  \Gamma_\delta(d)\defeq\int\bigl(\delta(d)(p)-\delta^{*\prime}(d)(p)\bigr)^2\mu(dp)\in[0,1],
\end{equation}
with $\Gamma_\delta$ Borel by Tonelli. Since $\delta^{*\prime}$ is a rule,
\eqref{eq:eb1-pyth} gives $R^*_n(G')=R(\delta^{*\prime},G')$ (the posterior
mean attains the Bayes risk) and $\mathrm{regret}(\delta,G')=\E_{G'}\Gamma_\delta(D)$.
If $G'\in[G]$, then
$\E_{G'}\Gamma_\delta(D)=\int\Gamma_\delta\,d\mathcal L_G=\E_{D\sim G}\norm{\delta(D)-\delta^*_{G'}(D)}^2_{L^2(\mu)}$,
because $\Gamma_\delta$ is a fixed function of $d$ and the two laws agree. This proves~(i).

\medskip\noindent
\textbf{Step 2: part (ii) (only the top moment is unidentified).} Fix
$G'\in[G]$. Write $Z,N,M_r$ for the kernels of $G$ and $Z',N',M'_r$ for those
of $G'$.

\emph{2.1.} For each $y$ and Borel $B$,
$\int_B Z'(\cdot,y)\,d\mu^{\otimes n}=\P_{G'}(X\in B,Y=y)=\P_G(X\in B,Y=y)=\int_B Z(\cdot,y)\,d\mu^{\otimes n}$.
Taking $B=\{Z'>Z\}$ and $B=\{Z'<Z\}$ shows that
$\mathsf E\defeq\bigcup_y\{x:Z'(x,y)\ne Z(x,y)\}$ is Borel with
$\mu^{\otimes n}(\mathsf E)=0$.

\emph{2.2.} For $r\in[n]$ let $B_r\defeq\{t\in[0,1]^r:M'_r(t)\ne M_r(t)\}$. For
$x\notin\mathsf E$, (K3) with $S=\{1,\dots,r\}$ gives
$M'_r(x_1,\dots,x_r)=\sum_{y:\,y_1=\dots=y_r=1}Z'(x,y)=\sum_{y:\,y_1=\dots=y_r=1}Z(x,y)=M_r(x_1,\dots,x_r)$.
So $B_r\times[0,1]^{n-r}\subseteq\mathsf E$ and
$\mu^{\otimes r}(B_r)=\mu^{\otimes n}(B_r\times[0,1]^{n-r})=0$.

\emph{2.3 (a fresh point).} For a proper subset $S\subsetneq[n]$ and
$r\defeq1+\abs S\in[n]$, the map $\mathrm{pr}_S(p,x)\defeq(p,x_S)$ selects $r$
distinct coordinates of a product measure, so
$(\mathrm{pr}_S)_\#\mu^{\otimes(n+1)}=\mu^{\otimes r}$ (the two measures agree
on rectangles, a generating $\pi$-system). Hence
$\mathsf F_S\defeq\{(p,x):M'_r(p,x_S)\ne M_r(p,x_S)\}=\mathrm{pr}_S^{-1}(B_r)$
is $\mu^{\otimes(n+1)}$-null, and so is
$\mathsf F\defeq\bigcup_{S\subsetneq[n]}\mathsf F_S$. No step assumes $\mu$
non-atomic: coincident coordinates are points of $[0,1]^r$ carrying their
$\mu^{\otimes r}$-mass, and 2.2 identifies $M_r$ on a $\mu^{\otimes r}$-full set,
diagonal points of positive mass included. Only the full $(n+1)$-tuple $(p,x)$
escapes identification; that is where $U_{G'}$ lives.

\emph{2.4.} By (K4), with $S_A\defeq T_1(y)\cup A$,
\[
  N'(p;x,y)-N(p;x,y)=\sum_{A\subseteq T_0(y)}(-1)^{\abs A}\,
  \bigl[M'_{1+\abs{S_A}}-M_{1+\abs{S_A}}\bigr](p,x_{S_A}).
\]
If $A\subsetneq T_0(y)$, then $S_A\subsetneq[n]$ and the bracket vanishes off
$\mathsf F$. If $A=T_0(y)$, then $S_A=[n]$, the sign is
$(-1)^{n-s(y)}=\sigma(y)$, and the bracket is $U_{G'}(p,x)$. Therefore
\begin{equation}\label{eq:eb1-star}
  N'(p;x,y)-N(p;x,y)=\sigma(y)\,U_{G'}(p,x)\qquad\text{for all }(p,x)\notin\mathsf F\text{ and all }y.
\end{equation}

\emph{2.5.} The set
$\Xi\defeq\{(p,x,y):x\in\mathsf E\text{ or }(p,x)\in\mathsf F\}$ has
$(\mu\otimes\nu)(\Xi)=0$, hence $(\mu\otimes\mathcal L_G)(\Xi)=0$ since
$\mathcal L_G\ll\nu$. For $(p,(x,y))\notin\Xi$ with $Z(x,y)>0$ we have
$Z'(x,y)=Z(x,y)>0$, and by Lemma~C and~\eqref{eq:eb1-star},
$\delta^*_{G'}(x,y)(p)-\delta^*_G(x,y)(p)=(N'-N)/Z=\sigma(y)\,U_{G'}(p,x)/Z_G(x,y)$.
This proves~(ii), for any versions by Lemma~C(e).

\medskip\noindent
\textbf{Step 3: part (iii) (Chebyshev-radius representation).} Write
$Z\defeq Z_G$, $Z_y\defeq Z(x,y)$ and $w\defeq w_G=\sum_y 1/Z_y$; $w$ is Borel
with values in $[2^n,\infty]$, and $\{w=\infty\}=\{x:\min_y Z_y=0\}$.

\emph{3.1 (the a.e.\ bound on $U$).} Let $G'\in[G]$, $(p,x)\notin\mathsf F\cup([0,1]\times\mathsf E)$
and $y$ arbitrary. By~\eqref{eq:eb1-star}, (K1) for $G$ and $G'$, the inequality
$\abs{u-v}\le\max\{u,v\}$ for $u,v\ge0$, and $Z'=Z$ off $\mathsf E$,
\[
  \abs{U_{G'}(p,x)}=\abs{N'-N}\le\max\{N',N\}\le\max\{Z'(x,y),Z(x,y)\}=Z_y.
\]
Hence $\abs U\le\min_y Z(x,y)$ $\mu^{\otimes(n+1)}$-a.e., for every representative.
Two consequences: (3.1a) $U=0$ a.e.\ on $\{w=\infty\}$; (3.1b) $wU^2\le1$ a.e.
Indeed, on $\{w<\infty\}$, $U^2/Z_y\le(\min_{y'}Z_{y'})^2/Z_y\le Z_y$, and we sum
over $y$ using (K2); on $\{w=\infty\}$, $\infty\cdot0=0$ a.e.

\emph{3.2 (the Hilbert space).} Let $\mathcal H_G$ be the space of
$\mu^{\otimes(n+1)}$-a.e.\ classes of Borel $c:[0,1]^{n+1}\to\R$ with
$\int\!\!\int wc^2<\infty$, with inner product
$\inner{c}{c'}\defeq\int_{\{w<\infty\}}w\,c\,c'\,d\mu^{\otimes(n+1)}$. Every
$c\in\mathcal H_G$ vanishes a.e.\ on $\{w=\infty\}$, and $\norm c=0$ forces
$c=0$ a.e.\ (as $w\ge1$). Restriction to $\{w<\infty\}$ is an isometric
bijection onto $L^2(\{w<\infty\},w\cdot\mu^{\otimes(n+1)})$ (the null sets
coincide since $1\le w<\infty$ there), so $\mathcal H_G$ is complete. By (3.1b),
$\mathcal U[G]\subseteq\mathcal H_G$ with $\norm U\le1$, and for
$c\in\mathcal H_G$, $U\in\mathcal U[G]$ the integral $\int\!\!\int w(c-U)^2$ of
(iii) equals $\norm{c-U}^2_{\mathcal H_G}$. Define
$f(c)\defeq\sup_{U\in\mathcal U[G]}\norm{c-U}^2$ and
$\Lambda\defeq\inf_{c\in\mathcal H_G}f(c)$. Since $G\in[G]$, $\mathcal U[G]$
contains $0$, so $0\le f(c)\le(\norm c+1)^2$ and $\Lambda\le1$.

\emph{3.3. Lemma D (representation).} For jointly measurable
$\eta:[0,1]\times\mathcal D\to\R$ put
$c_\eta(p,x,y)\defeq\sigma(y)\,Z(x,y)\,\bigl(\eta(x,y)(p)-\delta^*_G(x,y)(p)\bigr)$.
Then for every $G'\in[G]$,
\begin{multline}\label{eq:eb1-D}
  \int_{\mathcal D}\int_0^1\bigl(\eta(d)(p)-\delta^*_{G'}(d)(p)\bigr)^2\mu(dp)\,\mathcal L_G(\mathrm dd)\\
  =\sum_y\int\!\!\int\mathbf 1\{Z_y>0\}\,\frac{(c_\eta(p,x,y)-U_{G'}(p,x))^2}{Z_y}\,d\mu^{\otimes(n+1)}
  \in[0,\infty].
\end{multline}
\emph{Proof of Lemma D.} By the density and Tonelli, the left side is
$\sum_y\int\mu^{\otimes n}(dx)\,Z_y\int\mu(dp)\,(\eta-\delta^*_{G'})^2$, and
terms with $Z_y=0$ vanish. Off $\Xi$ and where $Z_y>0$, part (ii) gives
$\eta-\delta^*_{G'}=(\eta-\delta^*_G)-(\delta^*_{G'}-\delta^*_G)=\sigma(y)(c_\eta-U)/Z_y$
(using $\sigma^2=1$), so $Z_y(\eta-\delta^*_{G'})^2=(c_\eta-U)^2/Z_y$; and $\Xi$
is null. $\triangleleft$
By (i), for every $[0,1]$-valued rule $\delta$, $\mathrm{regret}(\delta,G')$
equals the right side of~\eqref{eq:eb1-D} with $\eta=\delta$.

\emph{3.4 (lower bound $\mathrm{Reg}_n\ge\Lambda$).} Fix a rule $\delta$ and set
$\bar c_\delta(p,x)\defeq w^{-1}\sum_y\sigma(y)\bigl(\delta(x,y)(p)-\delta^*_G(x,y)(p)\bigr)$
on $\{w<\infty\}$ and $\bar c_\delta\defeq0$ on $\{w=\infty\}$. It is Borel and
$\abs{\bar c_\delta}\le2^n/w$, so $\int\!\!\int w\bar c_\delta^{\,2}\le2^n$ and
$\bar c_\delta\in\mathcal H_G$. On $\{w<\infty\}$ every $Z_y>0$ and
$c_\delta(\cdot,y)/Z_y=\sigma(y)(\delta-\delta^*_G)$, so with $a_y\defeq1/Z_y$
and $c_y\defeq c_\delta(\cdot,y)$, $\bar c_\delta=\sum_y a_yc_y/w$ is the
$a$-weighted mean of the $c_y$. The weighted-mean identity
\[
  \sum_y a_y(c_y-u)^2=\sum_y a_y(c_y-\bar c)^2+w(\bar c-u)^2\ \ge\ w(\bar c-u)^2
\]
holds because the cross term $2(\bar c-u)\sum_y a_y(c_y-\bar c)$ vanishes. So
on $\{w<\infty\}$ the integrand of~\eqref{eq:eb1-D} is at least
$w(\bar c_\delta-U)^2$, and on $\{w=\infty\}$, a.e., $U=0=\bar c_\delta$.
Integrating, $\mathrm{regret}(\delta,G')\ge\norm{\bar c_\delta-U_{G'}}^2$ for
every $G'\in[G]$, so $\sup_{G'\in[G]}\mathrm{regret}(\delta,G')\ge f(\bar c_\delta)\ge\Lambda$.
Take the infimum over $\delta$.

\emph{3.5 (upper bound $\mathrm{Reg}_n\le\Lambda$, by clipping).} For
$c\in\mathcal H_G$ let
$\delta_c(x,y)(p)\defeq\delta^*_G(x,y)(p)+\sigma(y)c(p,x)/Z_y$ on $\{Z_y>0\}$ and
$\delta_c\defeq\delta^*_G$ elsewhere, and let
$\tilde\delta_c\defeq\operatorname{clip}\circ\delta_c$, a jointly measurable
$[0,1]$-valued rule. On $\{Z_y>0\}$, $c_{\delta_c}=c$. The clip is 1-Lipschitz
and fixes $[0,1]$, and $\delta^*_{G'}\in[0,1]$, so
$\abs{\operatorname{clip}(t)-v}\le\abs{t-v}$ for $v\in[0,1]$. With (i)
and~\eqref{eq:eb1-D} for $\eta=\delta_c$,
\[
  \mathrm{regret}(\tilde\delta_c,G')
  =\E_{D\sim G}\norm{\operatorname{clip}\circ\delta_c(D)-\delta^*_{G'}(D)}^2
  \le\sum_y\int\!\!\int\mathbf 1\{Z_y>0\}\,\frac{(c-U_{G'})^2}{Z_y}\,d\mu^{\otimes(n+1)}.
\]
The integrand on the right is $w(c-U)^2$ on $\{w<\infty\}$ and, a.e.,
$0=w(c-U)^2$ on $\{w=\infty\}$ ($c=0$ a.e.\ there because $c\in\mathcal H_G$,
and $U=0$ by (3.1a)). Hence $\mathrm{regret}(\tilde\delta_c,G')\le\norm{c-U_{G'}}^2$
for all $G'\in[G]$, so $\mathrm{Reg}_n\le f(c)$; the infimum over $c$ gives
$\mathrm{Reg}_n\le\Lambda$.

\emph{3.6.} Therefore
$\mathrm{Reg}_n(G;\mathcal G)=\Lambda=\inf_c\sup_U\int\!\!\int w(c-U)^2
=\bigl(\inf_c\sup_U\norm{c-U}_{\mathcal H_G}\bigr)^2$, the squared Chebyshev
radius of $\mathcal U[G]$ in $\mathcal H_G$.

\emph{3.7 (the Chebyshev center exists and is unique).} The parallelogram law
with $u=c_1-U$, $v=c_2-U$ gives
$\norm{\tfrac{c_1+c_2}2-U}^2=\tfrac12\norm{c_1-U}^2+\tfrac12\norm{c_2-U}^2-\tfrac14\norm{c_1-c_2}^2$;
taking the supremum over $U$,
$f\bigl(\tfrac{c_1+c_2}2\bigr)\le\tfrac12f(c_1)+\tfrac12f(c_2)-\tfrac14\norm{c_1-c_2}^2$.
A minimizing sequence with $f(c_j)\le\Lambda+1/j$ is therefore Cauchy, since
$\tfrac14\norm{c_i-c_j}^2\le\tfrac12(1/i+1/j)$; so $c_j\to c^*$ in
$\mathcal H_G$. The map $c\mapsto\sup_U\norm{c-U}$ is finite and 1-Lipschitz,
so $f(c^*)=\Lambda$. Uniqueness: $f(c_1)=f(c_2)=\Lambda$ forces $\norm{c_1-c_2}=0$.

\emph{3.8 (a minimax rule).} The rule $\hat\delta\defeq\tilde\delta_{c^*}$ has
$\sup_{G'}\mathrm{regret}(\hat\delta,G')\le f(c^*)=\Lambda=\mathrm{Reg}_n$, so it
is minimax. On $\{Z_G(D)>0\}$, which has full probability by (K5),
$\hat\delta(D)=\operatorname{clip}\bigl(\delta^*_G(D)+\sigma(Y)\,c^*(\cdot,X)/Z_G(X,Y)\bigr)$,
the rule displayed in (iii). Its values on null sets and the choice of a
representative of $c^*$ do not change any regret, by (i) and Tonelli.

\emph{3.9 (real-valued versus $[0,1]$-valued rules).} For jointly measurable
real-valued rules ($R\in[0,\infty]$ by Tonelli), $(\operatorname{clip}\circ\delta-g)^2\le(\delta-g)^2$
pointwise because $g\in[0,1]$. So $R(\operatorname{clip}\circ\delta,G')\le R(\delta,G')$, the
Bayes risk over real rules equals $R^*_n(G')$,
$\mathrm{regret}(\delta,G')\ge\mathrm{regret}(\operatorname{clip}\circ\delta,G')$,
and the infimum over real $\delta$ of $\sup_{G'}\mathrm{regret}$ equals
$\mathrm{Reg}_n$. The infimum over real $c$ in (iii) corresponds to unclipped
rules, and 3.5 clips them at no cost. This proves~(iii).

\medskip\noindent
\textbf{Step 4: part (iv) (monotonicity in the prior class).} Let
$G\in\mathcal G_1\subseteq\mathcal G_2$. The identified set
$[G]_1=\{G'\in\mathcal G_1:\mathcal L_{G'}=\mathcal L_G\}$ is contained in
$[G]_2$ and contains $G$. Neither the class of rules nor $R^*_n(G')$ depends on
the prior class, so
$\sup_{[G]_1}\mathrm{regret}(\delta,\cdot)\le\sup_{[G]_2}\mathrm{regret}(\delta,\cdot)$
for every rule $\delta$; take the infimum over $\delta$.
\end{proof}

\begin{remark}\label{rem:eb1-convexity}
The proof of (i)--(iv) does not use the convexity of the prior class; convexity is kept in the statement
for uniformity with the definition of the identification regret (\cref{def:seb}).
\end{remark}

%% file: sections/9b2_binomial.tex
\subsection{Constant maps: the moment-range formula}
\label{app:binomial}

For a prior on constant maps $g\equiv\theta$ the data carry only the label
counts, the identified set is a moment class, and the unidentified set
$\mathcal U[H]$ is an interval of constants. The Chebyshev radius of
\cref{prop:chebradius} then becomes half the length of the classical range of
the $(n+1)$-th moment. Throughout, $T_N$ is the Chebyshev polynomial of the
first kind and $\widetilde T_N(t)\defeq2^{-(2N-1)}T_N(2t-1)$ is the monic shifted
Chebyshev polynomial of degree $N$ on $[0,1]$, so
$\sup_{[0,1]}\abs{\widetilde T_N}=2^{-(2N-1)}$ (since $\abs{T_N}\le1$ on $[-1,1]$
with equality at its extremal points).

\begin{proposition}[Constant maps]\label{prop:binomial}
Consider the setting of \cref{prop:chebradius} with the convex class $\Gconst$
of priors on constant maps: $g\equiv\theta$ with $\theta\sim H$, $H$ a Borel
probability measure on $[0,1]$ (we identify the prior with $H$). Let
$m_k(H)\defeq\int\theta^k\,dH$ and
$b_s(H)\defeq\int\theta^s(1-\theta)^{n-s}\,dH(\theta)$ for $s=0,\dots,n$. Let
$\rho_n(H)\defeq m^+_{n+1}(H)-m^-_{n+1}(H)$, where $m^\pm_{n+1}(H)$ are the
supremum and infimum of $m_{n+1}(H')$ over probability measures $H'$ on $[0,1]$
with $m_k(H')=m_k(H)$ for $k=1,\dots,n$. Then for every $\mu$ and every $n\ge1$:
\begin{enumerate}[label=(\alph*)]
\item $\displaystyle \Reg_n(H;\Gconst)=\Bigl(\frac{\rho_n(H)}2\Bigr)^2\sum_{s=0}^n\frac{\binom ns}{b_s(H)}$,
  with the convention that the right side is $0$ when $\rho_n(H)=0$ (and
  $0\cdot(+\infty)=0$); $W\defeq\sum_s\binom ns/b_s(H)=+\infty$ forces
  $\rho_n(H)=0$, so $0\cdot\infty$ is the only indeterminate form that occurs.
\item $\rho_n(H)\le4^{-n}$ for every $H$.
\item For the arcsine law
  $H_{\mathrm{arc}}(d\theta)=d\theta/\bigl(\pi\sqrt{\theta(1-\theta)}\bigr)$ on
  $[0,1]$: $\rho_n(H_{\mathrm{arc}})=4^{-n}$ and
  \[
    \Reg_n(H_{\mathrm{arc}};\Gconst)=\frac1{8\binom{2n}n}
    =(1+o(1))\,\frac{\sqrt{\pi n}}{8\cdot4^n};
  \]
  at $n=1$, $\Reg_1(H_{\mathrm{arc}};\Gconst)=1/16=R^*_1(H_{\mathrm{arc}})$.
\end{enumerate}
\end{proposition}

Here $H_{\mathrm{arc}}=\omega$, the arcsine law of \cref{sec:eb}. Part (b) and the value
$\rho_n(H_{\mathrm{arc}})=4^{-n}$ are classical. For $H$
not finitely supported, $\rho_n(H)=\prod_{i=1}^n p_i(1-p_i)$, where the $p_i$
are the canonical moments of $H$ \citep{skibinsky1967range}; see
\citet{dette1997canonical} for the theory of canonical moments. Since
$p_i(1-p_i)\le1/4$, this gives (b), and the arcsine law, all of whose canonical
moments equal $1/2$, attains it. Part (a) also has classical roots: the quadratic structure of the
Bayes risk in the $(n+1)$-th moment and the value $1/16$ at $n=1$ go back to
\citet[eq.~(8) and Section~5]{skibinsky1968minimax} and \citet[Section~5]{robbins1964empirical}, and the
regret formula (a) is the midpoint case of the bound in \citet[eq.~(24)]{skibinsky1968minimax}
(\cref{rem:skibinsky}). The proof below is self-contained: (b) comes
from the extremal property of $\widetilde T_{n+1}$, and the arcsine value from
splitting the Gauss--Lobatto--Chebyshev quadrature into its even and odd nodes.

\begin{proof}
We use standard facts: the Beta integral, the formula for
$\Gamma(k+\tfrac12)$, Stirling's asymptotics for $\binom{2n}n$, and, for the
Chebyshev polynomials, $T_N(\cos\varphi)=\cos(N\varphi)$, $\abs{T_N}\le1$ on
$[-1,1]$, and leading coefficient $2^{N-1}$ (which makes $\widetilde T_N$ monic).

\emph{Step 1 (the identified set).} $\Gconst$ is convex: a mixture of laws of
constant random maps is the law of a constant random map whose value has the
mixed law. The map $g(\varpi,t)=\theta(\varpi)$, with $\theta$ independent of
$X$, is jointly measurable and fits the product construction (F1)--(F3) of the
setting. For $g\equiv\theta$,
$\prod_k g(x_k)^{y_k}(1-g(x_k))^{1-y_k}=\theta^{s(y)}(1-\theta)^{n-s(y)}$, so
$Z_H(x,y)=b_{s(y)}(H)$ does not depend on $x$. Under every prior
$X\sim\mu^{\otimes n}$, so $\mathrm{Law}_{H'}(D)=\mathrm{Law}_H(D)$ if and only if
$b_s(H')=b_s(H)$ for all $s=0,\dots,n$: the density is constant in $x$,
$\mu^{\otimes n}$ is a probability measure, and every $s$ is realized by some
$y$. The polynomials $\theta^s(1-\theta)^{n-s}$, $s=0,\dots,n$, form a basis of
the polynomials of degree at most $n$ (the Bernstein basis up to binomial
factors). Hence $b_s(H')=b_s(H)$ for all $s$ if and only if
$m_k(H')=m_k(H)$ for $k=0,\dots,n$ ($m_0=1$ for probability measures). So $[H]$,
taken within $\Gconst$, is exactly the set of probability measures $H'$ with
$m_k(H')=m_k(H)$ for $k=1,\dots,n$.

\emph{Step 2 (the unidentified set).} For $H'\in[H]$,
$M^{H'}_{n+1}(p,x)=\E_{H'}\theta^{n+1}=m_{n+1}(H')$ for all $(p,x)$, so
$U_{H'}$ is the constant function $m_{n+1}(H')-m_{n+1}(H)$. The set
$\mathcal I\defeq\{m_{n+1}(H')-m_{n+1}(H):H'\in[H]\}$ is convex, as the image of
the convex set $[H]$ under an affine functional, with supremum and infimum
$m^\pm_{n+1}(H)-m_{n+1}(H)$. So its closure is an interval $[u_-,u_+]$ with
$u_+-u_-=\rho_n(H)$, and $\mathcal U[H]=\{\text{constant functions }u:u\in\mathcal I\}$.

\emph{Step 3 (the weight).}
$w_H(x)=\sum_y1/b_{s(y)}(H)=\sum_{s=0}^n\binom ns/b_s(H)\eqqcolon W\in(0,+\infty]$,
constant in $x$.

\emph{Step 4 (part (a) when $W<\infty$).} By \cref{prop:chebradius}(iii),
\[
  \Reg_n(H;\Gconst)=\inf_c\sup_{u\in\mathcal I}W\int\!\!\int(c-u)^2\,d\mu\,d\mu^{\otimes n},
\]
the infimum running over measurable $c$ with $\int\!\!\int Wc^2<\infty$ (so $c$ is
square integrable). \emph{Lower bound.} Fix such a $c$ and let
$\bar c\defeq\int\!\!\int c\,d\mu\,d\mu^{\otimes n}$. For each constant $u$, the
variance decomposition gives $\int\!\!\int(c-u)^2\ge(\bar c-u)^2$. If
$\rho_n(H)=0$ the lower bound below is trivial. Otherwise, for every
$\eps\in(0,\rho_n(H)/2)$ there are $u_1,u_2\in\mathcal I$ with $u_1<u_-+\eps$
and $u_2>u_+-\eps$, and
$\max\{\abs{\bar c-u_1},\abs{\bar c-u_2}\}\ge(u_2-u_1)/2>\rho_n(H)/2-\eps>0$.
Letting $\eps\downarrow0$,
$\sup_{u\in\mathcal I}W\int\!\!\int(c-u)^2\ge W(\rho_n(H)/2)^2$.
\emph{Upper bound.} The constant $c\equiv(u_-+u_+)/2$ gives
$\sup_{u\in\mathcal I}W\bigl((u_-+u_+)/2-u\bigr)^2\le W(\rho_n(H)/2)^2$. Hence
$\Reg_n(H;\Gconst)=W\rho_n(H)^2/4$, which is (a).

\emph{Step 5 (part (a) when $W=\infty$).} Then $b_s(H)=0$ for some $s$, so
$\theta^s(1-\theta)^{n-s}=0$ $H$-a.s. If $s=0$ then $H=\delta_1$; if $s=n$ then
$H=\delta_0$; if $0<s<n$ then $H$ is supported in $\{0,1\}$. Every $H'\in[H]$
has $b_s(H')=b_s(H)=0$, so $H'$ is supported in the same set; a probability
measure on $\{0,1\}$ is determined by its first moment, and $m_1(H')=m_1(H)$,
so $H'=H$. Hence $[H]=\{H\}$, $\rho_n(H)=0$ and $\mathcal U[H]=\{0\}$, and by
\cref{prop:chebradius}(iii) (admissible $c$ vanish a.e.\ where $w=\infty$, and
$0\cdot\infty=0$) $\Reg_n(H;\Gconst)=0$, which is the right side of (a) by the
stated convention.

\emph{Step 6 (part (b)).} Let $H^+,H^-\in[H]$. The polynomial
$t^{n+1}-\widetilde T_{n+1}(t)$ has degree at most $n$, so its integrals under
$H^+$ and $H^-$ coincide, and
\[
  m_{n+1}(H^+)-m_{n+1}(H^-)=\int\widetilde T_{n+1}\,d(H^+-H^-)
  \le2\sup_{[0,1]}\abs{\widetilde T_{n+1}}=2\cdot2^{-(2n+1)}=4^{-n}.
\]
Taking the supremum over $H^+$ and the infimum over $H^-$ gives
$\rho_n(H)\le4^{-n}$.

\emph{Step 7 (part (c): $\rho_n(H_{\mathrm{arc}})=4^{-n}$).} Put $N\defeq n+1$,
nodes $\xi_j\defeq\cos(j\pi/N)$ for $j=0,\dots,N$, and Lobatto weights
$\lambda_0=\lambda_N\defeq1/(2N)$, $\lambda_j\defeq1/N$ for $1\le j\le N-1$.

\emph{Fact 1.} For integers $0\le k\le2N-1$,
$\sum_{j=0}^N\lambda_j\cos(kj\pi/N)$ equals $1$ if $k=0$ and $0$ if
$1\le k\le2N-1$. Indeed, for $k=0$ the weights sum to $2/(2N)+(N-1)/N=1$. For
$1\le k\le2N-1$, $\cos(kj\pi/N)=\cos(k(2N-j)\pi/N)$, so
$\sum_{j=0}^N\lambda_j\cos(kj\pi/N)=\frac1{2N}\sum_{j=0}^{2N-1}\cos(kj\pi/N)
=\frac1{2N}\operatorname{Re}\sum_{j=0}^{2N-1}z^j$ with $z\defeq e^{ik\pi/N}$;
since $z^{2N}=1$ and $z\ne1$ ($k$ is not a multiple of $2N$), the geometric sum
vanishes.

\emph{Fact 2.}
$\int_{-1}^1T_k(s)\,\frac{ds}{\pi\sqrt{1-s^2}}=\frac1\pi\int_0^\pi\cos(k\varphi)\,d\varphi$,
which is $1$ if $k=0$ and $0$ if $k\ge1$.

Since $T_k(\xi_j)=\cos(kj\pi/N)$, the probability measure
$L\defeq\sum_j\lambda_j\delta_{\xi_j}$ and the arcsine law $A$ on $[-1,1]$
integrate $T_0,\dots,T_{2N-1}$ identically (Facts~1 and~2), hence all
polynomials of degree at most $2N-1=2n+1$, and in particular those of degree at
most $n$. Let $\nu\defeq\sum_j(-1)^j\lambda_j\delta_{\xi_j}$. Since
$\sin(j\pi)=0$, $(-1)^j\cos(kj\pi/N)=\cos((N+k)j\pi/N)$, so
$\int T_k\,d\nu=\sum_j\lambda_j\cos((N+k)j\pi/N)$; this is $0$ for
$0\le k\le N-1$ (then $1\le N+k\le2N-1$, Fact~1), and equals
$\sum_j\lambda_j=1$ for $k=N$. Write $\nu=\nu^+-\nu^-$ with
$\nu^+\defeq\sum_{j\text{ even}}\lambda_j\delta_{\xi_j}$ and
$\nu^-\defeq\sum_{j\text{ odd}}\lambda_j\delta_{\xi_j}$. Then $\nu^++\nu^-=L$
has mass $1$ and $\nu^+-\nu^-=\nu$ has mass $\int T_0\,d\nu=0$, so each has mass
$1/2$. Hence $K^+\defeq2\nu^+=L+\nu$ and $K^-\defeq2\nu^-=L-\nu$ are probability
measures on $[-1,1]$ with $\int T_k\,dK^\pm=\int T_k\,dL=\int T_k\,dA$ for
$k\le n$, and $\int T_N\,d(K^+-K^-)=2\int T_N\,d\nu=2$.

Push forward by $t=(1+s)/2$: $A$ goes to $H_{\mathrm{arc}}$, and $K^\pm$ go to
probability measures $H^\pm$ on $[0,1]$. Polynomials of degree at most $n$ in
$t$ are polynomials of degree at most $n$ in $s$, so
$m_k(H^\pm)=m_k(H_{\mathrm{arc}})$ for $k\le n$, i.e.\
$H^\pm\in[H_{\mathrm{arc}}]$. As in Step~6,
\begin{align*}
  m_{n+1}(H^+)-m_{n+1}(H^-)&=\int\widetilde T_N\,d(H^+-H^-)\\
  &=2^{-(2n+1)}\int T_N\,d(K^+-K^-)=2^{-(2n+1)}\cdot2=4^{-n}.
\end{align*}
Hence $\rho_n(H_{\mathrm{arc}})\ge4^{-n}$, and with (b), equality holds.

\emph{Step 8 (part (c): the weight and the value).} By the Beta integral,
$b_s(H_{\mathrm{arc}})=\frac1\pi\int_0^1\theta^{s-1/2}(1-\theta)^{n-s-1/2}\,d\theta
=\frac{\Gamma(s+\frac12)\Gamma(n-s+\frac12)}{\pi\,\Gamma(n+1)}$, and
$\Gamma(k+\tfrac12)=(2k)!\sqrt\pi/(4^kk!)$, so
$b_s(H_{\mathrm{arc}})=\frac{(2s)!\,(2n-2s)!}{4^n\,s!\,(n-s)!\,n!}>0$. Therefore
\[
  \sum_{s=0}^n\frac{\binom ns}{b_s}
  =\sum_{s=0}^n\frac{n!}{s!\,(n-s)!}\cdot\frac{4^n\,s!\,(n-s)!\,n!}{(2s)!\,(2n-2s)!}
  =\frac{4^n(n!)^2}{(2n)!}\sum_{s=0}^n\binom{2n}{2s}
  =\frac{4^n}{\binom{2n}n}\,2^{2n-1},
\]
where, for $n\ge1$, $\sum_s\binom{2n}{2s}=2^{2n-1}$ (add the expansions of
$(1+1)^{2n}$ and $(1-1)^{2n}$). By (a) and Step~7,
\[
  \Reg_n(H_{\mathrm{arc}};\Gconst)=\Bigl(\frac{4^{-n}}2\Bigr)^2\cdot\frac{4^n\,2^{2n-1}}{\binom{2n}n}
  =\frac1{8\binom{2n}n}.
\]
Since $\binom{2n}n=(1+o(1))\,4^n/\sqrt{\pi n}$ (Stirling),
$\Reg_n(H_{\mathrm{arc}};\Gconst)=(1+o(1))\sqrt{\pi n}/(8\cdot4^n)$. At $n=1$,
$\Reg_1=1/16$. Under $H_{\mathrm{arc}}=\mathrm{Beta}(\frac12,\frac12)$ the
posterior of $\theta$ given one label $y$ is
$\mathrm{Beta}(\frac12+y,\frac32-y)$, whose variance
$(\frac12+y)(\frac32-y)/(2^2\cdot3)$ equals $3/48=1/16$ for $y=0$ and for $y=1$.
The posterior mean is the Bayes rule (Step~1 of the proof of
\cref{prop:chebradius}), so $R^*_1(H_{\mathrm{arc}})$ is the expected posterior
variance, $1/16$.
\end{proof}

\subsection{Diagonal domination for sorted vectors}
\label{app:ch1}

The upper bound for monotone maps rests on one elementary inequality: matching
two sorted lists index by index never beats evaluating the polynomial with the
second list as roots at the best point of the first. With the Chebyshev zeros
as roots, this turns a multivariate product into a univariate sup-norm.

\LemCHone*

\begin{proof}
\emph{Part (a).} We induct on $n$. For $n=0$ both sides equal
$\abs{x_0-y_0}$. Let $n\ge1$ and assume (a) for every pair of nondecreasing
sequences of length $n$ (indices $0,\dots,n-1$). Put
$\bar P(t)\defeq\prod_{i=0}^{n-1}(t-y_i)$, so $P(t)=\bar P(t)(t-y_n)$. The
truncated sequences $(x_0,\dots,x_{n-1})$ and $(y_0,\dots,y_{n-1})$ are
nondecreasing, so by the induction hypothesis there is $k\le n-1$ with
$\prod_{i=0}^{n-1}\abs{x_i-y_i}\le\abs{\bar P(x_k)}$. Hence
\begin{equation}\label{eq:ch1-one}
  \Pi\defeq\prod_{i=0}^n\abs{x_i-y_i}\le\abs{\bar P(x_k)}\cdot\abs{x_n-y_n}.
\end{equation}
\emph{Case A: $\abs{x_n-y_n}\le\abs{x_k-y_n}$.} Then by~\eqref{eq:ch1-one},
$\Pi\le\abs{\bar P(x_k)}\,\abs{x_k-y_n}=\abs{P(x_k)}$.

\emph{Case B: $\abs{x_n-y_n}>\abs{x_k-y_n}$.} Since $k\le n-1$ and $x$ is
nondecreasing, $x_k\le x_n$; if $x_k=x_n$ then $\abs{x_n-y_n}=\abs{x_k-y_n}$,
contradicting Case~B. So $x_k<x_n$. Put $\bar x\defeq(x_k+x_n)/2$. For real
$v$, $(x_n-v)^2-(x_k-v)^2=(x_n-x_k)(x_n+x_k-2v)$, which is positive if and only
if $v<\bar x$ (because $x_n-x_k>0$); so $\abs{x_n-v}>\abs{x_k-v}$ if and only if
$v<\bar x$. Hence $y_n<\bar x$, and since the $y$'s are nondecreasing,
$y_i\le y_n<\bar x$ for every $i\le n-1$. Fix such an $i$. As
$y_i<\bar x<x_n$, $\abs{x_n-y_i}=x_n-y_i$. If $y_i\le x_k$, then
$\abs{x_k-y_i}=x_k-y_i\le x_n-y_i$. If $x_k<y_i<\bar x$, then
$\abs{x_k-y_i}=y_i-x_k<\bar x-x_k=x_n-\bar x<x_n-y_i$. In both sub-cases
$\abs{x_k-y_i}\le\abs{x_n-y_i}$. Multiplying over $i=0,\dots,n-1$ gives
$\abs{\bar P(x_k)}\le\abs{\bar P(x_n)}$, and with~\eqref{eq:ch1-one},
$\Pi\le\abs{\bar P(x_n)}\,\abs{x_n-y_n}=\abs{P(x_n)}$.

In both cases $\Pi\le\max_k\abs{P(x_k)}$, which proves (a) for $n$.

\emph{Part (b).} We use standard Chebyshev facts, which follow from
$T_m(\cos\varphi)=\cos(m\varphi)$ and the recurrence
$T_{m+1}(s)=2sT_m(s)-T_{m-1}(s)$, $T_0=1$, $T_1=s$: for $m\ge1$, $T_m$ has
degree $m$ and leading coefficient $2^{m-1}$; $\abs{T_m}\le1$ on $[-1,1]$, with
equality at $s=\pm1$; and its zeros are $\cos((2i+1)\pi/(2m))$,
$i=0,\dots,m-1$. Hence $T_{n+1}$ has leading coefficient $2^n$, and the
substitution $s=2t-1$ multiplies the leading coefficient by $2^{n+1}$; since
$2^{-(2n+1)}\cdot2^n\cdot2^{n+1}=1$, $\widetilde T_{n+1}$ is monic of degree
$n+1$. Its $n+1$ zeros $\bigl(1+\cos((2i+1)\pi/(2n+2))\bigr)/2$ are real,
distinct and lie in $(0,1)$; listed increasingly they are the $y_i$ of the
statement (the map $i\mapsto n-i$ reverses the order, since
$\cos(\pi-a)=-\cos a$). So $\widetilde T_{n+1}=P$. Moreover
$\sup_{t\in[0,1]}\abs{\widetilde T_{n+1}(t)}=2^{-(2n+1)}\sup_{s\in[-1,1]}\abs{T_{n+1}(s)}=2^{-(2n+1)}$.
Each $x_k$ lies in $[0,1]$, so $\abs{P(x_k)}\le2^{-(2n+1)}$, and (a) gives (b).
\end{proof}

\begin{remark}[Sharpness]\label{rem:ch1-sharp}
Sortedness of both sequences is needed in (a): for $x=(1,0)$ and $y=(0,1)$ the
left side is $1$, while $P(t)=t(t-1)$ vanishes at both $x_k$; and, by symmetry of
$P$ in $y$, $x=(0,1)$ and $y=(1,0)$ show the same for $y$. The constant in
(b) is attained at $x\equiv0$ and at $x\equiv1$, and it cannot be improved by
any other choice of the $y_i$: on the diagonal $x_0=\dots=x_n=x$ the left side
is $\abs{\prod_i(x-y_i)}$, a monic polynomial of degree $n+1$, whose sup-norm
on $[0,1]$ is at least $2^{-(2n+1)}$ by the minimality of the Chebyshev
polynomial. Equivalently, $\prod_i(x_i-y_i)$ is a polynomial of degree $n+1$
whose top-degree coefficients sum to one, and $2^{-(2n+1)}$ is the least
sup-norm of such polynomials on the ordered simplex
$\{0\le x_0\le\dots\le x_n\le1\}$ \citep[Theorem~3.1 and
Example~3.3]{dressler2026least}; part (b) shows that the multi-affine product
with matched Chebyshev roots attains this value.
\end{remark}

%% file: sections/9b3_monotone.tex
\subsection{The upper bound for monotone maps}
\label{app:monotone-eb}

The mechanism is this. Fix an $(n+1)$-tuple of design points and sort it. For a
nondecreasing path, the vector of values $(\theta_0,\dots,\theta_n)$ at the
sorted points lies in the ordered simplex. Every lower-order product moment is
identified (Step~1). So the difference between the top moments of two
observationally equivalent priors equals the difference between the
expectations of \emph{any} multi-affine polynomial in $\theta$ with top term
$\prod_i\theta_i$. We choose the one whose roots are the Chebyshev zeros,
matched index by index to the sorted coordinates.
\Cref{lem:ch1} bounds it by $2^{-(2n+1)}$ on the whole simplex. The bound on the
unidentified moment therefore depends neither on $\mu$ nor on the prior.
\Cref{prop:chebradius}(i)--(ii) turn it into a regret bound, with one factor
$4^{-n}$ for each of the $2^n$ label patterns.

\ThmMonotoneEB*

\begin{proof}
Throughout, ``a.e.''\ refers to $\mu^{\otimes(n+1)}$ on tuples
$t=(t_0,t_1,\dots,t_n)\defeq(p,x_1,\dots,x_n)$. For a prior $G$ and $r\ge1$,
$M^G_r(s_1,\dots,s_r)\defeq\E_G[g(s_1)\cdots g(s_r)]$ is defined for every
tuple, coinciding coordinates allowed. The identified set depends only on the
law of $D$: if $G''\in[G]$, then $[G'']=[G]$, since both are the set of priors
in $\Gmono$ whose law of $D$ equals $\mathrm{Law}_G(D)$.

\emph{Step 1 (identified product moments, every Borel $\mu$).} Let
$G',G''\in[G]$. Then $\mathrm{Law}_{G'}(D)=\mathrm{Law}_{G''}(D)$, and both laws
have densities $Z_{G'},Z_{G''}$ with respect to the same $\sigma$-finite measure
$\mu^{\otimes n}\otimes(\text{counting})$ on $[0,1]^n\times\{0,1\}^n$. By
uniqueness of Radon--Nikodym derivatives, $Z_{G'}(x,y)=Z_{G''}(x,y)$ for every
$y$ and $\mu^{\otimes n}$-a.e.\ $x$. For every prior, every $x\in[0,1]^n$ (no
distinctness needed) and every $S=\{k_1<\dots<k_r\}\subseteq\{1,\dots,n\}$,
\[
  M_r(x_{k_1},\dots,x_{k_r})=\sum_{y:\,y_k=1\text{ for }k\in S}Z(x,y),
\]
as one sees by expanding $\prod_{k\notin S}\bigl(g(x_k)+(1-g(x_k))\bigr)=1$
inside the expectation. Taking $S=\{1,\dots,r\}$ with $r\le n$ gives
$M^{G'}_r(s)=M^{G''}_r(s)$ for $\mu^{\otimes r}$-a.e.\ $s\in[0,1]^r$: the set of
$x$ where the $Z$'s differ is $\mu^{\otimes n}$-null, and the identity reads
$M_r(x_1,\dots,x_r)$ on its left side. Now let $S$ be a proper subset of
$\{0,\dots,n\}$ with $\abs S=r\le n$, and list $t_S$ in increasing index order.
The set of $t\in[0,1]^{n+1}$ with $M^{G'}_r(t_S)\ne M^{G''}_r(t_S)$ is the
preimage of a $\mu^{\otimes r}$-null set under the coordinate projection
$t\mapsto t_S$, which pushes $\mu^{\otimes(n+1)}$ forward to $\mu^{\otimes r}$
(distinct \emph{indices} are selected; values may coincide). Hence it is
$\mu^{\otimes(n+1)}$-null. These sets are measurable because
$s\mapsto M^G_r(s)$ is Borel, by Tonelli and the joint measurability of the
priors. Let $\mathsf N$ be the union of these finitely many null sets (at most
$2^{n+1}-2$ of them; $S=\emptyset$ contributes nothing). Since $M_r$ is a
symmetric function of its arguments, for $t\notin\mathsf N$,
\[
  \E_{G'}\prod_{k\in S'}g(t_k)=\E_{G''}\prod_{k\in S'}g(t_k)
  \qquad\text{for every proper subset }S'\subsetneq\{0,\dots,n\}.
\]

\emph{Step 2 (part (a)).} Fix $t\notin\mathsf N$. Let $\psi$ be the permutation
of $\{0,\dots,n\}$ that sorts the coordinates nondecreasingly, ties broken by
index, so that $t_{\psi(0)}\le\dots\le t_{\psi(n)}$. For a map $g$ put
$\theta_i\defeq g(t_{\psi(i)})$. For nondecreasing $g$,
$0\le\theta_0\le\theta_1\le\dots\le\theta_n\le1$ (equal coordinates of $t$
give equal $\theta$'s). Let $a_0\le\dots\le a_n$ be the zeros of
$\widetilde T_{n+1}$, as in \cref{lem:ch1}(b). Expand
\[
  \prod_{i=0}^n(\theta_i-a_i)=\prod_{i=0}^n\theta_i+q(\theta),
\]
where $q$ is a linear combination, with coefficients depending only on $a$, of
the monomials $\prod_{i\in S}\theta_i=\prod_{k\in\psi(S)}g(t_k)$ over proper
subsets $S$ of $\{0,\dots,n\}$ ($S=\emptyset$ gives $1$). Each $\psi(S)$ is a
proper subset of $\{0,\dots,n\}$, so by Step~1,
$\E_{G'}q(\theta)=\E_{G''}q(\theta)$. Since
$\prod_i\theta_i=g(t_0)\cdots g(t_n)$,
\[
  M^{G'}_{n+1}(t)-M^{G''}_{n+1}(t)=\E_{G'}Q(\theta)-\E_{G''}Q(\theta),
  \qquad Q(\theta)\defeq\prod_{i=0}^n(\theta_i-a_i).
\]
By \cref{lem:ch1}(b), $\abs{Q(\theta)}\le2^{-(2n+1)}$ for every nondecreasing
$g$, and $G'$-a.s.\ and $G''$-a.s.\ the path is nondecreasing. So the difference
is at most $2\cdot2^{-(2n+1)}=4^{-n}$ in absolute value. The particular case
$G''=G$ is $\abs{U_{G'}}\le4^{-n}$ a.e.

\emph{Step 3 (part (b)).} Let $G'\in[G]$. By \cref{prop:chebradius}(i),
$R(\delta^*_G,G')-R^*_n(G')=\E_{D\sim G}\norm{\delta^*_G(D)-\delta^*_{G'}(D)}^2_{L^2(\mu)}$.
We first record which measure each a.e.\ statement lives in. The integral is
against
$\mu\otimes\mathrm{Law}_G(D)=Z_G\cdot\bigl(\mu^{\otimes(n+1)}\otimes(\text{counting})\bigr)$,
and $\{Z_G=0\}$ is $\mathrm{Law}_G(D)$-null. Since $Z_G\le1$, every
$\mu^{\otimes(n+1)}$-null set of $(p,x)$ is
$\mu\otimes\mathrm{Law}_G(D)$-null, so the a.e.\ statements of
\cref{prop:chebradius}(ii) and of part (a) combine inside the integral.
Conversely, on $\{Z_G>0\}$ a set that is null for
$Z_G\cdot(\mu^{\otimes(n+1)}\otimes\text{counting})$ is null for
$\mu^{\otimes(n+1)}\otimes\text{counting}$, because the density is positive
there. So the bound $\abs{U_{G'}}/Z_G\le1$ below holds
$\mu^{\otimes(n+1)}\otimes\text{counting}$-a.e.\ on $\{Z_G>0\}$.

By \cref{prop:chebradius}(ii), for $\mathrm{Law}_G(D)$-a.e.\ $D$ with
$Z_G(D)>0$ and $\mu$-a.e.\ $p$,
$\abs{\delta^*_{G'}(D)(p)-\delta^*_G(D)(p)}=\abs{U_{G'}(p,X)}/Z_G(X,Y)$, and
this is at most $1$ because both posterior means lie in $[0,1]$. Hence
\begin{align*}
  \E_{D\sim G}\norm{\delta^*_G(D)-\delta^*_{G'}(D)}^2
  &=\int\!\!\int\sum_{y:\,Z_G(x,y)>0}\frac{U_{G'}(p,x)^2}{Z_G(x,y)}\,d\mu(p)\,d\mu^{\otimes n}(x)\\
  &\le\int\!\!\int\sum_{y:\,Z_G(x,y)>0}\abs{U_{G'}(p,x)}\cdot\frac{\abs{U_{G'}(p,x)}}{Z_G(x,y)}\,d\mu(p)\,d\mu^{\otimes n}(x)\\
  &\le\int\!\!\int2^n\cdot4^{-n}\,d\mu(p)\,d\mu^{\otimes n}(x)=2^{-n},
\end{align*}
using $\abs{U_{G'}}\le4^{-n}$ from part (a), $\abs{U_{G'}}/Z_G\le1$, and the
fact that there are $2^n$ label patterns $y$. Taking $\delta=\delta^*_G$ in the
definition of $\Reg_n$ gives $\Reg_n(G;\Gmono)\le2^{-n}$.

\emph{Step 4 (part (c)).} Let $G'',G'\in[G]$. Then $G''\in\Gmono$, and
$[G'']=[G]$ contains $G'$. Steps~1--3 used nothing about $G$ beyond
$G\in\Gmono$, so we may apply Step~3 with $G''$ in the role of $G$:
$R(\delta^*_{G''},G')-R^*_n(G')\le2^{-n}$.
\end{proof}

\begin{remark}[Monotonicity in a known direction; sharpness]\label{rem:monotone-needed}
Monotonicity in one \emph{known} direction cannot be dropped from (a). Take $n=1$ and $\mu$ uniform on
$[0,1]$; the law of $D$ depends only on $M_1$. The priors
$\frac12\delta_{\{g=\mathrm{id}\}}+\frac12\delta_{\{g=1-\mathrm{id}\}}$ and
$\frac12\delta_{\{g\equiv0\}}+\frac12\delta_{\{g\equiv1\}}$ both have $M_1\equiv1/2$, while their second
moments are $M_2(s,t)=\frac12(1-s-t+2st)$ and $M_2\equiv\frac12$. The gap $\frac12(s+t-2st)$ exceeds
$4^{-1}$ exactly where $(s-\frac12)(t-\frac12)<0$, a set of $\mu^{\otimes2}$-measure $\frac12$ (at
$(s,t)=(0,1)$ it equals $\frac12$). Conversely, (a) is attained at
$n=1$: $\delta_{\{g\equiv1/2\}}$ and $\frac12\delta_{\{g\equiv0\}}+\frac12\delta_{\{g\equiv1\}}$ both have
$M_1\equiv1/2$, and their second moments are $M_2\equiv1/4$ and $M_2\equiv1/2$, a gap of exactly $4^{-1}$ at
every $t$.
\end{remark}

%% file: sections/9b4_betastar.tex
\subsection{The exact base for constant maps}
\label{app:betastar}

By \cref{prop:binomial}(a), for a constant-map prior
$\Reg_n(H;\Gconst)=\tfrac14\rho_n(H)^2\sum_s\binom ns/b_s(H)$. This is the
product of a squared moment range and a sum of inverse label-pattern
probabilities. Both factors are governed by the logarithmic potential of the
arcsine law $\omega$. The lower half of part (iii) bounds $4^n\rho_n(H)$ below
by half the geometric mean of the density of $H$ with respect to $\omega$. The
upper half bounds it above by the ratio of the $H$-mean to the
$\omega$-geometric mean of any positive polynomial of degree at most $n$. The
upper bound (i) applies the upper half of (iii) once for each label count $s$,
with a polynomial that trades $\rho_n$ against $b_s$. The lower bound (ii) uses
a prior that starves $b_{\lfloor n/2\rfloor}$: on a set of arcsine mass $1/2$
where $t^s(1-t)^{n-s}$ is largest, its density is inversely proportional to
that weight. Catalan's constant enters as the arcsine average of
$\log(t(1-t))$ over the central half $[\sin^2(\pi/8),\cos^2(\pi/8)]$.

Part (iii) is known. The lower half, with constant $1$ in place of $\tfrac12$,
follows from the exact formula of \citet[Theorem~3.1(i)]{gamboa2004large}, and
the upper half follows from their Remark~3.1(i). In the language of orthogonal
polynomials on the unit circle it is a finite-$n$ form of Szeg\H{o}'s theorem
\citep[Theorem~8.1]{simon2005opuc}; see also \citet[eq.~(2.10)]{gamboa2016sum}.
The proof below is self-contained (Fej\'er--Riesz for the upper half, duality
and Jensen for the lower half), and parts (i) and (ii) apply its two halves
directly.

{\setlength{\emergencystretch}{12em}\ThmBetaStar*}%

\begin{proof}
\emph{Notation.} We write
$\operatorname{GM}(f)\defeq\exp\int\log f\,d\omega$ for the geometric mean
under $\omega$, $v_s(t)\defeq t^s(1-t)^{n-s}$ (so $b_s(H)=\int v_s\,dH$), and
$\mathcal P_n$ for the real polynomials of degree at most $n$. The substitution
$t=(1-\cos\vartheta)/2$ carries $d\vartheta/\pi$ on $[0,\pi]$ (or
$d\vartheta/(2\pi)$ on $[0,2\pi)$) to $\omega$. Let
$\mathcal C\defeq[a_*,1-a_*]$ with $a_*=\sin^2(\pi/8)$, i.e.\
$\vartheta\in[\pi/4,3\pi/4]$, so $\omega(\mathcal C)=1/2$. We abbreviate
$\rho\defeq\rho_n(H)$ and $\Reg\defeq\Reg_n(H;\Gconst)$.

\emph{Step 1 (two integrals).} (a) For $x\in[0,1]$ write
$x=(1-\cos\varphi)/2$. Then
$\abs{t-x}=\abs{\sin((\vartheta+\varphi)/2)}\,\abs{\sin((\vartheta-\varphi)/2)}$.
Substituting $\vartheta\to-\vartheta$ in the second factor,
\[
  \int\log\abs{t-x}\,d\omega(t)=\frac1\pi\int_{-\pi}^{\pi}\log\abs{\sin((\vartheta+\varphi)/2)}\,d\vartheta
  =\frac2\pi(-\pi\log2)=-\log4,
\]
using $\int\log\abs{\sin u}\,du=-\pi\log2$ over any interval of length $\pi$.
(b) Since $t(1-t)=\sin^2\vartheta/4$,
\[
  \int_{\mathcal C}\log(t(1-t))\,d\omega=\frac2\pi\int_{\pi/4}^{3\pi/4}\log\sin\vartheta\,d\vartheta-\log2
  =\frac{2\Gcat}\pi-2\log2 .
\]
For the last equality we use $\int_0^{\pi/4}\log\sin=-\frac\pi4\log2-\frac\Gcat2$.
Indeed, $\log\sin=\frac12[\log(\sin\cos)+\log\tan]$ on $(0,\pi/4]$. Next,
$\int_0^{\pi/4}\log\tan=-\Gcat$, because
$\int_0^{\pi/4}\log\tan\vartheta\,d\vartheta=\int_0^1\frac{\log u}{1+u^2}\,du
=-\sum_{k\ge0}\frac{(-1)^k}{(2k+1)^2}=-\Gcat$. And
$\int_0^{\pi/4}\log(\sin\cos)=\int_0^{\pi/4}\log(\sin(2\vartheta)/2)
=\frac12\int_0^{\pi/2}\log\sin u\,du-\frac\pi4\log2=-\frac\pi2\log2$. Then
$\int_{\pi/4}^{3\pi/4}\log\sin=2\bigl(\int_0^{\pi/2}-\int_0^{\pi/4}\bigr)\log\sin
=-\frac\pi2\log2+\Gcat$. By the symmetry $t\leftrightarrow1-t$,
$\int_{\mathcal C}\log t\,d\omega=\int_{\mathcal C}\log(1-t)\,d\omega=\Gcat/\pi-\log2$,
so
\begin{equation}\label{eq:betastar-central}
  \int_{\mathcal C}\log v_s\,d\omega=n\Bigl(\frac\Gcat\pi-\log2\Bigr)\qquad\text{for every }s.
\end{equation}

\emph{Step 2 (Lemma U: the upper half of (iii)).} \emph{Weak duality.} If
$Q',Q\in\mathcal P_n$ and $Q'\le t^{n+1}\le Q$ on $[0,1]$, then every $H'$ with
the same moments $0,\dots,n$ as $H$ satisfies
$\int Q'\,dH\le m_{n+1}(H')\le\int Q\,dH$. Hence $\rho\le\int(Q-Q')\,dH$.

\emph{Fej\'er--Riesz (proved inline).} Write
$\widetilde P(\vartheta)\defeq P((1-\cos\vartheta)/2)=\sum_{\abs k\le d}a_ke^{ik\vartheta}>0$
with $d=\deg P\le n$ and $a_{-k}=\overline{a_k}$; the $a_k$ are real since
$\widetilde P$ is even, and $a_{-d}=a_d\ne0$ (the case $d=0$ is trivial). The
polynomial $R(z)\defeq\sum_ka_kz^{k+d}$ has degree $2d$, $R(0)\ne0$,
$\overline{R(1/\bar z)}=z^{-2d}R(z)$, and no roots on $\abs z=1$. By the
functional equation its roots pair, with multiplicity, as $z_j,1/\bar z_j$ with
$\abs{z_j}>1$, and for $\abs z=1$, $\abs{z-1/\bar z_j}=\abs{z-z_j}/\abs{z_j}$.
Hence $\widetilde P=\abs{p(e^{i\vartheta})}^2$ with
$p(z)\defeq c\prod_j(z-z_j)$, $\deg p\le n$, $p$ zero-free on the closed unit
disk, and $c$ chosen so that $p(0)>0$. Since $\log\abs p$ is harmonic near the
closed disk, $\operatorname{GM}(P)=p(0)^2$.

\emph{The gap pair.} Write $p^2=\sum_{j\le2d}c_jz^j$ (so $c_0=p(0)^2$), and let
$J(\vartheta)\defeq\operatorname{Re}\bigl(e^{i(n+1)\vartheta}\,\overline{p(e^{i\vartheta})}^{\,2}\bigr)$
and
$J_e(\vartheta)\defeq[J(\vartheta)+J(-\vartheta)]/2=\sum_j\operatorname{Re}(c_j)\cos((n+1-j)\vartheta)$.
Then $\abs{J_e}\le\abs p^2=\widetilde P$. Only $j=0$ reaches frequency $n+1$, so
$J_e$ corresponds to a polynomial
$\mathcal J(t)=c_0T_{n+1}(1-2t)+(\text{degree}\le n)$, whose $t^{n+1}$
coefficient is $c_0(-1)^{n+1}2^{2n+1}$. With $\varsigma\defeq(-1)^{n+1}$ and
$K\defeq2\cdot4^n\operatorname{GM}(P)$, put
\[
  A\defeq(P+\varsigma\mathcal J)/K\ge0,\qquad B\defeq(P-\varsigma\mathcal J)/K\ge0 .
\]
$A$ is monic of degree $n+1$ and $B$ has leading coefficient $-1$, so
$Q'\defeq t^{n+1}-A$ and $Q\defeq t^{n+1}+B$ lie in $\mathcal P_n$ and satisfy
$Q'\le t^{n+1}\le Q$. Therefore
\begin{equation}\label{eq:lemma-U}
  \rho\le\int(A+B)\,dH=\frac2K\int P\,dH=4^{-n}\,\frac{\int P\,dH}{\operatorname{GM}(P)}.
\end{equation}

\emph{Step 3 (Lemma L: the lower half of (iii)).} Assume $H=h\omega$ with
$h\ge\underline h>0$.

(a) \emph{$\rho>0$.} The polynomial $T^*(t)\defeq T_{n+1}(1-2t)$ has
$\abs{T^*}\le1$, is $\omega$-orthogonal to $\mathcal P_n$, and satisfies
$\int t^{n+1}T^*\,d\omega=(-1)^{n+1}4^{-(n+1)}$, since its leading coefficient
is $(-1)^{n+1}2^{2n+1}$ and $\int(T^*)^2\,d\omega=1/2$. The measures
$(h\pm\underline hT^*)\omega$ are probability measures with the moments
$1,\dots,n$ of $H$, so $\rho\ge2\underline h\,4^{-(n+1)}>0$.

(b) \emph{Strong duality.} We use the supporting-hyperplane theorem and
compactness of the moment set; these are standard convex analysis. The moment
set $\mathcal M_{n+1}$ (the vectors $(m_1,\dots,m_{n+1})$ of moments of
probability measures on $[0,1]$) is convex and compact, and the extremes
$m^\pm$ of the $(n+1)$-th moment given $\mathbf m\defeq(m_1(H),\dots,m_n(H))$
are attained. At $x^+=(\mathbf m,m^+)$ a supporting hyperplane gives
$\sum_{j=1}^{n+1}\gamma_jx_j\le\gamma_0$ on $\mathcal M_{n+1}$, with equality at
$x^+$. Evaluating at $(\mathbf m,m^-)$ and using $\rho>0$ gives
$\gamma_{n+1}\ge0$. If $\gamma_{n+1}=0$, then
$\Phi\defeq\gamma_0-\sum_{j=1}^n\gamma_jt^j$ is a nonzero polynomial with
$\Phi\ge0$ on $[0,1]$ and $\int\Phi\,dH=0$, so $H$ would be supported on the
finitely many zeros of $\Phi$, which is impossible for $H=h\omega$. Normalizing
$\gamma_{n+1}=1$ yields $Q\in\mathcal P_n$ with $Q\ge t^{n+1}$ on $[0,1]$ and
$\int Q\,dH=m^+$. Likewise, the mirrored supporting hyperplane at
$x^-=(\mathbf m,m^-)$, where $\gamma_{n+1}<0$ by the same finite-support
argument, normalized to $\gamma_{n+1}=-1$, yields $Q'\le t^{n+1}$ with
$\int Q'\,dH=m^-$. Hence $\rho=\int(Q-Q')\,dH$.

(c) \emph{A potential bound.} Let $A\defeq t^{n+1}-Q'\ge0$ and
$B\defeq Q-t^{n+1}\ge0$; $A$ and $-B$ are monic of degree $n+1$. Put
$\Delta\defeq A+B\ge2\sqrt{AB}$. For each complex root $z$ of $A$ (resp.\ of
$-B$), let $\bar x$ be $\operatorname{Re}z$ clamped to $[0,1]$. Then
$\abs{t-z}\ge\abs{t-\bar x}$ on $[0,1]$, so by Step~1(a),
$\int\log\abs{t-z}\,d\omega\ge-\log4$. Since $A=\abs A=\prod_i\abs{t-z_i}$ over
its $n+1$ roots, and likewise for $B$,
$\operatorname{GM}(A),\operatorname{GM}(B)\ge4^{-(n+1)}$, so
$\operatorname{GM}(\Delta)\ge2\cdot4^{-(n+1)}=4^{-n}/2$.

(d) By Jensen's inequality,
\begin{equation}\label{eq:lemma-L}
  \rho=\int\Delta h\,d\omega\ge\operatorname{GM}(\Delta h)
  =\operatorname{GM}(\Delta)\operatorname{GM}(h)\ge\tfrac12\,4^{-n}\operatorname{GM}(h).
\end{equation}
Steps~2 and~3 prove (iii).

\emph{Step 4 (part (i)).} If $\rho=0$ then $\Reg=0$ by
\cref{prop:binomial}(a). If $b_s(H)=0$ for some $s$, apply~\eqref{eq:lemma-U}
with $P=1+v_s/\eps$. Since $\log(1+y)\ge\mathbf 1_{\mathcal C}\log y$ pointwise
and $\omega(\mathcal C)=1/2$,
\[
  \rho\le4^{-n}\,\frac{1+b_s/\eps}{\operatorname{GM}(1+v_s/\eps)}
  \le4^{-n}\,\eps^{1/2}\,e^{-\int_{\mathcal C}\log v_s\,d\omega}\longrightarrow0\qquad(\eps\to0),
\]
so $\rho=0$. Otherwise all $b_s\defeq b_s(H)>0$. For each $s$ take
$Q_s\defeq1+v_s/b_s>0$ (degree $n$), so $\int Q_s\,dH=2$. Pointwise
$\log(1+y)\ge\mathbf 1_{\mathcal C}\log y$ for $y>0$, so
$\operatorname{GM}(Q_s)\ge\exp\bigl(\int_{\mathcal C}\log v_s\,d\omega-\omega(\mathcal C)\log b_s\bigr)
=b_s^{-1/2}e^{\int_{\mathcal C}\log v_s\,d\omega}$. Then~\eqref{eq:lemma-U}
and~\eqref{eq:betastar-central} give
\[
  \rho^2\le4\cdot16^{-n}\,b_s\,e^{-2\int_{\mathcal C}\log v_s\,d\omega}=4\cdot4^{-n}\,b_s\,e^{-2\Gcat n/\pi}.
\]
By \cref{prop:binomial}(a),
\[
  \Reg=\frac14\rho^2\sum_s\frac{\binom ns}{b_s}\le\sum_s\binom ns4^{-n}e^{-2\Gcat n/\pi}
  =2^n4^{-n}e^{-2\Gcat n/\pi}=\bstar^{-n}.
\]

\emph{Step 5 (a general lower-bound prior).} Fix $s$. Choose a level $\ell_s$
such that $\mathcal C_s\defeq\{v_s\ge\ell_s\}$ has $\omega$-mass $1/2$; this is
possible because $\omega$ has no atoms and each level set of $v_s$ is finite.
Put $h_s\defeq1$ off $\mathcal C_s$ and $h_s\defeq\kappa_s/v_s$ on
$\mathcal C_s$, where
$\kappa_s\defeq\bigl(2\int_{\mathcal C_s}v_s^{-1}\,d\omega\bigr)^{-1}$. Then
$H_s\defeq h_s\omega$ is a probability measure
($\int h_s\,d\omega=\frac12+\frac12$), and $h_s\ge\min(1,\kappa_s)>0$. On
$\mathcal C_s$, $v_s\ge\ell_s$, so $\int_{\mathcal C_s}v_s^{-1}\,d\omega\le1/(2\ell_s)$,
i.e.\ $\kappa_s\ge\ell_s$. Hence
$b_s(H_s)=\int_{\mathcal C_s^c}v_s\,d\omega+\kappa_s/2\le\ell_s/2+\kappa_s/2\le\kappa_s$.
Also $\operatorname{GM}(h_s)^2=\kappa_s\,e^{-2\int_{\mathcal C_s}\log v_s\,d\omega}$.
By~\eqref{eq:lemma-L} and \cref{prop:binomial}(a), keeping only the $s$-th term
of the sum,
\[
  \Reg_n(H_s;\Gconst)\ge\frac14\rho^2\,\frac{\binom ns}{b_s}
  \ge\frac14\cdot\frac14\,16^{-n}\operatorname{GM}(h_s)^2\,\frac{\binom ns}{\kappa_s}
  =\frac1{16}\,16^{-n}\binom ns\,e^{-2\int_{\mathcal C_s}\log v_s\,d\omega}.
\]

\emph{Step 6 (even $n=2k$, $s=k$).} Here $v_k=(t(1-t))^k$, whose superlevel
set of $\omega$-mass $1/2$ is $\mathcal C$. By~\eqref{eq:betastar-central},
$e^{-2\int_{\mathcal C}\log v_k\,d\omega}=4^ne^{-2\Gcat n/\pi}$, so the prior
$H_k\defeq H_s$ with $s=k$ has
\[
  \Reg_n(H_k;\Gconst)\ge\frac1{16}\binom{2k}k4^{-n}e^{-2\Gcat n/\pi}
  =\frac1{16}\,\frac{\binom{2k}k}{4^k}\,\bstar^{-n}.
\]

\emph{Step 7 (odd $n=2k+1$, $s=k$).} Here $v_k=(t(1-t))^k(1-t)$. By the
bathtub principle, among sets $E$ of $\omega$-mass $1/2$, $\mathcal C$
maximizes $\int_E\log(t(1-t))\,d\omega$: with $f\defeq\log(t(1-t))$ and $\ell$
its value on the boundary of $\mathcal C$, $f\ge\ell$ on $\mathcal C$ and
$f\le\ell$ off $\mathcal C$, so
$\int_{\mathcal C}f-\int_Ef\ge\ell\,(\omega(\mathcal C\setminus E)-\omega(E\setminus\mathcal C))=0$.
Also $\log(1-t)\le0$. So
$\int_{\mathcal C_k}\log v_k\,d\omega\le k\int_{\mathcal C}\log(t(1-t))\,d\omega=(n-1)(\Gcat/\pi-\log2)$,
and Step~5 gives
\[
  \Reg_n(H_k;\Gconst)\ge\frac1{16}\,16^{-n}\binom nk4^{n-1}e^{-2\Gcat(n-1)/\pi}
  =\frac{e^{2\Gcat/\pi}}{64}\,\frac{\binom nk}{2^n}\,\bstar^{-n}.
\]

\emph{Step 8 (constants).} First, $\binom{2k}k/4^k\ge1/(2\sqrt k)$: the ratio of
successive terms of $\bigl(\binom{2k}k/4^k\bigr)\cdot2\sqrt k$ is
$\frac{2k+1}{2k+2}\sqrt{\frac{k+1}k}\ge1$, and the $k=1$ value of
$\binom{2k}k/4^k$ is $\frac12$. Next,
$\binom{2k+1}k/2^{2k+1}=\binom{2k+2}{k+1}/4^{k+1}\ge1/(2\sqrt{k+1})=1/\sqrt{2n+2}$.
For even $n$, Step~6 gives
$\Reg_n(H_k;\Gconst)\ge\bstar^{-n}/(32\sqrt{n/2})\ge\bstar^{-n}/(72\sqrt n)$.
For odd $n$, Step~7 gives
$\Reg_n(H_k;\Gconst)\ge\frac{e^{2\Gcat/\pi}}{64}\,\frac{\bstar^{-n}}{\sqrt{2n+2}}\ge\frac{e^{2\Gcat/\pi}}{128}\,\frac{\bstar^{-n}}{\sqrt n}$,
and $e^{2\Gcat/\pi}/128=0.013997\ldots\ge1/72$. In both cases
$\sup_H\Reg_n(H;\Gconst)\ge\bstar^{-n}/(72\sqrt n)$, and Step~6 is the
even-$n$ bound witnessed by $H_k$. This proves (ii).

\emph{Step 9 (consequences).} By (i) and (ii), and since
$(72\sqrt n)^{1/n}\to1$,
\[
  \bigl(\sup_H\Reg_n(H;\Gconst)\bigr)^{1/n}\to1/\bstar.
\]
Moreover $4^n\sup_H\Reg_n(H;\Gconst)\ge(4/\bstar)^n/(72\sqrt n)$ with
$4/\bstar=2e^{-2\Gcat/\pi}=1.116306\ldots>1$, which exceeds $Cn^\alpha$ for large
$n$, for every $C$ and $\alpha$. So no bound of the form
$\sup_H\Reg_n(H;\Gconst)\le Cn^\alpha4^{-n}$ holds.
\end{proof}

%% file: sections/9b5_lower_main.tex
\subsection{The lower bound for monotone priors}
\label{app:eb-lower}

Enlarging the prior class can only enlarge each identified set, so it can only
increase the identification regret (\cref{prop:chebradius}(iv)). Every constant
map is nondecreasing, so the constant-map lower bounds of \cref{prop:binomial}
and \cref{thm:betastar} transfer verbatim to monotone priors. The hard
instances are priors on constant maps, for every output law $\mu$.

\begin{lemma}[Lower bound for monotone priors]\label{lem:eb-lower}
Consider the setting of \cref{prop:chebradius} with any Borel probability
measure $\mu$ on $[0,1]$, and let $\Gmono$ be the class of \cref{sec:eb}: the
priors admitting a jointly measurable version whose paths are nondecreasing
maps $[0,1]\to[0,1]$. Then for every $n\ge1$,
\[
  \sup_{G\in\Gmono}\Reg_n(G;\Gmono)\ \ge\ \max\Bigl\{\frac1{8\binom{2n}n},\ \frac{\bstar^{-n}}{72\sqrt n}\Bigr\},
\]
and in particular
$\liminf_{n\to\infty}\bigl(\sup_{G\in\Gmono}\Reg_n(G;\Gmono)\bigr)^{1/n}\ge1/\bstar$.
Moreover, for the arcsine constant-map prior,
$\Reg_n(H_{\mathrm{arc}};\Gmono)\ge1/(8\binom{2n}n)=(1+o(1))\sqrt{\pi n}/(8\cdot4^n)$,
and at $n=1$, $\Reg_1(H_{\mathrm{arc}};\Gmono)\ge1/16=R^*_1(H_{\mathrm{arc}})$.
The same bounds hold with $\Gmono$ replaced by the class $\mathcal G_{\mathrm{path}}$
of laws of jointly measurable random maps all of whose paths are nondecreasing
maps $[0,1]\to[0,1]$.
In words: with $n$ audited labels per question, even when the law of one
question's data is known exactly, no recalibration rule can guarantee, uniformly
over monotone-map priors, an excess risk over the prior-oracle Bayes rule below
these values. At $n=1$ the unavoidable excess is at least the oracle Bayes risk
of the arcsine constant-map prior.
\end{lemma}

\begin{proof}
Let $\Gconst$ be the class of priors on constant maps, as in
\cref{prop:binomial}. A constant map $g\equiv\theta$ is jointly measurable and
nondecreasing, and a jointly measurable map whose paths are all nondecreasing is
its own jointly measurable nondecreasing version, so
\[
  \Gconst\subseteq\mathcal G_{\mathrm{path}}\subseteq\Gmono .
\]
All three classes are convex: a mixture of two priors in a class is the law of
$Bg_1+(1-B)g_2$, with $g_1,g_2$ jointly measurable representatives (versions,
for $\Gmono$) of the two priors and $B$ an independent Bernoulli variable, and
this map is again jointly measurable, with nondecreasing (respectively constant)
paths. For every $H\in\Gconst$, \cref{prop:chebradius}(iv), applied to each
inclusion, gives
\begin{equation}\label{eq:eb-lower-transfer}
  \Reg_n(H;\Gmono)\ \ge\ \Reg_n(H;\mathcal G_{\mathrm{path}})\ \ge\ \Reg_n(H;\Gconst).
\end{equation}
Below we write the argument for $\Gmono$; the same lines, with the middle term
of~\eqref{eq:eb-lower-transfer}, give the bounds for $\mathcal G_{\mathrm{path}}$.

\emph{The arcsine term.} Take $H=H_{\mathrm{arc}}\in\Gconst$.
By~\eqref{eq:eb-lower-transfer} and \cref{prop:binomial}(c),
$\Reg_n(H_{\mathrm{arc}};\Gmono)\ge1/(8\binom{2n}n)$ for every $\mu$ and every
$n\ge1$; the asymptotic form is that of \cref{prop:binomial}(c). At $n=1$,
\cref{prop:binomial}(c) gives $\Reg_1(H_{\mathrm{arc}};\Gconst)=1/16=R^*_1(H_{\mathrm{arc}})$,
so $\Reg_1(H_{\mathrm{arc}};\Gmono)\ge R^*_1(H_{\mathrm{arc}})$.

\emph{The $\bstar$ term.} Taking the supremum of~\eqref{eq:eb-lower-transfer}
over $H\in\Gconst\subseteq\Gmono$ and using \cref{thm:betastar}(ii), which holds
for every Borel $\mu$,
\[
  \sup_{G\in\Gmono}\Reg_n(G;\Gmono)\ge\sup_{H\in\Gconst}\Reg_n(H;\Gconst)\ge\frac{\bstar^{-n}}{72\sqrt n}.
\]
Together with the arcsine term this gives the maximum. The statement about the
$\liminf$ follows from $(72\sqrt n)^{1/n}\to1$.
\end{proof}

\begin{remark}[Which term is larger]\label{rem:eb-lower-crossover}
Let $r_n\defeq8\binom{2n}n/(72\sqrt n\,\bstar^n)$ be the ratio of the second
term to the first. Then
$r_{n+1}/r_n=\bigl(4-\tfrac2{n+1}\bigr)\sqrt{n/(n+1)}/\bstar$. This ratio is
increasing in $n$, since both factors are. It equals $0.99399\ldots$ at $n=8$
and $1.00607\ldots$ at $n=9$, so $r$ decreases on $\{1,\dots,9\}$ and increases
from $n=9$ on. Direct evaluation gives $r_1=0.0620\ldots<1$ and
$r_{62}<1\le r_{63}$. Hence the $\bstar$ term is the larger exactly for
$n\ge63$. The constant $1/72$ is not optimized, so the crossover point reflects
the constants of the proof rather than the behaviour of the supremum itself.
\end{remark}

\subsection{Proof of the empirical-Bayes theorem}
\label{app:main-eb}

\ThmEB*

\begin{proof}
All objects are those of \cref{prop:chebradius} with $\mathcal G=\Gmono$, the
class of \cref{sec:eb} (laws admitting a jointly measurable version with
nondecreasing paths): the identified set $[G]$, the identification regret
$\Reg_n(G)=\Reg_n(G;\Gmono)$, and the Bayes risk $R^*_n$. \Cref{lem:eb-lower}
and \cref{thm:monotone-eb} are stated for this same class.

\emph{Lower side.} For every Borel $\mu$ and every $n\ge1$,
\cref{lem:eb-lower} gives
\[
  \sup_{G\in\Gmono}\Reg_n(G;\Gmono)\ge\max\bigl\{1/(8\tbinom{2n}n),\bstar^{-n}/(72\sqrt n)\bigr\};
\]
its proof uses only that $\Gconst$ and $\mathcal G_{\mathrm{path}}$ are convex
subclasses of $\Gmono$ (\cref{prop:chebradius}(iv)).

\emph{Upper side.} For every Borel $\mu$, \cref{thm:monotone-eb}(b) gives
$\Reg_n(G;\Gmono)\le2^{-n}$ for every $G\in\Gmono$, so the supremum is at most
$2^{-n}$. \Cref{thm:monotone-eb}(c) gives the plug-in statement: for every
$G\in\Gmono$, the posterior-mean rule of any $G''\in[G]$ has excess risk at
most $2^{-n}$ against every $G'\in[G]$.

\emph{Constant maps.} For every Borel $\mu$, \cref{thm:betastar}(i)--(ii) give
the two-sided bound
\[
  \frac{\bstar^{-n}}{72\sqrt n}\ \le\ \sup_{H\in\Gconst}\Reg_n(H;\Gconst)\ \le\ \bstar^{-n},
\]
and $(72\sqrt n)^{1/n}\to1$ gives
$\bigl(\sup_H\Reg_n(H;\Gconst)\bigr)^{1/n}\to1/\bstar$.

\emph{The case $n=1$.} The lower side is $1/(8\binom21)=1/16$, which equals
$R^*_1(H_{\mathrm{arc}})$, the oracle Bayes risk of the arcsine constant-map
prior, by \cref{prop:binomial}(c) (as recorded in \cref{lem:eb-lower}).
\end{proof}

%% file: sections/9b6_finite_m.tex
\subsection{The identification regret at every finite \texorpdfstring{$m$}{m}}
\label{app:finite-m}

The proof treats the data of the other questions as an auxiliary randomization. Its law is the same
under every prior in the identified set, so averaging over it produces a single rule for one question.

\CorFiniteM*

\begin{proof}
Write $\mathcal D\defeq[0,1]^n\times\set{0,1}^n$ for the space of one question's data, as in Step~0 of the
proof of \cref{prop:chebradius}, and $\mathcal A_j$ for the $\sigma$-field of block $j$ in the product
construction of the statement.

\emph{Step 1 (the other questions are an auxiliary randomization with an identified law).} Fix
$G\in\cG$, $G'\in[G]$, $m$, $\gh$ and $j$. Put $W_j\defeq((D_i)_{i\ne j},\xi_0)$, with values in
$\mathcal W\defeq\mathcal D^{m-1}\times[0,1]$. Under $G'$, the pair $(g_j,D_j)$ is independent of $W_j$, since
the $m$ blocks are independent and $\xi_0$ is independent of them, and
\[
  \Law_{G'}(W_j)=\Law_{G'}(D)^{\otimes(m-1)}\otimes\mathrm{Uniform}[0,1]
  =\Law_{G}(D)^{\otimes(m-1)}\otimes\mathrm{Uniform}[0,1]\eqqcolon Q_W ,
\]
because $G'\in[G]$ means $\Law_{G'}(D)=\Law_G(D)$. So $Q_W$ does not depend on $G'\in[G]$.

\emph{Step 2 (the averaged rule).} Write $\gh_j(p;d,w)$ for $\gh_j$ evaluated with $D_j=d$ and $W_j=w$ (the
arguments re-ordered by a measurable bijection), and define
$\bar\delta_j(d)(p)\defeq\int_{\mathcal W}\gh_j(p;d,w)\,Q_W(dw)$. Since $\gh_j$ is jointly measurable and
$[0,1]$-valued, the Fubini--Tonelli theorem makes $(p,d)\mapsto\bar\delta_j(d)(p)$ jointly measurable and
$[0,1]$-valued: $\bar\delta_j$ is a rule in the sense of \cref{def:seb}, and it does not depend on $G'$
(only on $Q_W$, $G$ and $\gh$).

\emph{Step 3 (Jensen's inequality, question by question).} For every $(p,d)$ and every $a\in[0,1]$,
convexity of $b\mapsto(b-a)^2$ gives
\[
  \int\bigl(\gh_j(p;d,w)-a\bigr)^2\,Q_W(dw)\ \ge\ \bigl(\bar\delta_j(d)(p)-a\bigr)^2 .
\]
Put $\mathcal A'_j\defeq\mathcal A_j\vee\sigma(X_j,Y_j)$. Since $W_j$ is a function of the blocks $i\ne j$
and of $\xi_0$, the product construction makes $\sigma(W_j)$ independent of $\mathcal A'_j$, and $W_j$ has
law $Q_W$ by Step~1. The function $F(\varpi,p,w)\defeq(\gh_j(p;D_j(\varpi),w)-g_j(\varpi,p))^2$ is
nonnegative and $\mathcal A'_j\otimes\mathcal B([0,1])\otimes\mathcal B(\mathcal W)$-measurable ($g_j$ is
$\mathcal A_j\otimes\mathcal B([0,1])$-measurable, $D_j$ is $\mathcal A'_j$-measurable, and $\gh_j$ is jointly
measurable). By Tonelli's theorem over $\mu$ and, for each fixed $p$, the freezing lemma (Lemma~A in Step~0.2
of the proof of \cref{prop:chebradius}, applied with $\mathcal A'_j$ in place of $\mathcal H$ and
$\zeta=W_j$),
\begin{align*}
  \E_{G'}\nmu{\gh_j-g_j}^2
  &=\int\E_{G'}\Bigl[\int\bigl(\gh_j(p;D_j,w)-g_j(p)\bigr)^2\,Q_W(dw)\Bigr]\,d\mu(p)\\
  &\ge\int\E_{G'}\Bigl[\bigl(\bar\delta_j(D_j)(p)-g_j(p)\bigr)^2\Bigr]\,d\mu(p)
   \;=\;R(\bar\delta_j,G'),
\end{align*}
the last equality because for every $p$ the pair $(g_j(p),D_j)$ has the law of $(g(p),D)$ under $G'$, and,
by Tonelli's theorem over $\mu$, $R(\delta,G')$ depends on $G'$ only through these laws.

\emph{Step 4 (average over questions).} The rule $\bar\delta\defeq\frac1m\sum_j\bar\delta_j$ is again
$[0,1]$-valued, jointly measurable, and independent of $G'$. Pointwise convexity of $b\mapsto(b-a)^2$ gives
$\frac1m\sum_j(\bar\delta_j(d)(p)-a)^2\ge(\bar\delta(d)(p)-a)^2$. Evaluating at $(d,a)=(D,g(p))$ under
$G'$, taking $\E_{G'}$ for each $p$ and integrating over $\mu$ (Tonelli's theorem; $g$ is jointly
measurable), $\frac1m\sum_jR(\bar\delta_j,G')\ge R(\bar\delta,G')$. With Step~3,
\[
  \mathcal E_m(\gh,G')\ \ge\ R(\bar\delta,G')-R^*_n(G')\qquad\text{for every }G'\in[G]
\]
(both risks lie in $[0,1]$, so the difference is well defined).

\emph{Step 5 (proof of (i)).} The rule $\bar\delta$ is the same for all $G'\in[G]$, hence
\[
  \sup_{G'\in[G]}\mathcal E_m(\gh,G')\ \ge\ \sup_{G'\in[G]}\bigl(R(\bar\delta,G')-R^*_n(G')\bigr)
  \ \ge\ \inf_\delta\sup_{G'\in[G]}\bigl(R(\delta,G')-R^*_n(G')\bigr)=\Reg_n(G;\cG).
\]

\emph{Step 6 (proof of (ii)).} For every $\gh$ and every convex $\cG$, since $[G]\subseteq\cG$ for
$G\in\cG$,
$\sup_{G\in\cG}\mathcal E_m(\gh,G)\ge\sup_{G\in\cG}\sup_{G'\in[G]}\mathcal E_m(\gh,G')\ge\sup_{G\in\cG}\Reg_n(G;\cG)$
by (i). The classes $\Gmono$ and $\Gconst$ are convex (proof of \cref{lem:eb-lower}); convexity is needed only
because the definitions of \cref{def:seb} are made for convex classes, and the argument of (i) does not use
it. With $\cG=\Gmono$, \cref{thm:main-eb} gives
$\sup_G\Reg_n(G;\Gmono)\ge\max\set{1/(8\binom{2n}n),\bstar^{-n}/(72\sqrt n)}$; with $\cG=\Gconst$, its
constant-map clause gives $\sup_H\Reg_n(H;\Gconst)\ge\bstar^{-n}/(72\sqrt n)$. The bounds do not depend on
$\gh$, so taking the infimum over $\gh$ gives (ii).
\end{proof}